%% file: mcgra_arxiv.tex
\documentclass[twoside,11pt]{article}

\usepackage{amsmath}
\usepackage{amsthm}
\usepackage[preprint]{jmlr2e}
\usepackage[T1]{fontenc}
\usepackage[utf8]{inputenc}
\usepackage{amssymb}
\usepackage{amsfonts}
\usepackage{graphicx}
\usepackage{algorithm}
\usepackage{algorithmic}
\usepackage{array}
\usepackage{booktabs}
\usepackage{caption}
\usepackage{subfig}
\usepackage{bm}
\usepackage[table]{xcolor}
\usepackage{hyperref}
\usepackage{url}
\usepackage{nicefrac}
\usepackage{microtype}
\usepackage{multirow}
\usepackage{enumitem}
\usepackage{mathtools}
\usepackage{float}
\usepackage{lastpage}
\usepackage{placeins}
\usepackage{tcolorbox}
\usepackage{comment}
\usepackage{bbm}
\usepackage{tabularx}
\usepackage{ragged2e}

\usepackage{etoc}
\etocdepthtag.toc{mtchapter}
\etocsettagdepth{mtchapter}{subsection}
\etocsettagdepth{mtappendix}{none}
\usepackage[textsize=tiny]{todonotes}

\newcolumntype{L}[1]{>{\RaggedRight\arraybackslash}p{#1}}
\newcolumntype{Y}{>{\RaggedRight\arraybackslash}X}

\usepackage{etoolbox}
\usepackage{pgf} 
\definecolor{high}{HTML}{32B897}  
\definecolor{low}{HTML}{FA7F6F}  
\newcommand*{\opacity}{90}
\definecolor{lightgrayrule}{gray}{0.5}
\newcommand{\rowgrayrule}{\arrayrulecolor{lightgrayrule}\midrule\arrayrulecolor{black}}
\newcommand*{\minval}{0.0}
\newcommand*{\maxval}{1.0}
\newcommand{\gradient}[1]{
    \ifdimcomp{#1pt}{>}{\maxval pt}{#1}{
        \ifdimcomp{#1pt}{<}{\minval pt}{#1}{
            \pgfmathparse{int(round(100*(#1/(\maxval-\minval))-(\minval*(100/(\maxval-\minval)))))}
            \xdef\tempa{\pgfmathresult}
            \cellcolor{high!\tempa!low!\opacity} #1
    }}
}

\newcommand{\ie}{\textit{i.e.}}
\newcommand{\eg}{\textit{e.g.}}

\newcommand{\wrt}{\textit{w.r.t. }}

\theoremstyle{plain}
\newtheorem{observation}{Observation}

\newcommand{\up}[1]{\textcolor[rgb]{0.8 0 0}{#1}}
\newcommand{\down}[1]{\textcolor[rgb]{0 0.6 0}{#1}}

\definecolor{citeColor}{RGB}{0,20,115}
\hypersetup{colorlinks,linkcolor={citeColor},citecolor={citeColor},urlcolor={citeColor}}

\jmlrheading{26}{2026}{1-\pageref{LastPage}}{5/26}{---}{21-0000}{Zhanke Zhou, Bo Han, Xuan Li, Jiangchao Yao, Sanmi Koyejo, Michael K. Ng}
\ShortHeadings{Beyond Homophily: Towards Generalized Graph Reconstruction Attack and Defense}{Zhanke Zhou, Bo Han, Xuan Li, Jiangchao Yao, Sanmi Koyejo, Michael K. Ng}
\firstpageno{1}

\begin{document}

\title{Beyond Homophily: Towards Generalized Graph Reconstruction Attack and Defense}


\newcommand{\addrHKBUCS}{$^1$~TMLR Group, Department of Computer Science, Hong Kong Baptist University, Kowloon Tong, Hong Kong SAR}
\newcommand{\addrSJTU}{$^2$~Cooperative Medianet Innovation Center, Shanghai Jiao Tong University, Shanghai, China}
\newcommand{\addrStanford}{$^3$~Department of Computer Science, Stanford University, Stanford, CA 94305, USA}
\newcommand{\addrHKBUMathCS}{$^4$~Department of Mathematics and Department of Computer Science, Hong Kong Baptist University, Kowloon Tong, Hong Kong SAR}

\author{
\!\!\name Zhanke Zhou$^{1}$ \email cszkzhou@comp.hkbu.edu.hk
\AND
\name Bo Han$^{1}$ \email bhanml@comp.hkbu.edu.hk
\AND
\name Xuan Li$^{1}$ \email csxuanli@comp.hkbu.edu.hk
\AND
\name Jiangchao Yao$^{2}$ \email sunarker@sjtu.edu.cn
\AND
\name Sanmi Koyejo$^{3}$ \email sanmi@cs.stanford.edu
\AND
\name Michael K.\ Ng$^{4}$ \email michael-ng@hkbu.edu.hk
\AND
\addr \addrHKBUCS \\
\addr \addrSJTU \\
\addr \addrStanford \\
\addr \addrHKBUMathCS
}

\editor{My editor}

\maketitle

\input{sections/1-abstract}
\input{sections/2-introduction}
\input{sections/3-preliminaries}
\input{sections/4-related-work}
\input{sections/5-understanding}
\input{sections/6-method-attack}
\input{sections/7-method-defense}
\input{sections/8-experiments}
\input{sections/9-further-discussion}
\input{sections/10-conclusion}


\newpage
\vskip 0.2in
\bibliography{mcgra}

\newpage
\etocdepthtag.toc{mtappendix}
\etocsettagdepth{mtchapter}{none}
\etocsettagdepth{mtappendix}{subsection}
\renewcommand{\contentsname}{Appendix}
\tableofcontents

\newpage
\appendix
\input{appendix/1-theoretical-justification}
\input{appendix/2-full-empirical-study}
\input{appendix/3-further-related-work}

\end{document}

%% file: sections/1-abstract.tex
\begin{abstract}
Graph neural networks (GNNs) are widely deployed on relational data, yet they can leak sensitive or proprietary information about the training graph adjacency, \eg, social ties, transactions, and interactions.
This work studies \emph{graph reconstruction attacks} (GRA), a form of model inversion that reconstructs the training adjacency from a trained GNN, given different levels of attacker-side information. 
We first provide a systematic characterization of when and why adjacency becomes recoverable through features, labels, embeddings, and predictions, with leakage modulated by graph homophily, heterophily, and the model's inductive bias.
Motivated by these findings, we view GNN inference through a \emph{Markov chain approximation} lens, treating the layered forward computation as a chain of topology-dependent representations. Building on this view, we develop complementary attack and defense methods.
On the attack side, we propose \emph{MC-GRA~(+)}, which reconstructs the adjacency by optimizing a surrogate adjacency whose GNN-induced representations align with those of the target model at each layer.
On the defense side, we propose \emph{MC-GPB~(+)}, which suppresses adjacency-dependent information throughout the representation chain while aiming to preserve classification accuracy under a privacy-utility trade-off.
Experiments across homophilic/heterophilic graph benchmarks and GNNs show that our attacks improve reconstruction fidelity over prior methods, while our defenses reduce reconstruction success with only minor accuracy loss.
\end{abstract}


%% file: sections/2-introduction.tex
\section{Introduction}
\label{sec: introduction}

\begin{figure}[H]
\centering
\includegraphics[width=0.8\textwidth]{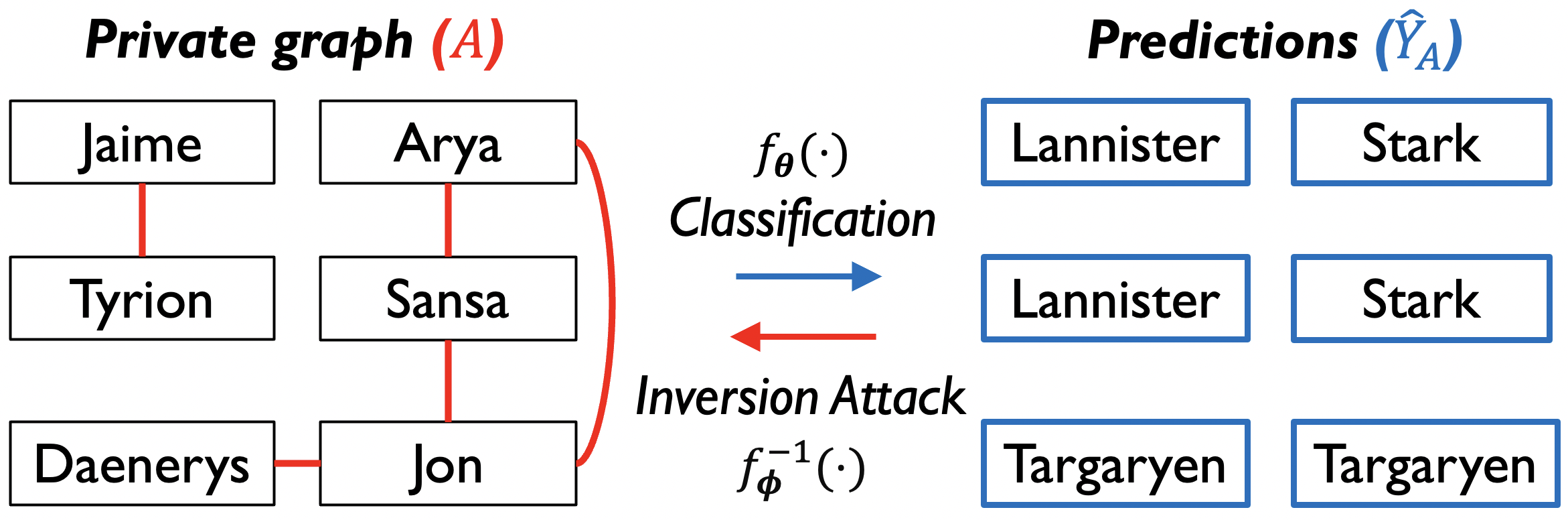}

\caption{Illustration of a graph reconstruction attack using characters from \textit{Game of Thrones}. In the \emph{forward} classification, a trained GNN $f_{\bm{\theta}}(\cdot)$ predicts node categories $\widehat{\bm{Y}}_{A}$ (family names shown in blue boxes) from node attributes and the private training graph $A$ (kinship relations shown as red edges). In the \emph{backward} inversion attack, an adversary reconstructs the original graph adjacency $A$ despite not having access to the ground-truth adjacency.}
\label{fig: motivation}
\end{figure}

Deep learning has expanded beyond Euclidean data (\eg, images and text) to relational, non-Euclidean structures such as graphs. Graph neural networks (GNNs)~\citep{kipf2016semi,gilmer2017neural,zhang2018link} learn over graphs by coupling node/edge attributes with message passing on an adjacency structure. This enables strong performance in applications where relational inductive bias is essential, including social networks~\citep{fan2019graph}, recommender systems~\citep{wu2020graphRecsys}, and molecular property prediction and drug discovery~\citep{ioannidis2020few}. However, the increasing deployment of GNNs in privacy-sensitive settings raises a basic concern: trained models may inadvertently encode and expose confidential information about their training data, especially the \emph{adjacency}.

A prominent threat is the \emph{model inversion attack}~\citep{fredrikson2015model,zhang2020secret,struppek2022ppa}, where an adversary with access to a trained model (and possibly auxiliary, non-sensitive attributes) aims to reconstruct sensitive information correlated with the training set. For GNNs, a natural inversion target is the \emph{training graph adjacency}. The adjacency used during training can encode sensitive relationships (\eg, social ties, transactions, communications) and may also be proprietary (\eg, curated interaction graphs or knowledge graphs). Recovering the training adjacency from a trained GNN can therefore compromise both individual privacy and organizational intellectual property. We refer to this class of attacks as \emph{graph reconstruction attacks} (GRA), with a representative case in Fig.~\ref{fig: motivation}. Existing work on GRA remains limited and often tailored to specific assumptions, \eg, restricted to homophilic graphs, fixed side-information configurations, or particular access models~\citep{he2021stealing,zhang2021graphmi}. As a result, the general mechanisms of topology leakage and the corresponding principles for defenses remain insufficiently understood.

In this work, we address this gap through a unified view of GRA as a \emph{Markov chain approximation}
\footnote{Throughout this paper, ``Markov chain'' is used by analogy to denote the layered conditional-dependence structure of GNN forward computation, following~\citep{zhou2023mcgra}, rather than a stationary stochastic process in the classical sense. Since the forward pass is deterministic and generally non-stationary, \emph{GNN computation chain} is the more precise term; this caveat applies throughout and is not repeated.}
problem (Fig.~\ref{fig: GRA-two-chains}). 
A trained GNN induces a layered sequence of topology-dependent representations; the attacker’s goal is to find a surrogate adjacency whose induced computation behaves similarly to that generated by the true (private) training graph, under the variables available to the adversary. This framing is crucial because topology can leak through \emph{multiple correlated channels}---node features, supervision signals, intermediate embeddings, and prediction scores (formalized in Sec.~\ref{sec: overview})---rather than solely through the final output. Moreover, these channels can interact in ways that depend on both architecture and graph regime---for instance, as we show in Sec.~\ref{sec: overview}, hidden embeddings can carry more adjacency signal than predictions under heterophily-aware architectures.

To ground this analysis, we conduct a systematic empirical study across diverse domains in homophilic and heterophilic settings, comparing canonical and heterophily-aware target models under matched training protocols (Sec.~\ref{sec: overview}). As a preview of findings, three consistent patterns emerge: (i)~leakage depends on the interaction between graph homophily and the model’s inductive bias, not on dataset characteristics alone; (ii)~uniform averaging of multiple leaked variables often provides limited or even negative marginal gains, due to redundancy and scale mismatches among different similarity scores; and (iii)~label-derived signals become unreliable under heterophily, where label agreement is weakly coupled to connectivity.

\begin{figure}[t!]
\centering
\includegraphics[width=0.8\textwidth]{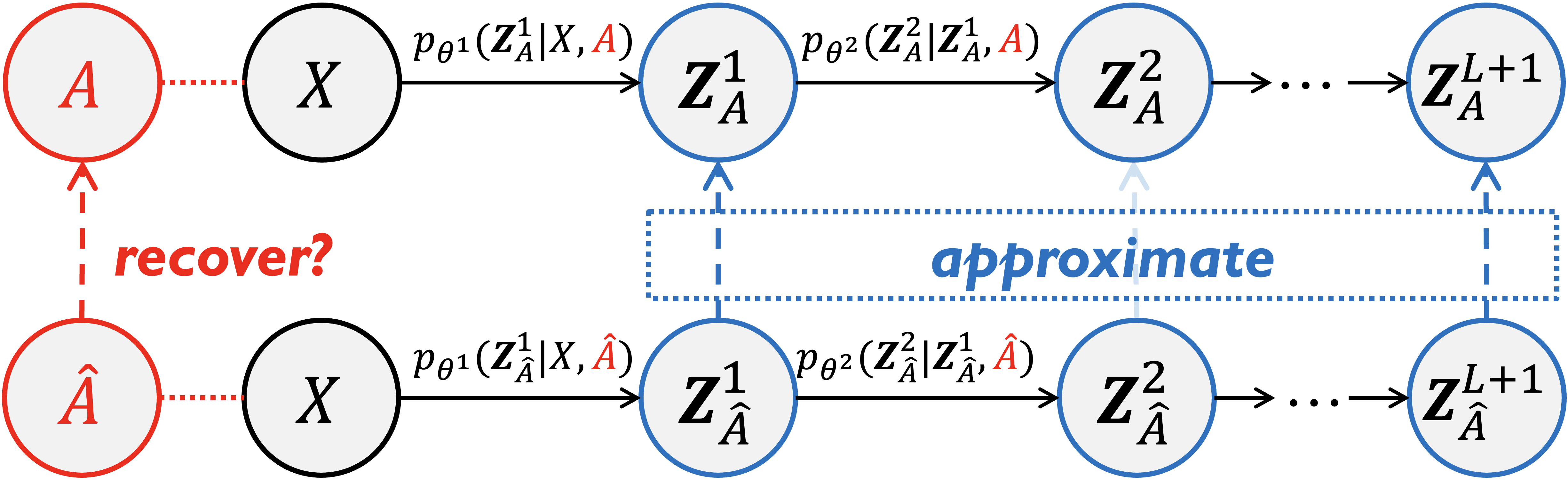}

\caption{
Modeling graph reconstruction as \emph{Markov-chain approximation} (here ``Markov chain'' in the sense of layered conditional independence; see Sec.~\ref{sec: overview}). 
The upper chain is the target GNN’s forward computation induced by the private adjacency $A$ and node feature $X$, while the lower chain is the attacker’s surrogate computation induced by the reconstructed adjacency $\hat{A}$ and node feature $X$ (if available).
To recover the ground truth adjacency $A$, the attack seeks an $\hat{A}$ whose induced chain best matches the original chain under all the available observations, \ie, approximating the latent variables $\bm{Z}_{A}^i$ through $\bm{Z}_{\hat{A}}^i$.
}
\label{fig: GRA-two-chains}
\end{figure}

Building on the chain view, we propose the \emph{Markov Chain-based Graph Reconstruction Attack} (MC-GRA). MC-GRA reconstructs topology by aligning observable variables at corresponding depths of the induced computation chains, turning adjacency recovery into a structured chain-matching problem that naturally supports heterogeneous attacker knowledge. The recovered adjacency is parameterized through a learnable continuous relaxation with injected stochasticity that facilitates gradient-based optimization over the discrete search space, while a complexity-aware regularizer discourages degenerate solutions (\eg, edge probabilities clustered near $0.5$ rather than near $0$ or $1$; details in Sec.~\ref{sec: GRA attack}). For strongly heterophilic graphs, we further introduce MC-GRA+, which incorporates a heterophily-aware prior derived from the model’s predicted labels to bias reconstruction toward cross-label connectivity when homophily-driven criteria become systematically misaligned with the true topology (see Theorems~\ref{thm: sufficiency_label_induced_matrix}--\ref{thm: noisy_labels_contraction_polished} in Sec.~\ref{ssec: attack understanding} for the theoretical justification).

On the defense side, we propose the \emph{Markov Chain-based Graph Privacy Bottleneck} (MC-GPB), which trains GNNs to remain accurate while suppressing adjacency information that is not essential for predicting labels. MC-GPB is motivated by an information bottleneck perspective with the adjacency treated as the sensitive source: it promotes label-predictive representations while penalizing adjacency-dependent information throughout the representation chain, and it regularizes layer-to-layer dependence to limit memorization of graph-specific details. Practically, we combine differentiable dependence estimation with injected stochasticity (\eg, edge dropping) to discourage deterministic encoding of training edges. Finally, MC-GPB+ introduces an explicit term tailored to heterophilic edges that further suppresses adjacency dependence, reducing leakage of high-value cross-label connectivity while preserving task utility under a privacy-utility trade-off (Fig.~\ref{fig: adjacency-demo}).
\footnote{We use MC-GRA~(+) to denote MC-GRA and MC-GRA+ collectively, and likewise for MC-GPB~(+).}

Experiments (in Sec.~\ref{sec: experiments}) across nine datasets spanning citation, political, web, air-traffic, and chemical graphs, and across multiple GNN backbones (GCN, GPR-GNN, GAT, GraphSAGE), confirm the effectiveness of both attack and defense:
\begin{itemize}[leftmargin=*]
\setlength\itemsep{0.1em}
\item \emph{Attack:} MC-GRA~(+) consistently improve edge-recovery AUC over the non-learnable ensemble and over GraphMI, with the largest relative gains on heterophilic graphs and strong absolute performance on homophilic ones.
\item \emph{Defense:} MC-GPB~(+) reduce leakage against both similarity-based and learnable probes while largely preserving classification accuracy. When evaluated directly against our attack, protected models lower reconstruction AUC across prior-knowledge settings.
\item \emph{Ablations:} Hidden-layer alignment, output alignment, and injected stochasticity each contribute to the attack, and the defense transfers across the tested architectures.
\end{itemize}

\begin{figure}[t!]
\centering
\includegraphics[width=0.8\textwidth]{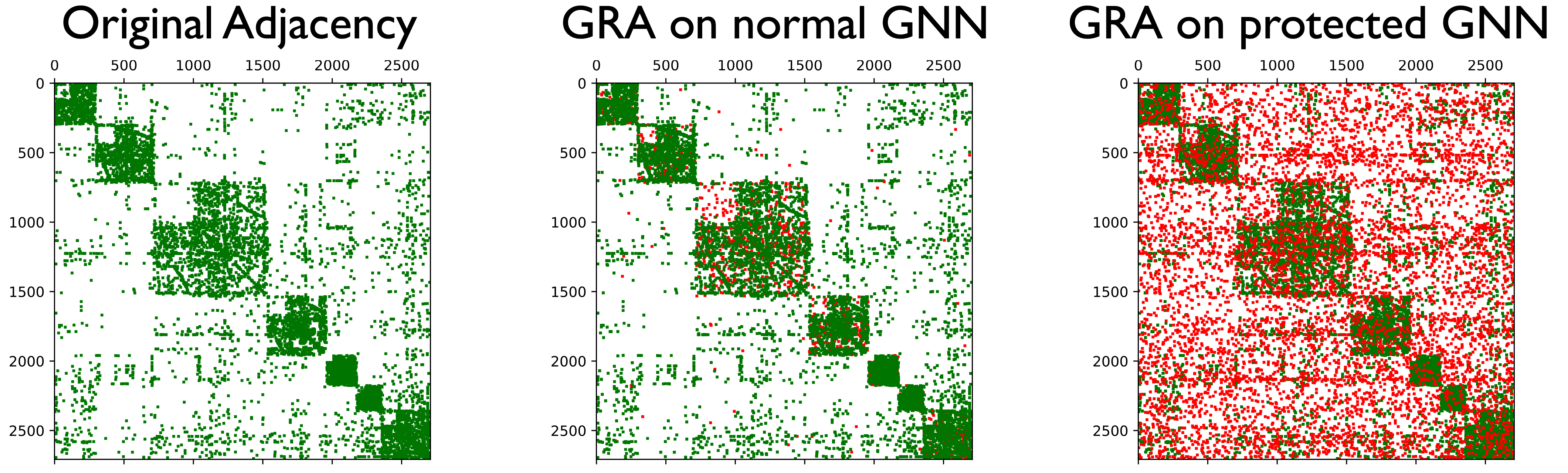}
\caption{
Recovered adjacency on Cora.
Green dots denote correctly recovered edges; red dots denote errors.
An unprotected GNN enables substantial adjacency leakage under our attack (MC-GRA~(+)), whereas training with MC-GPB~(+) markedly degrades reconstruction quality and makes the recovered adjacency less informative.}
\label{fig: adjacency-demo}
\end{figure}

In summary, our main contributions are as follows.
\begin{itemize}[leftmargin=*]
\setlength\itemsep{0.1em}
\item \textbf{A unified Markov-chain approximation formulation.}
We model GRA through a chain-approximation lens induced by GNN computation, enabling principled comparison of leakage channels and threat models and motivating both attack and defense objectives.

\item \textbf{A systematic study of graph reconstruction attacks.}
We provide an empirical characterization of GRA, clarifying when and why adjacency becomes recoverable from node features, intermediate embeddings, and prediction scores across graph regimes.

\item \textbf{Stronger attacks via chain matching.}
We propose MC-GRA, a chain-based reconstruction method that adaptively exploits arbitrary subsets of exposed variables, and MC-GRA+ for heterophilic regimes via a heterophily prior and robust stochastic optimization.

\item \textbf{Principled defenses via a graph privacy bottleneck.}
We introduce MC-GPB, an information-theory-guided defense that reduces adjacency-dependent information in representations while preserving label-relevant semantics, and MC-GPB+ for heterophily.

\item \textbf{Information-theoretic analysis linking reconstruction to adjacency-representation dependence.}
We provide formal results connecting reconstruction fidelity to cross-chain dependence and characterizing irreducible leakage under homophily and heterophily.

\item \textbf{Comprehensive empirical validation.}
We empirically demonstrate substantial improvements over prior attack and defense methods across diverse benchmarks spanning nine datasets and multiple GNN architectures.
\end{itemize}

\paragraph{Relation to the conference version.}
This paper is an extended version of our previous work~\citep{zhou2023mcgra} published in ICML~2023. The main differences are summarized in Sec.~\ref{sec: discussion}, including (1) heterophily-aware attack and defense variants, (2) new information-theoretic analysis, and (3) broader experimental validation.

%% file: sections/3-preliminaries.tex
\section{Preliminaries}
\label{sec: preliminaries}

This section establishes notation, formalizes the GNN forward model and graph reconstruction attacks, and recalls the information-theoretic tools used in later sections.

\subsection{Notation and Setup}
Let $N$ denote the number of nodes and $D$ the feature dimension. We write $[C]:=\{1,\dots,C\}$ for the label set, where $C$ is the number of classes; this notation is used throughout. We consider undirected, unweighted graphs without self-loops. We define the graph by its symmetric adjacency matrix $A\in\{0,1\}^{N\times N}$ with $A_{ij}=A_{ji}$ and $A_{ii}=0$, and node-feature matrix $X\in\mathbb{R}^{N\times D}$. We list each edge once by convention $i<j$; accordingly, the edge set $\mathcal{E}=\{(v_i,v_j):A_{ij}=1,\,i<j\}$ is determined by $A$ (one-to-one with the upper-triangular entries of $A$), and the node set is $\mathcal{V}=\{v_i\}_{i=1}^{N}$. We write $\mathcal{G}=(\mathcal{V},\mathcal{E},A,X)$ for the full graph object (a convenience tuple; $\mathcal{E}$ and $\mathcal{V}$ are uniquely determined by $A$ and $N$ and carry no independent information). Self-loops are introduced only when required by the GNN architecture.
Each node $v_i$ has label $y_i\in[C]$, and we denote the node-label vector by $Y=(y_1,\dots,y_N)\in[C]^{N}$. When used in information-theoretic statements (Secs.~\ref{sec: overview}--\ref{sec: GRA defense} and the appendix), $A$, $X$, and $Y$ denote random variables under the data-generating distribution; when used in optimization or algorithm descriptions, they denote the observed realizations.

The node-classification task is to estimate the label of each node from $(A,X)$ using a parameterized GNN $f_{\bm{\theta}}(\cdot)$.
We write
\begin{equation}
f_{\bm{\theta}}(A,X)=\bm{\hat{Y}}_{A}\in[0,1]^{N\times C},
\end{equation}
where the $i$-th row of $\bm{\hat{Y}}_{A}$ is the predicted class-probability vector for node $v_i$ (each row lies in the probability simplex $\Delta^{C-1}$).
Tab.~\ref{tab: notations} summarizes the notation used most frequently throughout the paper.

\begin{table}[H]
\centering
\small
\setlength{\tabcolsep}{10pt}
\renewcommand{\arraystretch}{1.1}
\fontsize{9}{9}\selectfont
\begin{tabular}{c|c}
\toprule
\textbf{Notation}  & \textbf{Meaning} \\
\midrule
$[C]:=\{1,\dots,C\}$ & label set ($C$ = number of classes) \\ \midrule
$\mathcal{V} = \{v_i\}_{i=1}^{N}$ & node set \\ \midrule
$\mathcal{E}$ & edge set \\ \midrule
$A\in\{0,1\}^{N \times N}$ & binary adjacency matrix \\ \midrule
$X\in\mathbb{R}^{N \times D}$ & node-feature matrix \\ \midrule
$\mathcal{G}  =  (\mathcal{V}, \mathcal{E}, A, X)$ & graph input to a GNN \\ \midrule
$Y=(y_1,\dots,y_N)$ & node-label vector \\ \midrule
$\bm{H}_{A}^{(\ell)}$ & layer-$\ell$ node-representation matrix induced by $A$ \\ \midrule
$\bm{H}_{A}^{\mathrm{all}}$ & full set of hidden layers $\{\bm{H}_A^{(\ell)}\}_{\ell=1}^{L}$ (used in Sec.~\ref{sec: overview} tables) \\ \midrule
$\mathcal{I}_{H}(\mathcal{K})$ & set of hidden-layer indices in $\mathcal{K}$ \\ \midrule
$\mathcal{H}_A^{\bullet}$ & set of available hidden layers $\{\bm{H}_A^{(\ell)}:\ell\in\mathcal{I}_{H}(\mathcal{K})\}$ \\ \midrule
$\bm{\hat{Y}}_{A}$ & node-wise prediction matrix induced by $A$ \\ \midrule
$H(U)$ & entropy of $U$: $H(U)=-\sum_u p(u)\log p(u)$ \\ \midrule
$D_{\mathrm{KL}}(p\|q)$ & Kullback--Leibler divergence from $q$ to $p$ \\ \midrule
$I(U ; V)$ & mutual information between $U$ and $V$ \\ \midrule
$I(U ; V \mid W)$ & conditional mutual information between $U$ and $V$ given $W$ \\ \midrule
$\mathcal{N}(i)$ & neighbor set of node $v_i$ (no self-loop unless stated) \\ \midrule
\texttt{ORI-chain} & chain of representations induced by the true adjacency $A$ (Sec.~\ref{sec: overview}) \\ \midrule
\texttt{GRA-chain} & chain induced by a candidate adjacency $\bm{\hat{A}}$ in a reconstruction attack (Sec.~\ref{sec: overview}) \\ \midrule
$\rho(\cdot)$ & generic componentwise activation (\eg, ReLU) in GNN layer updates \\ \midrule
$\sigma(\cdot)$ & elementwise sigmoid $\sigma(x)=1/(1+e^{-x})$ in Secs.~\ref{sec: GRA attack}--\ref{sec: GRA defense} \\
\bottomrule
\end{tabular}

\caption{Frequently used notation. 
$\mathcal{I}_{H}(\mathcal{K})$ denotes the set of hidden-layer indices $\ell\in\{1,\dots,L\}$ for which $\bm{H}_A^{(\ell)}\in\mathcal{K}$; $\mathcal{H}_A^{\bullet}=\{\bm{H}_A^{(\ell)}:\ell\in\mathcal{I}_{H}(\mathcal{K})\}$. 
Formal definitions are in Sec.~\ref{sec: GRA attack}.}
\label{tab: notations}
\end{table}

\subsection{Graph Neural Networks}
Both our attack and defense exploit how the adjacency influences representations through message passing. We therefore formalize the GNN forward computation in a model-agnostic way.
In experiments, we instantiate this framework with GCN~\citep{kipf2016semi}, GAT~\citep{velickovic2018graph}, GraphSAGE~\citep{hamilton2017inductive}, and GPR-GNN~\citep{chien2021adaptive}. Throughout, $\rho(\cdot)$ denotes a generic component-wise activation function (\eg, ReLU) in GNN layer updates, while $\sigma(\cdot)$ is reserved exclusively for the sigmoid function $\sigma(x)=1/(1+e^{-x})$ used in the attack and defense parameterizations (Secs.~\ref{sec: GRA attack}--\ref{sec: GRA defense}).

An $L$-layer GNN produces hidden representations through $L$ rounds of message passing and feature transformation.
Let $\bm{H}_{A}^{(0)}=X$, and for $\ell=1,\dots,L$, let $\bm{H}_{A}^{(\ell)}\in\mathbb{R}^{N\times d_{\ell}}$ denote the layer-$\ell$ representation matrix, whose $i$th row is the node representation $\bm{h}_i^{(\ell)}$.
For standard message-passing GNNs, a generic layer update can be written as
\begin{align}
\bm{m}_i^{(\ell)}
&=
\mathrm{AGGREGATE}^{(\ell)}
\Bigl(
\bigl\{
\mathrm{MESS}^{(\ell)}(\bm{h}_i^{(\ell-1)},\bm{h}_j^{(\ell-1)})
: j\in\mathcal{N}(i)
\bigr\}
\Bigr), \label{eqn: gnn-update-agg}\\
\bm{h}_i^{(\ell)}
&=
\mathrm{COMBINE}^{(\ell)}\!\bigl(\bm{h}_i^{(\ell-1)},\bm{m}_i^{(\ell)}\bigr),
\end{align}
where $d_{\ell}$ is the layer-$\ell$ output dimension; $\mathrm{MESS}^{(\ell)}:\mathbb{R}^{d_{\ell-1}}\times\mathbb{R}^{d_{\ell-1}}\to\mathbb{R}^{d'_{\ell}}$ returns a message vector, $\mathrm{AGGREGATE}^{(\ell)}$ maps a multiset of such vectors to a single vector in $\mathbb{R}^{d'_{\ell}}$, and $\mathrm{COMBINE}^{(\ell)}:\mathbb{R}^{d_{\ell-1}}\times\mathbb{R}^{d'_{\ell}}\to\mathbb{R}^{d_{\ell}}$ produces the layer-$\ell$ node representation $\bm{h}_i^{(\ell)}\in\mathbb{R}^{d_{\ell}}$. We define $\mathcal{N}(i)$ as the neighbor set of node $v_i$ \emph{excluding} $v_i$; each architecture may then add self-loops or otherwise modify the effective neighborhood. For example, GCN uses a self-loop-augmented adjacency $A^{+}=A+I$ (see Tab.~\ref{tab: GNN-architectures}), so the effective neighborhood includes $v_i$ itself. A unified chain notation $\bm{Z}_A^{(0)},\dots,\bm{Z}_A^{(L+1)}$ that indexes features, hidden states, and predictions as a single sequence is introduced in Sec.~\ref{sec: overview}.
For simplicity, we omit optional edge attributes from the generic update; when present, they can be incorporated into the message function.
After $L$ layers, the final representation of node $v_i$ is $\bm{h}_i^{(L)}$, and stacking these vectors yields $\bm{H}_{A}^{(L)}\in\mathbb{R}^{N\times d_L}$.
The classifier head maps $\bm{H}_{A}^{(L)}$ to the prediction matrix $\bm{\hat{Y}}_{A}\in[0,1]^{N\times C}$, with each row summing to one.

Let $\mathcal{V}_{\mathrm{train}}\subseteq\mathcal{V}$ denote the set of labeled training nodes. For supervised node classification, we optimize the cross-entropy loss:
\begin{equation}
\mathcal L_{\mathrm{CE}}
=
-\sum_{i\in\mathcal V_{\mathrm{train}}}
\log (\bm{\hat{Y}}_{A})_{i,y_i},
\end{equation}
where $(\bm{\hat{Y}}_{A})_{i,y_i}$ denotes the $(i,y_i)$-entry of the prediction matrix (the predicted probability of the true class for node $v_i$).

\begin{table}[tbp]
\centering
\setlength{\tabcolsep}{8pt}
\renewcommand{\arraystretch}{1.8}
\fontsize{9}{9}\selectfont
\begin{tabular}{cc}
\toprule
\textbf{GNN} & \textbf{Schematic update} \\ \hline
GCN & $\bm{H}_A^{(\ell)} = \rho \bigl((D^{+})^{-1/2} A^{+} (D^{+})^{-1/2}\bm{H}_A^{(\ell-1)}\bm{W}^{(\ell)}\bigr)$ \\
GAT & $\bm{h}_i^{(\ell)} = \rho \Bigl(\sum_{j\in\mathcal{N}(i)\cup\{i\}}\alpha_{ij}^{(\ell)} \bm{W}^{(\ell)}\bm{h}_j^{(\ell-1)}\Bigr)$ \\
GraphSAGE & $\bm{h}_i^{(\ell)} = \rho \Bigl(\bm{W}^{(\ell)}\bigl[\bm{h}_i^{(\ell-1)} \,\|\, \mathrm{mean}_{j\in\mathcal{N}(i)}\bm{h}_j^{(\ell-1)}\bigr]\Bigr)$ \\
GPR-GNN & $\bm H_A = \sum_{k=0}^{K}\gamma_k P_A^{k}\bm E$, where $\bm E$ is the pre-propagation base embedding from $X$ \\
\bottomrule
\end{tabular}
\caption{Schematic forms of the GNN architectures. Exact implementations follow the original papers. The GPR-GNN entry summarizes its propagation stage rather than a single layerwise recurrence. Here $A^{+}=A+I$ and $D^{+}_{ii}=\sum_j A^{+}_{ij}$ for the self-loop-augmented normalization used by GCN; $A^{+}$ is used only in this table to avoid confusion with the recovered adjacency $\bm{\hat{A}}$ introduced later. In GAT, $\alpha_{ij}^{(\ell)}$ denotes the learned attention coefficient from node $j$ to node $i$ at layer $\ell$; $\mathcal{N}(i)$ there includes $i$ (self-loop). GPR-GNN uses a single propagation stage rather than layerwise recurrence: $\bm H_A$ is the final representation (no layer index), $P_A$ is the propagation operator induced by the normalized adjacency, and $\bm E$ is the pre-propagation base embedding (distinct from the chain notation $\bm{Z}_A^{(\ell)}$ used in Secs.~\ref{sec: overview}--\ref{sec: GRA defense}).}
\label{tab: GNN-architectures}
\end{table}

\subsection{Graph Reconstruction Attacks}
We study model inversion attacks on GNNs in which the adversary seeks to reconstruct the training adjacency. We refer to this threat as a \emph{graph reconstruction attack} (GRA)---a specialization of the general model inversion attack (MIA) to the graph setting, where the inversion target is the adjacency structure rather than, \eg, images or text. We use GRA as the primary term throughout.

\begin{definition}[Graph Reconstruction Attack]
\label{def: graph reconstruction attack}
Let $f_{\bm{\theta}}(\cdot)$ be a trained target GNN. Let $\mathcal{O}=\{X, Y, \bm{H}_A^{(1)},\ldots,\bm{H}_A^{(L)}, \bm{\hat{Y}}_A\}$ be the catalog of observable quantities. The attacker's available side information is a subset $\mathcal{K}\subseteq\mathcal{O}$; thus $\mathcal{K}$ can contain any subset of the layer-wise representations (e.g., only $\bm{H}_{A}^{(1)}$ and $\bm{H}_{A}^{(3)}$). When the layer index is omitted, we write $\mathcal{H}_A^{\bullet}$ for the set of available hidden representations under the given threat model (a single layer or multiple layers; see Remark~\ref{rem: def1}). In formal statements we always use the explicit layer index $\bm{H}_A^{(\ell)}$.
A graph reconstruction attack aims to infer the training adjacency matrix $A$ of the training graph $\mathcal{G}_{\mathrm{train}}=(\mathcal V,\mathcal E,A,X)$ by producing an estimate
\begin{align}
\bm{\hat{A}}^{*}
\in
\arg\max_{\bm{\hat{A}}\in\mathcal A}
S(\bm{\hat{A}};f_{\bm{\theta}},\mathcal K),
\label{eqn: graph_MI_attack}
\end{align}
where $\mathcal A$ is the \emph{relaxed} feasible set of candidate adjacency matrices: the set of symmetric matrices in $[0,1]^{N\times N}$ with zero diagonal. 
Since the true adjacency $A$ is symmetric with zero diagonal (by the undirected, no-self-loop assumption in Sec.~\ref{sec: preliminaries}), it lies in the binary subset ${0,1}^{N\times N}\cap\mathcal{A}$. We therefore optimize over a continuous relaxation $\mathcal{A}$ only to permit gradient-based updates, whereas each forward evaluation uses a binarized adjacency induced from the relaxed variables (by thresholding or sampling, depending on the implementation), and the final output is the binary adjacency.
$S(\cdot)$ is an attack-specific measurable reconstruction score, designed so that higher values indicate better agreement between $\bm{\hat{A}}$ and the true adjacency $A$ under the available information; uniqueness of the maximizer is not assumed.
\end{definition}

\begin{remark}[Threat model and feasible set]
\label{rem: def1}
The set $\mathcal{I}_{H}(\mathcal{K})\subseteq\{1,\dots,L\}$ specifies which hidden-layer indices are available under a given threat model; it is defined formally in Sec.~\ref{ssec: attack objective}. Thus $\mathcal{H}_A^{\bullet}=\{\bm{H}_A^{(\ell)}:\ell\in\mathcal{I}_{H}(\mathcal{K})\}$. Implementation details for $\mathcal{A}$ and thresholding are in Sec.~\ref{sec: GRA attack}.
\end{remark}

The feasible set $\mathcal{A}$ is used as stated for gradient-based optimization; implementation details are in Sec.~\ref{sec: GRA attack}. GRA is conducted post hoc, \ie, after the target model $f_{\bm{\theta}}(\cdot)$ has been trained. In the definition, $A$ is the private object to be inferred and $\mathcal{K}$ denotes the quantities observable to the attacker.
We now specify the threat model in more detail, distinguishing model access, side information, and the attack target.

\paragraph{Access to the target model.}
Generally, there are black-box and white-box settings \wrt attack and defense in the community.
In the black-box setting, the attacker can only query the model and observe its outputs.
Black-box access is strictly weaker than white-box, so any black-box attack is also feasible in the white-box setting.
In the white-box setting, the attacker additionally observes the trained parameters of $f_{\bm{\theta}}$ and may rerun the model on candidate adjacencies; if intermediate representations are released, the attacker can compare against those released states.
Our methods target the white-box setting; the black-box setting is included for completeness.

\paragraph{Access to side information.}
The attacker may additionally observe side information $\mathcal{K}\subseteq\mathcal{O}$ (with $\mathcal{O}$ as in Definition~\ref{def: graph reconstruction attack}).
Depending on the application, such observations may be exact, partial, or noisy.
For example, a deployment pipeline may expose prediction scores $\bm{\hat{Y}}_{A}$ for downstream decision making, or intermediate embeddings (elements of $\mathcal{H}_A^{\bullet}$) for retrieval, recommendation, or debugging.
Partial-label access or subset-restricted observations can be modeled by restricting the corresponding object in $\mathcal K$.
Prior work has also considered stronger side information such as partial subgraphs or auxiliary datasets~\citep{he2021stealing}.

\paragraph{Attack target.}
Both $A$ and $X$ are inputs to a GNN and could in principle serve as inversion targets.
We focus on recovering $A$ because it directly encodes sensitive or private relationships, such as social ties, communication links, transactions, or curated interaction structures.
The training graph is $\mathcal G_{\mathrm{train}}=(\mathcal V,\mathcal E,A,X)$; labels $Y$ are supervision signals used during training and are not part of the graph object $\mathcal{G}$ itself.
Recovering links is privacy-sensitive whenever edges encode confidential relations (e.g., in social, financial, biomedical, or scientific graphs). In realistic workflows, node-level outputs or embeddings may be shared with downstream services, exposing side information useful for reconstruction. This threat is distinct from adversarial examples: adversarial examples manipulate test-time predictions, whereas GRA aims at reconstructing training data.

Model inversion has been studied extensively in vision and language~\citep{fredrikson2014privacy, hidano2017model, chen2021knowledge, wang2021variational, zhao2021exploiting, kahla2022label, carlini2021extracting, zhang2022text}, and analogous threats exist for graphs~\citep{he2021stealing, zhang2021graphmi}. A formal study of GRA exposes privacy weaknesses in GNN pipelines and motivates defenses before such leakage occurs in deployment.

\subsection{Information-Theoretic Preliminaries}
\label{ssec: information measures}

Having specified the graph-learning setup and the threat model, we next introduce the information-theoretic tools and notational conventions used in the later analysis. We assume familiarity with KL divergence $D_{\mathrm{KL}}(p\|q)$, mutual information $I(U;V)$, and the standard chain rules for entropy and mutual information (please see \citet{cover1999elements} for background). All entropies and mutual information use the natural logarithm (nats).

\paragraph{Population vs.\ empirical quantities.}
Throughout the paper, information-theoretic quantities appear at three distinct levels:
\begin{enumerate}[leftmargin=*,label=(\roman*)]
\setlength\itemsep{0.1em}
\item \textbf{Theory:} mutual information $I(\cdot;\cdot)$ under the data distribution over $(A,X,Y)$.
\item \textbf{Empirical diagnosis (Sec.~\ref{sec: overview}):} the AUC-based recoverability proxy $R_{\mathrm{AUC}}$.
\item \textbf{Optimization (Secs.~\ref{sec: GRA attack}--\ref{sec: GRA defense}):} differentiable surrogates $d(\cdot,\cdot)$ such as HSIC or CKA.
\end{enumerate}
We write $I(\cdot;\cdot)$ for population MI and $d(\cdot,\cdot)$ for the surrogate; uppercase $A$ for the random variable and lowercase $a$ for a realization when the distinction matters.
To aid the reader, we adopt the following signposting convention: statements involving population MI are introduced with phrases such as ``under the data-generating distribution'' or ``at the population level,'' while instance-level surrogate computations are indicated by ``for the observed graph'' or ``in the implemented objective.''

\begin{remark}[Scope of surrogates]
\label{rem: surrogate-scope}
We use RAUC, HSIC, and CKA as dependence and recoverability surrogates rather than as numerical MI estimators. This distinction is appropriate for our goal: the attack and defense require tractable signals that rank adjacency dependence, not exact estimation of mutual information. The justification for HSIC and CKA is not merely empirical correlation: HSIC is a kernel-based dependence measure whose population value is zero if and only if independence holds under characteristic kernels, and CKA is a normalized form of HSIC that retains this dependence-comparison role while improving scale invariance for representation matching. Hence, although neither quantity estimates the numerical value of $I(U;V)$, both are principled proxies for whether and how strongly two variables remain statistically coupled.
\end{remark}

In our setting, this is the relevant property: adjacency leakage is mediated by dependence between adjacency-induced representations and the attacker-visible variables. Thus, larger HSIC/CKA values should be read as stronger residual dependence, not as quantitative estimates of MI. Empirically, on all nine benchmark datasets the ranking of side-information configurations by surrogate value (HSIC or CKA) is consistent with the ranking by edge-recovery AUC (cf.\ Sec.~\ref{sec: experiments}), supporting the practical adequacy of these surrogates for optimization and diagnosis. This caveat applies throughout the paper whenever surrogates are used in place of population MI. We do not repeat it at every occurrence.

\begin{remark}[Terminology: ``Markov chain'' in this paper]
\label{rem: markov-terminology}
The GNN forward pass is deterministic and non-stationary, so it is not a stochastic process in the classical sense. Throughout Secs.~\ref{sec: overview}--\ref{sec: GRA defense}, we use the term \emph{GNN computation chain} (or simply ``chain'') to refer to the layered forward computation, in which each layer's output depends on previous layers only through the immediately preceding representation given the adjacency. When we write ``Markov chain'' by analogy, following the convention of the conference version~\citep{zhou2023mcgra}, this always refers to this layered conditional-independence structure, not to a stationary stochastic process. This caveat applies throughout and is not repeated at each occurrence.
\end{remark}

For reference, a stochastic process $(W_t)_{t\ge 1}$ is \emph{first-order Markov} if $p(w_{t+1}\mid w_t,\dots,w_1)=p(w_{t+1}\mid w_t)$ for all $t$; it is \emph{stationary} if the marginal distribution is time-invariant. A formal definition and the associated entropy lemma are given in Appendix~\ref{ssec: background markov lemma}.

\begin{remark}[Information contraction in GNN chains]
\label{rem: contraction-preview}
For a stationary first-order Markov chain, the conditional entropy $H(W_t\mid W_1)$ is nondecreasing in $t$ (Appendix~\ref{ssec: background markov lemma}). GNN computation chains are deterministic and non-stationary, so this classical result does not apply directly. A related but distinct result---concerning the \emph{cross-chain} mutual information $I(\bm{Z}_A^{(i)};\bm{Z}_{\hat{A}}^{(i)})$ between aligned variables of two chains induced by different adjacencies---is given in Theorem~\ref{theorem: reducing MI with two chains} (Sec.~\ref{sec: GRA attack}). That result uses the data processing inequality for deterministic maps and does not require stationarity.
\end{remark}

%% file: sections/4-related-work.tex
\section{Related Work}
\label{sec: related_work}

This section reviews prior work on graph reconstruction attacks (GRAs) and defenses, situating our chain-based framework relative to existing approaches. We use ``GRA'' as the primary term for model inversion attacks that target the training graph. We organize the attack review by threat-model access level (embedding-release, black-box query, and white-box optimization), then summarize defenses, and finally identify the gaps that motivate Secs.~\ref{sec: overview}--\ref{sec: GRA defense}. Additional related work on image and text MIAs is provided in Appendix~\ref{app: related-work}.

\subsection{Model Inversion Attacks on Graphs}

Graphs introduce a qualitatively different inversion problem from images or text, because the sensitive object is the relational structure itself: the search space of possible adjacency matrices is $2^{\binom{N}{2}}$, \ie, exponential in the number of possible undirected edges (equivalently, exponential in $N^2$). Existing GRAs differ mainly in attacker access (released embeddings, black-box queries, or white-box parameters), side information, and the reconstruction target. We organize the review by access level; Tab.~\ref{tab: related-work-attacks} summarizes the comparison.
The broader GNN literature provides relevant background on spectral/non-Euclidean graph learning, topology stability, and scalable message-passing training~\citep{bruna2014spectral,bronstein2017geometric,gama2020stability,chiang2019cluster,zou2019layer}.

\paragraph{Embedding-release setting.}
Early work assumes that the adversary observes released node embeddings.
\citet{chanpuriya2021deepwalking} and \citet{duddu2020quantifying} show that a decoder trained on auxiliary graphs can recover graph topology from embeddings alone, but these methods rely on released representations and auxiliary training data.
Under related assumptions, \citet{zhang2022inference} uses an auto-encoder framework to infer graph statistics, subgraph membership, and candidate full-graph structure from released graph representations.

\paragraph{Black-box query setting.}
\citet{he2021stealing} proposes \emph{link stealing}, a black-box attack that queries a target GNN under different combinations of side information, including node features, a partial target graph, and a shadow graph.
Its strongest variants require either access to a sensitive partial graph or to an auxiliary graph, which limits applicability in deployments where no auxiliary graph is available.

\paragraph{White-box optimization setting.}
In the white-box setting, GraphMI~\citep{zhang2021graphmi} optimizes a candidate adjacency by matching the target model's predictions under the recovered graph to label information.
GraphMI is closest in spirit to the white-box setting, but it relies on prediction-label matching rather than a chain-level alignment objective.
Related neural graph-matching work studies differentiable solvers for noisy, partial, and multiple-graph correspondence~\citep{wang2022neural,wang2023combinatorial,wang2023partial,wang2024pygmtools}; these methods are not GRAs, but they provide adjacent tools for reasoning about combinatorial graph recovery.

\paragraph{Query-based setting.}
LinkTeller~\citep{wu2022linkteller} studies graph reconstruction in a query-based setting, recovering private edges from changes in node predictions under perturbed inputs. Relative to our formulation, this corresponds to a more restricted side-information regime centered on output-level observations rather than full white-box access. Since our framework is defined for arbitrary subsets of observable variables, such query-based settings can, in principle, be represented within the same general formulation. In this work, we concentrate on the white-box chain-matching regime, which is the primary threat model for our attack design and analysis.

\paragraph{Graph-level and heterogeneous-graph extensions.}
Beyond node-classification settings, \citet{zhang2022inference} also studies graph-level inversion, where a single graph embedding is decoded into a candidate graph.
More recently, HomoGMI and HeteGMI~\citep{liu2023model} extend GRAs to homogeneous graphs (single node/edge type) and heterogeneous graphs (multiple node/edge types) by combining label-consistency objectives with graph-proximity constraints.
These methods improve reconstruction under their target settings, but they remain tied to particular side-information assumptions and do not provide a unified formulation for heterogeneous attacker knowledge or for explicit heterophily-aware analysis.

\begin{table}[tbp]
\centering
\small
\setlength{\tabcolsep}{4pt}
\renewcommand{\arraystretch}{1.25}
\fontsize{9}{9}\selectfont
\begin{tabularx}{\linewidth}{L{2.7cm}L{1.8cm}Y L{1.6cm}Y}
\toprule
\textbf{Method} & \textbf{Access} & \textbf{Side information} & \textbf{Target} & \textbf{Limitation} \\
\midrule
\citet{chanpuriya2021deepwalking};  \citet{duddu2020quantifying}
& Embedding release 
& Released embeddings, auxiliary graphs 
& Topology 
& Requires auxiliary data \\

\citet{zhang2022inference}
& Embedding release 
& Graph representations 
& Stats / subgraph / full graph 
& Release-channel specific \\

\citet{he2021stealing}
& Black-box 
& Features, partial/shadow graph 
& Links 
& Strongest variants need auxiliary graph \\

\citet{zhang2021graphmi}
& White-box 
& $X$, $Y$ 
& Adjacency 
& Prediction-label matching; homophily-implicit \\

\citet{liu2023model}
& White-box 
& Labels, graph proximity 
& Adjacency 
& Fixed side-info; no unified heterophily analysis \\
\midrule

MC-GRA (ours) 
& White-box 
& Arbitrary $\mathcal{K}$ subset ($X$, $Y$, $\bm{H}_A$, $\bm{\hat{Y}}_A$) 
& Adjacency 
& White-box access \\ 
\bottomrule
\end{tabularx}

\caption{Comparison of GRA methods by access level, side information, target, and limitation.}
\label{tab: related-work-attacks}
\end{table}

\paragraph{Recent and concurrent developments.}
Since the conference version of this work~\citep{zhou2023mcgra}, the graph privacy landscape has continued to evolve. \citet{zhang2022federated} study graph reconstruction from shared gradients in federated GNN training, a threat model distinct from the centralized post-hoc setting studied here. \citet{wu2024graphguard} propose a graph-level unlearning framework that removes the influence of specific edges from trained models, addressing data-deletion rather than reconstruction resistance. \citet{dai2024comprehensive} provide a comprehensive survey of privacy attacks on GNNs, categorizing link inference, membership inference, and reconstruction attacks under a unified taxonomy. While these works address related privacy concerns, they operate under different threat models or target different objectives, and none addresses the limitations identified in Sec.~\ref{ssec: GRA analysis}---namely, implicit homophily bias, fixed side-information assumptions, and the lack of a unified chain-level framework for both attack and defense.

While these attacks demonstrate that topology leakage is a genuine threat across access models, none addresses the question of how to systematically defend against it. We now review defenses designed to protect against such threats, grouping them by when they act (training time vs.\ inference or release time) and by the type of guarantee they provide (formal, e.g.\ differential privacy, vs.\ empirical protection).

\subsection{Defending Model Inversion Attacks on Graphs}

Existing defenses can be grouped into \emph{training-time} defenses (which modify the learning procedure) and \emph{inference-time or release-channel} defenses (which act on released artifacts). The two groups protect different threat surfaces and are not always directly comparable.

\paragraph{Training-time defenses.}
Training-time defenses reduce the dependence of learned representations or outputs on sensitive training structure.
These include formal privacy mechanisms (e.g., DP), empirical obfuscation (e.g., graph perturbation), and representation-regularization methods.
Differential-privacy (DP) mechanisms such as Degree-Preserving Randomized Response (DPRR) and GAP inject calibrated noise into node attributes or message passing and provide formal $(\varepsilon,\delta)$ guarantees~\citep{hidano2022degree,sajadmanesh2023gap}; at privacy budgets that preserve acceptable utility; however, empirical protection against inversion can remain limited, and the formal guarantees do not directly bound adjacency reconstruction risk.
Adversarial graph-perturbation methods, such as NetFense, learn perturbations that obfuscate private links while preserving utility~\citep{hsieh2021netfense}.
Other graph-security and deployment work studies poisoning/robustness and serving-time graph compression for GNNs~\citep{liu2019unified,wang2019graphdefense,si2023serving}; these objectives differ from adjacency reconstruction defense.

\paragraph{Inference-time and release-channel defenses.}
At a broader deployment level, other defenses act on the information released after training, for example, by perturbing exposed logits, gradients, or explanations at deployment time.
Such methods protect released artifacts rather than the training procedure itself and therefore address a different release channel from representation-level training defenses.
Similarly, masking explanation subgraphs produced by post-hoc explainers protects explanation outputs rather than the predictive interface of the trained model.
Combined defense strategies may therefore mix training-time protection with deployment-time controls depending on what information is exposed.
Recent surveys of graph privacy and security emphasize that effective protection will likely require combining privacy accounting, task-aware regularization, and release control mechanisms~\citep{khosla2022privacy}.

\subsection{Limitations of Existing Attack or Defense Methods}
\label{ssec: GRA analysis}

Reflecting on the above landscape, four key limitations relevant to our problem emerge.
\begin{itemize}[leftmargin=*]
\setlength\itemsep{0.1em}
\item
\textbf{Non-graph assumptions:} Many existing GRAs rely on assumptions that do not fully capture the relational and structural properties of graphs, \eg, generative priors (in image MIAs), or continuous latent spaces (in text MIAs). When adapted to graphs, such methods are often specialized to a particular domain, release channel, or attack formulation.

\item
\textbf{Implicit homophily bias:} GraphMI's prediction-label matching objective rewards candidate adjacencies under which connected nodes produce similar predictions, which aligns with the homophilic prior that neighbors share labels. This can misguide reconstruction on heterophilic graphs, where edges frequently connect nodes with different labels (experiments in Sec.~\ref{sec: overview}). This domain specificity is compounded when side information varies.

\item
\textbf{Fixed side-information assumptions:} Because existing GRAs are tied to fixed side-information assumptions, a key challenge is to design a single attack objective that can exploit arbitrary subsets of available side information without requiring a separate method for each threat model.
This challenge becomes more acute because graph domains vary widely, and GNN inductive biases can interact strongly with the available signals.
Privacy and partial-information graph-learning studies further show that conclusions can depend on realistic attacker assumptions, sampling decisions, and feedback models~\citep{jayaraman2021revisiting,gu2013selective,gu2014online}.

\item
\textbf{Privacy-utility trade-off:} Existing defenses either provide formal protection (e.g., differential privacy) at substantial performance cost, or rely on empirical obfuscation without a general representation-level framework for controlling adjacency leakage in GNNs. In particular, no existing method offers a principled, architecture-agnostic mechanism for suppressing adjacency information across GNN layers while preserving task utility.
\end{itemize}

Taken together, these limitations motivate a method that is graph-specific, adaptive to heterogeneous side-information settings, and paired with a defense that directly regularizes adjacency leakage while preserving predictive utility. A central idea that we adopt is the \emph{chain abstraction}: we view the GNN forward pass as a chain of representations (inputs $\to$ hidden layers $\to$ outputs). This abstraction provides a common language for both attack and defense, \ie, the attacker seeks to match chain states induced by a candidate adjacency to the released states, while the defender seeks to decorrelate those states from the adjacency. Sec.~\ref{sec: overview} develops this chain-based perspective and empirical study; Secs.~\ref{sec: GRA attack} and \ref{sec: GRA defense} derive the corresponding attack and defense objectives.

%% file: sections/5-understanding.tex
\section{Empirical Analysis of Adjacency Leakage in GNN Representations}
\label{sec: overview}

This section develops the empirical and conceptual basis for our attack and defense formulations.
We first introduce a Markov-chain view of graph reconstruction, then quantify how adjacency information is recoverable from features, labels, embeddings, and predictions across graph regimes and architectures, and finally study how this recoverability evolves during training. The goal is to identify empirical patterns that inform our later attack and defense objectives, not to provide a complete theory of leakage in GNNs.

We use the three-level distinction established in Sec.~\ref{ssec: information measures}: $R_{\mathrm{AUC}}$ for empirical diagnosis, $d(\cdot,\cdot)$ for optimization surrogates, and $I(\cdot;\cdot)$ for population-level theory. All quantities in this section are empirical (AUC-based).

\subsection{A Markov-Chain View of GRA}
\label{ssec: problem statement}

As in prior work~\citep{rahman2022markovgnn,zhao2022comprehensive,zhao2023exploring}, a GNN's forward computation can be viewed through a chain abstraction (recall Remark~\ref{rem: markov-terminology} for our use of ``Markov chain'').
\begin{equation}
f_{\bm\theta}:\ \bigl(A,X\bigr)\ \rightarrow\ \bm{H}_A\ \rightarrow\ \bm{\hat{Y}}_{A},
\end{equation}
where $A$ denotes the adjacency matrix, $X$ the node features, $\bm{H}_A$ the topology-dependent node embeddings produced by message passing on $A$, and $\bm{\hat{Y}}_{A}$ the prediction scores obtained from $\bm{H}_A$ by the classifier head. The forward computation is deterministic. Conditioned on $A$, the forward computation forms a deterministic layer-to-layer chain $X \to \bm{Z}_A^{(1)} \to \cdots \to \bm{Z}_A^{(L+1)}$. Equivalently, we use a conditional Markov-chain interpretation in which each stage depends on previous stages only through the immediately preceding representation once the adjacency is fixed. The value of this abstraction is the stage-wise decomposition it provides for analyzing where adjacency information enters and persists.

This viewpoint is particularly useful in the white-box setting, where the attacker can inspect the target model and reason about how the private adjacency affects intermediate states throughout the forward pass.
We model GRA as approximating the target model's \texttt{ORI-chain} using a surrogate \texttt{GRA-chain}, as illustrated in Fig.~\ref{fig: GRA-two-chains}.
The \texttt{ORI-chain} is induced by the true adjacency $A$, whereas the \texttt{GRA-chain} is induced by a recovered adjacency $\bm{\hat{A}}$:
\begin{equation}
\begin{split}
\texttt{ORI-chain:}& \quad \bm{Z}^{(0)}
\xrightarrow[\bm{\theta}^{(1)}]{A}  \bm{Z}_{A}^{(1)}
\xrightarrow[\bm{\theta}^{(2)}]{A}  \bm{Z}_{A}^{(2)}
\to  \cdots
\xrightarrow[\bm{\theta}^{(L+1)}]{A}  \bm{Z}_{A}^{(L+1)},
\\
\texttt{GRA-chain:}& \quad \bm{Z}^{(0)}
\xrightarrow[\bm{\theta}^{(1)}]{\bm{\hat{A}}}  \bm{Z}_{\bm{\hat{A}}}^{(1)}
\xrightarrow[\bm{\theta}^{(2)}]{\bm{\hat{A}}}  \bm{Z}_{\bm{\hat{A}}}^{(2)}
\to  \cdots
\xrightarrow[\bm{\theta}^{(L+1)}]{\bm{\hat{A}}}  \bm{Z}_{\bm{\hat{A}}}^{(L+1)}.
\end{split}
\label{eqn: two-Markov-chains}
\end{equation}
We omit the subscript $A$ from $\bm{Z}^{(0)}$ because $X$ does not depend on the adjacency.
\begin{definition}[Chain-variable mapping]
\label{def: chain-variable-mapping}
$\bm{Z}^{(0)}=X$, $\bm{Z}_{A}^{(\ell)}=\bm{H}_{A}^{(\ell)}$ for $\ell=1,\dots,L$, and $\bm{Z}_{A}^{(L+1)}=\bm{\hat{Y}}_{A}$. For index $0$, we write $\bm{Z}^{(0)}$ without the subscript $A$ because $X$ does not depend on the adjacency; when $0\in\mathcal{I}_{\mathrm{all}}(\mathcal{K})$, the corresponding element of $\mathcal{S}_A$ is simply $X$. The prediction head is indexed as layer $L+1$ for notational uniformity; it is produced by applying the classifier head to $\bm{H}_A^{(L)}$ without additional message passing. The contraction result (Theorem~\ref{theorem: reducing MI with two chains}) applies to the hidden-state chain; the classifier head is discussed after that theorem. This mapping is used throughout Secs.~\ref{sec: overview}--\ref{sec: GRA defense}.
\end{definition}
Under the induced probabilistic view, each stage variable depends on previous stages only through the immediately preceding representation and the transition determined by $A$ (or $\bm{\hat{A}}$) together with the layer parameters.

Let $\mathcal{K}$ denote the attacker's observable side information.
We write $\mathcal{I}_{\mathrm{all}}(\mathcal{K})$ for the set of chain indices whose variables are observable under $\mathcal{K}$: index $0$ if $X\in\mathcal{K}$; $\ell\in\{1,\dots,L\}$ if $\bm{H}_A^{(\ell)}\in\mathcal{K}$; and $L+1$ if $\bm{\hat{Y}}_A\in\mathcal{K}$. The hidden-layer subset is $\mathcal{I}_{H}(\mathcal{K})=\mathcal{I}_{\mathrm{all}}(\mathcal{K})\cap\{1,\dots,L\}$; indicators for output/labels are formalized in Sec.~\ref{sec: GRA attack}. We define
\begin{equation}
\mathcal{S}_A=\{\bm{Z}_A^{(i)}: i\in\mathcal{I}_{\mathrm{all}}(\mathcal{K})\},
\qquad
\mathcal{S}_{\bm{\hat{A}}}=\{\bm{Z}_{\bm{\hat{A}}}^{(i)}: i\in\mathcal{I}_{\mathrm{all}}(\mathcal{K})\}.
\end{equation}
In a reconstruction attack, the attacker does not observe $A$ but may observe a subset of the \texttt{ORI-chain} variables; they search for a candidate adjacency $\bm{\hat{A}}$ such that the induced \texttt{GRA-chain} matches those observations. The basic idea of GRA is then to choose $\bm{\hat{A}}$ so that $\mathcal{S}_{\bm{\hat{A}}}$ matches $\mathcal{S}_A$ as closely as possible. We refer to this as \emph{chain matching}: aligning the \texttt{GRA-chain} states to the released \texttt{ORI-chain} states at corresponding depths. This interpretation will later motivate depth-aligned chain matching for attack (Sec.~\ref{sec: GRA attack}) and layerwise representation control for defense (Sec.~\ref{sec: GRA defense}).

\subsection{Experimental Settings}
\label{ssec: experimental settings}

\paragraph{Datasets.}
We use datasets from five domains.
Cora and Citeseer~\citep{sen2008collective} are citation networks in which nodes are documents and edges denote citation links.
Polblogs~\citep{adamic2005political} is a network of political blogs in which nodes represent blogs and edges denote hyperlinks between them.
Texas, Cornell, and Wisconsin~\citep{pei2020geom} are webpage graphs with web pages as nodes and hyperlinks as edges.
USA and Brazil~\citep{ribeiro2017struc2vec} are air-traffic networks in which nodes are airports and edges denote airline routes.
AIDS~\citep{riesen2008iam} is a chemical graph in which nodes are atoms and edges are chemical bonds.
The dataset statistics are summarized in Tab.~\ref{tab: statistics}.

\begin{table}[tbp]
\centering\fontsize{9}{9}\selectfont
\renewcommand{\arraystretch}{1.3} 
\setlength\tabcolsep{5.2pt} 
\begin{tabular}{cccccccccc}
\toprule
Dataset & Cora & Citeseer & Polblogs & USA & Brazil & AIDS & Texas & Cornell & Wisconsin \\
\midrule
\# Nodes & 2,708 & 3,327 & 1,490 & 1,190 & 131 & 1,429 & 183 & 183 & 251 \\
\# Edges & 5,278 & 4,676 & 33,430 & 27,164 & 2,077 & 2,948 & 298 & 325 & 515 \\
\# Classes & 7 & 6 & 2 & 4 & 4 & 14 & 5 & 5 & 5 \\
\# Features & 1,433 & 3,703 & N/A & N/A & N/A & 4 & 1,703 & 1,703 & 1,703 \\
Feature Homophily & 0.83 & 0.81 & N/A & N/A & N/A & 0.06 & N/A & N/A & N/A \\
Edge Homophily & 0.81 & 0.74 & 0.91 & 0.70 & 0.45 & 0.51 & 0.13 & 0.09 & 0.19 \\
\bottomrule
\end{tabular}
\caption{Dataset statistics.
Edge homophily ratio (label homophily) is the average of $\mathbbm{1}\{y_i=y_j\}$ over edges $(i,j)\in\mathcal E$~\citep{pei2020geom}; feature homophily is the average edgewise cosine similarity over available node features.
For Texas, Cornell, and Wisconsin we omit feature homophily because the raw bag-of-words features are extremely sparse and high-dimensional, and the cosine summary is not informative in our setting. ``N/A'' indicates that the dataset has no node features or that feature homophily is omitted as above.}
\label{tab: statistics}
\end{table}

\paragraph{Homophilic and heterophilic graphs.}
Most widely used GNNs are designed around homophily, where nodes with similar labels are more likely to be connected~\citep{mcpherson2001birds}.
In such regimes, similarity-based reconstruction criteria are better aligned with the graph structure.
In heterophilic graphs, edges frequently connect nodes with different labels~\citep{pei2020geom}, so similarity-based reconstruction is more easily misspecified.
For discussion purposes, we refer to graphs with edge homophily below roughly $0.2$ as strongly heterophilic, following conventions in the heterophily literature~\citep{pei2020geom}; this threshold is for expository convenience rather than a universal definition.
In these settings, criteria that favor within-class similarity tend to over-recover homophilic substructures and under-recover cross-class edges, which makes adjacency reconstruction more difficult.

\paragraph{Target model.}
We use GCN~\citep{kipf2016semi} and GPR-GNN~\citep{chien2021adaptive} as the target models $f_{\bm{\theta}}$ in this section to span homophilic and heterophily-aware designs; Sec.~\ref{sec: experiments} evaluates additional architectures (GAT, GraphSAGE).
GCN serves as a homophily-oriented baseline, whereas GPR-GNN offers a contrastive architecture that better accommodates heterophily through learnable multi-hop propagation weights.
For a fair comparison, models use hidden width $d_{\mathrm{hid}}=16$, depth $L=2$, learning rate $\eta=0.01$, and the same training split.
Across all datasets, we train each model for $E=200$ epochs. The choice $d_{\mathrm{hid}}=16$ is consistent with common benchmarks; we do not claim that the reported patterns are insensitive to width, and Sec.~\ref{sec: experiments} includes additional architectures and settings.

\paragraph{Featureless graphs.}
When node features $X$ are absent (e.g., Polblogs, USA, Brazil), the target model uses a default input (e.g., learned or constant embeddings) as $\bm{Z}^{(0)}$; the chain view still applies with that input as the initial state. The attacker uses the same default when rerunning the model in the chain-matching framework (Sec.~\ref{sec: GRA attack}). In the tables below, $R_{\mathrm{AUC}}(A;X)$ is reported as N/A for featureless datasets, meaning that no meaningful node features exist (not that the model input is absent; featureless datasets use a default input such as constant or learned embeddings, as described above).

\paragraph{Evaluation metrics.}
Variables in the \texttt{ORI-chain} may carry adjacency-dependent information because the GNN transition operators are functions of $A$.
Directly computing the true mutual information between $A$ and a high-dimensional variable $\bm{Z}$ is intractable in our setting, so we use an empirical leakage proxy rather than exact MI.

For a node-level observable $\bm{Z}$, we define a similarity-based adjacency score $\hat A_{\bm{Z}}=\sigma(\bm{Z}\bm{Z}^\top)\in\mathbb R^{N\times N}$, where $\sigma(\cdot)$ denotes the elementwise sigmoid function $\sigma(x)=1/(1+e^{-x})$ (not the GNN activation $\rho$; see Sec.~\ref{sec: preliminaries}).
The \emph{adjacency-leakage proxy} (recoverability score) is
\begin{equation}
R_{\mathrm{AUC}}(A;\bm{Z})
\triangleq
\mathrm{AUC}\bigl(A,\hat A_{\bm{Z}}\bigr),
\end{equation}
i.e., the edge-recovery AUC computed over the upper-triangular entries of the ground-truth adjacency and the score matrix induced by $\bm{Z}$. $R_{\mathrm{AUC}}$ is an empirical proxy for recoverability, not a consistent estimator of $I(A;\bm{Z})$.

For categorical variables such as $Y$, we form a score matrix proportional to $\mathbbm{1}\{y_i=y_j\}$ (e.g., from the inner product of one-hot label vectors) and define $R_{\mathrm{AUC}}(A;Y)$ as the corresponding AUC. Edge homophily is the average of $\mathbbm{1}\{y_i=y_j\}$ over edges $(i,j)\in\mathcal{E}$; $R_{\mathrm{AUC}}(A;Y)$ is the AUC of this score versus $A$ over all pairs (edges and non-edges), so it is not identical to the edge-homophily ratio. Intuitively, when connected nodes tend to have more similar representations, the resulting similarity matrix is more informative for edge recovery.


\subsection{Which \texttt{ORI-chain} Variables Leak Adjacency Information?}
\label{ssec: understanding-by-auc-proxy}

We first examine leakage patterns under GCN across datasets spanning homophilic and heterophilic regimes.
The goal is to identify which observable variables in the \texttt{ORI-chain} are most useful for recovering adjacency under our AUC-based leakage proxy. The analysis proceeds in four tables: Tabs.~\ref{tab: understanding-MI-term-gcn}--\ref{tab: understanding-MI-term-ensemble-gcn} for GCN and Tabs.~\ref{tab: understanding-MI-term-gprgnn}--\ref{tab: understanding-MI-ensemble-gprgnn} for GPR-GNN; the reader primarily interested in the attack or defense may skip to the observations summarizing each pair.

\begin{table}[tbp]
\centering\fontsize{8}{10}\selectfont
\renewcommand{\arraystretch}{1.2}
\setlength\tabcolsep{4.8pt}
\renewcommand{\minval}{0.18}
\renewcommand{\maxval}{0.90}
\begin{tabular}{cccccccccc}
\toprule
Proxy & Cora & Citeseer & Polblogs & USA & Brazil & AIDS & Texas & Cornell & Wisconsin \\
\midrule
$R_{\mathrm{AUC}}(A; X)$ &  \gradient{0.781} & \gradient{0.881} & N/A & N/A & N/A &  \gradient{0.521}  &\gradient{0.561} &\gradient{0.626} &\gradient{0.626}\\
$R_{\mathrm{AUC}}(A; \bm{H}_A^{\mathrm{all}})$ & \gradient{0.766} & \gradient{0.760} & \gradient{0.763} & \gradient{0.850} & \gradient{0.758} & \gradient{0.584} &\gradient{0.353} &\gradient{0.346} &\gradient{0.574}\\
$R_{\mathrm{AUC}}(A; \bm{\hat{Y}}_{A})$  &  \gradient{0.712} & \gradient{0.743} & \gradient{0.772} & \gradient{0.826} & \gradient{0.732} & \gradient{0.561}  &\gradient{0.27} &\gradient{0.316} &\gradient{0.572}\\
$R_{\mathrm{AUC}}(A;Y)$  & \gradient{0.815} & \gradient{0.779} & \gradient{0.705} & \gradient{0.728} & \gradient{0.613} & \gradient{0.536} &\gradient{0.347} &\gradient{0.422} &\gradient{0.424}\\
\bottomrule
\end{tabular}
\caption{AUC-based adjacency-leakage proxy $R_{\mathrm{AUC}}(A;\bm{Z})$ under a two-layer GCN. We write $\bm{H}_A^{\mathrm{all}}:=\{\bm{H}_A^{(\ell)}\}_{\ell=1}^{L}$ for the full set of hidden layers (distinct from the threat-model-dependent subset $\mathcal{H}_A^{\bullet}$); in this table, $R_{\mathrm{AUC}}(A;\bm{H}_A^{\mathrm{all}})$ uses all hidden layers. Higher values indicate stronger adjacency recoverability. ``N/A'' indicates that the dataset has no node features.}
\label{tab: understanding-MI-term-gcn}
\end{table}

\begin{observation}[Recoverability varies across graph regimes under GCN (AUC-based proxy)]
\label{obs: homophily-amplifies-adjacency-leakage}
Under the $R_{\mathrm{AUC}}$ proxy, Tab.~\ref{tab: understanding-MI-term-gcn} shows that under GCN, adjacency is highly recoverable from multiple \texttt{ORI-chain} variables on homophilic graphs (e.g., Cora: $R_{\mathrm{AUC}}(A;X)\approx 0.78$, $R_{\mathrm{AUC}}(A;\bm{H}_A^{\mathrm{all}})\approx 0.77$, $R_{\mathrm{AUC}}(A;Y)\approx 0.82$).
On more heterophilic graphs such as Texas, the recoverability pattern becomes more variable-dependent, and similarity-based signals are less aligned with cross-class connectivity.

We hypothesize that this pattern arises because GCN's normalized averaging operator smooths features over local neighborhoods: on homophilic graphs, connected nodes already share similar features and labels, so message passing reinforces pairwise similarity along edges, making the inner-product score $\bm{Z}\bm{Z}^\top$ a strong edge predictor. Under heterophily, neighbors have dissimilar labels, so averaging can \emph{reduce} distinguishability between edge and non-edge pairs under a similarity-based score, weakening recoverability.
\end{observation}

A natural baseline is to test whether uniform averaging of decoded similarity scores improves recovery.
For each accessible object $\mathcal K_i\in\mathcal K$, we first map it to a similarity matrix $\hat A_{\mathcal K_i}$ as above and then average these matrices:
\begin{equation}
\bm{\hat{A}}
=
\frac{1}{|\mathcal K|}\sum_{i=1}^{|\mathcal K|}\hat A_{\mathcal K_i}.
\end{equation}
This baseline asks whether averaging decoded similarity matrices improves adjacency recovery.

\begin{table}[tbp]
\centering\fontsize{8}{10}\selectfont
\renewcommand{\arraystretch}{1.2}
\setlength\tabcolsep{5.2pt}
\renewcommand{\minval}{0.3}
\renewcommand{\maxval}{0.91}
\begin{tabular}{ccccccccccccc}
\toprule
$X$ & $\bm{H}_A^{\mathrm{all}}$ & $\bm{\hat{Y}}_{A}$ & $Y$ & Cora & Citeseer &  Polblogs  & USA   & Brazil & AIDS & Texas & Cornell & Wisconsin \\
\midrule
$\checkmark$ & $\checkmark$ &  &  & \gradient{.781}  &\gradient{.881} & \gradient{.763} & \gradient{.850}  & \gradient{.758} & \gradient{.521} &\gradient{.574} &\gradient{.381} & \gradient{.654} \\
$\checkmark$ &  & $\checkmark$ &   &  \gradient{.781} & \gradient{.881} & \gradient{.772} & \gradient{.826} & \gradient{.732} & \gradient{.521} &\gradient{.570} &\gradient{.373} &\gradient{.657}\\
$\checkmark$ &  &  & $\checkmark$  &  \gradient{.849} & \gradient{.907}  & \gradient{.705} & \gradient{.728} & \gradient{.613} & \gradient{.522} &\gradient{.414} &\gradient{.529} &\gradient{.548}\\
\midrule
$\checkmark$ & $\checkmark$ & $\checkmark$ &   &  \gradient{.781}  & \gradient{.881} & \gradient{.763} & \gradient{.848} & \gradient{.756} & \gradient{.521} &\gradient{.576} &\gradient{.360} &\gradient{.647}\\
$\checkmark$ & $\checkmark$ & & $\checkmark$   & \gradient{.849} & \gradient{.907}  & \gradient{.779} & \gradient{.850} & \gradient{.743} & \gradient{.522} &\gradient{.438} &\gradient{.359} &\gradient{.548}\\
$\checkmark$ &  & $\checkmark$ & $\checkmark$   &  \gradient{.842} & \gradient{.907} & \gradient{.785} & \gradient{.842} & \gradient{.730} & \gradient{.522} &\gradient{.432} &\gradient{.350} &\gradient{.546}\\
\midrule
$\checkmark$ & $\checkmark$  & $\checkmark$ & $\checkmark$  & \gradient{.849} & \gradient{.907} & \gradient{.781} & \gradient{.852} & \gradient{.717} & \gradient{.522} &\gradient{.455} &\gradient{.328} &\gradient{.593}\\
\bottomrule
\end{tabular}
\caption{Uniform averaging of decoded similarity scores under the same AUC-based leakage proxy for GCN.
We assume that $X$ is accessible whenever it exists and evaluate all seven combinations obtained by adding at least one variable from $\{\bm{H}_A^{\mathrm{all}},\bm{\hat{Y}}_{A},Y\}$.
``$\checkmark$'' indicates that the corresponding variable is included in the fusion baseline.}
\label{tab: understanding-MI-term-ensemble-gcn}
\end{table}

\begin{observation}[Uniform averaging yields limited gains]
\label{obs: simple-combination-cannot-extract-more-private-information}
Tab.~\ref{tab: understanding-MI-term-ensemble-gcn} shows that uniform averaging often yields limited or even negative marginal gains: adding more variables to the average can \emph{decrease} recovery AUC, particularly on heterophilic graphs.
Two hypotheses can explain this: (1)~decoded similarity scores from different representation spaces are not commensurately scaled, so uniform averaging can be dominated by the highest-magnitude source; (2)~the sources carry largely redundant adjacency information, so additional variables add noise without new signal.
The ablation study in Sec.~\ref{ssec: ablation study} provides evidence primarily for hypothesis~(1): MC-GRA's learned per-layer weighting substantially outperforms uniform averaging, indicating that scale alignment is the dominant factor. Hypothesis~(2) likely plays a secondary role, particularly under GCN where hidden-layer and output representations are highly correlated by construction.
These results motivate a fusion strategy that weights or aligns sources by depth and dependence structure (Sec.~\ref{sec: GRA attack}) rather than uniform averaging.
\end{observation}

\begin{observation}[Label-derived signals and heterophily]
\label{obs: label-signal-misalignment}
Tabs.~\ref{tab: understanding-MI-term-gcn} and \ref{tab: understanding-MI-term-ensemble-gcn} show that $R_{\mathrm{AUC}}(A;Y)$ decreases monotonically with edge homophily ratio across datasets (e.g., $0.82$ on Cora with homophily $0.81$ versus $0.35$ on Texas with homophily $0.13$; cf.\ Tab.~\ref{tab: statistics}). This is expected: our similarity-based proxy scores label agreement, so it can only detect edges between same-label nodes. The operational consequence is that on heterophilic graphs, label-based scores systematically \emph{miss} cross-class edges, and $\bm{\hat{Y}}_{A}$ often provides little benefit beyond feature- or embedding-based signals under naive fusion. MC-GRA+ addresses this gap by incorporating a heterophily-aware prior that targets cross-label connectivity (Sec.~\ref{sec: GRA attack}).
\end{observation}

\noindent To disentangle dataset-level effects from architecture-dependent leakage behavior, we next repeat the analysis under GPR-GNN, an architecture designed to better accommodate heterophily through learnable multi-hop propagation weights.

\begin{table}[tbp]
\centering\fontsize{8}{10}\selectfont
\renewcommand{\arraystretch}{1.2}
\setlength\tabcolsep{4.6pt}
\begin{tabular}{cccccccccc}
\toprule
Proxy & Cora & Citeseer & Polblogs & USA & Brazil & AIDS & Texas & Cornell & Wisconsin \\
\midrule
$R_{\mathrm{AUC}}(A; X)$ & \gradient{0.794} & \gradient{0.886} & N/A & N/A & N/A & \gradient{0.499}  &\gradient{0.55} &\gradient{0.626} &\gradient{0.626}\\
$R_{\mathrm{AUC}}(A; \bm{H}_A^{\mathrm{all}})$ & \gradient{0.663} & \gradient{0.512} & \gradient{0.849} & \gradient{0.628} & \gradient{0.815} & \gradient{0.623} &\gradient{0.757} &\gradient{0.493} &\gradient{0.433} \\
$R_{\mathrm{AUC}}(A; \bm{\hat{Y}}_{A})$  & \gradient{0.5} & \gradient{0.5} & \gradient{0.802} & \gradient{0.275} & \gradient{0.528} & \gradient{0.5} &\gradient{0.66} &\gradient{0.448} &\gradient{0.407} \\
$R_{\mathrm{AUC}}(A;Y)$  & \gradient{0.816} & \gradient{0.779} & \gradient{0.714} & \gradient{0.75} & \gradient{0.601} & \gradient{0.536} &\gradient{0.347} &\gradient{0.422} &\gradient{0.424} \\
\bottomrule
\end{tabular}
\caption{AUC-based adjacency-leakage proxy $R_{\mathrm{AUC}}(A;\bm{Z})$ under a two-layer GPR-GNN. As in Tab.~\ref{tab: understanding-MI-term-gcn}, $\bm{H}_A^{\mathrm{all}}=\{\bm{H}_A^{(\ell)}\}_{\ell=1}^{L}$. Higher values indicate stronger adjacency recoverability. ``N/A'' indicates that the dataset has no node features.}
\label{tab: understanding-MI-term-gprgnn}
\end{table}

\begin{table}[tbp]
\centering\fontsize{8}{10}\selectfont
\renewcommand{\arraystretch}{1.2}
\setlength\tabcolsep{4.0pt}
\renewcommand{\minval}{0.26}
\renewcommand{\maxval}{0.943}
\begin{tabular}{ccccccccccccc}
\toprule
$X$ & $\bm{H}_A^{\mathrm{all}}$ & $\bm{\hat{Y}}_{A}$ & $Y$ & Cora & Citeseer &  Polblogs   & USA   & Brazil & AIDS & Texas & Cornell & Wisconsin \\
\midrule
$\checkmark$ & $\checkmark$ &  &  & \gradient{0.816} & \gradient{0.886} & \gradient{0.849} & \gradient{0.628} & \gradient{0.815} & \gradient{0.625} & \gradient{0.734} & \gradient{0.562} & \gradient{0.482} \\
$\checkmark$ &  & $\checkmark$ &   & \gradient{0.794} & \gradient{0.886} & \gradient{0.802} & \gradient{0.275} & \gradient{0.528} & \gradient{0.499} & \gradient{0.647} & \gradient{0.522} & \gradient{0.495} \\
$\checkmark$ &  &  & $\checkmark$  & \gradient{0.892} & \gradient{0.903} & \gradient{0.714} & \gradient{0.75} & \gradient{0.601} & \gradient{0.537} & \gradient{0.414} & \gradient{0.528} & \gradient{0.548} \\
\midrule
$\checkmark$ & $\checkmark$ & $\checkmark$ &   & \gradient{0.816} & \gradient{0.886} & \gradient{0.84} & \gradient{0.463} & \gradient{0.735} & \gradient{0.625} &\gradient{0.72} &\gradient{0.51} & \gradient{0.458} \\
$\checkmark$ & $\checkmark$ & & $\checkmark$   & \gradient{0.899} & \gradient{0.903} & \gradient{0.863} & \gradient{0.763} & \gradient{0.755} & \gradient{0.597} & \gradient{0.536} & \gradient{0.527} & \gradient{0.485} \\
$\checkmark$ &  & $\checkmark$ & $\checkmark$   & \gradient{0.892} & \gradient{0.903} & \gradient{0.828} & \gradient{0.627} & \gradient{0.556} & \gradient{0.537} & \gradient{0.494} & \gradient{0.479} & \gradient{0.481} \\
\midrule
$\checkmark$ & $\checkmark$  & $\checkmark$ & $\checkmark$  & \gradient{0.899} & \gradient{0.903} & \gradient{0.86} & \gradient{0.685} & \gradient{0.7} & \gradient{0.597} &\gradient{0.536} &\gradient{0.479} &\gradient{0.456} \\
\bottomrule
\end{tabular}
\caption{Uniform averaging of decoded similarity scores under the same AUC-based leakage proxy for GPR-GNN.
We assume that $X$ is accessible whenever it exists and evaluate all seven combinations obtained by adding at least one variable from $\{\bm{H}_A^{\mathrm{all}},\bm{\hat{Y}}_{A},Y\}$.
``$\checkmark$'' indicates that the corresponding variable is included in the fusion baseline.}
\label{tab: understanding-MI-ensemble-gprgnn}
\vspace{-15pt}
\end{table}

\begin{observation}[GPR-GNN vs.\ GCN: heterophily and ensembling]
\label{obs: homophily-amplifies-adjacency-leakage-gprgnn}
Under GPR-GNN (Tabs.~\ref{tab: understanding-MI-term-gprgnn} and \ref{tab: understanding-MI-ensemble-gprgnn}), output-derived recoverability is generally weaker on several heterophilic graphs than under GCN; $\bm{\hat{Y}}_{A}$ tends to carry less adjacency information on heterophilic datasets, while hidden-state patterns remain architecture- and dataset-dependent. Including $\bm{H}_A^{\mathrm{all}}$ or $\bm{\hat{Y}}_{A}$ in the averaging baseline often fails to improve recovery on heterophilic graphs; the strongest baselines tend to rely on $X$ and $Y$, contrasting with GCN where multiple chain variables contribute on homophilic graphs.
\end{observation}

\subsection{Training Dynamics of Adjacency Leakage in the \texttt{ORI-chain}}
\label{ssec: tracking by graph information plane}

We next ask how leakage and utility evolve during training, rather than at a single trained state. To that end, we track both adjacency recoverability and task performance over time.
Specifically, for $\bm{Z}\in\{\bm{H}_A^{(1)},\bm{H}_A^{(2)},\bm{\hat{Y}}_{A}\}$, we monitor the adjacency-leakage proxy $R_{\mathrm{AUC}}(A;\bm{Z})$ together with the corresponding prediction accuracy.

Inspired by information-plane analyses~\citep{tishby2015deep, shwartz2017opening}, we define the \emph{graph information plane} (GIP) as a two-dimensional plot that maps each representation $\bm{Z}$ computed from $(A,X)$ to $\bigl(R_{\mathrm{AUC}}(A;\bm{Z}),\mathrm{Acc}(Y;\bm{Z})\bigr)$. Here, $R_{\mathrm{AUC}}(A;\bm{Z})$ measures adjacency recoverability and $\mathrm{Acc}(Y;\bm{Z})$ measures task utility. 

The quantity $\mathrm{Acc}(Y;\bm{Z})$ is obtained by applying the trained downstream subnetwork from that layer onward. Specifically, for $\bm{H}_A^{(1)}$, we apply the trained second message-passing layer and the final classifier head; for $\bm{H}_A^{(2)}$, we apply the trained classifier head; and for $\bm{\hat{Y}}_A$, we use the model output directly. Here, the vertical axis is not the ideal information-theoretic quantity $I(Y;\bm{Z})$, but its empirical proxy $\mathrm{Acc}(Y;\bm{Z})$. The GIP is therefore better interpreted as an accuracy--leakage plane than as a strict information plane. For consistency with the information-plane literature, some figure labels may still denote the vertical axis as $I(Y;\bm{Z})$; throughout this paper, these labels should be understood as referring to the empirical proxy $\mathrm{Acc}(Y;\bm{Z})$. This proxy has limitations: accuracy is threshold-dependent and may saturate near the top of the plane, so trajectories in that regime should be interpreted with caution. A softer metric, such as cross-entropy loss, could in principle reveal finer-grained dynamics. We nevertheless retain accuracy for interpretability and provide cross-entropy based--GIP plots for Cora in Appendix~\ref{sec: full qualitative results} as supplementary evidence.


\begin{figure}[t!]
\centering
\includegraphics[width=0.48\linewidth]{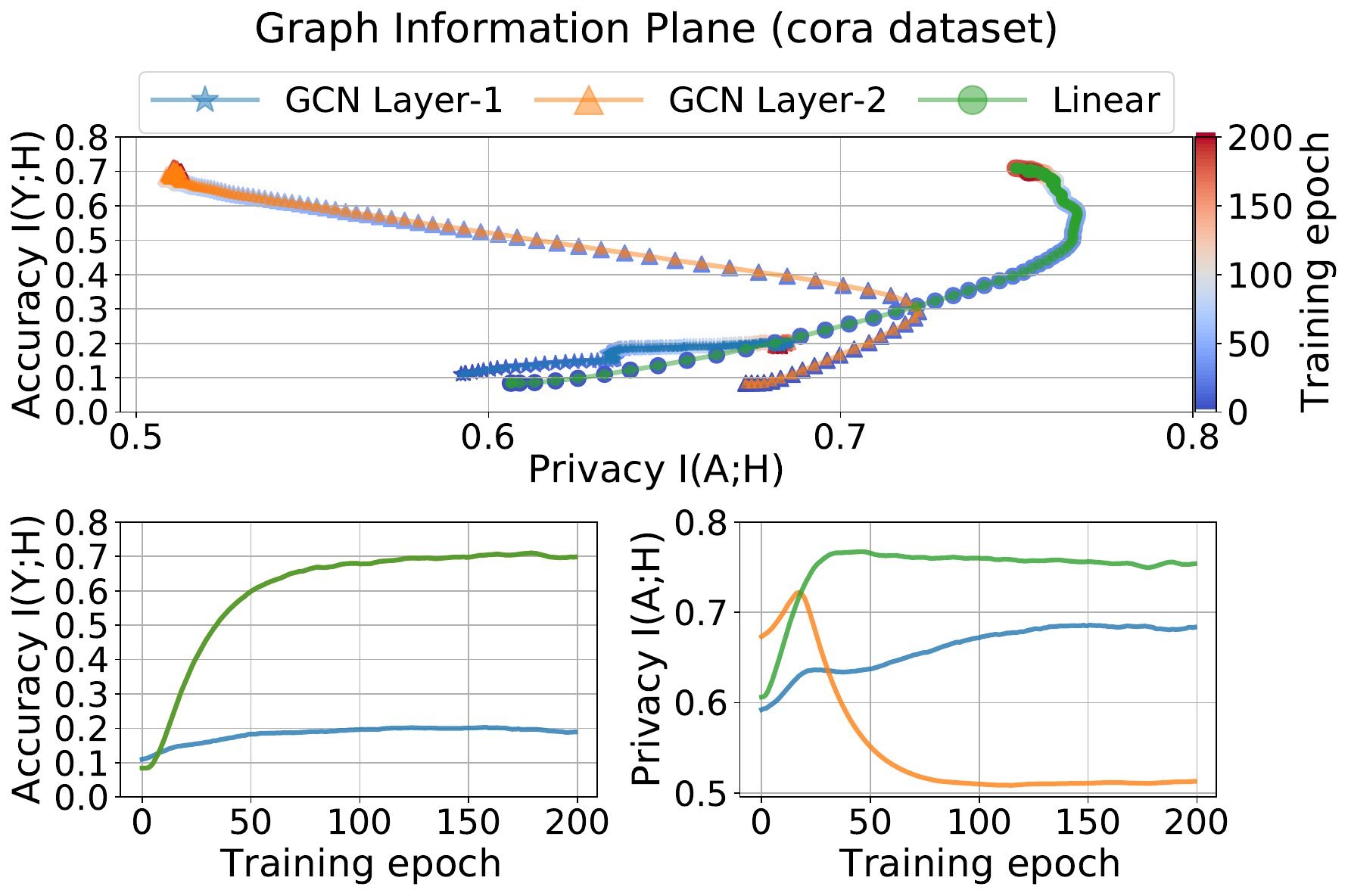}
\includegraphics[width=0.48\linewidth]{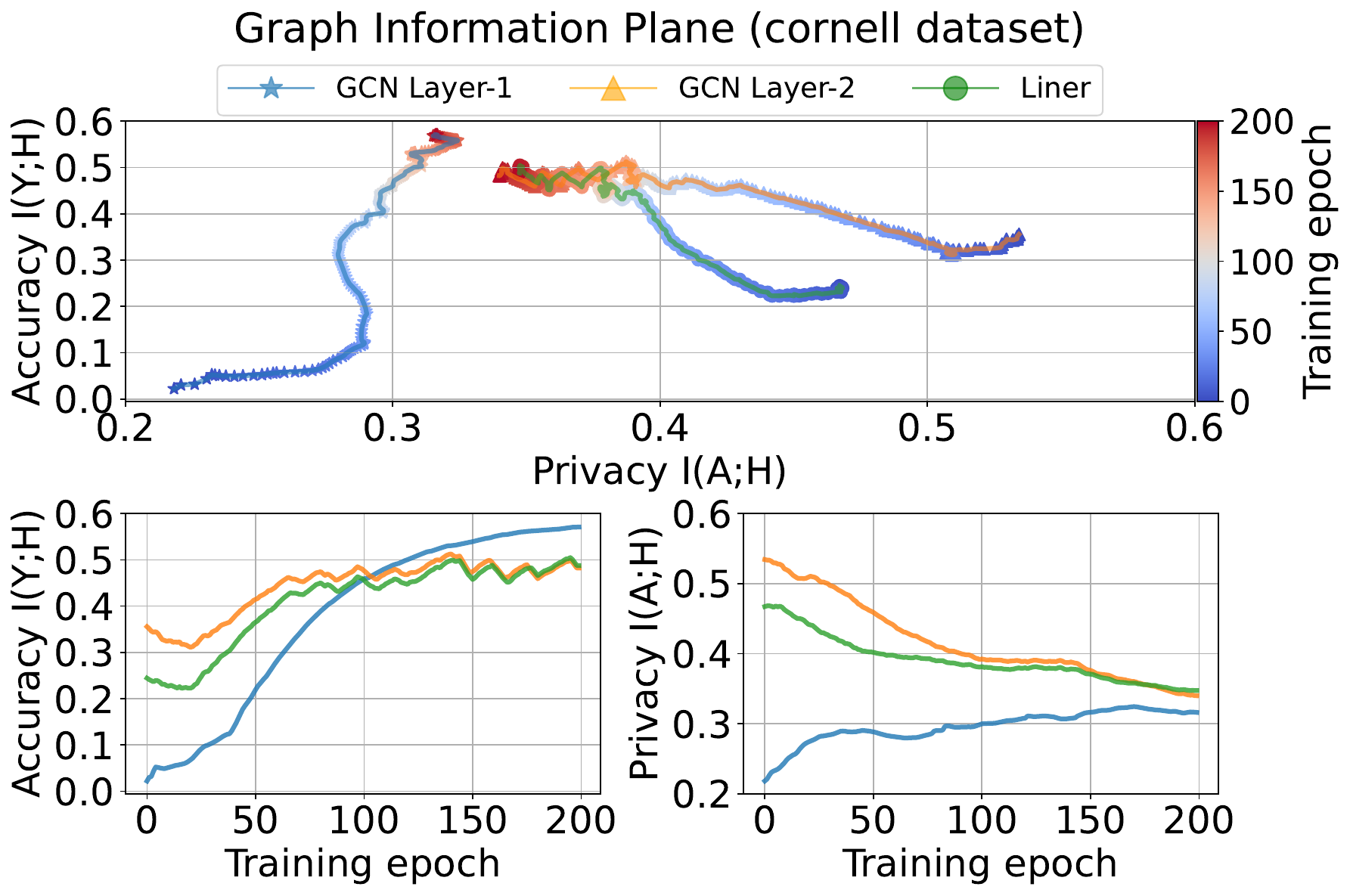}
\caption{Graph information plane for a two-layer GCN on the homophilic graph Cora and the heterophilic graph Cornell.
The curves for layer 2 and the linear head overlap in the vertical (accuracy) coordinate because $\bm{\hat{Y}}_{A}=\mathrm{Linear}(\bm{H}_A^{(2)})$ yields the same accuracy; $R_{\mathrm{AUC}}(A;\bm{Z})$ can differ between $\bm{H}_A^{(2)}$ and $\bm{\hat{Y}}_{A}$. Different line styles (dashed vs.\ solid) and markers distinguish the overlapping curves.}
\label{fig: graph-information-plane-gcn}
\end{figure}

\begin{figure}[t!]
\centering
\includegraphics[width=0.48\linewidth]{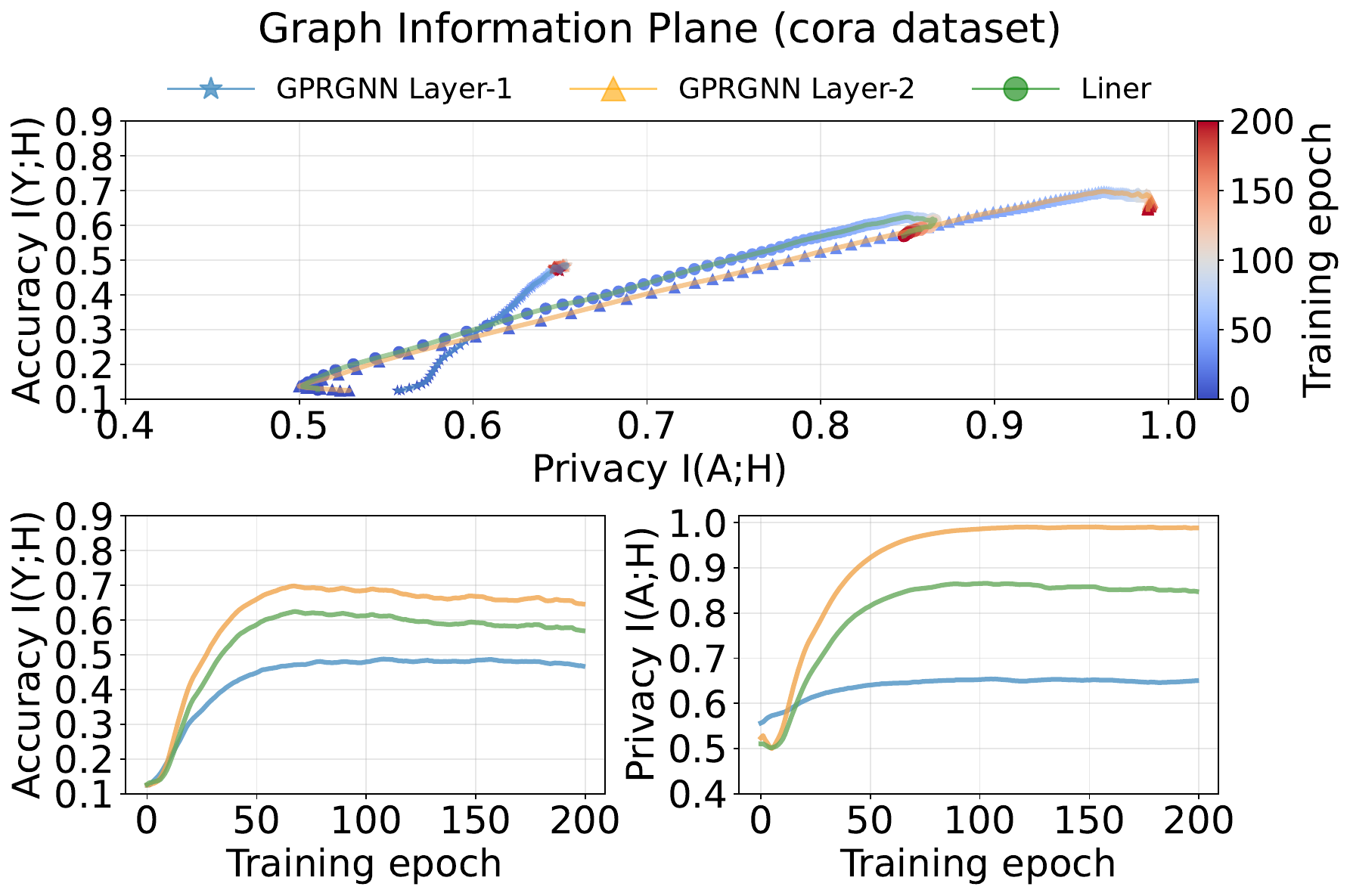}
\includegraphics[width=0.48\linewidth]{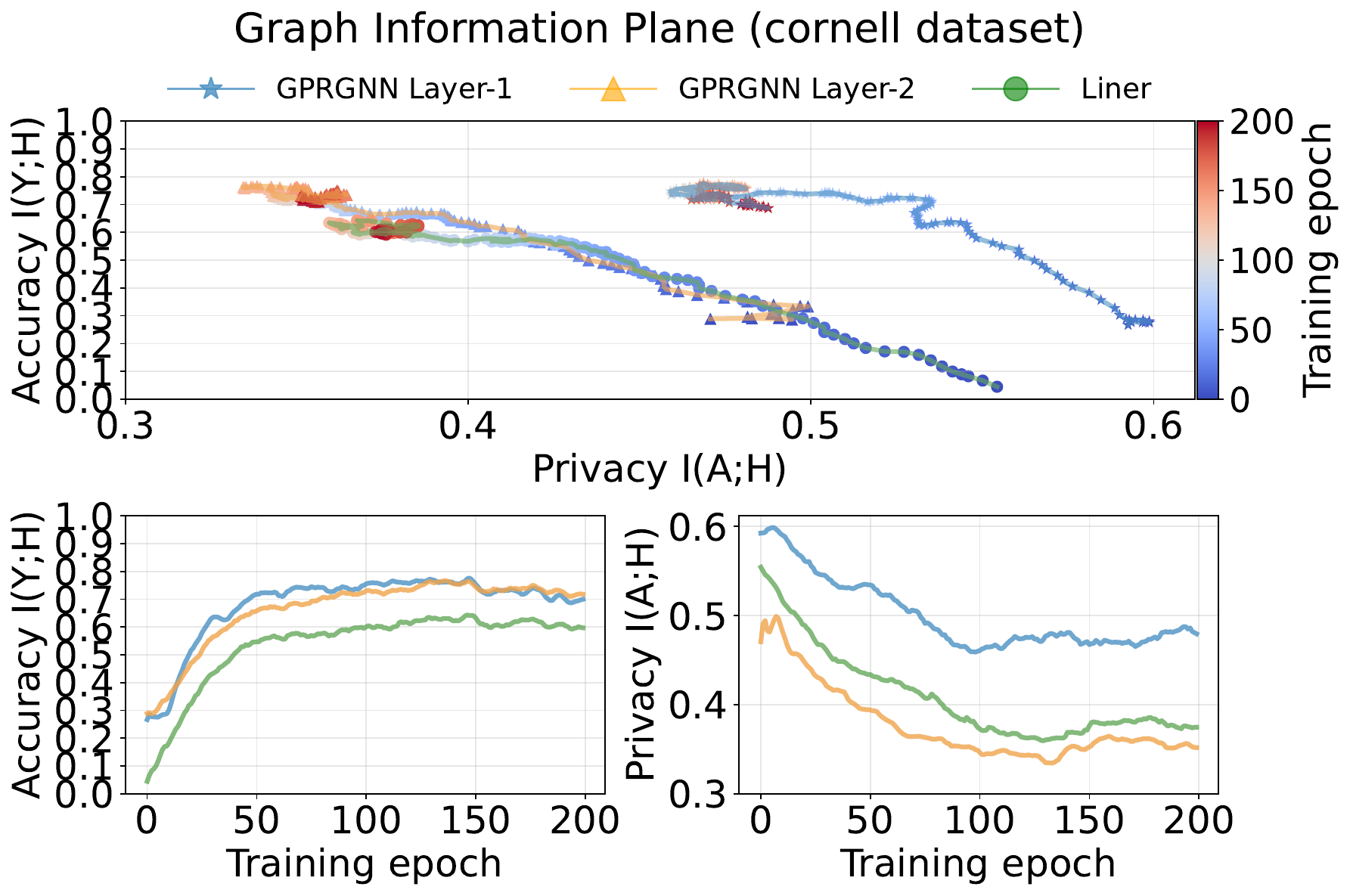}
\caption{Graph information plane for a two-layer GPR-GNN on the homophilic graph Cora and the heterophilic graph Cornell.}
\label{fig: graph-information-plane-gpr-gnn}
\end{figure}

At each training epoch, we compute these proxy coordinates for $\bm{Z}\in\{\bm{H}_A^{(1)},\bm{H}_A^{(2)},\bm{\hat{Y}}_{A}\}$ and visualize the resulting trajectories.

\begin{observation}[Training dynamics on the graph information plane]
\label{obs: training-dynamics-gip}
\textbf{GCN} (Fig.~\ref{fig: graph-information-plane-gcn}): the trajectories exhibit a two-stage pattern on both homophilic and heterophilic graphs. On the homophilic graph, adjacency recoverability first increases, then decreases while accuracy continues to improve (a V-shaped trajectory). The heterophilic example shows a qualitatively similar pattern with different magnitude and timing. This pattern is reminiscent of, though not necessarily explained by, the fit-then-compress behavior discussed in information-bottleneck analyses~\citep{shwartz2017opening}; the underlying mechanisms in GNNs may well differ~\citep{saxe2019information}.
\textbf{GPR-GNN} (Fig.~\ref{fig: graph-information-plane-gpr-gnn}): the trajectories are more topology-dependent. On the heterophilic example, $R_{\mathrm{AUC}}(A;\bm{Z})$ decreases rapidly while accuracy increases; on the homophilic example, both quantities increase more consistently. This is consistent with the weaker output-derived leakage observed on several heterophilic datasets in Tabs.~\ref{tab: understanding-MI-term-gprgnn} and \ref{tab: understanding-MI-ensemble-gprgnn}.
\end{observation}

Overall, GCN fits local structural information early and then shifts toward representations more useful for prediction than for similarity-based edge recovery. GPR-GNN tends to suppress adjacency-aligned signals when local topology is less reliable for prediction, while preserving them on homophilic graphs.

%% file: sections/6-method-attack.tex
\section{Markov Chain-based Graph Reconstruction Attack (MC-GRA)}
\label{sec: GRA attack}

\begin{figure*}[t!]
\centering
\includegraphics[width=0.8\textwidth]{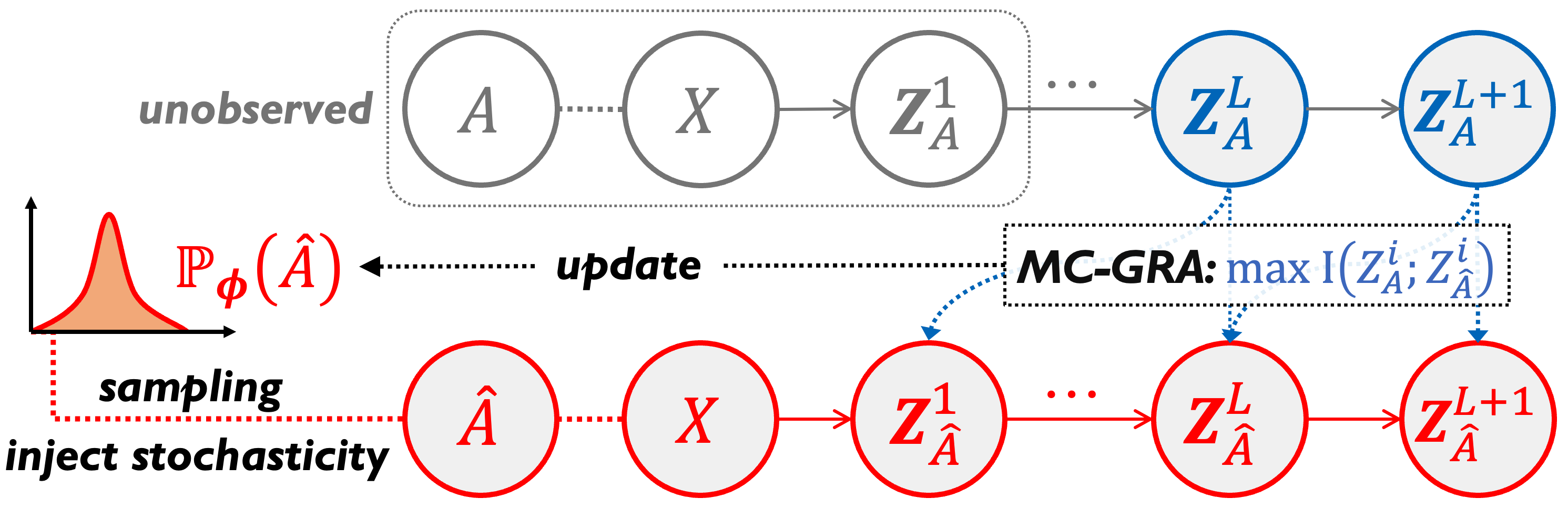}
\caption{
\textbf{The MC-GRA attack framework.} 
\emph{Forward (in red):} an estimated adjacency $\bm{\hat{A}}$ is sampled from a parameterized distribution $\mathbb{P}_{\bm{\phi}}(\bm{\hat{A}})$ and perturbed via injected stochasticity to enable differentiable optimization and exploration. The perturbed $\bm{\hat{A}}$ is then propagated through the target model to generate the corresponding \texttt{GRA-chain} variables. 
\emph{Backward (in blue):} the sampling parameters $\bm{\phi}$ are updated by maximizing the MC-GRA objective in Eq.~\eqref{eqn: MC-GRA}, which provides supervision by matching aligned \texttt{ORI-chain}/\texttt{GRA-chain} representations while regularizing the decisiveness of $\bm{\hat{A}}$.
}
\label{fig: markov-attack}
\end{figure*}

Building on the chain-level analysis in Sec.~\ref{sec: overview}, we now formalize the attack. The key insight from that analysis is that deeper layers can carry less adjacency signal under the proxy, motivating depth-aligned matching: comparing earlier layers can provide stronger alignment signals than matching only outputs. MC-GRA applies to threat models in which at least one of $Y$, $\bm{H}_A^{(\ell)}$, or $\bm{\hat{Y}}_A$ is released. The feature-only case $\mathcal{K}=\{X\}$ is not covered by the chain-matching objective, since no released variable from the GNN computation is available for alignment. Observations~\ref{obs: homophily-amplifies-adjacency-leakage}-\ref{obs: homophily-amplifies-adjacency-leakage-gprgnn} show that adjacency leakage varies across chain variables and graph regimes, while naive fusion fails to exploit complementary signals.
This motivates an attack that flexibly combines multiple alignment signals at matching chain depths.

The attacker's goal is to reconstruct the training adjacency $A$ from the observable side information $\mathcal{K}$ and the known target model. Side information $\mathcal{K}$ can be heterogeneous (inputs $X$, labels $Y$, intermediate representations $\bm{H}_{A}^{(\ell)}$, predictions $\bm{\hat{Y}}_{A}$, or subsets thereof).
The target reconstructed adjacency is binary, but the optimization is carried out over a continuous relaxation to permit gradient-based updates; in the forward computation, this relaxed parameterization is converted into a binary adjacency via thresholding or stochastic sampling.
We address these challenges by developing MC-GRA as a chain-matching framework: the attacker searches for a candidate $\bm{\hat{A}}$ whose induced \texttt{GRA-chain} aligns with the released \texttt{ORI-chain} states at corresponding depths. We first define the attack objective (Sec.~\ref{ssec: attack objective}), then describe practical parameterizations and optimization (Sec.~\ref{ssec: attack implementation}), and finally present information-theoretic interpretations (Sec.~\ref{ssec: attack understanding}).

\subsection{The Optimization Objectives for Attack}
\label{ssec: attack objective}

\paragraph{A motivating unstructured objective.}
Before introducing the chain-matching objective, we first consider a simple \emph{motivating heuristic} that highlights why the structure of the GNN forward computation matters. Let $\mathcal{K}\subseteq\{X,Y,\bm{H}_{A},\bm{\hat{Y}}_{A}\}$ denote the attacker’s available side information.%
\footnote{The candidate sources yield $2^{4}-1=15$ nonempty combinations when $\bm{H}_{A}$ is treated as a single object; the count grows further when individual layers are distinguished.}
Here $\bm{H}_{A}$ denotes the set (or concatenation) of hidden layers; in Eq.~\eqref{eqn: MC-GRA} we use layer-specific $\bm{H}_{A}^{(\ell)}$. In the population view, $\bm{\hat{A}}$ is a random variable induced by the attack given $\mathcal{K}$, so $I(\bm{\hat{A}};\mathcal{K}_i)$ is well defined under the joint distribution over graphs, labels, and attack randomness.
A naive attack objective is to seek an adjacency estimate $\bm{\hat{A}}$ that has strong dependence on each released object:
\begin{equation}
\bm{\hat{A}}^{*}
\in
\arg\max_{\bm{\hat{A}}}
\sum_{\mathcal{K}_i\in\mathcal{K}} \alpha_i\, I(\bm{\hat{A}};\mathcal{K}_i).
\label{eqn: Basic-GRA}
\end{equation}
This objective is intended only as an \emph{unstructured heuristic}, not as a strong attacker baseline. When the elements of $\mathcal{K}$ are highly collinear, as with multiple hidden layers, summing marginal dependence terms can overcount redundant information and ignore higher-order interactions among the released variables. A stronger attacker could instead use a learned fusion module over concatenated observations. Our point here is more specific: without leveraging the layered computation induced by the GNN, such flat objectives cannot distinguish observations generated at the same stage of message passing from correlations inherited indirectly from earlier layers.
A similar caveat applies to including $X$: maximizing $I(\bm{\hat{A}};X)$ does not directly enforce fidelity to the private adjacency $A$ and may instead reward spurious coupling between the recovered graph and the observed features when $A$ and $X$ are only weakly related.

The limitation of the unstructured objective is not merely that it is weak, but that it treats $\mathcal{K}$ as a flat collection of released variables. By contrast, MC-GRA performs \emph{depth-aligned cross-chain matching}, comparing variables at corresponding depths of the \texttt{ORI-chain} and \texttt{GRA-chain}. In this way, it exploits the compositional structure of the GNN forward computation rather than aggregating dependence terms without regard to their origin.

\paragraph{The Markov chain-based attack objective.}
The target model $f_{\bm{\theta}}(\cdot)$ induces an \texttt{ORI-chain} under the private adjacency $A$ and a corresponding \texttt{GRA-chain} under a candidate adjacency $\bm{\hat{A}}$, with the same chain-variable mapping as in Sec.~\ref{sec: overview} ($\bm{Z}_{A}^{(0)}=X$, $\bm{Z}_{A}^{(\ell)}=\bm{H}_{A}^{(\ell)}$ for $\ell=1,\dots,L$, $\bm{Z}_{A}^{(L+1)}=\bm{\hat{Y}}_{A}$). We define $\bm{Z}_{\bm{\hat{A}}}^{(0)},\dots,\bm{Z}_{\bm{\hat{A}}}^{(L+1)}$ analogously by rerunning the same model on $\bm{\hat{A}}$. In practice, only a subset of these variables is observed; the rest are latent. The MI terms in the objective are defined with respect to the joint data-generating distribution over $(A,X,Y)$ and the quantities induced by the target model and the attack procedure; for a single observed graph they are estimated via the surrogates in Sec.~\ref{ssec: attack implementation}.

We introduce the following notation for the attacker's observable variables. Let $\mathcal{I}_{H}(\mathcal{K})\subseteq\{1,\dots,L\}$ denote the set of hidden-layer indices whose representations are available under the threat model; this set can be empty, a singleton, or contain multiple layers. Let $\mathbbm{1}_{\hat Y}(\mathcal{K})\in\{0,1\}$ indicate whether the prediction scores $\bm{\hat{Y}}_{A}$ are available, and $\mathbbm{1}_{Y}(\mathcal{K})\in\{0,1\}$ indicate whether the true labels $Y$ are available. The attack requires at least one released supervision object from $\{Y,\bm{H}_{A}^{(\ell)},\bm{\hat{Y}}_{A}\}$; otherwise the objective below contains no positive alignment term. The population-level MC-GRA objective is
\begin{equation}
\begin{split}
\mathcal{L}_{\mathrm{MC\text{-}GRA}}(\bm{\hat{A}};\mathcal{K})
:={}&
\sum_{\ell\in\mathcal{I}_{H}(\mathcal{K})}
\alpha_{p}^{(\ell)} I\!\left(\bm{H}_{A}^{(\ell)};\bm{H}_{\bm{\hat{A}}}^{(\ell)}\right)\\
&+\mathbbm{1}_{\hat Y}(\mathcal{K})\,\alpha_{o}\, I\!\left(\bm{\hat{Y}}_{A};\bm{\hat{Y}}_{\bm{\hat{A}}}\right)
+ \mathbbm{1}_{Y}(\mathcal{K})\,\alpha_{s}\, I\!\left(Y;\bm{\hat{Y}}_{\bm{\hat{A}}}\right)
- \alpha_{c} R_{c}(\bm{\hat{A}}).
\end{split}
\label{eqn: MC-GRA}
\end{equation}
\paragraph{Level convention.} The alignment terms in Eq.~\eqref{eqn: MC-GRA} are stated at the \emph{population} level (under the data-generating distribution over $(A,X,Y)$). The sharpening term $R_c(\bm{\hat{A}})$ is an \emph{instance-level} regularizer that does not have a direct population-level MI interpretation; it is included in Eq.~\eqref{eqn: MC-GRA} for notational compactness. In the implemented attack (Sec.~\ref{ssec: attack implementation}), every MI term is replaced by the differentiable dependence surrogate $d(\cdot,\cdot)=\texttt{HSIC}$ with normalized inputs (chosen for computational efficiency; the defense uses \texttt{CKA} for its scale invariance---see Sec.~\ref{ssec: attack implementation} for details and alternatives), evaluated on the observed graph instance, while $R_c$ is computed directly from the current relaxed edge probabilities.

\begin{remark}[Surrogate vs.\ mutual information]
Maximizing the cross-chain surrogates is \emph{not} equivalent to maximizing $I(A;\bm{\hat{A}})$; the caveats established in Remark~\ref{rem: surrogate-scope} apply throughout this section and are not repeated.
\end{remark}

The generic weights $\alpha_i$ in Eq.~\eqref{eqn: Basic-GRA} are here split into layer-specific hidden-layer weights $\alpha_p^{(\ell)}$, output weight $\alpha_o$, supervision weight $\alpha_s$, and sharpening coefficient $\alpha_c$. In all experiments, we set $\alpha_p^{(\ell)}=\alpha_p$ for all $\ell$ (equal across layers), so only four scalar coefficients $\{\alpha_p,\alpha_o,\alpha_s,\alpha_c\}$ are tuned; the layer-specific notation is retained for generality. Specifically:
\begin{itemize}[itemsep=0pt]
    \item $\alpha_p$: hidden-layer alignment weight;
    \item $\alpha_o$: output alignment weight;
    \item $\alpha_s$: supervision (label) weight;
    \item $\alpha_c$: sharpening coefficient;
    \item $\alpha_h$: heterophily prior weight (MC-GRA+ only; scalar, not layer-indexed).
\end{itemize}

We optimize over $\bm{\hat{A}}$ and take $\bm{\hat{A}}^*\in\arg\max_{\bm{\hat{A}}}\mathcal{L}_{\mathrm{MC\text{-}GRA}}(\bm{\hat{A}};\mathcal{K})$ (any element of the argmax in case of ties). Unavailable supervision terms are omitted automatically through the index set and indicators. Note that the label term $\alpha_s I(Y;\bm{\hat{Y}}_{\bm{\hat{A}}})$ is structurally different from the other alignment terms: it pairs the true labels $Y$ (a fixed supervision target, not a chain state) with the \texttt{GRA-chain} predictions, rather than matching two chain variables at the same depth. This is appropriate because $Y$ is not produced by the GNN forward pass but serves as an external reference that the reconstructed adjacency should reproduce through the model.
The remaining alignment terms compare variables at the same depth of the two chains, while $R_c(\bm{\hat{A}})$ denotes a \emph{sharpening} (decisiveness) regularizer. In implementation, $R_c(\bm{\hat{A}})$ is the edgewise Bernoulli entropy $\sum_{i<j} h_b(\bm{\hat{A}}_{ij})$, where $h_b(p)=-p\log p-(1-p)\log(1-p)$; minimizing $-R_c$ pushes edge probabilities toward $\{0,1\}$, encouraging decisive rather than diffuse edge scores.
As in Sec.~\ref{sec: overview}, we summarize the released chain variables by
\begin{equation}
\mathcal{S}_{A}
=
\{\bm{H}_{A}^{(\ell)}:\ell\in\mathcal{I}_{H}(\mathcal{K})\}
\cup
\bigl\{\bm{\hat{Y}}_{A}: \mathbbm{1}_{\hat Y}(\mathcal{K})=1\bigr\},
\end{equation}
with
\begin{equation}
\mathcal{S}_{\bm{\hat{A}}}
=
\{\bm{H}_{\bm{\hat{A}}}^{(\ell)}:\ell\in\mathcal{I}_{H}(\mathcal{K})\}
\cup
\bigl\{\bm{\hat{Y}}_{\bm{\hat{A}}}: \mathbbm{1}_{\hat Y}(\mathcal{K})=1\bigr\}.
\end{equation}
Here, if $Y$ is observed, it acts as a fixed supervision target paired with $\bm{\hat{Y}}_{\bm{\hat{A}}}$. 

\paragraph{Heterophily-aware attack.}
Most graph reconstruction attacks implicitly benefit from homophily, namely that adjacent nodes tend to share the same label. In heterophilic graphs, edges often connect nodes with different labels, so similarity-based reconstruction criteria can become systematically misaligned with the true adjacency. MC-GRA+ augments the standard objective with a heterophily-aware prior constructed from the released prediction scores $\bm{\hat{Y}}_{A}$, using predicted label disagreement as a proxy for cross-label connectivity.

Let $\Pi_{\mathrm{sym},0}(\cdot)$ denote symmetrization with the diagonal set to zero. Since the rows of $\bm{\hat{Y}}_{A}$ are probability vectors, the entries of $\bm{\hat{Y}}_{A}\bm{\hat{Y}}_{A}^{\top}$ already lie in $[0,1]$. We form the label-similarity score matrix $\bm{S}_{Y}=\Pi_{\mathrm{sym},0}(\bm{\hat{Y}}_{A}\bm{\hat{Y}}_{A}^{\top})$ without an additional sigmoid, so as to preserve dynamic range, and define
\begin{equation}
\bm{P}_{\text{hetero}}
\;=\;
\Pi_{\mathrm{sym},0}\!\bigl(\bm{1}\bm{1}^{\top}-\bm{S}_{Y}\bigr),
\end{equation}
so that node pairs with more dissimilar predicted labels receive larger scores. This matrix is fixed for a given threat-model instance, as it is computed entirely from the released $\bm{\hat{Y}}_{A}$ and does not depend on the candidate $\bm{\hat{A}}$.

This construction is heuristic rather than calibrated. It is useful only to the extent that the released predictions are informative about the underlying class structure. In particular, when the target GNN has poor predictive accuracy, which is a realistic possibility on strongly heterophilic graphs without specialized architectures, predicted disagreement may fail to reflect true cross-label connectivity and can even mislead the attack. 
Accordingly, MC-GRA+ is best viewed as a helpful prior when $\bm{\hat{Y}}_{A}$ is reasonably reliable, rather than as a universally valid correction for heterophily. The later theorems in Sec.~\ref{ssec: attack understanding} clarify when label-derived quantities of this kind are informative and when they lose information.

We incorporate this score by augmenting Eq.~\eqref{eqn: MC-GRA} with an additional term. Under the induced population view where released predictions are random, one may analogously consider $I(\bm{P}_{\text{hetero}};\bm{\hat{A}})$; 
in practice we only optimize the instance-level surrogate $d(\bm{P}_{\text{hetero}},\bm{\hat{A}})$, a pointwise alignment score between the fixed observed $\bm{P}_{\text{hetero}}$ and the candidate $\bm{\hat{A}}$:
\begin{equation}
\mathcal{L}_{\mathrm{MC\text{-}GRA+}}(\bm{\hat{A}};\mathcal{K})
\;=\;
\mathcal{L}_{\mathrm{MC\text{-}GRA}}(\bm{\hat{A}};\mathcal{K})
\;+\;
\mathbbm{1}_{\hat Y}(\mathcal{K})\,\alpha_{h}\,d(\bm{P}_{\text{hetero}},\bm{\hat{A}}).
\label{eqn: MC-GRA-hetero}
\end{equation}

MC-GRA+ targets heterophilic edges by biasing reconstruction toward cross-label connectivity, which homophily-based criteria systematically underweight. 
We maximize $d(\bm{P}_{\text{hetero}},\bm{\hat{A}})$ as the differentiable alignment signal. 
When predictions are not released ($\bm{\hat{Y}}_A \notin \mathcal{K}$), the heterophily prior cannot be constructed, and MC-GRA+ reduces to MC-GRA.

\subsection{Implementation Details of MC-GRA (+)}
\label{ssec: attack implementation}

Translating the population-level objective (Eq.~\eqref{eqn: MC-GRA}) into a practical algorithm requires three approximation steps: (i)~replacing the intractable MI terms with differentiable dependence surrogates (HSIC for the attack; CKA for the defense), (ii)~parameterizing the discrete candidate adjacency $\bm{\hat{A}}$ through a continuous relaxation to enable gradient-based optimization, and (iii)~injecting stochasticity into the relaxed adjacency to facilitate exploration and prevent premature convergence to local optima. We describe each step below.

\paragraph{Parameterizing Eq.~\eqref{eqn: MC-GRA} and Eq.~\eqref{eqn: MC-GRA-hetero}.}
To optimize the attack objectives with gradient-based methods, we model the recovered adjacency $\bm{\hat{A}}$ as a random variable drawn from a learnable distribution $\mathbb{P}_{\bm{\phi}}(\bm{\hat{A}})$, parameterized by $\bm{\phi}$. The present section focuses on the white-box rerun setting, which assumes an attack-side input $X_{\mathrm{atk}}$ compatible with forwarding the target model. In the standard white-box setting this input is the true feature matrix $X$; on featureless graphs it is the same default input used by the target architecture.

Because the paper assumes simple undirected graphs, every parameterization enforces symmetry and a zero diagonal. Throughout this section, $\sigma(\cdot)$ denotes the elementwise sigmoid function $\sigma(x)=1/(1+e^{-x})$; the generic activation in GNN layer updates is denoted $\rho$ (see Theorem~\ref{theorem: reducing MI with two chains}). For any matrix $M\in\mathbb{R}^{N\times N}$, define
\begin{equation}
\Pi_{\mathrm{sym},0}(M)
\;=\;
\frac{1}{2}(M+M^\top)-\operatorname{diag}\!\left(\frac{1}{2}(M+M^\top)\right).
\end{equation}
At each iteration, we sample
\begin{equation}
\bm{\hat{A}} \sim \mathbb{P}_{\bm{\phi}}(\bm{\hat{A}}),
\end{equation}
propagate the sampled (and possibly perturbed) adjacency through the \texttt{GRA-chain}, and update $\bm{\phi}$ by stochastic gradient ascent on the chosen surrogate objective. In implementation, we optimize the expected surrogate objective under $\mathbb{P}_{\bm{\phi}}(\bm{\hat{A}})$ using a single-sample Monte Carlo estimate per iteration (reparameterization trick). For the attack method description, we take HSIC as the default dependence surrogate; Tab.~\ref{tab: ablation-similarity-metrics} ablates DP, HSIC, CKA, and KDE, and shows that HSIC and CKA generally perform best. The defense (Sec.~\ref{sec: GRA defense}) uses CKA instead, for its scale invariance under representation scaling during training; the choice of default surrogate therefore differs between attack and defense.

The complexity term in Eq.~\eqref{eqn: MC-GRA} is instantiated in practice by the edgewise Bernoulli entropy of the relaxed edge-probability matrix:
\begin{equation}
R_c(\bm{\hat{A}})
\;=\;
\sum_{1\le i<j\le N} h_b(\bm{\hat{A}}_{ij}),
\end{equation}
where $h_b(p)=-p\log p-(1-p)\log(1-p)$. Thus, every MI term is replaced by the chosen dependence surrogate $d(\cdot,\cdot)$, while the complexity term is computed directly from the current relaxed edge probabilities. This regularizer penalizes high-entropy edge-probability assignments, encouraging decisive edge scores rather than diffuse uncertain solutions, and is well defined for all three parameterizations.

We consider three instantiations of $\mathbb{P}_{\bm{\phi}}(\bm{\hat{A}})$, listed in increasing order of expressiveness:
\begin{itemize}[leftmargin=*]
\setlength\itemsep{0.1em}

\item \textbf{Direct optimization (degenerate distribution).}
We optimize an unconstrained score matrix $\bm{\Omega}\in\mathbb{R}^{N\times N}$ and set
\begin{equation}
\bm{\phi}\equiv \bm{\Omega},
\qquad
\bm{\hat{A}}=\Pi_{\mathrm{sym},0}\!\bigl(\sigma(\bm{\Omega})\bigr)\in[0,1]^{N\times N}.
\end{equation}
Equivalently, $\mathbb{P}_{\bm{\phi}}$ is a point mass (Dirac delta) at $\Pi_{\mathrm{sym},0}\!\bigl(\sigma(\bm{\Omega})\bigr)$.

\item \textbf{Gaussian parameterization (factorized noise model).}
We use a fully factorized Gaussian over an unconstrained score matrix. For each $i<j$, we sample independent Gaussian logits
\begin{equation}
G_{ij}\sim\mathcal{N}(\mu_{ij},\varsigma_{ij}^{2}),
\qquad
\bm{\phi}=\{(\mu_{ij},\varsigma_{ij}) : 1\le i<j\le N\},
\end{equation}
where $\varsigma_{ij}>0$ denotes the standard deviation (distinct from the dependence surrogate $d(\cdot,\cdot)$ and the sigmoid $\sigma(\cdot)$), and set $G_{ji}=G_{ij}$ and $G_{ii}=0$. Using the reparameterization trick,
\begin{equation}
\bm{G}
=
\bm{\mu}+\bm{\epsilon}\odot\bm{\varsigma},
\qquad
\epsilon_{ij}\stackrel{\mathrm{i.i.d.}}{\sim}\mathcal{N}(0,1)\ \ \text{for }i<j,
\end{equation}
and define
\begin{equation}
\bm{\hat{A}}
=
\Pi_{\mathrm{sym},0}\!\bigl(\sigma(\bm{G})\bigr).
\end{equation}
This factorized parameterization facilitates exploration around a mean score matrix while keeping the sampling step differentiable with respect to $(\bm{\mu},\bm{\varsigma})$.

\item \textbf{Generator-based parameterization (implicit model).}
We parameterize $\bm{\hat{A}}$ implicitly through a learnable generator $f_{\bm{\phi}}(\cdot)$. The generator may share the target architecture; initialization from the target weights is used only in the white-box setting, while otherwise the generator is initialized independently. As one practical canonical choice, we use the identity matrix as a seed adjacency and compute
\begin{equation}
\bm{H}_{I}=f_{\bm{\phi}}(\bm{I},X_{\mathrm{atk}}),
\end{equation}
then induce an adjacency via
\begin{equation}
\bm{\hat{A}}
=
\Pi_{\mathrm{sym},0}\!\bigl(\sigma(\bm{H}_{I}\bm{H}_{I}^{\top})\bigr)\in[0,1]^{N\times N}.
\end{equation}
The identity graph is used only as a neutral seed from which the generator produces node embeddings; other seed graphs could also be used. Sharing the target architecture is a heuristic inductive bias: it biases the recovered adjacency toward structures capable of reproducing the target model's own representation geometry. This construction ties edge probabilities to learned node representations and captures structural regularities beyond independent entrywise parameterizations.
\end{itemize}
The direct parameterization is simplest but lacks exploration; the Gaussian parameterization adds controlled stochasticity that facilitates escaping local optima; the generator-based parameterization is the most expressive and can capture structural regularities, but is harder to optimize. Unless otherwise specified, we use the Gaussian parameterization in all experiments (Sec.~\ref{sec: experiments}), as it provides a good trade-off between expressiveness and optimization stability. As a practical guideline: the Gaussian parameterization is preferred on featureless or sparse-feature graphs, where the generator-based variant lacks a meaningful input signal; the generator-based parameterization is preferred when informative node features are available, as it can exploit the target architecture's inductive bias (see the ablation in Tab.~\ref{tab: ablation-parameterization}).

\paragraph{Optimizing Eq.~\eqref{eqn: MC-GRA} and Eq.~\eqref{eqn: MC-GRA-hetero} with injected stochasticity.}
Along the forward computation, both $\bm{\hat{A}}$ and $X_{\mathrm{atk}}$ influence variables such as $\bm{H}_{\bm{\hat{A}}}^{(\ell)}$ and $\bm{\hat{Y}}_{\bm{\hat{A}}}$. Without perturbation, the optimization can drift toward degenerate solutions in which the candidate adjacency plays only a limited role, because highly informative features may already suffice to mimic released hidden states or outputs without faithfully recovering the underlying edges. Injected stochasticity mitigates this effect by reducing brittle co-adaptation between features and adjacency, thereby encouraging solutions whose agreement with the released variables is more robust to perturbation. Empirically, this improves reconstruction quality.

Concretely, we perturb the attack inputs as
\begin{equation}
\tilde{X}=\Phi_{\mathrm{atk}}(X_{\mathrm{atk}}, X_{\epsilon}),
\qquad
\tilde{\bm{A}}=\Phi_{\mathrm{atk}}(\bm{\hat{A}}, \bm{A}_{\epsilon}),
\end{equation}
where $X_{\epsilon}$ and $\bm{A}_{\epsilon}$ denote injected noise. Here $\Phi_{\mathrm{atk}}$ denotes the \emph{attack-side} perturbation operator, implemented through a binary Concrete relaxation for the adjacency (see below). This is distinct from the defense-side operator in Sec.~\ref{sec: GRA defense}: there, $\Phi_{\mathrm{def}}$ denotes \emph{edge-drop noise} (DropEdge) used as a lightweight stochastic regularizer during training, while the main additional cost of the defense comes from the layerwise dependence penalties introduced later. We keep this distinction explicit throughout to avoid conflating attack-side relaxation noise with defense-side regularization. In experiments, adjacency relaxation noise is always used, while feature perturbation is applied only when features are available.

We implement adjacency stochasticity by sampling (or relaxing) binary edges. For each unordered node pair $i<j$ with probability $p_{ij}\triangleq \bm{\hat{A}}_{ij}\in[0,1]$, we sample one relaxed edge variable and then mirror it to enforce symmetry:
\begin{equation}
a_{ij}\sim \mathrm{Bern}(p_{ij}),
\qquad
a_{ij}=a_{ji}\in\{0,1\}.
\end{equation}
Since Bernoulli sampling is non-differentiable, we adopt a binary Concrete relaxation~\citep{kool2019stochastic,xie2019reparameterizable}. Writing $\operatorname{logit}(p_{ij})=\log p_{ij}-\log(1-p_{ij})$, we sample $u_{ij}\sim\mathrm{Uniform}(0,1)$ and set
\begin{equation}
\tilde{a}_{ij}
=
\sigma\!\left(
\frac{\operatorname{logit}(p_{ij})+\log u_{ij}-\log(1-u_{ij})}{\tau}
\right),
\end{equation}
where $\tau>0$ is the temperature hyperparameter controlling the sharpness of the relaxation: as $\tau\to 0$, the relaxation converges to an exact Bernoulli sample (but gradients degenerate), while large $\tau$ yields a smooth but biased approximation~\citep{maddison2017concrete,jang2017categorical}. In all experiments we use a fixed $\tau=0.5$ without annealing; a decreasing schedule $\tau_t\to 0$ could sharpen the final solution but was not necessary for competitive AUC in our setting. The resulting relaxed adjacency is mirrored across the diagonal and then zeroed on the diagonal before being passed through the target model.

\paragraph{Computational cost.}
Each attack iteration requires one forward pass through the target GNN on the candidate (perturbed) adjacency and one evaluation of the dependence surrogate and sharpening term. With a dense relaxed candidate adjacency, the forward pass is typically $O(L(N^2 d + N d^2))$ for $L$ layers and hidden dimension $d$; if a sparse parameterization or thresholded candidate is used, the cost can be reduced accordingly. The surrogate is computed over node-wise representations and costs $O(N^2)$ for the Gram-style terms. The overall cost per iteration is dominated by the GNN forward pass and the $N\times N$ dependence evaluation.

\paragraph{The algorithm.}
Equipped with the parameterizations and perturbations described above, we summarize the resulting optimization procedure in Algorithm~\ref{alg: GRA}. The two methods share the same optimization loop; MC-GRA+ differs only by adding the heterophily-score term when released predictions are available.

\begin{algorithm}[ht]
\caption{Generalized graph reconstruction attack (MC-GRA / MC-GRA+).}
\begin{algorithmic}[1]
\REQUIRE Target model $f_{\bm{\theta}}$, attack-side feature input $X_{\mathrm{atk}}$ used to rerun the model, released side information $\mathcal{K}$, differentiable dependence estimator $d(\cdot,\cdot)$, coefficients $\{\alpha_p^{(\ell)}\},\alpha_o,\alpha_s,\alpha_c$, and (optional) $\alpha_h$
\STATE Initialize $\mathbb{P}_{\bm{\phi}}(\bm{\hat{A}})$; collect released variables from $\mathcal{K}$
\STATE If MC-GRA+ ($\alpha_h>0$) and $\bm{\hat{Y}}_{A}\in\mathcal{K}$, precompute $\bm{P}_{\text{hetero}}=\Pi_{\mathrm{sym},0}(\bm{1}\bm{1}^{\top}-\Pi_{\mathrm{sym},0}(\bm{\hat{Y}}_{A}\bm{\hat{Y}}_{A}^{\top}))$; \textbf{otherwise} set $\alpha_h=0$ (MC-GRA+ reduces to MC-GRA when $\bm{\hat{Y}}_{A}\notin\mathcal{K}$)
\FOR{$t=1,\dots,n$}
    \STATE Sample $\bm{\hat{A}}\sim \mathbb{P}_{\bm{\phi}}(\bm{\hat{A}})$; set $\tilde{X}=\Phi_{\mathrm{atk}}(X_{\mathrm{atk}}, X_{\epsilon})$, $\tilde{\bm{A}}=\Phi_{\mathrm{atk}}(\bm{\hat{A}}, \bm{A}_{\epsilon})$
    \STATE Run \texttt{GRA-chain} forward on $f_{\bm{\theta}}(\tilde{\bm{A}},\tilde{X})$ to get terms induced by $\tilde{\bm{A}}$
    \STATE Compute surrogate from Eq.~\eqref{eqn: MC-GRA-hetero} (or Eq.~\eqref{eqn: MC-GRA} if $\alpha_h=0$), with MI terms replaced by $d(\cdot,\cdot)$
    \STATE Update $\bm{\phi}$ by gradient ascent on the surrogate
\ENDFOR
\STATE \textbf{return} the mean score matrix $\mathbb{E}_{\bm{\phi}^{*}}[\bm{\hat{A}}]$ (thresholded if a binary graph is required). Concretely: for the direct parameterization, this is $\Pi_{\mathrm{sym},0}(\sigma(\bm{\Omega}^*))$; for the Gaussian parameterization, $\Pi_{\mathrm{sym},0}(\sigma(\bm{\mu}^*))$; for the generator-based parameterization, a single forward pass through the converged generator.
\end{algorithmic}
\label{alg: GRA}
\end{algorithm}

\begin{figure*}[t!]
\centering
\subfloat[Standard training.]
{\includegraphics[width=0.31\textwidth]{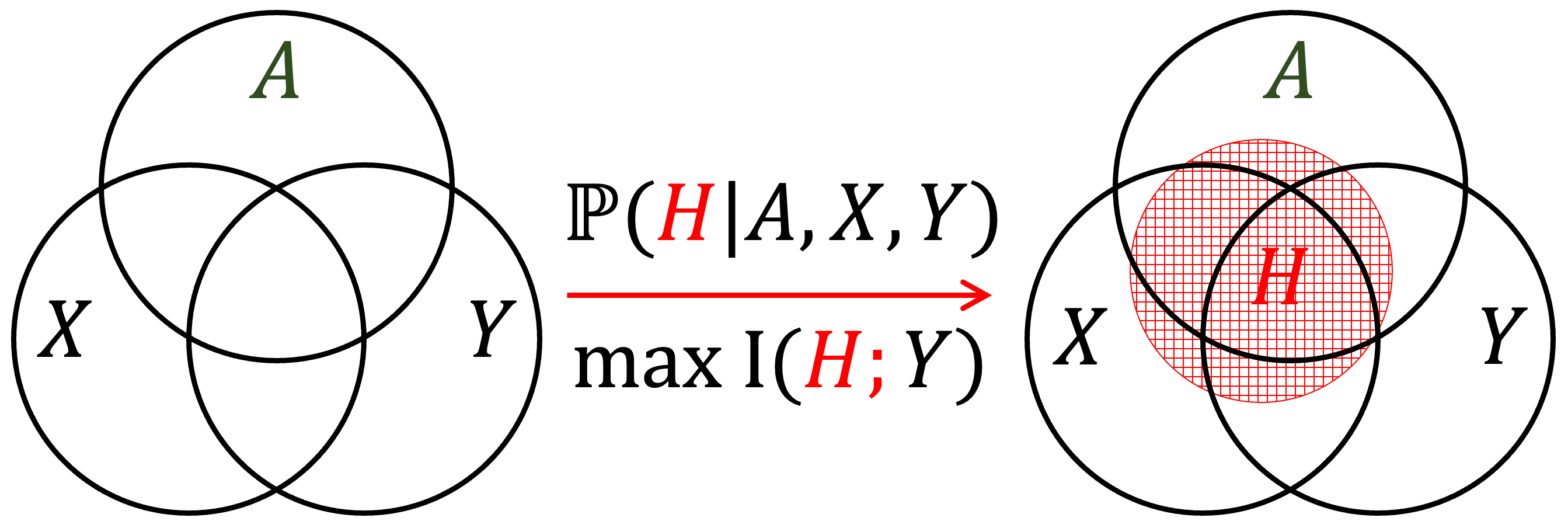}
\label{fig: three-ball-standard-training}}
\hfill
\subfloat[MC-GRA attacking.]
{\includegraphics[width=0.31\textwidth]{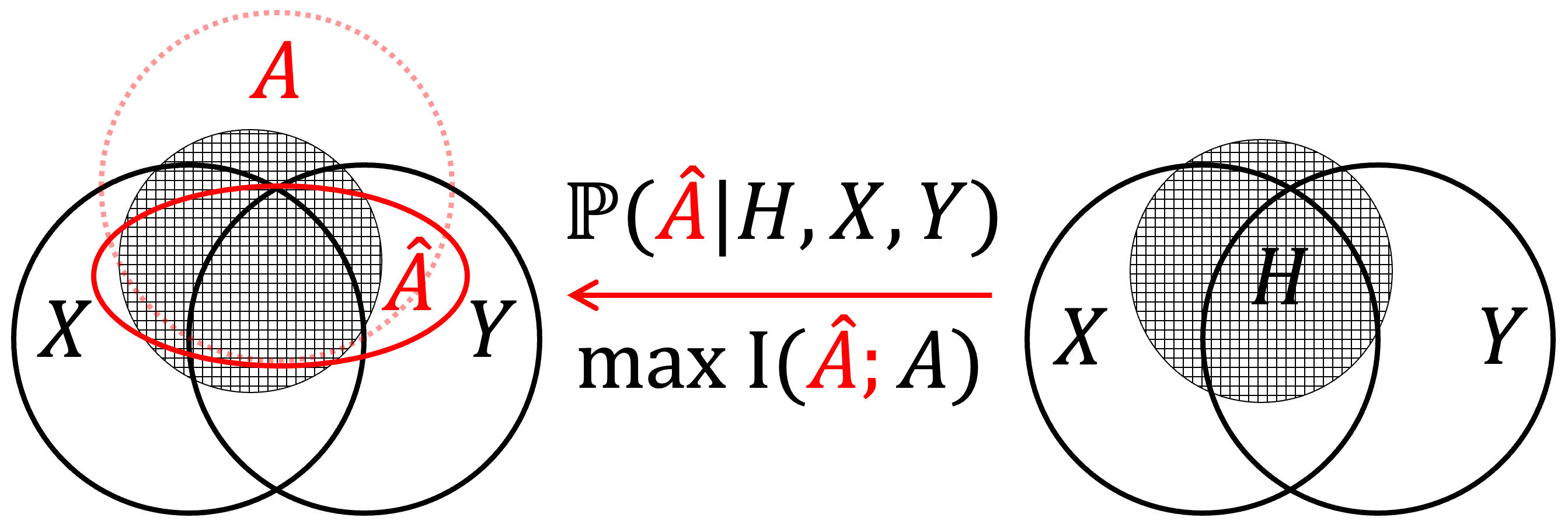}
\label{fig: three-ball-GRA}}
\hfill
\subfloat[MC-GPB defensive training.]
{\includegraphics[width=0.31\textwidth]{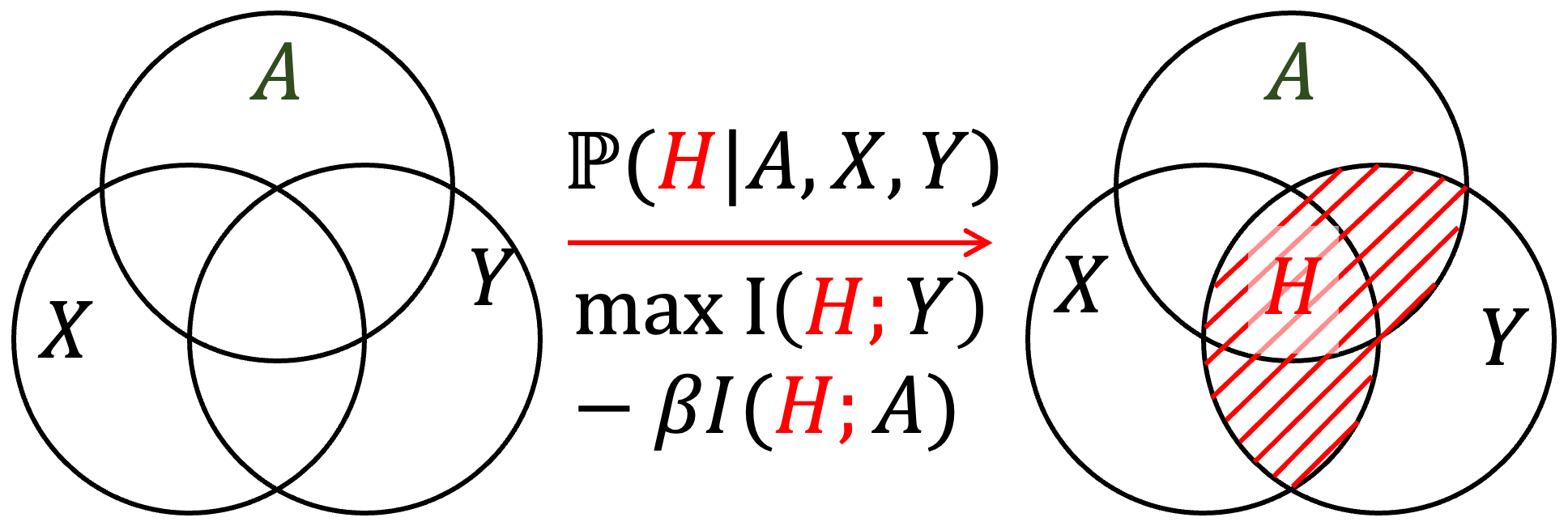}
\label{fig: three-ball-GPB}}
\caption{
\textbf{An information-theoretic view of training, attack, and defense.} Illustration of how information flows and is (progressively) compressed along the GNN-induced Markov chains during standard training, how MC-GRA exploits dependence between aligned \texttt{ORI-chain}/\texttt{GRA-chain} variables for reconstruction, and how MC-GPB reduces adjacency leakage by discouraging dependence between hidden representations and the adjacency.
}
\label{fig: information properties of training and attack}
\end{figure*}

\subsection{Theoretical Understanding}
\label{ssec: attack understanding}

This subsection presents a \emph{stylized population analysis}. It separates three components: a cross-chain dependence result, a conditional fidelity bound based on quantized representations, and an upper bound determined by the attacker’s knowledge set. These results explain why matching released chain variables can help reconstruction, but they do not imply that the practical surrogate objectives optimized in Sec.~\ref{ssec: attack implementation} are exact mutual informations.

\paragraph{Marginal value beyond features.}
When the attacker reruns the target model with a feature input, the incremental explanatory value of a candidate adjacency $\bm{\hat{A}}$ beyond features for an induced variable $\bm{Z}$ can be expressed through
\begin{equation}
I(\bm{Z};\bm{\hat{A}}\mid X)
=
I(\bm{Z};X,\bm{\hat{A}}) - I(\bm{Z};X).
\label{eqn: cmi-marginal}
\end{equation}
The identity shows that the marginal value of $\bm{\hat{A}}$ beyond $X$ is exactly this conditional mutual information. Without perturbations, the optimizer can find degenerate solutions where $\bm{\hat{A}}$ is unnecessary because $X$ alone suffices to reproduce $\bm{Z}$ (e.g., when representations are overdetermined by features). Injected stochasticity is intended to reduce such feature-only shortcuts so that the adjacency-dependent component of $\bm{Z}$ relevant to reconstruction is preserved and the optimizer is encouraged to recover structure rather than exploit coincidental feature-adjacency correlations.

\paragraph{Information contraction across aligned chain variables.}
We first consider a stylized population setting with deterministic layer maps and well-defined mutual information. For clarity, the theorem below is stated for the hidden-state chain $\{\bm{Z}_{A}^{(i)}\}_{i=0}^{L}$; the same argument applies if an output layer is appended as state $L+1$. The key condition enabling the data processing inequality is that both chains share the same input $X=\bm{Z}_{A}^{(0)}=\bm{Z}_{\bm{\hat{A}}}^{(0)}$, so at layer $0$ the two chains coincide exactly and their cross-chain dependence is maximal relative to later layers; each deterministic layer map can only contract this dependence.

\begin{theorem}[Layerwise non-increase of MI across two chains (GCN-type layer maps)]
\label{theorem: reducing MI with two chains}
We state the theorem for the following simple algebraic template (which includes GCN). Analogous extensions to other deterministic architectures (e.g., GAT, GraphSAGE) are plausible---the proof uses only the data processing inequality for deterministic maps---but we omit them here. Let $A$ and $\bm{\hat{A}}$ be fixed (deterministic) adjacency matrices. Let $\{(\bm{Z}_{A}^{(i)},\bm{Z}_{\bm{\hat{A}}}^{(i)})\}_{i=0}^{L}$ be random vectors defined on a common probability space, with the randomness arising from the shared input features $\bm{Z}_{A}^{(0)}=\bm{Z}_{\bm{\hat{A}}}^{(0)}=X$. For each layer index $i=0,1,\dots,L-1$, define the deterministic layer maps
\begin{equation}
T_{A}^{(i)}(\bm{z}) \;\triangleq\; \rho\bigl(\psi(A)\,\bm{z}\,\bm{\theta}^{(i)}\bigr),
\qquad
T_{\bm{\hat{A}}}^{(i)}(\bm{z}) \;\triangleq\; \rho\bigl(\psi(\bm{\hat{A}})\,\bm{z}\,\bm{\theta}^{(i)}\bigr),
\end{equation}
and assume the layerwise updates satisfy
\begin{equation}
\bm{Z}_{A}^{(i+1)}=T_{A}^{(i)}(\bm{Z}_{A}^{(i)}),
\qquad
\bm{Z}_{\bm{\hat{A}}}^{(i+1)}=T_{\bm{\hat{A}}}^{(i)}(\bm{Z}_{\bm{\hat{A}}}^{(i)}).
\end{equation}
Here $\bm{\theta}^{(i)}$ is a fixed weight matrix, $\psi(\cdot)$ is a fixed graph-dependent linear operator (for GCN, $\psi(A)$ is the normalized adjacency, optionally with self-loop as in Tab.~\ref{tab: GNN-architectures}), and $\rho$ is a fixed componentwise activation (we use $\rho$ to distinguish it from the sigmoid $\sigma$). Assume $H(X)<\infty$ (so that $I(\bm{Z}_{A}^{(0)};\bm{Z}_{\bm{\hat{A}}}^{(0)})=H(X)$ is finite; this is satisfied, \eg, when $X$ is discrete or has a bounded density). Then the mutual informations $I(\bm{Z}_{A}^{(i)};\bm{Z}_{\bm{\hat{A}}}^{(i)})$ are finite for all $i=0,\dots,L$, and for every $i=0,\dots,L-1$,
\begin{equation}
I\!\left(\bm{Z}_{A}^{(i+1)}\,;\,\bm{Z}_{\bm{\hat{A}}}^{(i+1)}\right)
\;\le\;
I\!\left(\bm{Z}_{A}^{(i)}\,;\,\bm{Z}_{\bm{\hat{A}}}^{(i)}\right).
\label{eqn: MI-nonincrease-two-chains}
\end{equation}

\noindent
\textbf{Proof (sketch).} Eq.~\eqref{eqn: MI-nonincrease-two-chains} follows from the data processing inequality applied to the
deterministic mappings $\bm{Z}_{A}^{(i)}\mapsto \bm{Z}_{A}^{(i+1)}$ and $\bm{Z}_{\bm{\hat{A}}}^{(i)}\mapsto \bm{Z}_{\bm{\hat{A}}}^{(i+1)}$.
A detailed proof is in Appendix~\ref{ssec: proof of reducing MI with two chains}.
\end{theorem}

\paragraph{Equality and scope.}
Equality in Eq.~\eqref{eqn: MI-nonincrease-two-chains} holds if both $T_{A}^{(i)}$ and $T_{\bm{\hat{A}}}^{(i)}$ are bijective a.e., or generally if the joint mapping $(\bm{Z}_{A}^{(i)},\bm{Z}_{\bm{\hat{A}}}^{(i)})\mapsto(\bm{Z}_{A}^{(i+1)},\bm{Z}_{\bm{\hat{A}}}^{(i+1)})$ is invertible on the support up to null sets. In settings with non-invertible nonlinearities (e.g., ReLU) and non-degenerate inputs, the inequality is strict. We state and prove the theorem for the GCN-type template. The conclusion extends to any architecture whose layer map is a deterministic measurable function of the representation (given fixed adjacency and parameters), since the proof relies only on the data processing inequality for deterministic maps. In particular, for GAT, the layer map $T_A^{(i)}(\bm{z})$ includes data-dependent attention coefficients $\alpha(\bm{z},A)$, but remains a deterministic function of $\bm{z}$ given fixed $A$ and $\bm{\theta}$; likewise for GraphSAGE, whose aggregation is deterministic once the adjacency and features are fixed. However, we have not verified the regularity conditions (finiteness of MI, well-defined densities) for these architectures, so we state the extension as a conjecture rather than a proven result; the experimental evaluation in Sec.~\ref{sec: experiments} empirically confirms the non-increase pattern for both architectures.

When $X$ is continuous, $I(\bm{Z}_A^{(0)};\bm{Z}_{\bm{\hat{A}}}^{(0)})=H(X)$ may be infinite under differential entropy. The non-increase result remains valid ($+\infty \ge \cdots$), but the theorem is most informative for characterizing the \emph{rate} of decrease between deeper layers where MI is finite.

The contraction in Eq.~\eqref{eqn: MI-nonincrease-two-chains} is stated for the hidden-state chain $\{\bm{Z}_A^{(i)}\}_{i=0}^L$. If a nonlinear classifier head (\eg, softmax) is appended as layer $L+1$, the data processing inequality still ensures non-increase, but the specific GCN-type template above does not directly cover it.

\paragraph{A conditional fidelity bound from quantized representation matching.}
The next result is also stylized. It treats the feature input as fixed side information. Because exact equality of high-dimensional continuous representations occurs with probability zero, we introduce a measurable quantizer $Q$ (e.g., a uniform partition of the representation space into $|\mathcal{H}|$ cells, or $k$-means-style quantization) that discretizes the representation space into finitely many bins; quantization makes the matching event $P_e(x)$ well-defined.

\begin{theorem}[Auxiliary conditional fidelity bound via quantized representations]
\label{theorem: GRA attack lower bound}
Treat the feature input $X$ as fixed side information, so the statement is pointwise in the realized feature value. Let $\bar{\bm{H}}_{A}=Q(\bm{H}_{A})$ and $\bar{\bm{H}}_{\bm{\hat{A}}}=Q(\bm{H}_{\bm{\hat{A}}})$ for a measurable quantizer $Q$ with finite support $\mathcal{H}$ satisfying $2\le |\mathcal{H}|<\infty$. For almost every $X=x$, define
\begin{equation}
P_e(x)\triangleq \mathbb{P}\!\left(\bar{\bm{H}}_{A}\neq \bar{\bm{H}}_{\bm{\hat{A}}}\mid X=x\right),
\end{equation}
where $\neq$ denotes matrix-wise inequality. Then
\begin{equation}
I(A;\bm{\hat{A}}\mid X=x)
\ge
I(\bar{\bm{H}}_{A};\bar{\bm{H}}_{\bm{\hat{A}}}\mid X=x)
\ge
H(\bar{\bm{H}}_{A}\mid X=x)-H_b\!\bigl(P_e(x)\bigr)-P_e(x)\log\!\bigl(|\mathcal{H}|-1\bigr).
\end{equation}
Replacing $\log(|\mathcal{H}|-1)$ by $\log|\mathcal{H}|$ yields a weaker but simpler bound.

\noindent
\textbf{Interpretation.}
The theorem should be read as a technical auxiliary result for a discretized surrogate problem, not as a literal analysis of the continuous optimization used by MC-GRA. Its role is to show that, after quantization, low mismatch between hidden representations implies a nontrivial lower bound on conditional reconstruction fidelity.

\noindent
\textbf{Proof (sketch).}
Given $X$, the quantized representations are deterministic functions of $A$ and $\bm{\hat{A}}$, so the first inequality is an application of data processing. The second is obtained by applying Fano's inequality to the problem of recovering $\bar{\bm{H}}_{A}$ from $\bar{\bm{H}}_{\bm{\hat{A}}}$. Full details of the proof are given in Appendix~\ref{ssec: proof of GRA attack lower bound}.
\end{theorem}

Combined with Theorem~\ref{theorem: reducing MI with two chains}, this result provides a concrete justification for matching hidden layers rather than only outputs: earlier layers admit a tighter bound because representation entropy $H(\bar{\bm{H}}_A \mid X)$ is higher and mismatch probability $P_e(x)$ can be lower before contraction reduces the signal. The bound is informative only when the quantization preserves meaningful structure and the attacker achieves low mismatch probability $P_e(x)$; with coarse quantization or high $P_e(x)$, the right-hand side can be negative or vacuous. The quantizer $Q$ is a theoretical device, not part of the attack procedure; the bound can in principle be optimized over $Q$ to yield the tightest result, but we do not pursue this here. To give a concrete sense of scale: for a $d_{\mathrm{hid}}$-dimensional hidden space ($d_{\mathrm{hid}}=16$ in our experiments) with $k$-bit uniform quantization per dimension, $|\mathcal{H}|=2^{d_{\mathrm{hid}} k}$ and $\log|\mathcal{H}|=16k$ nats. With $k=1$ (binary sign quantization), $|\mathcal{H}|=2^{16}\approx 65{,}000$; the bound is informative only when $P_e(x)\lesssim H(\bar{\bm{H}}_A\mid x)/\log|\mathcal{H}|$. For our attack on Cora (where the reconstruction AUC exceeds 0.9), the empirical mismatch rate between quantized ORI-chain and GRA-chain representations is small at the first hidden layer, consistent with the bound predicting a nontrivial lower bound on fidelity at that layer. We do not compute the bound numerically because the quantizer $Q$ is not part of the attack procedure, and optimizing over $Q$ is itself intractable. The assumption $2\le|\mathcal{H}|<\infty$ excludes the trivial single-point quantizer. Theorem~\ref{theorem: GRA attack lower bound} is a theoretical motivation rather than a constructive recipe: it formalizes the intuition that stronger representation matching can increase a lower bound on conditional reconstruction fidelity. The theorem is tractable only in threat models where $\bm{H}_{A}^{(\ell)}$ (or its quantized version) is observable; $\bm{H}_{A}^{(\ell)}$ and $\bm{H}_{\bm{\hat{A}}}^{(\ell)}$ denote the same layer(s) under $\mathcal{I}_{H}(\mathcal{K})$. Thus, the bound supports the design choice of matching hidden layers when they are in $\mathcal{K}$, in addition to matching outputs.

\paragraph{An information-theoretic ceiling on reconstruction.}
We next characterize the fundamental upper bound on recoverable information under a given knowledge set. The result is a direct application of the data processing inequality to the Markov chain $A \to \mathcal{K} \to \bm{\hat{A}}$; we state it as a proposition to emphasize its application to the GRA setting.

\begin{proposition}[Optimal fidelity under a given knowledge set]
\label{prop: optimal fidelity in GRA}
Let $\mathcal{K}$ denote all information available to the attacker, including any released side information and any fixed model information such as the trained parameters $\bm{\theta}$ in the white-box setting. View $\mathcal{K}$ as a random object jointly distributed with the unknown training adjacency $A$. Let the attacker output an estimate $\bm{\hat{A}}$ via an arbitrary possibly randomized procedure
\begin{equation}
\bm{\hat{A}} = g(\mathcal{K},U),
\qquad
U \perp\!\!\!\perp (A,\mathcal{K}),
\end{equation}
where $g$ is measurable. Then
\begin{equation}
I(A;\bm{\hat{A}})\;\le\; I(A;\mathcal{K}),
\label{eqn: optimal-fidelity-gap}
\end{equation}
\ie, the fidelity achievable by any attack using only $\mathcal{K}$ is upper-bounded by the information about $A$ contained in
$\mathcal{K}$. Define the optimal fidelity (under knowledge $\mathcal{K}$) as
\begin{equation}
\mathcal{F}^{*}(\mathcal{K}) \;\triangleq\; \sup_{g} I\bigl(A;g(\mathcal{K},U)\bigr),
\end{equation}
where the supremum is over measurable $g$ and auxiliary randomness $U\perp\!\!\!\perp(A,\mathcal{K})$.
Then $\mathcal{F}^{*}(\mathcal{K}) \le I(A;\mathcal{K})$. The substantive content is the upper bound in Eq.~\eqref{eqn: optimal-fidelity-gap}: it provides an information-theoretic ceiling on reconstruction under the available knowledge set. Equality holds iff $\bm{\hat{A}}^{*}$ is a sufficient statistic of $\mathcal{K}$ for $A$ (\ie, $I(A;\mathcal{K}\mid \bm{\hat{A}}^{*})=0$); this is idealized and rarely achieved in practice.

\noindent
\textbf{Proof (sketch).}
Eq.~\eqref{eqn: optimal-fidelity-gap} follows from the data processing inequality applied to the Markov chain
$A \to \mathcal{K} \to \bm{\hat{A}}$. The equality condition holds iff $\bm{\hat{A}}^{*}$ is a sufficient statistic of $\mathcal{K}$ for $A$. Details are in Appendix~\ref{ssec: proof of optimal fidelity in GRA}.
\end{proposition}

Proposition~\ref{prop: optimal fidelity in GRA} identifies an information-theoretic ceiling: even an optimal attack cannot recover more about $A$ than what is in $\mathcal{K}$, so $H(A\mid \bm{\hat{A}}^{*}) = H(A\mid \mathcal{K})$ at optimality (Fig.~\ref{fig: three-ball-GRA}).

\paragraph{Heterophily-aware side information and label-derived priors.}
The next results use a label-dependent random-graph model to understand the information content of label-derived priors; they are not assumptions of the attack algorithm. The sequence moves from the ideal affinity matrix $\bm{P}(Y)$, to the coarser disagreement indicator $\bm{D}$, and finally to noisy predicted labels $\hat Y$, mirroring the practical approximation underlying $\bm{P}_{\text{hetero}}$ in MC-GRA+.

\begin{theorem}[Label-induced edge-probability matrix is sufficient for $A$ under a known block model]
\label{thm: sufficiency_label_induced_matrix}
This is a stylized population result under idealized assumptions. The theorem assumes the block matrix $B$ is \emph{known} in the generative model; in practice the attack does not know $B$ and cannot compute $\bm{P}(Y)$ directly. The result motivates the use of coarser label-derived scores such as $\bm{P}_{\mathrm{hetero}}$ in MC-GRA+. Let $Y_i\in[C]$ and assume an undirected conditional edge-independence block model with known symmetric $B\in[0,1]^{C\times C}$:
for all $i<j$,
\begin{equation}
A_{ij}\mid (Y_i=a,Y_j=b)\sim \mathrm{Bern}(B_{ab}),
\qquad
\{A_{ij}\}_{i<j}\ \text{independent given }Y,
\end{equation}
where $A_{ji}=A_{ij}$ and $A_{ii}=0$.
Define $\bm{P}(Y)\in[0,1]^{N\times N}$ by $\bm{P}(Y)_{ij}=B_{Y_iY_j}$ for $i\neq j$ (and $0$ on the diagonal).
Then
\begin{equation}
A \perp\!\!\!\perp Y \mid \bm{P}(Y),
\qquad\text{and hence}\qquad
I(A;\bm{P}(Y))=I(A;Y).
\end{equation}

\noindent
\textbf{Proof (sketch).}
For any $y$, the conditional law factorizes as
\begin{equation}
\mathbb{P}(A\mid Y=y)=\prod_{i<j}\mathrm{Bern}\!\left(A_{ij};B_{y_iy_j}\right).
\end{equation}
If $y,y'$ satisfy $\bm{P}(y)=\bm{P}(y')$, then the Bernoulli parameters coincide edgewise, hence
$\mathbb{P}(A\mid Y=y)=\mathbb{P}(A\mid Y=y')$.
Therefore conditioning on $\bm{P}(Y)$ yields the same conditional law as conditioning on $Y$,
\ie, $\mathbb{P}(A\mid Y)=\mathbb{P}(A\mid \bm{P}(Y))$ a.s., which is equivalent to $A\perp\!\!\!\perp Y\mid \bm{P}(Y)$.
Since $\bm{P}(Y)$ is a deterministic function of $Y$, the mutual-information identity follows from the chain rule and
$I(A;Y\mid \bm{P}(Y))=0$.
Full proof in Appendix~\ref{app: proof_sufficiency_posterior_matrix}.
\end{theorem}

Theorem~\ref{thm: sufficiency_label_induced_matrix} shows that, under the block model, the adjacency-relevant content of labels is entirely captured by the induced pairwise affinity matrix $\bm{P}(Y)$. SBM-like structure arises approximately when graphs exhibit community structure with roughly homogeneous within- vs.\ between-community connectivities (e.g., some social or biological networks); when such structure is weak or absent, the model is only a crude approximation. In MC-GRA+, the practical score $\bm{P}_{\text{hetero}}$ is a prediction-based analogue of this idea, but with the richer affinity matrix replaced by a coarser disagreement-style construction. The next theorem quantifies the resulting information loss and gives a sharp condition under which this collapse is lossless.

\begin{theorem}[Information loss from collapsing labels to a heterophily indicator]
\label{thm: collapse_to_disagreement_info_loss_polished}
Assume the model of Theorem~\ref{thm: sufficiency_label_induced_matrix} and that each class has positive probability,
\ie, $\mathbb{P}(Y_i=a)>0$ for all $a\in[C]$.
Define $\bm{D}\in\{0,1\}^{N\times N}$ by
$\bm{D}_{ij}\triangleq \mathbbm{1}\{Y_i\neq Y_j\}$ (and $\bm{D}_{ii}=0$). Then
\begin{equation}
I(A;Y)=I(A;\bm{D})+I(A;Y\mid \bm{D}).
\end{equation}
Moreover, $I(A;Y\mid \bm{D})=0$ (equivalently $I(A;\bm{D})=I(A;Y)$) if and only if there exist
$p_{\mathrm{in}},p_{\mathrm{out}}\in[0,1]$ such that
\begin{equation}
B_{ab}=
\begin{cases}
p_{\mathrm{in}}, & a=b,\\
p_{\mathrm{out}}, & a\neq b,
\end{cases}
\end{equation}
\ie, the edge law depends on $(Y_i,Y_j)$ only through the indicator $\mathbbm{1}\{Y_i=Y_j\}$ (equivalently, through $\bm{D}_{ij}$).

\noindent
\textbf{Proof (sketch).}
Since $\bm{D}$ is a deterministic function of $Y$, the chain rule gives the decomposition.
If $B$ has the two-parameter form above, then $\mathbb{P}(A\mid Y)$ depends on $Y$ only through $\bm{D}$, yielding
$A\perp\!\!\!\perp Y\mid \bm{D}$ and hence $I(A;Y\mid \bm{D})=0$.
Conversely, if $A\perp\!\!\!\perp Y\mid \bm{D}$, then the conditional edge probability can only depend on whether labels match.
Under the positive-support assumption on $Y$, this implies all diagonal (resp.\ off-diagonal) entries of $B$ coincide, yielding the
two-parameter form above.
Full proof in Appendix~\ref{app: proof_collapse_to_disagreement}.
\end{theorem}

The two-parameter form (same $p_{\mathrm{in}}$ for all same-class pairs and $p_{\mathrm{out}}$ for all cross-class pairs) is a very specific structural property; most real-world graphs have more complex community structure (e.g., multiple blocks with different affinities), so information loss from collapsing to the heterophily indicator $\bm{D}$ is the typical case rather than the exception.

Finally, MC-GRA+ typically constructs heterophily-aware priors from \emph{predicted} labels $\hat Y$, rather than the true $Y$. The key qualitative fact is that prediction noise can only reduce the information available from label-derived priors: for any $V$ such that $V\to Y\to \hat Y$ forms a Markov chain, the data processing inequality gives $I(V;\hat Y)\le I(V;Y)$, and any side information $S=h(\hat Y)$ satisfies $I(V;S)\le I(V;\hat Y)$. More precisely, under a stylized symmetric-noise model with uniform class marginals, one can derive an explicit contraction coefficient $\eta(\varepsilon,C)=\bigl(1-\frac{C}{C-1}\varepsilon\bigr)^2$ such that $I(V;\hat Y)\le \eta(\varepsilon,C)\,I(V;Y)$; the formal statement and proof are given in Theorem~\ref{thm: noisy_labels_contraction_polished} below. The uniform-marginal assumption is restrictive (e.g., Cora has 7 classes with ratios ranging from roughly 3\% to 27\%), so the closed-form coefficient should be read as a qualitative scaling law. Without the uniform assumption, a contraction coefficient still exists but takes a different form; the qualitative conclusion---that noisier predictions reduce the usable information from label-derived priors---holds in general.

\begin{theorem}[Noisy predicted labels: an SDPI contraction bound]
\label{thm: noisy_labels_contraction_polished}
Let $Y\in[C]^N$ and let $\hat Y$ be obtained from $Y$ by independent $C$-ary symmetric noise with error
$\varepsilon\in\bigl[0,\frac{C-1}{C}\bigr]$:
\begin{equation}
\mathbb{P}(\hat Y_i=Y_i\mid Y_i)=1-\varepsilon,
\qquad
\mathbb{P}(\hat Y_i=b\neq Y_i\mid Y_i)=\frac{\varepsilon}{C-1},
\quad\text{independently over }i.
\end{equation}
Assume additionally that each $Y_i$ is marginally uniform on $[C]$ (a restrictive idealization; see the discussion above).\footnote{Under non-uniform class marginals, the contraction coefficient generally differs and the bound does not take this simple form; qualitatively, larger class imbalance tends to reduce the effective information in $\hat Y$ about $Y$, so the contraction intuition still applies.}
For any $V$ such that $V\to Y\to \hat Y$,
\begin{equation}
I(V;\hat Y)\le \eta(\varepsilon,C)\, I(V;Y),
\qquad
\eta(\varepsilon,C)\triangleq\left(1-\frac{C}{C-1}\varepsilon\right)^2.
\end{equation}
Moreover, for any heterophily prior $S=h(\hat Y)$ (so that $V\to \hat Y\to S$),
\begin{equation}
I(V;S)\le I(V;\hat Y)\le \eta(\varepsilon,C)\,I(V;Y).
\end{equation}

\noindent
\textbf{Proof (sketch).}
Under the uniform stationary distribution, the $C$-ary symmetric channel is reversible and satisfies an SDPI for mutual information
with coefficient equal to the square of its second-largest singular value,
$1-\frac{C}{C-1}\varepsilon$, yielding $\eta(\varepsilon,C)$.
Because the channel acts independently across coordinates, the same contraction coefficient applies to the product channel on $[C]^N$ by tensorization.
The inequality $I(V;S)\le I(V;\hat Y)$ follows by data processing along $V\to \hat Y\to S$.
Full proof in Appendix~\ref{app: proof_noisy_labels_contraction}.
\end{theorem}

%% file: sections/7-method-defense.tex
\section{Markov Chain-based Graph Privacy Bottleneck (MC-GPB)}
\label{sec: GRA defense}

As shown in Sec.~\ref{sec: overview}, standard GNN training retains substantial adjacency-related information in intermediate states, which motivates the question: \emph{how can we train GNNs to be more resistant to graph reconstruction attacks?} A principled defense should limit the information about $A$ that is recoverable from the learned representations while preserving task utility.

When $X$ is public, the relevant quantity is the \emph{additional} leakage $I(A;\bm{H}_{A}^{(\ell)}\mid X)$; we use the unconditional $I(A;\bm{H}_{A}^{(\ell)})$ as a tractable population proxy (see Eq.~\eqref{eqn: tighter-defense-MI} and Sec.~\ref{ssec: GPB understanding}). The relationship between these quantities follows from the chain rule: $I(A;\bm{H}_{A}^{(\ell)},X)=I(A;X)+I(A;\bm{H}_{A}^{(\ell)}\mid X)=I(A;\bm{H}_{A}^{(\ell)})+I(A;X\mid \bm{H}_{A}^{(\ell)})$. Rearranging gives
\begin{equation}
\label{eqn: defense-MI-decomposition}
I(A;\bm{H}_{A}^{(\ell)}) \;=\; I(A;\bm{H}_{A}^{(\ell)}\mid X) \;+\; I(A;X) \;-\; I(A;X\mid \bm{H}_{A}^{(\ell)}).
\end{equation}
When $A$ and $X$ are independent, $I(A;X)=0$, so $I(A;\bm{H}_{A}^{(\ell)})=I(A;\bm{H}_{A}^{(\ell)}\mid X)-I(A;X\mid \bm{H}_{A}^{(\ell)})\le I(A;\bm{H}_{A}^{(\ell)}\mid X)$; the unconditional MI is then at most as large as the conditional. When $A$ and $X$ are dependent, either direction is possible: for example, if $X$ encodes structural information correlated with $A$, then $I(A;X)>0$ can make $I(A;\bm{H}_{A}^{(\ell)})>I(A;\bm{H}_{A}^{(\ell)}\mid X)$. In short, when $A$ and $X$ are independent, the unconditional MI upper-bounds the additional leakage; otherwise, either direction is possible. The unconditional proxy is easier to estimate; when $A$ and $X$ are independent, it serves as an upper bound on the additional leakage $I(A;\bm{H}_{A}^{(\ell)}\mid X)$, and when they are dependent, it provides a combined measure of both feature-mediated and representation-mediated leakage. We use it for the implemented objective. For featureless graphs, the two quantities coincide. We formalize the goal as reducing this leakage across layers while preserving the information needed to predict $Y$, yielding a controlled balance between utility and privacy.

We address this challenge with the \emph{Markov Chain-based Graph Privacy Bottleneck} framework (MC-GPB). We first formulate population-level privacy objectives that treat adjacency as the sensitive source and couple task relevance with adjacency suppression (Sec.~\ref{ssec: GPB objective}). We then derive a surrogate training loss based on perturbed graphs and differentiable dependence scores (Sec.~\ref{ssec: GPB implementation}). Finally, we give stylized information-theoretic interpretations that clarify irreducible leakage, excess leakage, and the role of homophily/heterophily; these results are population analyses rather than direct guarantees for the implemented defense (Sec.~\ref{ssec: GPB understanding}).

\begin{figure*}[t!]
\centering
\includegraphics[width=0.8\textwidth]{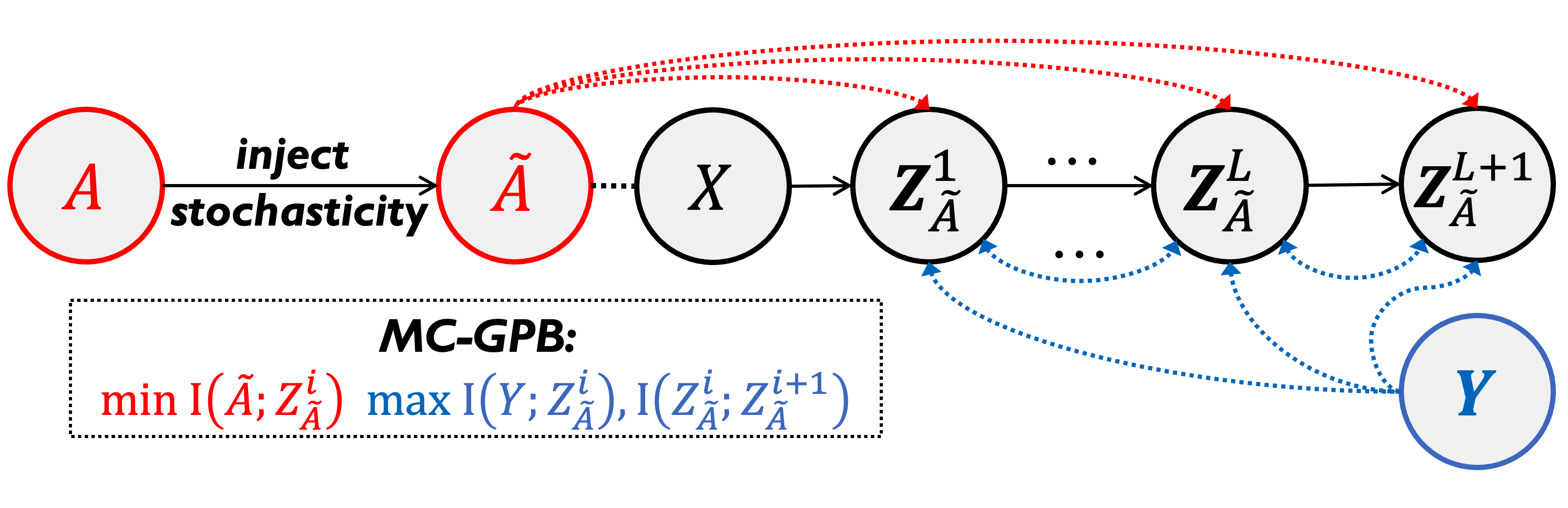}
\caption{\textbf{The MC-GPB defense framework.} 
MC-GPB addresses the accuracy-privacy trade-off through the population objective in Eq.~\eqref{eqn: tighter-defense-MI} and its practical surrogate in Sec.~\ref{ssec: GPB implementation}: it regularizes layerwise graph representations to reduce adjacency leakage while maintaining label-predictive information. During training, injected stochasticity (\eg, edge dropping) further promotes suppression of adjacency dependence by limiting deterministic memorization of training edges, thereby reducing reconstruction risk when releasing the trained model.}
\label{fig: markov-defense}
\end{figure*}

\subsection{The Optimization Objectives for Defense}
\label{ssec: GPB objective}

To address the privacy-utility trade-off under graph reconstruction attacks, we propose the \textit{Markov Chain-based Graph Privacy Bottleneck} (MC-GPB) framework (Fig.~\ref{fig: markov-defense}). The goal is to learn model parameters $\bm{\theta}$ from the training data $(A,X,Y)$ such that the resulting hidden representations remain predictive of labels while containing as little recoverable information about the training adjacency as possible. Our formulation is inspired by the information bottleneck principle, which promotes representations that preserve task-relevant information while compressing nuisance information~\citep{tishby2000information,shwartz2017opening,wu2020graph}. In our setting the natural graph-aware analogue is not the classical chain $X\to Z\to Y$, but rather
\begin{equation}
(A,X)\;\to\; Z\;\to\; Y,
\end{equation}
with $Z$ instantiated by the GNN's hidden states. Here $A$ plays the role of the variable to compress for privacy, replacing the generic input in the classical IB; the goal is to retain $Y$-relevant information while suppressing residual dependence on $A$.

\paragraph{The Markov chain-based defense objective.}
Let $\bm{H}_A^{(\ell)}$ denote the layer-$\ell$ hidden representation produced by the GNN when trained or evaluated on adjacency $A$, and let $L$ be the number of layers. Intuitively, the defense balances three competing goals at every layer: (i)~\emph{utility}---each layer should retain enough information to predict labels accurately; (ii)~\emph{privacy}---each layer should reveal as little as possible about the private adjacency $A$; and (iii)~\emph{complexity}---consecutive layers should not faithfully transmit fine-grained adjacency details that are unnecessary for the task. At the population level, we define MC-GPB by penalizing lack of label informativeness, leakage of adjacency information, and overly strong layer-to-layer information transmission:
\begin{equation}
\begin{split}
\bm{\theta}
&=\arg\min_{\bm{\theta}}\;
\sum_{\ell=1}^{L}
\Bigl(
\underbrace{-\,I(Y;\bm{H}_{A}^{(\ell)})}_{\text{utility}}
+
\underbrace{\beta_{p}^{(\ell)}\,I(A;\bm{H}_{A}^{(\ell)})}_{\text{privacy}}
\Bigr)
\;+\;
\sum_{\ell=1}^{L-1}
\underbrace{\beta_{c}^{(\ell)}\,I(\bm{H}_{A}^{(\ell)};\bm{H}_{A}^{(\ell+1)})}_{\text{complexity}}.
\end{split}
\label{eqn: tighter-defense-MI}
\end{equation}
Here $\beta_{p}^{(\ell)}\ge 0$ and $\beta_{c}^{(\ell)}\ge 0$ control the balance between utility and privacy and the inter-layer regularization. (We use $\beta$ for defense coefficients to distinguish them from the attack coefficients $\alpha$ introduced in Sec.~\ref{sec: GRA attack}.) The utility term $-I(Y;\bm{H}_{A}^{(\ell)})$ is summed over \emph{all} layers $\ell=1,\dots,L$, departing from the information bottleneck (where utility is measured at the output only). This ensures that every layer retains label-relevant information; without it, the trivial constant representation $\bm{H}_{A}^{(\ell)}=c$ achieves $I(A;\bm{H}_{A}^{(\ell)})=0$ (perfect privacy) but also $I(Y;\bm{H}_{A}^{(\ell)})=0$ (zero utility), collapsing all intermediate representations. The privacy term penalizes adjacency leakage per layer. The third term regularizes layer-to-layer information flow. Its purpose is to discourage the GNN from propagating adjacency-specific detail beyond what is needed for the task: if a representation at layer $\ell$ encodes idiosyncratic edges, penalizing $I(\bm{H}_A^{(\ell)};\bm{H}_A^{(\ell+1)})$ discourages the next layer from transmitting that detail. By itself this is a compression term, not a privacy guarantee; its privacy role is indirect, through limiting the propagation of adjacency-dependent information along the chain. When $A$ is fixed and the layer map is deterministic, $I(\bm{H}_{A}^{(\ell)};\bm{H}_{A}^{(\ell+1)})=H(\bm{H}_{A}^{(\ell+1)})$; minimizing it alone would encourage degenerate (constant) representations. The utility term prevents collapse by requiring label-relevant information at each layer.

\noindent\textbf{Probability model for Eq.~\eqref{eqn: tighter-defense-MI}.}
Because we observe a single graph rather than a distribution over graphs, the population MI terms in Eq.~\eqref{eqn: tighter-defense-MI} cannot be optimized; we instead minimize surrogate dependence measures under stochastic edge perturbation (Sec.~\ref{ssec: GPB implementation}). In detail: all terms are defined under the \emph{same} joint distribution over $(A,X,Y)$: both $A$ and $Y$ are treated as random variables drawn from a data-generating distribution. In the practical single-graph setting, however, $A$ is observed and fixed. The population objective is a conceptual target that motivates the surrogate loss. When $A$ is fixed (non-random), the population MI $I(A;\bm{H}_{A}^{(\ell)})$ is technically zero, so the implemented objective does not optimize MI directly. Instead, randomness is reintroduced through DropEdge perturbation $\tilde{A}=\Phi_{\mathrm{def}}(A,A_\epsilon)$, and the practical loss minimizes the \emph{surrogate dependence measure} $d(\hat{A}_{\tilde{H}}^{(\ell)}, A)$ between the decoded adjacency score (computed from the stochastic representation $\bm{H}_{\tilde{A}}^{(\ell)}$) and the fixed true adjacency $A$. This surrogate is well-defined for fixed $A$ and serves as a tractable proxy for the population privacy term.
We note that the per-layer utility term $-I(Y;\bm{H}_{A}^{(\ell)})$ is a population-level idealization. In practice, the cross-entropy surrogate (Eq.~\eqref{eqn: practical-defense-loss}) measures utility only at the final output, which effectively sidesteps the difficulty of optimizing MI at intermediate layers while all parameters are updated jointly. The population objective thus motivates the design; the practical loss avoids the associated optimization instabilities.

In the practical single-graph setting (fixed $A$), the surrogate $d(\bm{H}_{\tilde{A}}^{(\ell)},\bm{H}_{\tilde{A}}^{(\ell+1)})$ operates on representations computed from the \emph{perturbed} adjacency $\tilde{A}$, which introduces randomness and prevents the degeneracy above. The surrogate thus measures representation similarity across layers under stochastic perturbation rather than entropy. The coefficients are set by validation. Eq.~\eqref{eqn: tighter-defense-MI} is a population-level idealization; in practice, each term is replaced by a differentiable surrogate (Sec.~\ref{ssec: GPB implementation}).

\paragraph{Conditional interpretation of the MC-GPB objective.}
Eq.~\eqref{eqn: tighter-defense-MI} can also be interpreted through conditional mutual information. Since
\begin{equation}
I(A;\bm{H}_{A}^{(\ell)}\mid Y)=I(A;\bm{H}_{A}^{(\ell)})-I(Y;\bm{H}_{A}^{(\ell)})+I(Y;\bm{H}_{A}^{(\ell)}\mid A),
\end{equation}
simultaneously reducing $I(A;\bm{H}_{A}^{(\ell)})$ and maintaining large $I(Y;\bm{H}_{A}^{(\ell)})$ encourages representations with lower adjacency dependence at fixed task relevance. This interpretation is only partial: the residual term $I(Y;\bm{H}_{A}^{(\ell)}\mid A)$ measures how much the representation reveals about $Y$ beyond what $A$ already determines; it is model-dependent and relates to the feature contribution (since the GNN takes $(A,X)$ as input).

\begin{remark}[Nature of the MC-GPB guarantee]
MC-GPB encourages, rather than guarantees, representations that are label-informative yet less adjacency-informative. The defense provides no formal privacy guarantee analogous to differential privacy; its effectiveness is empirical and depends on the surrogate quality and hyperparameter tuning.
\end{remark}

\paragraph{Heterophily-aware defense.}
In heterophilic graphs, edges predominantly connect nodes with different labels, so cross-label links can be particularly informative for reconstruction attacks. To explicitly bias the defense toward such links, we define the label-agreement mask $M\in\{0,1\}^{N\times N}$ by
\begin{equation}
\begin{aligned}
M_{ij}&=\mathbbm{1}\{y_i=y_j\} \quad (i\neq j),
\\
M_{ii}&=1 \text{(ensuring $(\bm{1}-M)_{ii}=0$, so that $A_{\text{hetero}}$ has zero diagonal)},
\end{aligned}
\end{equation}
and then define the heterophilic adjacency mask
\begin{equation}
A_{\text{hetero}} = A \odot (\bm{1}-M),
\end{equation}

where $\odot$ denotes elementwise multiplication. Since the graph has zero diagonal, $A_{\text{hetero}}$ also has zero diagonal. MC-GPB+ requires access to ground-truth labels at training time to construct $M$; in our experiments we use the benchmark training labels. Note the asymmetry with the attack: MC-GRA+ constructs its heterophily prior from \emph{predicted} labels $\bm{\hat{Y}}_A$ (since the attacker may not have ground-truth labels), while MC-GPB+ uses \emph{ground-truth} labels $Y$ (which are available to the defender during training). In a white-box setting, the attacker could in principle construct a stronger prior; however, our experiments show that MC-GPB+ remains effective even against this stronger attacker (Sec.~\ref{sec: experiments}).

We extend MC-GPB into MC-GPB+ by adding an additional penalty that reduces the mutual information between hidden representations and $A_{\text{hetero}}$:
\begin{equation}
\bm{\theta}\!=\!\arg\min_{\bm{\theta}}
\sum_{\ell=1}^{L}
\Bigl(
-I(Y;\bm{H}_{A}^{(\ell)})
+
\beta_{p}^{(\ell)}\,I(A;\bm{H}_{A}^{(\ell)})
+
\beta_{h}^{(\ell)}\,I(A_{\text{hetero}};\bm{H}_{A}^{(\ell)})
\Bigr)
+\sum_{\ell=1}^{L-1}\beta_{c}^{(\ell)}\,I(\bm{H}_{A}^{(\ell)};\bm{H}_{A}^{(\ell+1)}).
\label{eqn: defense-hetero-MI}
\end{equation}
Here $\beta_{h}^{(\ell)}\ge 0$ controls the strength of the heterophily-aware penalty (layer-indexed, unlike the scalar $\alpha_h$ in the attack, because the defense operates on all layers during training). Conditioned on the label mask $M$ (or when labels are treated as fixed information), $A_{\text{hetero}}$ is a deterministic function of $A$, so $I(A_{\text{hetero}};\bm{H}_{A}^{(\ell)}\mid M)\le I(A;\bm{H}_{A}^{(\ell)}\mid M)$ by data processing; the heterophily term reweights the defense toward cross-label edges rather than adding new information beyond the global penalty. Its role is to \emph{reweight} the defense toward cross-label edges, which are often more informative for reconstruction on heterophilic graphs. We keep the global penalty as the primary regularizer and add the heterophily term to emphasize suppression on those edges without separately tuning homo/hetero weights. As with Eq.~\eqref{eqn: tighter-defense-MI}, Eq.~\eqref{eqn: defense-hetero-MI} is a population objective; practical training uses surrogate penalties described next.

\subsection{Implementation Details of MC-GPB (+)}
\label{ssec: GPB implementation}

\textbf{Promoting suppression of adjacency dependence in Eq.~\eqref{eqn: tighter-defense-MI} via injected stochasticity.}
A direct way to reduce privacy risk is to prevent the trained GNN from deterministically encoding fine-grained details of the training adjacency. We inject stochasticity by adopting DropEdge~\citep{rong2020dropedge}, which independently removes each observed edge of $A$ with probability $p\in(0,1)$. DropEdge was originally proposed as a regularizer against overfitting~\citep{rong2020dropedge}; we repurpose it here as a heuristic privacy-promoting mechanism. \emph{DropEdge provides no formal privacy guarantee} (unlike differential privacy); training with it does not guarantee lower leakage after optimization, because the optimizer may still encode adjacency through other channels. Denote the perturbed adjacency by
\begin{equation}
\tilde{A} \;=\; \Phi_{\mathrm{def}}(A, A_{\epsilon}),
\end{equation}
where $A_{\epsilon}$ represents the edge-deletion noise and $\Phi_{\mathrm{def}}$ denotes the defense-side perturbation operator (DropEdge). Note that $\Phi_{\mathrm{def}}$ (random edge deletion) and $\Phi_{\mathrm{atk}}$ (binary Concrete relaxation, Sec.~\ref{sec: GRA attack}) are formally different perturbation operators; they are distinguished by their subscripts throughout. The noise $A_{\epsilon}$ is sampled independently of $(X,Y)$ and independently across edges. We emphasize, however, that a representation $\bm{Z}$ computed from $(\tilde{A},X)$ is generally \emph{not} independent of $A_{\epsilon}$, because it depends on the perturbed adjacency.

At the channel level, $\Phi_{\mathrm{def}}$ yields $A\to \tilde{A}$, so any $U$ computed from $\tilde{A}$ satisfies $A \to \tilde{A} \to U$ and data processing bounds the information in $U$ about $A$ by what remains in $\tilde{A}$. The defense implements a tractable regularization objective that encourages reliance on graph patterns stable under stochastic edge perturbations, reducing brittle memorization of idiosyncratic training edges~\citep{zhao2022learning}.

\paragraph{Promoting feasibility via differentiable dependence measures.}
Directly optimizing Eq.~\eqref{eqn: tighter-defense-MI} is infeasible, so we train MC-GPB with a concrete surrogate loss. The utility term is instantiated by the standard cross-entropy on the perturbed-graph predictions over the labeled training nodes, while the privacy terms are instantiated by differentiable dependence scores. We use the symmetrization operator $\Pi_{\mathrm{sym},0}(\cdot)$ defined in Sec.~\ref{ssec: attack implementation}.

For each layer $\ell$, define the decoded adjacency score $\hat A_{\tilde H}^{(\ell)}=\Pi_{\mathrm{sym},0}\!\Bigl(\sigma\bigl(\bm H_{\tilde A}^{(\ell)}{\bm H_{\tilde A}^{(\ell)}}^{\top}\bigr)\Bigr)$, i.e., the same inner-product--sigmoid construction used for the leakage proxy in Sec.~\ref{ssec: understanding-by-auc-proxy}. This maps the perturbed representation $\bm H_{\tilde A}^{(\ell)}$ to an $N\times N$ symmetric score matrix comparable with the (sensitive) adjacency $A$. The implemented loss penalizes adjacency recoverability via this decoded score rather than estimating $I(A;\bm{H}_A^{(\ell)})$ directly.

\paragraph{Target of the privacy penalty.}
The penalty compares the decoded score to the \emph{original} adjacency $A$, not the perturbed $\tilde A$: representations are computed from $\tilde A$ during training, but the goal is to suppress leakage of the true $A$. Thus we compare $\hat A_{\tilde H}^{(\ell)}$ (or $\hat A^{(\ell)}(\tilde A)$; the subscript $\tilde H$ indicates dependence on $\bm{H}_{\tilde A}^{(\ell)}$) to $A$. This penalizes \emph{similarity-based} recoverability. The penalty targets similarity-based adjacency recoverability, matching the leakage proxy and reconstruction objectives used in our evaluation.
Because $\hat A_{\tilde H}^{(\ell)}$ is decoded from representations computed on the perturbed $\tilde{A}$, the penalty partially reflects noise injected by DropEdge rather than pure adjacency leakage. This coupling is intentional---the defense seeks to ensure that even noisy representations do not reveal $A$---but the penalty may over-penalize when DropEdge distortion is large.

This surrogate preserves the privacy and inter-layer penalties layerwise, but approximates the multi-layer utility term using only the final supervised loss. This simplification is standard in deep information bottleneck implementations, where intermediate-layer MI is replaced by end-to-end supervised loss~\citep{shwartz2017opening}. We use the practical surrogate objective, where $d(\cdot,\cdot)$ denotes CKA (chosen for scale invariance; see below):
\begin{equation}
\begin{split}
\widehat{\mathcal L}_{\mathrm{MC\text{-}GPB}}
:={}&
\mathcal L_{\mathrm{CE}}(\hat{\bm Y}_{\tilde A},Y)
\;+\;
\sum_{\ell=1}^{L}\beta_p^{(\ell)}\,d\!\left(\hat A_{\tilde H}^{(\ell)},A\right)
\;+\;
\sum_{\ell=1}^{L-1}\beta_c^{(\ell)}\,d\!\left(\bm H_{\tilde A}^{(\ell)},\bm H_{\tilde A}^{(\ell+1)}\right),
\end{split}
\label{eqn: practical-defense-loss}
\end{equation}
and, for the heterophily-aware variant,
\begin{equation}
\widehat{\mathcal L}_{\mathrm{MC\text{-}GPB+}}
\;=\;
\widehat{\mathcal L}_{\mathrm{MC\text{-}GPB}}
\;+\;
\sum_{\ell=1}^{L}\beta_h^{(\ell)}\,d\!\left(\hat A_{\tilde H}^{(\ell)},A_{\text{hetero}}\right).
\label{eqn: practical-defense-loss-hetero}
\end{equation}
The sensitive target of the penalty is always the original $A$ (or $A_{\text{hetero}}$); $\tilde A$ is used only to produce the stochastic representation $\bm{H}_{\tilde A}^{(\ell)}$ during training. The ideal objectives in Eqs.~\eqref{eqn: tighter-defense-MI} and \eqref{eqn: defense-hetero-MI} are written in terms of $\bm{H}_A^{(\ell)}$; the implemented loss is evaluated on $\bm{H}_{\tilde A}^{(\ell)}$ as a stochastic approximation.

\paragraph{Choice of dependence surrogate.}
We use the same dependence surrogate notation $d(\cdot,\cdot)$ and the same family of candidates (DP, HSIC, CKA, KDE) as in Sec.~\ref{ssec: attack implementation}; the definitions are not repeated here.
For MC-GPB/MC-GPB+, $d(\cdot,\cdot)$ is instantiated by \texttt{CKA}, while the attack uses \texttt{HSIC} (with normalized inputs). The rationale for this difference is that scale invariance matters more for the defense: representation norms can change significantly across epochs and layers during training, and a scale-dependent penalty could conflate norm growth with adjacency leakage; CKA provides this invariance. For the attack, the candidate adjacency is optimized directly (not trained over epochs), so scale sensitivity is less problematic. Tab.~\ref{tab: ablation-similarity-metrics} compares all four choices for both attack and defense.

The specific surrogate options are:

\begin{itemize}[leftmargin=*]
\setlength\itemsep{0.1em}

\item \textbf{Dot product-based similarity (DP).}
We measure dependence through Gram matrices:
\begin{equation}
\texttt{DP}(\bm{U},\bm{V})
\;=\;
\mathrm{tr}\bigl(\bm{U}\bm{U}^{\top}\bm{V}\bm{V}^{\top}\bigr)
\;=\;
\|\bm{V}^{\top}\bm{U}\|_{F}^{2}.
\end{equation}
Here $(\bm{U}\bm{U}^{\top})_{ij}=\langle \bm{u}_i,\bm{u}_j\rangle$ is the dot product between $\bm{u}_i$ and $\bm{u}_j$; analogously for $\bm{V}\bm{V}^{\top}$.

\item \textbf{Hilbert-Schmidt independence criterion (HSIC).}
HSIC measures dependence via cross-covariance in a reproducing kernel Hilbert space. Given kernels $k(\cdot,\cdot)$ and $l(\cdot,\cdot)$ with Gram matrices $K_{ij}=k(\bm{u}_i,\bm{u}_j)$ and $L_{ij}=l(\bm{v}_i,\bm{v}_j)$, the empirical HSIC \citep{gretton2005measuring} is estimated as
\begin{equation}
\texttt{HSIC}(K,L)
\;=\;
\frac{1}{(N-1)^2}\,\mathrm{tr}(KHLH),
\end{equation}
where $H=I-\frac{1}{N}\bm{1}\bm{1}^{\top}$ is the $N\times N$ centering matrix.

\item \textbf{Centered kernel alignment (CKA).}
CKA normalizes HSIC to be invariant to isotropic scaling, \ie, $d(\bm{U},\bm{V})=d(\alpha \bm{U},\beta \bm{V})$ for all $\alpha,\beta>0$ when $d$ is CKA~\citep{cortes2012algorithms}. It is computed as
\begin{equation}
\texttt{CKA}(K,L)
\;=\;
\frac{\texttt{HSIC}(K,L)}{\sqrt{\texttt{HSIC}(K,K)\,\texttt{HSIC}(L,L)}}.
\end{equation}

\item \textbf{Kernel density estimation (KDE).}
KDE estimates the marginal densities $p_{\bm{U}}$, $p_{\bm{V}}$ and the joint density $p_{\bm{U}\bm{V}}$ using a kernel (\eg, a Gaussian kernel with bandwidth or covariance matrix $\mathbf{H}$), then forms a differentiable estimate of $I(\bm{U};\bm{V})$ via the usual identity $I(\bm{U};\bm{V})=H(\bm{U})+H(\bm{V})-H(\bm{U},\bm{V})$ with entropy terms replaced by plug-in estimates from the fitted densities (e.g., $\widehat{H}(\bm{U})=-\frac{1}{N}\sum_i \log \hat p_{\bm{U}}(\bm{u}_i)$). The resulting estimate is sensitive to bandwidth choice and dimension; when used in ablations we follow standard rules (e.g., Silverman or Scott) for bandwidth selection.
\end{itemize}

\noindent
\textbf{Remark.} Different choices of $d(\cdot,\cdot)$ capture different notions of dependence and operate at different levels of approximation. Throughout, we use them as tractable empirical surrogates rather than claiming equality with $I(\cdot;\cdot)$.

\paragraph{Computational cost.}
Each training iteration incurs one forward pass on the perturbed graph, plus evaluation of the dependence surrogates and the decoded adjacency scores at each layer; the cost is that of the base GNN plus $O(L\cdot N^2)$ for the layerwise surrogate terms.

\paragraph{The algorithm.}
Equipped with the stochastic perturbation and surrogate losses above, we summarize the resulting defensive training procedure in Algorithm~\ref{alg: GPB}. MC-GPB+ differs only by adding the heterophily-aware penalty in Eq.~\eqref{eqn: practical-defense-loss-hetero}.

\begin{algorithm}[ht]
\caption{Generalized defensive training (MC-GPB / MC-GPB+).}
\begin{algorithmic}[1]
\REQUIRE Training graph $(A,X,Y)$, GNN $f_{\bm{\theta}}$ with $L$ layers, dependence surrogate $d(\cdot,\cdot)$, coefficients $\{\beta_p^{(\ell)}\}_{\ell=1}^{L}$, $\{\beta_c^{(\ell)}\}_{\ell=1}^{L-1}$, and (optional) $\{\beta_h^{(\ell)}\}_{\ell=1}^{L}$, edge-drop probability $p$, number of iterations $n$
\STATE Initialize model parameters $\bm{\theta}$
\IF{heterophily-aware defense is enabled}
    \STATE Construct the heterophilic adjacency mask
    $A_{\text{hetero}} = A \odot (\bm{1}-M)$ with $M_{ij}=\mathbbm{1}\{y_i=y_j\}$
\ENDIF
\FOR{$t=1,\dots,n$}
    \STATE Inject stochasticity into the training graph (\eg, DropEdge):
    $\tilde{A}=\Phi_{\mathrm{def}}(A, A_{\epsilon})$,
    where each observed edge is independently dropped with probability $p$
    \STATE Forward propagate to obtain representations and outputs:
    $\{\bm{H}_{\tilde{A}}^{(\ell)}\}_{\ell=1}^{L},\; \bm{\hat{Y}}_{\tilde{A}}
    \leftarrow f_{\bm{\theta}}(\tilde{A},X)$
    \STATE For each layer $\ell$, form
    $\hat A_{\tilde H}^{(\ell)}=\Pi_{\mathrm{sym},0}\!\bigl(\sigma(\bm H_{\tilde A}^{(\ell)}{\bm H_{\tilde A}^{(\ell)}}^\top)\bigr)$
    \IF{heterophily-aware defense is enabled}
        \STATE Compute $\widehat{\mathcal L}_{\mathrm{MC\text{-}GPB+}}$ from Eq.~\eqref{eqn: practical-defense-loss-hetero}
    \ELSE
        \STATE Compute $\widehat{\mathcal L}_{\mathrm{MC\text{-}GPB}}$ from Eq.~\eqref{eqn: practical-defense-loss}
    \ENDIF
    \STATE Update $\bm{\theta}$ by minimizing the resulting objective (gradient descent on $\bm{\theta}$)
\ENDFOR
\STATE \textbf{return} the trained model $f_{\bm{\theta}}$
\end{algorithmic}
\label{alg: GPB}
\end{algorithm}

\subsection{Theoretical Understanding}
\label{ssec: GPB understanding}

A graph privacy defense must balance two competing goals: preserving task utility while limiting leakage of the private training adjacency. In our setting, leakage is mediated through learned representations. When node features are already public, reducing $I(A;\bm{H}_{A}\mid X)$ lowers the \emph{additional} adjacency information leaked through the representations beyond what is already exposed by $X$. By Proposition~\ref{prop: optimal fidelity in GRA}, smaller representation leakage can in turn reduce the attack ceiling associated with knowledge sets that include those released representations. The results below formalize this idea in a stylized way using sufficiency and information-bottleneck arguments.

\paragraph{Sufficiency as a utility requirement.}
We begin with the classical sufficiency notion in information-bottleneck form. This proposition is a warm-up concept; the graph-specific formulation used later replaces $X$ by the joint input $(A,X)$.

\begin{proposition}[Sufficient statistic of $X$ for $Y$ (information-theoretic form)]
\label{prop: sufficient graph representation}
Let $X$ and $Y$ be random variables, and let $\bm{Z}=f(X)$ for some measurable function $f$.
We say that $\bm{Z}$ is \emph{sufficient (as a statistic of $X$) for predicting $Y$} if it preserves all information in $X$
that is relevant for $Y$, \ie,
\begin{equation}
I(\bm{Z};Y) = I(X;Y).
\label{eqn: suff-stat-mi}
\end{equation}
Since $\bm{Z}$ is a deterministic function of $X$, the chain rule implies
$I(X;Y)=I(\bm{Z};Y)+I(X;Y\mid \bm{Z})$, hence Eq.~\eqref{eqn: suff-stat-mi} holds if and only if $I(X;Y\mid \bm{Z})=0$,
equivalently $Y \perp\!\!\!\perp X \mid \bm{Z}$.
This is the information-theoretic analogue of the classical sufficient-statistic concept: $\bm{Z}$ is sufficient iff no label-relevant information is lost.
In our graph setting, the input is the joint $(A,X)$ and the representation is $\bm{Z}=f(A,X)$ for some measurable $f$. We say $\bm{Z}$ is \emph{sufficient (as a statistic of $(A,X)$) for predicting $Y$} if $I(\bm{Z};Y)=I(A,X;Y)$, equivalently $Y \perp\!\!\!\perp (A,X) \mid \bm{Z}$. The propositions and theorems below use this graph-specific form.
\end{proposition}

\paragraph{Why sufficiency alone is insufficient for privacy.}
Sufficiency guarantees utility, but it does not constrain how much \emph{additional} information about $A$ may be encoded. The next elementary observation is included only to formalize this baseline impossibility point: in the worst case, a representation retains all information about the adjacency itself.
Observe that
\begin{equation}
I(A;\bm{H}_A) = H(A) - H(A\mid \bm{H}_A) \;\le\; H(A),
\label{eqn: max-adj-info}
\end{equation}
with equality iff $A$ is fully recoverable from the representation.
Thus, in the worst case $\max_{\bm{H}_{A}} I(A;\bm{H}_{A}) = H(A)$: a learned representation may retain all information about the private adjacency. Guaranteeing only task sufficiency (preserving label-relevant information) is therefore not \emph{safe}---even if $\bm{H}_{A}$ is sufficient for predicting $Y$, it may still encode substantial \emph{excess} adjacency information beyond what is needed for the task.

\paragraph{Minimal sufficient statistics and unavoidable adjacency information.}
To control excess leakage, it is natural to seek representations that are sufficient for the task yet maximally compressed. This motivates the concept of minimal sufficient statistics. 
We recall the classical information-bottleneck form before specializing to adjacency-aware representations.

\begin{proposition}[Minimal sufficient statistic (IB form)]
\label{prop: minimal graph representation}
Let $X$ and $Y$ be random variables and let $\bm{Z}=f(X)$ range over (measurable) representations of $X$ that are sufficient for $Y$,
\ie, $I(\bm{Z};Y)=I(X;Y)$. A \emph{minimal sufficient statistic} $\bm{Z}^{*}$ is any minimizer of
\begin{equation}
\bm{Z}^{*}
\in
\arg\min_{\bm{Z}:\; I(\bm{Z};Y)=I(X;Y)} I(\bm{Z};X),
\label{eqn: minimal-suff-stat}
\end{equation}
\ie, $\bm{Z}^{*}$ is a maximally compressed representation among all sufficient representations of $X$.
Note that this IB-based minimality criterion is equivalent to the classical minimal sufficient statistic in many standard parametric models (\eg, exponential families), but the two notions can differ in general; we use the IB definition throughout.
\end{proposition}

The next result identifies the minimum amount of adjacency information any representation must retain if it is to remain sufficient for predicting $Y$ when $Y$ depends on $A$ beyond what is already available in $X$. Intuitively, once $X$ is given, a task-sufficient representation should preserve all label-relevant information contributed by $A$.

\begin{theorem}[Population lower bound on adjacency information in task-sufficient representations]
\label{theorem: minimum adjacency information}
Let $(A,X,Y,\bm{H}_{A})$ be jointly distributed random variables, and suppose $\bm{H}_{A}$ is generated from $(A,X)$, possibly with additional internal randomness, as $\bm{H}_{A}=h(A,X,\xi)$, where $\xi\perp\!\!\!\perp Y\mid (A,X)$. Then $Y \to (A,X) \to \bm{H}_{A}$ forms a Markov chain.

\noindent
Assume that $\bm{H}_{A}$ is sufficient for predicting $Y$ given $X$, namely
\begin{equation}
Y \perp\!\!\!\perp A \mid (\bm{H}_{A},X)
\quad\Longleftrightarrow\quad
I(\bm{H}_{A};Y\mid X)=I(A;Y\mid X).
\label{eqn: suff_wrt_A_given_X}
\end{equation}
Then
\begin{equation}
I(A;\bm{H}_{A}\mid X)\;\ge\; I(A;Y\mid X).
\label{eqn: min-adj-info}
\end{equation}
Moreover, equality holds if and only if
\begin{equation}
I(A;\bm{H}_{A}\mid Y,X)=0,
\label{eqn: min-adj-info-equality}
\end{equation}
that is, $\bm{H}_{A}$ contains no adjacency information beyond what is already revealed by $(Y,X)$.

\noindent
\textbf{Interpretation.}
The theorem should be read as a population-level structural result rather than a guarantee for the implemented defense. It shows that task-sufficient representations generally cannot eliminate all adjacency information, thereby motivating a privacy-utility trade-off. Our practical method replaces these population quantities with tractable empirical surrogates (HSIC/CKA) and stochastic regularization, so the theorem serves as conceptual guidance, not as a certificate for the learned model.

\noindent
\textbf{Proof (sketch).}
Under the Markov chain $Y \to (A,X)\to \bm{H}_{A}$, the chain rule gives
\begin{equation}
I(A,\bm{H}_{A};Y\mid X)
=
I(A;Y\mid X)+I(\bm{H}_{A};Y\mid A,X)
=
I(A;Y\mid X),
\end{equation}
since $I(\bm{H}_{A};Y\mid A,X)=0$. On the other hand,
\begin{equation}
I(A,\bm{H}_{A};Y\mid X)
=
I(\bm{H}_{A};Y\mid X)+I(A;Y\mid \bm{H}_{A},X).
\end{equation}
Hence $I(\bm{H}_{A};Y\mid X)=I(A;Y\mid X)$ is equivalent to $I(A;Y\mid \bm{H}_{A},X)=0$, i.e.,
$Y \perp\!\!\!\perp A \mid (\bm{H}_{A},X)$.
Applying the chain rule once more,
\begin{equation}
I(A;\bm{H}_{A}\mid X)
=
I(A;Y\mid X)+I(A;\bm{H}_{A}\mid Y,X)-I(A;Y\mid \bm{H}_{A},X).
\end{equation}
Under Eq.~\eqref{eqn: suff_wrt_A_given_X}, the last term vanishes, so
\begin{equation}
I(A;\bm{H}_{A}\mid X)
=
I(A;Y\mid X)+I(A;\bm{H}_{A}\mid Y,X)
\ge
I(A;Y\mid X),
\end{equation}
with equality if and only if $I(A;\bm{H}_{A}\mid Y,X)=0$.
\end{theorem}

\noindent
As a coarse empirical indicator of excess leakage, one can compare $R_{\mathrm{AUC}}(A;\bm{H}_A)$ (total leakage from hidden representations, Tab.~\ref{tab: understanding-MI-term-gcn}) with $R_{\mathrm{AUC}}(A;Y)$ (a proxy for the task-relevant floor). On Cora, for example, $R_{\mathrm{AUC}}(A;\bm{H}_A)\approx 0.77$ while $R_{\mathrm{AUC}}(A;Y)\approx 0.82$; on Texas, $R_{\mathrm{AUC}}(A;\bm{H}_A)\approx 0.35$ while $R_{\mathrm{AUC}}(A;Y)\approx 0.35$---suggesting that excess leakage beyond label-relevant information varies substantially across regimes. These are rough indicators based on the AUC proxy (not exact MI); a rigorous quantification of the gap $I(A;\bm{H}_{A}\mid Y,X)$ remains an open problem because conditional MI estimation in high dimensions requires density estimation in the joint space of $(A,\bm{H}_A,Y,X)$, which is infeasible with standard plug-in or variational estimators at the sample sizes available from a single graph.

\paragraph{(i) Approximate sufficiency.}
The sufficiency condition in Theorem~\ref{theorem: minimum adjacency information} is idealized; in practice, GNN representations are only approximately sufficient, so the bound holds approximately. Theorem~\ref{theorem: minimum adjacency information} provides a sharp interpretation of ``excess leakage'': for any task-sufficient representation (given $X$), $I(A;\bm{H}_{A}\mid Y,X)$ measures exactly how much adjacency information is retained beyond the irreducible amount $I(A;Y\mid X)$ that is intrinsically coupled with the task. Ideally, a defense would minimize this excess term. The sufficiency condition Eq.~\eqref{eqn: suff_wrt_A_given_X} is strong and holds only approximately for practical GNNs. When it fails, the chain rule still gives a weakened bound: $I(A;\bm{H}_{A}\mid X)=I(A;Y\mid X)-I(A;Y\mid \bm{H}_{A},X)+I(A;\bm{H}_{A}\mid Y,X)$, so $I(A;\bm{H}_{A}\mid X)\ge I(A;Y\mid X)-I(A;Y\mid \bm{H}_{A},X)$. The term $I(A;Y\mid \bm{H}_{A},X)$ is the ``information gap'': when the representation is not sufficient for $Y$ given $X$, this gap is positive and the lower bound on $I(A;\bm{H}_{A}\mid X)$ is reduced.

\paragraph{(ii) Gap between theory and surrogate.}
The implemented MC-GPB objective penalizes $d(\hat A_{\tilde H}^{(\ell)},A)$ as a tractable surrogate for $I(A;\bm{H}_{A}^{(\ell)})$ rather than directly optimizing $I(A;\bm{H}_{A}\mid Y,X)$. The gap between the theoretical target ($I(A;\bm{H}_A\mid Y,X)$) and the practical surrogate ($d(\hat{A}_{\tilde{H}}^{(\ell)}, A)$) is inherent to using differentiable dependence scores and decoded similarity matrices instead of exact MI.

\paragraph{(iii) Empirical gap from sufficiency.}
Empirically, GNNs achieve roughly 70--85\% accuracy on the benchmarks, indicating a significant gap from exact sufficiency. Achieving $I(A;\bm{H}_{A}\mid Y,X)=0$ in practice is very difficult because GNN training does not directly optimize this quantity. The empirical results in Sec.~\ref{sec: overview} confirm that well-trained GNNs retain substantial but imperfect label-relevant information.

\paragraph{Connection to the information bottleneck and MC-GPB.}
We connect MC-GPB to the information bottleneck (IB) principle. The algebraic form ``cross-entropy plus MI penalty'' is wellknown as an IB objective; the contribution here is its application to the graph setting with $A$ as the compression target. The result below interprets a conditional-leakage variant of MC-GPB rather than Eq.~\eqref{eqn: tighter-defense-MI}; it is intended only as an IB analogue for the privacy term.

\begin{proposition}[MC-GPB as an information bottleneck objective (single-layer population form)]
\label{prop: GPB approximate the optimal representation}
The following observation is a \emph{single-layer population interpretation}; it does not capture the full multi-layer objective in Eq.~\eqref{eqn: tighter-defense-MI}.
Consider the single-layer population surrogate obtained by removing the inter-layer regularization, tying the privacy weights
($\beta_{p}^{(\ell)}=\beta$ for all $\ell$), and replacing the utility term by supervised negative log-likelihood. Fix a layer $\ell$ and let $\bm{Z}\triangleq \bm{H}_{A}^{(\ell)}$.
Interpret optimizing model parameters $\bm{\theta}$ as inducing a conditional channel $p_{\bm{\theta}}(\bm{Z}\mid A,X)$ (or
$p_{\bm{\theta}}(\bm{Z}\mid A)$ if $X$ is treated as fixed), and let $q_{\bm{\theta}}(Y\mid \bm{Z})$ denote the predictive head.
Under these simplifications, the population surrogate objective at layer $\ell$ takes the form
\begin{equation}
\mathbb{E}\bigl[-\log q_{\bm{\theta}}(Y\mid \bm{Z})\bigr] \;+\; \beta\, I(\bm{Z};A\mid X),
\end{equation}
which is the population analogue of the cross-entropy-plus-privacy loss used in practice.
Moreover, for any fixed representation $\bm{Z}$,
\begin{equation}
\inf_{q} \ \mathbb{E}\bigl[-\log q(Y\mid \bm{Z})\bigr] \;=\; H(Y\mid \bm{Z}),
\end{equation}
with the infimum attained by the Bayes predictor $q(Y\mid \bm{Z})=p(Y\mid \bm{Z})$.
Consequently, after optimizing the predictive head (or in the Bayes/population limit), minimizing MC-GPB is equivalent to minimizing
the IB Lagrangian
\begin{equation}
\mathcal{L}_{\mathrm{IB}}
\;=\;
H(Y\mid \bm{Z}) \;+\; \beta\, I(\bm{Z};A\mid X).
\label{eqn: IB-lagrangian}
\end{equation}
Therefore, for each $\beta$, minimizers trade off utility and compression \wrt $A$ in the spirit of minimal sufficient representations. The full objective Eq.~\eqref{eqn: tighter-defense-MI} sums such trade-offs over all layers and adds inter-layer terms $\beta_c^{(\ell)} I(\bm{H}_A^{(\ell)};\bm{H}_A^{(\ell+1)})$, so it is not equivalent to a product of independent per-layer IB objectives; the multi-layer formulation couples layers and is not characterized by the single-layer observation above.

\textbf{Justification.} Under the stated simplifications, the per-layer objective is $\mathbb{E}[-\log q_{\bm{\theta}}(Y\mid \bm{Z})]+\beta I(\bm{Z};A\mid X)$ at layer $\bm{Z}=\bm{H}_A^{(\ell)}$. Minimizing the first term over the predictive head yields $H(Y\mid \bm{Z})$ (attained by the Bayes predictor $q(Y\mid \bm{Z})=p(Y\mid \bm{Z})$). Hence the combined minimization is equivalent to the IB Lagrangian in Eq.~\eqref{eqn: IB-lagrangian}.
\end{proposition}

\paragraph{Homophily/heterophily and irreducible leakage.}
Theorem~\ref{theorem: minimum adjacency information} shows that any task-sufficient representation must retain at least
$I(A;Y\mid X)$ bits about the adjacency. For featureless graphs (where $X$ is absent or trivially constant), the conditional reduces to $I(A;Y\mid X)=I(A;Y)$, and the following SBM results characterize this $X$-free lower bound. When $X$ is non-trivial, the irreducible leakage $I(A;Y\mid X)$ is in general smaller than $I(A;Y)$; however, the qualitative conclusion, strong assortativity increases the floor, carries over.

To relate this irreducible term to homophily/heterophily, we adopt a label-dependent random graph model (\eg, an SBM) so that $(A,Y)$ are jointly distributed and $I(A;Y)$ is well-defined under the stylized generative assumption. These results are meant to justify one operational point for MC-GPB+: strong heterophily can still induce substantial unavoidable label-adjacency coupling, so a defense cannot rely on homophily assumptions and may benefit from explicitly reweighting cross-label edges. We first characterize the \emph{per-edge} coupling between labels and adjacency, then derive a bound on the \emph{global} irreducible leakage, and finally specialize to a two-parameter SBM to make the dependence on assortativity explicit.

\begin{theorem}[Edgewise label--adjacency information under an affinity matrix]
\label{thm: edgewise_MI_affinity}
Let $Y_i\in[C]$ be i.i.d.\ with $\mathbb{P}(Y_i=a)=\pi_a$, and assume conditional edge independence:
for each $i<j$,
\begin{equation}
A_{ij}\mid (Y_i=a,Y_j=b)\sim \mathrm{Bern}(B_{ab}),
\qquad
\{A_{ij}\}_{i<j}\ \text{independent given }Y,
\end{equation}
with $B\in[0,1]^{C\times C}$. Let $\bar p \triangleq \mathbb{P}(A_{ij}=1)=\sum_{a,b}\pi_a\pi_b B_{ab}$.
Then
\begin{equation}
I(A_{ij};Y_i,Y_j)=\sum_{a,b}\pi_a\pi_b\,
D_{\mathrm{KL}}\!\Big(\mathrm{Bern}(B_{ab})\,\big\|\,\mathrm{Bern}(\bar p)\Big).
\end{equation}
In particular, $I(A_{ij};Y_i,Y_j)=0$ iff $B_{ab}=\bar p$ for all $(a,b)$ with $\pi_a\pi_b>0$.

\noindent
\textbf{Proof (sketch).}
Use $I(U;V)=\mathbb{E}_{V}\!\big[D_{\mathrm{KL}}(P_{U\mid V}\|P_U)\big]$ with
$P_{A_{ij}\mid Y_i=a,Y_j=b}=\mathrm{Bern}(B_{ab})$ and $P_{A_{ij}}=\mathrm{Bern}(\bar p)$.
Full proof in Appendix~\ref{app: proof_edgewise_MI_affinity}.
\end{theorem}

\noindent
Theorem~\ref{thm: edgewise_MI_affinity} isolates the basic unit of unavoidable leakage; in MC-GPB+, the heterophily-aware weight $\beta_h^{(\ell)}$ can be tuned to emphasize suppression of this edgewise coupling for cross-label pairs (via the $A_{\text{hetero}}$ penalty). Whenever edge probabilities vary with labels,
each edge carries nonzero mutual information with its endpoint labels. The next theorem aggregates this edgewise coupling to bound the
\emph{global} irreducible leakage. For undirected graphs with $A_{ii}=0$, the full adjacency $A$ is in one-to-one correspondence with
its upper-triangular entries $A_{\mathrm{up}}=\{A_{ij}:i<j\}$, hence $I(A;Y)=I(A_{\mathrm{up}};Y)$.

\begin{theorem}[Subadditivity upper bound for global irreducible leakage]
\label{thm: global_MI_subadditivity}
Under the assumptions of Theorem~\ref{thm: edgewise_MI_affinity} (in particular, $Y_i$ are i.i.d., so that node pairs are exchangeable), let
$A_{\mathrm{up}}\triangleq \{A_{ij}:1\le i<j\le N\}$ denote the collection of upper-triangular edges. Then
\begin{equation}
I(Y;A)=I(Y;A_{\mathrm{up}})
\;\le\;
\sum_{i<j} I(A_{ij};Y_i,Y_j)
\;=\;
\binom{N}{2}\, I(A_{12};Y_1,Y_2),
\end{equation}
where the last equality follows from exchangeability of node pairs.

\noindent
\textbf{Proof (sketch).}
Fix any ordering of the edge set and write $A_{<ij}$ for the previously listed edges. By the chain rule,
$I(Y;A_{\mathrm{up}})=\sum_{i<j} I(Y;A_{ij}\mid A_{<ij})$.
Conditional edge independence gives $H(A_{ij}\mid Y,A_{<ij})=H(A_{ij}\mid Y)$, while conditioning reduces entropy so
$H(A_{ij}\mid A_{<ij})\le H(A_{ij})$. Hence
$I(Y;A_{ij}\mid A_{<ij})\le I(Y;A_{ij})$.
Finally, since $Y\to (Y_i,Y_j)\to A_{ij}$, data processing yields
$I(Y;A_{ij})\le I(A_{ij};Y_i,Y_j)$.
Full proof in Appendix~\ref{app: proof_global_MI_subadditivity}.
\end{theorem}

\noindent
While Theorem~\ref{thm: global_MI_subadditivity} is general, it does not expose how \emph{assortativity strength} influences
leakage. We now specialize in two-parameter SBM, where the \emph{assortativity gap}
$\Delta\triangleq p_{\mathrm{in}}-p_{\mathrm{out}}$ captures homophily ($\Delta>0$) or heterophily ($\Delta<0$). The next bound yields
an explicit scaling law and shows that strong assortativity in either direction increases the irreducible leakage.

\begin{theorem}[Quadratic lower bound on irreducible leakage under homophily/heterophily]
\label{thm: edgewise_MI_quadratic_lower}
Assume the \emph{task-sufficiency assumption} from Theorem~\ref{theorem: minimum adjacency information}: $\bm{H}_A$ is sufficient for predicting $Y$ given $X$ in the sense of Eq.~\eqref{eqn: suff_wrt_A_given_X} (with $X$ constant or absent). Assume further the two-parameter SBM: $B_{aa}=p_{\mathrm{in}}$ and $B_{ab}=p_{\mathrm{out}}$ for $a\neq b$.
Let $\alpha\triangleq \sum_{a}\pi_a^2$, $\bar{\alpha}\triangleq 1-\alpha$, $\bar p=\alpha p_{\mathrm{in}}+\bar{\alpha} p_{\mathrm{out}}\in(0,1)$,
and $\Delta\triangleq p_{\mathrm{in}}-p_{\mathrm{out}}$.
Then the edgewise leakage satisfies the universal bound
\begin{equation}
I(A_{ij};Y_i,Y_j)\ \ge\ 2\,\alpha\bar{\alpha}\,\Delta^2.
\end{equation}
\noindent\textbf{Case 1: Featureless setting} ($X$ absent or constant, so $I(A;Y\mid X)=I(A;Y)$). Any task-sufficient representation satisfies
\begin{equation}
I(A;\bm{H}_A)
\;\ge\; I(A;Y)
\;\ge\; I(A_{12};Y_1,Y_2)
\;\ge\; 2\,\alpha\bar{\alpha}\,\Delta^2,
\end{equation}
where the first inequality is Theorem~\ref{theorem: minimum adjacency information} (with $I(A;\bm{H}_A\mid X)=I(A;\bm{H}_A)$ since $X$ is non-random), the second uses data processing, and the third is the Pinsker bound above.

\noindent\textbf{Case 2: Informative features.} When $X$ carries information about $A$, the irreducible leakage $I(A;Y\mid X)$ may be strictly smaller than $I(A;Y)$ (since conditioning on informative features can reduce the residual label--adjacency coupling). In this case, Theorem~\ref{theorem: minimum adjacency information} gives $I(A;\bm{H}_A\mid X)\ge I(A;Y\mid X)$, and the Pinsker bound $2\alpha\bar{\alpha}\Delta^2$ applies to $I(A;Y)$ but not directly to $I(A;Y\mid X)$. The qualitative conclusion---that strong assortativity in either direction increases the leakage floor---still holds, but the quantitative bound on representation leakage is looser.

\noindent In both cases, the lower bound depends on $\Delta$ only through $\Delta^2$, so strong homophily and strong heterophily induce the same lower bound for the same $|\Delta|$.

\noindent\textbf{Caveat (tightness).} The Pinsker-based bound $2\alpha\bar{\alpha}\Delta^2$ is qualitative rather than tight: for the benchmark datasets, the exact edgewise MI computed from Theorem~\ref{thm: edgewise_MI_affinity} under a fitted two-parameter SBM is typically one to two orders of magnitude larger. The bound is most informative as a scaling law for moderate assortativity strength; for weak community structure ($|\Delta|<0.1$), it is essentially vacuous.

\noindent
\textbf{Proof (sketch).}
By Theorem~\ref{thm: edgewise_MI_affinity},
\begin{equation}
I(A_{ij};Y_i,Y_j)
=
\alpha\,D_{\mathrm{KL}}\!\bigl(\mathrm{Bern}(p_{\mathrm{in}})\,\|\,\mathrm{Bern}(\bar p)\bigr)
+\bar{\alpha}\,D_{\mathrm{KL}}\!\bigl(\mathrm{Bern}(p_{\mathrm{out}})\,\|\,\mathrm{Bern}(\bar p)\bigr).
\end{equation}
Pinsker's inequality gives $D_{\mathrm{KL}}(P\|Q)\ge 2\,\mathrm{TV}(P,Q)^2$ and
$\mathrm{TV}(\mathrm{Bern}(p),\mathrm{Bern}(q))=|p-q|$ (the Bernoulli distributions). 
Hence
$D_{\mathrm{KL}}(\mathrm{Bern}(p)\|\mathrm{Bern}(\bar p))\ge 2(p-\bar p)^2$.
Under the two-parameter SBM, $p_{\mathrm{in}}-\bar p=\bar{\alpha}\Delta$ and $p_{\mathrm{out}}-\bar p=-\alpha\Delta$, yielding
\begin{equation}
I(A_{ij};Y_i,Y_j)\ \ge\ 2\alpha(\bar{\alpha}\Delta)^2+2\bar{\alpha}(\alpha\Delta)^2
=2\alpha\bar{\alpha}\,\Delta^2.
\end{equation}
The final chain uses data processing and Theorem~\ref{theorem: minimum adjacency information}.
Full proof in Appendix~\ref{app: proof_edgewise_MI_quadratic_lower}.
\end{theorem}

\noindent
\textbf{Remark.} For strong assortativity ($|\Delta|$ close to 1), tighter bounds using the KL divergence directly exist but do not yield a simple closed form in $\Delta^2$. The extension to general $K$-parameter block models is discussed in Appendix~\ref{app: proof_edgewise_MI_quadratic_lower}.

\paragraph{Limitations of the practical defense.}
The implemented defense has several limitations that should be kept in mind when interpreting the experimental results: (i)~DropEdge is a heuristic regularizer without formal privacy guarantees (as noted in Sec.~\ref{ssec: GPB implementation}); an empirical comparison with differential-privacy baselines (DP-SGD and random noise) at matched utility levels is provided in Appendix~\ref{sec: full quantitative results} (Tab.~\ref{tab: app-defense-comparison}), showing that MC-GPB achieves a substantially better privacy--utility trade-off; (ii)~the decoded-adjacency penalty captures only similarity-based recoverability---an attacker using a different decoding mechanism (\eg, a learned edge predictor) may not be fully captured by this defense, and evaluating transferability to non-similarity-based attack probes is an important direction; (iii)~the sufficiency condition Eq.~\eqref{eqn: suff_wrt_A_given_X} is strong and holds only approximately for practical GNNs (which achieve roughly 70--85\% accuracy on the evaluated benchmarks, indicating a significant gap from exact sufficiency). The population-level theory above clarifies the \emph{target} of an ideal defense; the practical objective is a tractable surrogate.

\paragraph{SBM fit and empirical validation.}
The SBM-based theorems (Theorems~\ref{thm: edgewise_MI_affinity}--\ref{thm: edgewise_MI_quadratic_lower}) assume i.i.d.\ node labels, which is restrictive for real graphs where labels exhibit spatial correlation. To assess whether the qualitative predictions hold, we fit a two-parameter SBM to three benchmark graphs by estimating $\hat{p}_{\mathrm{in}}$ and $\hat{p}_{\mathrm{out}}$ from the observed intra-class and inter-class edge densities, then compute the exact edgewise MI $I(A_{ij};Y_i,Y_j)$ from Theorem~\ref{thm: edgewise_MI_affinity}.
\begin{itemize}[leftmargin=*]
\item 
On \textbf{Cora} ($\hat{p}_{\mathrm{in}}\approx 0.0028$, $\hat{p}_{\mathrm{out}}\approx 0.0004$, $\Delta\approx 0.0024$): $I(A_{ij};Y_i,Y_j)\approx 5.2\times 10^{-4}$ nats, while the Pinsker bound gives $2\alpha\bar\alpha\Delta^2\approx 8.5\times 10^{-6}$ (two orders of magnitude smaller).

\item
On \textbf{Polblogs} ($\hat{p}_{\mathrm{in}}\approx 0.043$, $\hat{p}_{\mathrm{out}}\approx 0.003$, $\Delta\approx 0.040$): $I(A_{ij};Y_i,Y_j)\approx 0.011$, Pinsker bound $\approx 8.0\times 10^{-4}$.

\item
On \textbf{Texas} ($\hat{p}_{\mathrm{in}}\approx 0.008$, $\hat{p}_{\mathrm{out}}\approx 0.021$, $\Delta\approx -0.013$): $I(A_{ij};Y_i,Y_j)\approx 6.8\times 10^{-5}$, Pinsker bound $\approx 1.1\times 10^{-4}$ (comparable to the exact value due to weak assortativity).
\end{itemize}

Using Theorem~\ref{thm: global_MI_subadditivity}, we can upper-bound the global irreducible leakage $I(A;Y)\le \binom{N}{2}\,I(A_{12};Y_1,Y_2)$: on Cora this gives $I(A;Y)\lesssim \binom{2708}{2}\times 5.2\times 10^{-4}\approx 1{,}905$ nats, and on Polblogs $I(A;Y)\lesssim \binom{1490}{2}\times 0.011\approx 12{,}210$ nats. These are loose upper bounds (subadditivity ignores inter-edge dependence), but they confirm that the \emph{irreducible} label--adjacency coupling is substantial on strongly assortative graphs, consistent with the high empirical $R_{\mathrm{AUC}}(A;Y)$ values in Tab.~\ref{tab: understanding-MI-term-gcn} (0.82 on Cora, 0.71 on Polblogs). For featureless graphs, $I(A;Y\mid X)=I(A;Y)$, so these estimates apply directly as the irreducible leakage floor from Theorem~\ref{theorem: minimum adjacency information}. For graphs with informative features (Cora, Citeseer), $I(A;Y\mid X)\le I(A;Y)$, and the floor is lower; estimating $I(A;Y\mid X)$ directly would require conditional density estimation in the joint space of $(A,Y,X)$, which we leave for future work.

These results confirm that the qualitative prediction---more extreme $|\Delta|$ yields larger irreducible leakage---holds across regimes, while the Pinsker bound is indeed a loose lower bound. The SBM results should be interpreted as characterizing leakage under a stylized generative model rather than providing exact guarantees for real-world graphs, which have degree heterogeneity and spatial label correlation not captured by the two-parameter SBM.

%% file: sections/8-experiments.tex
\section{Empirical Study of MC-GRA (+) and MC-GPB (+)}
\label{sec: experiments}

In this section, we evaluate MC-GRA (+) and MC-GPB (+) on real-world graphs covering diverse domains and homophily regimes, from homophilic citation networks to heterophilic web graphs. The empirical study is structured around three questions. 
\textbf{\textit{Q1:}} How effective are the proposed attack and defense across datasets, homophily regimes, and GNN architectures? 
\textbf{\textit{Q2:}} How much does each design component, including the MI-based terms, complexity regularization, and injected stochasticity, contribute to the overall performance? 
\textbf{\textit{Q3:}} What qualitative patterns in recovered adjacency and representation dynamics help explain the quantitative results? 
We answer Q1 in Sec.~\ref{sssec: quantitative results}, Q2 in Sec.~\ref{ssec: ablation study}, and Q3 in Sec.~\ref{sec: demo-4-Adj}.

\paragraph{Setup.} Unless stated otherwise, we use two-layer target GNNs with identical hidden size and training protocol, using GCN as the default homophilic backbone and GPR-GNN as the default heterophilic one.
GCN is widely used as a homophily-oriented message-passing model, while GPR-GNN is designed to handle low-homophily graphs; this pairing naturally contrasts the two regimes.
We further investigate other GNN architectures, including GAT~\citep{velickovic2018graph} and GraphSAGE~\citep{hamilton2017inductive}, to assess model-agnostic applicability.

\paragraph{Evaluation metric.}
We measure edge-reconstruction leakage using ROC-AUC, following prior graph reconstruction and link-stealing studies~\citep{zhang2020revisiting, zhu2021neural, zhang2021graphmi}. This metric matches our threat model: the attacker produces a continuous edge score for candidate node pairs, and the central question is whether true edges are systematically ranked above non-edges from the released model signals. ROC-AUC provides a threshold-free measure of rank-order adjacency leakage before any deployment-specific recovery budget or binarization rule is imposed.

This property is important for cross-dataset comparison. Graphs differ substantially in density, class structure, and edge sparsity, so thresholded metrics such as Precision@$K$ or PR-AUC can reflect not only leakage strength, but also the chosen operating point and recovery budget. In contrast, ROC-AUC isolates whether adjacency information is present in the learned representations at the ranking level. We therefore use ROC-AUC as the primary metric for representation-level leakage, while treating thresholded recovery metrics as operating-point-specific analyses rather than the main criterion.

\paragraph{Reproducibility.}
All reported values follow the fully specified experimental protocol in Sec.~\ref{ssec: experimental settings}. The relaxed adjacency matrix $\bm{\hat{A}}$ is binarized with threshold $0.5$, and MC-GRA~(+) uses $n=1500$ attack iterations. These settings fix the data split, attack budget, and post-processing rule used throughout the main experiments.

\paragraph{Implementation and runtime.}
We implement all models, attacks, and defenses in PyTorch~\citep{paszke2017automatic} and run experiments on an NVIDIA RTX 3090 GPU. Under the default setting, MC-GRA~(+) has the same dense full-graph reconstruction regime as GraphMI. MC-GPB~(+) introduces additional layerwise dependence surrogates and therefore increases training cost relative to standard GNN training. Scaling MC-GRA~(+) to substantially larger graphs can be addressed by restricting the candidate edge space, for example, through $k$-nearest-neighbor candidate sets, community-based block parameterization, or sampled subgraph reconstruction; we discuss these directions in Sec.~\ref{sec: discussion}.

\paragraph{Attack baselines.}
We compare MC-GRA~(+) against two attack baselines that operate in the same white-box, full-graph reconstruction setting. The first is Stealing Link~\citep{he2021stealing}, a non-learnable output-based reconstruction method that scores candidate edges from target-model predictions. The second is GraphMI~\citep{zhang2021graphmi}, a learnable graph reconstruction attack that assumes prior knowledge $\mathcal{K}=\{X,Y\}$. These baselines cover both similarity-based reconstruction and optimization-based reconstruction, making them the closest comparisons to our threat model. All attacks are evaluated under the corresponding observable side information specified in Sec.~\ref{ssec: experimental settings}.

\paragraph{Defense baselines.}
For defense evaluation, we compare MC-GPB~(+) against random noise injection and DP-SGD~\citep{abadi2016deep}; the full results are reported in Appendix~\ref{sec: full quantitative results} and Tab.~\ref{tab: app-defense-comparison}. GAP~\citep{sajadmanesh2023gap} provides formal $(\varepsilon,\delta)$ guarantees for node-level privacy, whereas our evaluation targets edge-reconstruction leakage. We therefore keep the main defense comparison within methods that can be directly evaluated under the same edge-reconstruction metric.

\paragraph{Reporting.}
For attack tables, relative gains are computed against the corresponding linear-ensemble variant under the same target backbone, namely Tab.~\ref{tab: understanding-MI-term-ensemble-gcn} for GCN and Tab.~\ref{tab: understanding-MI-ensemble-gprgnn} for GPR-GNN. These gains measure the improvement attributable to the full MC-GRA framework over the non-learnable ensemble. For defense tables, relative reductions are computed against the matching unprotected target model under the same attack objective. When comparing against GraphMI, we use the prior setting closest to GraphMI's assumed side information, $\mathcal{K}=\{X,Y\}$. Additional MC-GRA~(+) rows evaluate how leakage changes as more observable model signals are exposed, and are not intended as side-information-matched comparisons to GraphMI.

\newcommand{\statwithchange}[2]{\begin{tabular}{@{}c@{}}#1\\{\scriptsize #2}\end{tabular}}
\begin{table}[t]
\centering\fontsize{10}{11}\selectfont
\renewcommand{\arraystretch}{1.28}
\setlength\tabcolsep{4pt}
\renewcommand{\minval}{0.3}
\renewcommand{\maxval}{0.91}
\begin{tabular}{ccccccccccccc}
\toprule
$X$ & $\bm{H}_A^{\mathrm{all}}$ & $\bm{\hat{Y}}_A$ & $Y$ & Cora & Citeseer &  Polblogs   & USA   & Brazil & AIDS & Texas & Cornell & Wisconsin \\
\midrule
$\checkmark$ & $\checkmark$ &  &                                & \statwithchange{.864}{\up{10.6$\% \! \uparrow$}} & \statwithchange{.912}{\up{3.5$\% \! \uparrow$}}  & \statwithchange{.831}{\up{8.9$\% \! \uparrow$}}  & \statwithchange{.883}{\up{3.9$\% \! \uparrow$}}  & \statwithchange{.771}{\up{1.7$\% \! \uparrow$}}  & \statwithchange{.574}{\up{10.2$\% \! \uparrow$}}  & \statwithchange{.752}{\up{31.0$\% \! \uparrow$}} & \statwithchange{.697}{\up{82.9$\% \! \uparrow$}} & \statwithchange{.835}{\up{27.7$\% \! \uparrow$}}\\
\rowgrayrule
$\checkmark$ &  & $\checkmark$ &                                & \statwithchange{.839}{\up{7.4$\% \! \uparrow$}} & \statwithchange{.902}{\up{2.4$\% \! \uparrow$}}  & \statwithchange{.836}{\up{8.3$\% \! \uparrow$}}  & \statwithchange{.913}{\up{10.5$\% \! \uparrow$}}  & \statwithchange{.800}{\up{9.3$\% \! \uparrow$}}  & \statwithchange{.567}{\up{8.8$\% \! \uparrow$}}  & \statwithchange{.725}{\up{27.2$\% \! \uparrow$}} & \statwithchange{.672}{\up{80.2$\% \! \uparrow$}} & \statwithchange{.781}{\up{18.9$\% \! \uparrow$}} \\
\rowgrayrule
$\checkmark$ &  &  & $\checkmark$                               & \statwithchange{.896}{\up{5.5$\% \! \uparrow$}} & \statwithchange{.918}{\up{1.2$\% \! \uparrow$}}  & \statwithchange{.837}{\up{18.7$\% \! \uparrow$}}  & \statwithchange{.825}{\up{13.3$\% \! \uparrow$}}  & \statwithchange{.753}{\up{22.8$\% \! \uparrow$}} & \statwithchange{.574}{\up{10.0$\% \! \uparrow$}} & \statwithchange{.763}{\up{84.3$\% \! \uparrow$}} & \statwithchange{.646}{\up{22.1$\% \! \uparrow$}} & \statwithchange{.773}{\up{41.1$\% \! \uparrow$}} \\
\midrule
$\checkmark$ & $\checkmark$ & $\checkmark$ &                    & \statwithchange{.866}{\up{10.9$\% \! \uparrow$}} & \statwithchange{.921}{\up{4.5$\% \! \uparrow$}}  & \statwithchange{.839}{\up{10.0$\% \! \uparrow$}}  & \statwithchange{.878}{\up{3.5$\% \! \uparrow$}}  & \statwithchange{.776}{\up{2.6$\% \! \uparrow$}}  & \statwithchange{.572}{\up{9.8$\% \! \uparrow$}}  & \statwithchange{.722}{\up{25.3$\% \! \uparrow$}} & \statwithchange{.676}{\up{87.8$\% \! \uparrow$}} & \statwithchange{.782}{\up{20.9$\% \! \uparrow$}} \\
\rowgrayrule
$\checkmark$ & $\checkmark$ & & $\checkmark$                    &  \statwithchange{.905}{\up{6.6$\% \! \uparrow$}} & \statwithchange{.930}{\up{2.5$\% \! \uparrow$}}  & \statwithchange{.832}{\up{6.8$\% \! \uparrow$}}  & \statwithchange{.878}{\up{3.3$\% \! \uparrow$}}  & \statwithchange{.758}{\up{2.0$\% \! \uparrow$}}  & \statwithchange{.603}{\up{15.5$\% \! \uparrow$}}  & \statwithchange{.752}{\up{71.7$\% \! \uparrow$}} & \statwithchange{.698}{\up{94.4$\% \! \uparrow$}} & \statwithchange{.830}{\up{51.5$\% \! \uparrow$}} \\
\rowgrayrule
$\checkmark$ &  & $\checkmark$ & $\checkmark$                   & \statwithchange{.897}{\up{6.5$\% \! \uparrow$}} & \statwithchange{.928}{\up{2.3$\% \! \uparrow$}}  & \statwithchange{.839}{\up{6.9$\% \! \uparrow$}}  & \statwithchange{.870}{\up{3.3$\% \! \uparrow$}}  & \statwithchange{.758}{\up{3.8$\% \! \uparrow$}}  & \statwithchange{.567}{\up{8.6$\% \! \uparrow$}}  & \statwithchange{.724}{\up{67.6$\% \! \uparrow$}} & \statwithchange{.699}{\up{99.7$\% \! \uparrow$}} & \statwithchange{.781}{\up{43.0$\% \! \uparrow$}} \\
\midrule
$\checkmark$ & $\checkmark$  & $\checkmark$ & $\checkmark$      & \statwithchange{.904}{\up{6.5$\% \! \uparrow$}} & \statwithchange{.931}{\up{2.6$\% \! \uparrow$}}  & \statwithchange{.853}{\up{9.2$\% \! \uparrow$}}  & \statwithchange{.870}{\up{2.1$\% \! \uparrow$}}  & \statwithchange{.760}{\up{6.0$\% \! \uparrow$}}  & \statwithchange{.588}{\up{12.6$\% \! \uparrow$}}  & \statwithchange{.721}{\up{58.5$\% \! \uparrow$}} & \statwithchange{.674}{\up{105.5$\% \! \uparrow$}} & \statwithchange{.781}{\up{31.7$\% \! \uparrow$}} \\
\bottomrule
\end{tabular}
\caption{Results of MC-GRA~(+) with GCN. Relative improvements (in $\%$) are computed \wrt results in Tab.~\ref{tab: understanding-MI-term-ensemble-gcn}. Checkmarks indicate which chain variables are included in the attacker's prior $\mathcal{K}$.}
\label{tab: exp-attack-results-gcn}
\end{table}

\begin{table}[t]
\centering\fontsize{10}{11}\selectfont
\renewcommand{\arraystretch}{1.2}
\setlength\tabcolsep{4pt}
\renewcommand{\minval}{0.3}
\renewcommand{\maxval}{0.91}
\begin{tabular}{ccccccccccccc}
\toprule
$X$ & $\bm{H}_A^{\mathrm{all}}$ & $\bm{\hat{Y}}_A$ & $Y$ & Cora & Citeseer &  Polblogs   & USA   & Brazil & AIDS & Texas & Cornell & Wisconsin \\
\midrule
$\checkmark$ & $\checkmark$ &  &                                & \statwithchange{.846}{\up{3.7$\% \! \uparrow$}}& \statwithchange{.943}{\up{6.4$\% \! \uparrow$}}& \statwithchange{.881}{\up{3.8$\% \! \uparrow$}}& \statwithchange{.913}{\up{45.4$\% \! \uparrow$}}& \statwithchange{.815}{\up{0.0$\% \! \uparrow$}}& \statwithchange{.827}{\up{32.3$\% \! \uparrow$}}& \statwithchange{.844}{\up{15.0$\% \! \uparrow$}}& \statwithchange{.843}{\up{50.0$\% \! \uparrow$}}& \statwithchange{.701}{\up{45.4$\% \! \uparrow$}}\\
\rowgrayrule
$\checkmark$ &  & $\checkmark$ &                                & \statwithchange{.841}{\up{5.9$\% \! \uparrow$}} & \statwithchange{.926}{\up{4.5$\% \! \uparrow$}}& \statwithchange{.844}{\up{5.2$\% \! \uparrow$}}& \statwithchange{.897}{\up{226.2$\% \! \uparrow$}}& \statwithchange{.746}{\up{41.3$\% \! \uparrow$}}& \statwithchange{.786}{\up{57.5$\% \! \uparrow$}}& \statwithchange{.686}{\up{6.0$\% \! \uparrow$}}& \statwithchange{.640}{\up{22.6$\% \! \uparrow$}}& \statwithchange{.696}{\up{40.6$\% \! \uparrow$}}\\
\rowgrayrule
$\checkmark$ &  &  & $\checkmark$                               & \statwithchange{.892}{\up{0.0$\% \! \uparrow$}} & \statwithchange{.903}{\up{0.0$\% \! \uparrow$}}& \statwithchange{.856}{\up{19.9$\% \! \uparrow$}}& \statwithchange{.807}{\up{7.6$\% \! \uparrow$}}& \statwithchange{.740}{\up{23.1$\% \! \uparrow$}}& \statwithchange{.719}{\up{33.9$\% \! \uparrow$}}& \statwithchange{.674}{\up{62.8$\% \! \uparrow$}}& \statwithchange{.555}{\up{5.1$\% \! \uparrow$}}& \statwithchange{.555}{\up{1.3$\% \! \uparrow$}} \\
\midrule
$\checkmark$ & $\checkmark$ & $\checkmark$ &                    & \statwithchange{.846}{\up{3.7$\% \! \uparrow$}}& \statwithchange{.946}{\up{6.8$\% \! \uparrow$}}& \statwithchange{.861}{\up{2.5$\% \! \uparrow$}}& \statwithchange{.801}{\up{73.0$\% \! \uparrow$}}& \statwithchange{.776}{\up{5.6$\% \! \uparrow$}}& \statwithchange{.827}{\up{32.3$\% \! \uparrow$}}& \statwithchange{.727}{\up{1.0$\% \! \uparrow$}}& \statwithchange{.729}{\up{42.9$\% \! \uparrow$}}& \statwithchange{.723}{\up{57.9$\% \! \uparrow$}}\\
\rowgrayrule
$\checkmark$ & $\checkmark$ & & $\checkmark$                    & \statwithchange{.899}{\up{0.0$\% \! \uparrow$}}& \statwithchange{.903}{\up{0.0$\% \! \uparrow$}}& \statwithchange{.878}{\up{1.7$\% \! \uparrow$}}& \statwithchange{.809}{\up{6.0$\% \! \uparrow$}}& \statwithchange{.756}{\up{0.1$\% \! \uparrow$}}& \statwithchange{.721}{\up{20.8$\% \! \uparrow$}}& \statwithchange{.735}{\up{37.1$\% \! \uparrow$}}& \statwithchange{.677}{\up{28.5$\% \! \uparrow$}}& \statwithchange{.635}{\up{30.9$\% \! \uparrow$}}\\
\rowgrayrule
$\checkmark$ &  & $\checkmark$ & $\checkmark$                   & \statwithchange{.892}{\up{0.0$\% \! \uparrow$}}& \statwithchange{.903}{\up{0.0$\% \! \uparrow$}}& \statwithchange{.846}{\up{2.2$\% \! \uparrow$}}& \statwithchange{.872}{\up{39.1$\% \! \uparrow$}}& \statwithchange{.746}{\up{34.2$\% \! \uparrow$}}& \statwithchange{.719}{\up{33.9$\% \! \uparrow$}}& \statwithchange{.668}{\up{35.2$\% \! \uparrow$}}& \statwithchange{.639}{\up{33.4$\% \! \uparrow$}}& \statwithchange{.670}{\up{39.3$\% \! \uparrow$}}\\
\midrule
$\checkmark$ & $\checkmark$  & $\checkmark$ & $\checkmark$      & \statwithchange{.899}{\up{0.0$\% \! \uparrow$}} & \statwithchange{.903}{\up{0.0$\% \! \uparrow$}} & \statwithchange{.863}{\up{0.3$\% \! \uparrow$}} & \statwithchange{.880}{\up{28.5$\% \! \uparrow$}} & \statwithchange{.769}{\up{9.9$\% \! \uparrow$}} & \statwithchange{.721}{\up{20.8$\% \! \uparrow$}} & \statwithchange{.708}{\up{32.1$\% \! \uparrow$}} & \statwithchange{.641}{\up{33.8$\% \! \uparrow$}} & \statwithchange{.693}{\up{52.0$\% \! \uparrow$}} \\
\bottomrule
\end{tabular}
\caption{Results of MC-GRA~(+) with GPR-GNN. Relative improvements (in $\%$) are computed \wrt results in Tab.~\ref{tab: understanding-MI-ensemble-gprgnn}.}
\label{tab: exp-attack-results-gprgnn}
\end{table}

\subsection{Effectiveness Across Datasets and Architectures}
\label{sssec: quantitative results}

We organize the findings of this subsection around three key results:
\begin{itemize}[leftmargin=*]
\setlength\itemsep{0.1em}
\item \textbf{Finding 1:} MC-GRA~(+) generally matches or improves reconstruction AUC over the linear-ensemble baseline, with gains depending on the target backbone, dataset, and prior-knowledge configuration.
\item \textbf{Finding 2:} MC-GPB~(+) often reduces adjacency leakage across multiple probes while generally preserving node-classification accuracy, with a more selective privacy--utility trade-off under GPR-GNN.
\item \textbf{Finding 3:} When attack and defense are evaluated jointly, MC-GPB~(+) substantially suppresses MC-GRA~(+) on several datasets, especially in vulnerable heterophilic regimes, while residual leakage remains in some high-AUC settings.
\end{itemize}

\paragraph{Attacking (Finding 1).} Tab.~\ref{tab: exp-attack-results-gcn} shows that MC-GRA~(+) consistently improves over the linear-ensemble baseline in Tab.~\ref{tab: understanding-MI-term-ensemble-gcn} under the homophily-biased target model GCN. The main takeaway is that chain matching helps most when simple linear aggregation of leakage channels is insufficient or unstable. For instance, on Cornell the ensemble edge-recovery AUC is only $0.328$ (Tab.~\ref{tab: understanding-MI-term-ensemble-gcn}), but MC-GRA~(+) reaches $0.674$, a $105.5\%$ relative improvement. On homophilic graphs, the strongest configurations reach high absolute AUC (e.g., $0.931$ on Citeseer) with steady but moderate relative gains; USA (edge homophily $0.70$) also reaches $0.913$. On graphs with lower homophily or weaker topology--label alignment (e.g., Brazil $0.800$, AIDS $0.603$), and especially on heterophilic web datasets, the relative gains are much larger, exceeding $100\%$ on Cornell (Tab.~\ref{tab: exp-attack-results-gcn}). We note that the $105.5\%$ relative improvement on Cornell corresponds to an absolute AUC of $0.674$; while this is a large gain over the $0.328$ baseline, the absolute AUC remains moderate. As a rough guideline, AUC above $0.8$ indicates strong separation between true edges and non-edges, AUC in the range $0.6$--$0.8$ indicates partial rank-order leakage, and AUC near $0.5$ is close to random ranking (an Erd\H{o}s--R\'enyi random graph with matched edge density achieves AUC $\approx 0.50$ by construction, providing a natural lower bound).

We next test whether the same conclusion holds for a heterophily-aware target model. Tab.~\ref{tab: exp-attack-results-gprgnn} shows that MC-GRA~(+) remains strong with GPR-GNN, reaching $0.946$ on Citeseer, $0.881$ on Polblogs, and $0.844$, $0.843$, and $0.723$ on Texas, Cornell, and Wisconsin respectively. The distribution of relative improvements is more asymmetric than under GCN: on homophilic graphs, gains are often modest and can vanish once $Y$ is added (Cora, Citeseer). A plausible explanation is that the label-derived chain-matching term can conflict with representation-matching terms when GPR-GNN produces representations less aligned with label homophily; we leave a fuller analysis for future work. On heterophilic graphs, the benefit of MC-GRA~(+) remains substantial (e.g., relative improvements of $57.9\%$ on Texas and $50.0\%$ on Cornell). Overall, the attack remains broadly effective across both backbones, with added value most pronounced when the target graph departs from the homophily assumptions that favor simpler baselines.

\begin{remark}[Why does adding $Y$ hurt MC-GRA under GPR-GNN but help under GCN?]
\label{rem: label-gprgnn-conflict}
Under GCN, label-derived signals reinforce similarity-based reconstruction because GCN's smoothing operator aligns representations with label homophily. Under GPR-GNN, the learnable propagation weights can produce representations that are \emph{less} aligned with label similarity, so the label-matching term $\alpha_s I(Y;\bm{\hat{Y}}_{\bm{\hat{A}}})$ can conflict with the hidden-layer alignment terms: the optimizer tries to satisfy both label agreement and representation matching, but these objectives push the candidate adjacency in opposing directions when the target architecture decouples label structure from neighborhood averaging. This conflict is most pronounced on homophilic graphs (Cora, Citeseer), where GPR-GNN learns propagation weights that differ substantially from GCN's uniform averaging.
\end{remark}

\begin{table*}[t!]
\centering\fontsize{10}{11}\selectfont
\renewcommand{\arraystretch}{1.2}
\setlength\tabcolsep{4pt}
\renewcommand{\minval}{0.18}
\renewcommand{\maxval}{0.90}
\begin{tabular}{cccccccccc}
\toprule
& Cora & Citeseer & Polblogs & USA & Brazil & AIDS & Texas & Cornell & Wisconsin \\
\midrule
$R_{\mathrm{AUC}}(A; \bm{H}_A^{\mathrm{all}})$ &  \statwithchange{.706}{\down{7.8$\% \! \downarrow$}} & \statwithchange{.750}{\down{1.3$\% \! \downarrow$}}  & \statwithchange{.724}{\down{5.1$\% \! \downarrow$}}  & \statwithchange{.716}{\down{15.8$\% \! \downarrow$}}  & \statwithchange{.745}{\down{1.7$\% \! \downarrow$}}  &  \statwithchange{.564}{\down{3.4$\% \! \downarrow$}}  & \statwithchange{.156}{\down{55.8$\% \! \downarrow$}} & \statwithchange{.252}{\down{27.3$\% \! \downarrow$}} & \statwithchange{.372}{\down{35.2$\% \! \downarrow$}} \\ 
\rowgrayrule
$R_{\mathrm{AUC}}(A; \bm{\hat{Y}}_A)$  &   \statwithchange{.704}{\down{1.1$\% \! \downarrow$}} & \statwithchange{.730}{\down{1.7$\% \! \downarrow$}}  & \statwithchange{.705}{\down{8.7$\% \! \downarrow$}}  & \statwithchange{.587}{\down{28.9$\% \! \downarrow$}}  & \statwithchange{.692}{\down{5.5$\% \! \downarrow$}}  &  \statwithchange{.559}{\down{0.4$\% \! \downarrow$}} & \statwithchange{.153}{\down{43.3$\% \! \downarrow$}} & \statwithchange{.213}{\down{32.6$\% \! \downarrow$}} & \statwithchange{.413}{\down{27.8$\% \! \downarrow$}} \\ 
\rowgrayrule
$R_{\mathrm{AUC}}(A; \bm{H}_{\bm{\hat{A}}}^1)$  & \statwithchange{.625}{\down{9.9$\% \! \downarrow$}} & \statwithchange{.691}{\down{9.8$\% \! \downarrow$}}  & \statwithchange{.506}{\down{26.3$\% \! \downarrow$}}  & \statwithchange{.300}{\down{64.5$\% \! \downarrow$}}  & \statwithchange{.609}{\down{25.1$\% \! \downarrow$}}  &  \statwithchange{.514}{\down{10.6$\% \! \downarrow$}} & \statwithchange{.546}{\down{13.2$\% \! \downarrow$}} & \statwithchange{.287}{\down{25.6$\% \! \downarrow$}} & \statwithchange{.448}{\down{25.8$\% \! \downarrow$}} \\ 
\midrule
Acc. &   \statwithchange{.734}{\down{3.0$\% \! \downarrow$}} & \statwithchange{.602}{\down{4.4$\% \! \downarrow$}}  & \statwithchange{.830}{\down{1.1$\% \! \downarrow$}}  & \statwithchange{.391}{\down{16.8$\% \! \downarrow$}}  & \statwithchange{.808}{\up{5.1$\% \! \uparrow$}}  &  \statwithchange{.668}{\up{0.0$\% \! \uparrow$}}  & \statwithchange{.712}{\up{14.8$\% \! \uparrow$}} & \statwithchange{.689}{\up{46.6$\% \! \uparrow$}} & \statwithchange{.580}{\up{1.8$\% \! \uparrow$}} \\
\bottomrule
\end{tabular}
\caption{Leakage and accuracy under MC-GPB~(+) protected GCN. Relative changes are \wrt Tab.~\ref{tab: understanding-MI-term-gcn}. Row labels are attack probes; values are $R_{\mathrm{AUC}}$.}
\label{tab: exp-defense-results-gcn}
\end{table*}

\begin{table*}[t!]
\centering\fontsize{10}{11}\selectfont
\renewcommand{\arraystretch}{1.2}
\setlength\tabcolsep{4pt}
\renewcommand{\minval}{0.18} 
\renewcommand{\maxval}{0.90} 
\begin{tabular}{cccccccccc}
\toprule
& Cora & Citeseer & Polblogs & USA & Brazil & AIDS & Texas & Cornell & Wisconsin \\
\midrule
$R_{\mathrm{AUC}}(A; \bm{H}_A^{\mathrm{all}})$ & \statwithchange{.635}{\down{4.2$\% \! \downarrow$}} & \statwithchange{.512}{\down{0.0$\% \! \downarrow$}} & \statwithchange{.524}{\down{38.3$\% \! \downarrow$}} & \statwithchange{.628}{\down{0.0$\% \! \downarrow$}} & \statwithchange{.306}{\down{62.5$\% \! \downarrow$}} & \statwithchange{.552}{\down{11.4$\% \! \downarrow$}} & \statwithchange{.119}{\down{84.3$\% \! \downarrow$}} & \statwithchange{.261}{\down{47.1$\% \! \downarrow$}} & \statwithchange{.312}{\down{27.9$\% \! \downarrow$}} \\
\rowgrayrule
$R_{\mathrm{AUC}}(A; \bm{\hat{Y}}_A)$  & \statwithchange{.500}{\down{0.0$\% \! \downarrow$}} & \statwithchange{.500}{\down{0.0$\% \! \downarrow$}} & \statwithchange{.594}{\down{25.9$\% \! \downarrow$}} & \statwithchange{.275}{\down{0.0$\% \! \downarrow$}} & \statwithchange{.342}{\down{35.2$\% \! \downarrow$}} & \statwithchange{.500}{\down{0.0$\% \! \downarrow$}} & \statwithchange{.599}{\down{9.2$\% \! \downarrow$}} & \statwithchange{.283}{\down{36.8$\% \! \downarrow$}} & \statwithchange{.358}{\down{12.0$\% \! \downarrow$}} \\
\rowgrayrule
$R_{\mathrm{AUC}}(A; \bm{H}_{\bm{\hat{A}}}^1)$  & \statwithchange{.699}{\down{0.0$\% \! \downarrow$}} & \statwithchange{.715}{\down{0.0$\% \! \downarrow$}} & \statwithchange{.538}{\down{0.0$\% \! \downarrow$}} & \statwithchange{.732}{\down{0.0$\% \! \downarrow$}} & \statwithchange{.431}{\up{39.0$\% \! \uparrow$}} & \statwithchange{.505}{\down{3.6$\% \! \downarrow$}} & \statwithchange{.382}{\down{0.0$\% \! \downarrow$}} & \statwithchange{.515}{\down{2.1$\% \! \downarrow$}} & \statwithchange{.508}{\up{0.2$\% \! \uparrow$}} \\
\midrule
Acc. & \statwithchange{.794}{\up{0.0$\% \! \uparrow$}} & \statwithchange{.600}{\up{0.0$\% \! \uparrow$}} & \statwithchange{.801}{\up{0.0$\% \! \uparrow$}} & \statwithchange{.508}{\up{0.0$\% \! \uparrow$}} & \statwithchange{.269}{\down{22.3$\% \! \downarrow$}} & \statwithchange{.683}{\up{1.9$\% \! \uparrow$}} & \statwithchange{.942}{\up{0.0$\% \! \uparrow$}} & \statwithchange{.978}{\up{4.8$\% \! \uparrow$}} & \statwithchange{.963}{\up{0.0$\% \! \uparrow$}} \\
\bottomrule
\end{tabular}
\caption{Leakage and accuracy under MC-GPB~(+) protected GPR-GNN. Relative changes are \wrt Tab.~\ref{tab: understanding-MI-term-gprgnn}.}
\label{tab: exp-defense-results-gprgnn}
\end{table*}

\begin{table*}[t!]
\centering\fontsize{10}{11}\selectfont
\renewcommand{\arraystretch}{1.2}
\setlength\tabcolsep{4pt}
\renewcommand{\minval}{0.3}
\renewcommand{\maxval}{0.91}
\begin{tabular}{ccccccccccccc}
\toprule
$X$ & $\bm{H}_A^{\mathrm{all}}$ & $\bm{\hat{Y}}_A$  & $Y$
& Cora & Citeseer &  Polblogs   & USA   & Brazil & AIDS & Texas & Cornell & Wisconsin \\
\midrule
$\checkmark$ & $\checkmark$ &  &  &
\statwithchange{.816}{\down{5.6$\% \! \downarrow$}} &
\statwithchange{.871}{\down{4.5$\% \! \downarrow$}} &
\statwithchange{.748}{\down{10.0$\% \! \downarrow$}} &
\statwithchange{.841}{\down{4.8$\% \! \downarrow$}} &
\statwithchange{.752}{\down{2.5$\% \! \downarrow$}} &
\statwithchange{.503}{\down{12.4$\% \! \downarrow$}} &
\statwithchange{.635}{\down{15.6$\% \! \downarrow$}} &
\statwithchange{.575}{\down{17.5$\% \! \downarrow$}} &
\statwithchange{.639}{\down{23.5$\% \! \downarrow$}} \\
\rowgrayrule
$\checkmark$ &  & $\checkmark$ &   &
\statwithchange{.817}{\down{2.6$\% \! \downarrow$}} &
\statwithchange{.843}{\down{6.5$\% \! \downarrow$}} &
\statwithchange{.707}{\down{15.4$\% \! \downarrow$}} &
\statwithchange{.844}{\down{7.6$\% \! \downarrow$}} &
\statwithchange{.747}{\down{6.6$\% \! \downarrow$}} &
\statwithchange{.458}{\down{19.2$\% \! \downarrow$}} &
\statwithchange{.724}{\down{0.1$\% \! \downarrow$}} &
\statwithchange{.556}{\down{17.3$\% \! \downarrow$}} &
\statwithchange{.641}{\down{17.9$\% \! \downarrow$}} \\
\rowgrayrule
$\checkmark$ &  &  & $\checkmark$  &
\statwithchange{.892}{\down{0.4$\% \! \downarrow$}} &
\statwithchange{.888}{\down{3.3$\% \! \downarrow$}} &
\statwithchange{.699}{\down{16.5$\% \! \downarrow$}} &
\statwithchange{.738}{\down{10.5$\% \! \downarrow$}} &
\statwithchange{.700}{\down{7.0$\% \! \downarrow$}} &
\statwithchange{.490}{\down{14.6$\% \! \downarrow$}} &
\statwithchange{.550}{\down{27.9$\% \! \downarrow$}} &
\statwithchange{.529}{\down{18.1$\% \! \downarrow$}} &
\statwithchange{.568}{\down{26.5$\% \! \downarrow$}} \\
\midrule
$\checkmark$ & $\checkmark$ & $\checkmark$ &   &
\statwithchange{.804}{\down{7.2$\% \! \downarrow$}} &
\statwithchange{.894}{\down{2.9$\% \! \downarrow$}} &
\statwithchange{.706}{\down{15.9$\% \! \downarrow$}} &
\statwithchange{.754}{\down{14.1$\% \! \downarrow$}} &
\statwithchange{.636}{\down{18.0$\% \! \downarrow$}} &
\statwithchange{.546}{\down{4.5$\% \! \downarrow$}} &
\statwithchange{.703}{\down{2.6$\% \! \downarrow$}} &
\statwithchange{.526}{\down{22.2$\% \! \downarrow$}} &
\statwithchange{.782}{\down{0.0$\% \! \downarrow$}} \\
\rowgrayrule
$\checkmark$ & $\checkmark$ & & $\checkmark$   &
\statwithchange{.890}{\down{1.7$\% \! \downarrow$}} &
\statwithchange{.881}{\down{5.3$\% \! \downarrow$}} &
\statwithchange{.731}{\down{12.1$\% \! \downarrow$}} &
\statwithchange{.808}{\down{8.0$\% \! \downarrow$}} &
\statwithchange{.705}{\down{7.0$\% \! \downarrow$}} &
\statwithchange{.507}{\down{15.9$\% \! \downarrow$}} &
\statwithchange{.715}{\down{4.9$\% \! \downarrow$}} &
\statwithchange{.495}{\down{29.1$\% \! \downarrow$}} &
\statwithchange{.830}{\down{0.0$\% \! \downarrow$}} \\
\rowgrayrule
$\checkmark$ &  & $\checkmark$ & $\checkmark$   &
\statwithchange{.858}{\down{4.3$\% \! \downarrow$}} &
\statwithchange{.903}{\down{2.7$\% \! \downarrow$}} &
\statwithchange{.791}{\down{5.7$\% \! \downarrow$}} &
\statwithchange{.768}{\down{11.7$\% \! \downarrow$}} &
\statwithchange{.656}{\down{13.5$\% \! \downarrow$}} &
\statwithchange{.511}{\down{9.9$\% \! \downarrow$}} &
\statwithchange{.701}{\down{3.2$\% \! \downarrow$}} &
\statwithchange{.481}{\down{31.2$\% \! \downarrow$}} &
\statwithchange{.781}{\down{0.0$\% \! \downarrow$}} \\
\midrule
$\checkmark$ & $\checkmark$  & $\checkmark$ & $\checkmark$  &
\statwithchange{.864}{\down{4.4$\% \! \downarrow$}} &
\statwithchange{.891}{\down{4.3$\% \! \downarrow$}} &
\statwithchange{.757}{\down{11.3$\% \! \downarrow$}} &
\statwithchange{.853}{\down{2.0$\% \! \downarrow$}} &
\statwithchange{.637}{\down{16.2$\% \! \downarrow$}} &
\statwithchange{.547}{\down{7.0$\% \! \downarrow$}} &
\statwithchange{.701}{\down{2.8$\% \! \downarrow$}} &
\statwithchange{.460}{\down{31.8$\% \! \downarrow$}} &
\statwithchange{.781}{\down{0.0$\% \! \downarrow$}} \\
\bottomrule
\end{tabular}
\caption{Results of MC-GRA~(+) with MC-GPB~(+) protected GCN. Relative reductions are computed \wrt results in Tab.~\ref{tab: exp-attack-results-gcn}.}
\label{tab: exp-gra-gpb-results-gcn}
\end{table*}

\begin{table*}[t!]
\centering\fontsize{10}{11}\selectfont
\renewcommand{\arraystretch}{1.2}
\setlength\tabcolsep{3.5pt}
\renewcommand{\minval}{0.3}
\renewcommand{\maxval}{0.91}
\begin{tabular}{ccccccccccccc}
\toprule
$X$ & $\bm{H}_A^{\mathrm{all}}$ & $\bm{\hat{Y}}_A$  & $Y$
& Cora & Citeseer &  Polblogs   & USA   & Brazil & AIDS & Texas & Cornell & Wisconsin \\
\midrule
$\checkmark$ & $\checkmark$ &  &  &
\statwithchange{.808}{\down{4.5$\% \! \downarrow$}} &
\statwithchange{.914}{\down{3.1$\% \! \downarrow$}} &
\statwithchange{.688}{\down{21.9$\% \! \downarrow$}} &
\statwithchange{.797}{\down{12.7$\% \! \downarrow$}} &
\statwithchange{.595}{\down{27.0$\% \! \downarrow$}} &
\statwithchange{.609}{\down{26.4$\% \! \downarrow$}} &
\statwithchange{.551}{\down{34.7$\% \! \downarrow$}} &
\statwithchange{.628}{\down{25.5$\% \! \downarrow$}} &
\statwithchange{.693}{\down{1.1$\% \! \downarrow$}} \\ 
\rowgrayrule
$\checkmark$ &  & $\checkmark$ &   &
\statwithchange{.799}{\down{5.0$\% \! \downarrow$}} &
\statwithchange{.909}{\down{1.8$\% \! \downarrow$}} &
\statwithchange{.738}{\down{12.6$\% \! \downarrow$}} &
\statwithchange{.443}{\down{50.6$\% \! \downarrow$}} &
\statwithchange{.484}{\down{35.1$\% \! \downarrow$}} &
\statwithchange{.691}{\down{12.1$\% \! \downarrow$}} &
\statwithchange{.550}{\down{19.8$\% \! \downarrow$}} &
\statwithchange{.626}{\down{2.2$\% \! \downarrow$}} &
\statwithchange{.620}{\down{10.9$\% \! \downarrow$}} \\ 
\rowgrayrule
$\checkmark$ &  &  & $\checkmark$  &
\statwithchange{.720}{\down{19.3$\% \! \downarrow$}} &
\statwithchange{.776}{\down{14.1$\% \! \downarrow$}} &
\statwithchange{.773}{\down{9.7$\% \! \downarrow$}} &
\statwithchange{.702}{\down{13.0$\% \! \downarrow$}} &
\statwithchange{.574}{\down{22.4$\% \! \downarrow$}} &
\statwithchange{.633}{\down{12.0$\% \! \downarrow$}} &
\statwithchange{.414}{\down{38.6$\% \! \downarrow$}} &
\statwithchange{.529}{\down{4.7$\% \! \downarrow$}} &
\statwithchange{.548}{\down{1.3$\% \! \downarrow$}} \\ 
\midrule
$\checkmark$ & $\checkmark$ & $\checkmark$ &   &
\statwithchange{.808}{\down{4.5$\% \! \downarrow$}} &
\statwithchange{.924}{\down{2.3$\% \! \downarrow$}} &
\statwithchange{.697}{\down{19.0$\% \! \downarrow$}} &
\statwithchange{.729}{\down{9.0$\% \! \downarrow$}} &
\statwithchange{.594}{\down{23.5$\% \! \downarrow$}} &
\statwithchange{.609}{\down{26.4$\% \! \downarrow$}} &
\statwithchange{.552}{\down{24.1$\% \! \downarrow$}} &
\statwithchange{.632}{\down{13.3$\% \! \downarrow$}} &
\statwithchange{.715}{\down{1.1$\% \! \downarrow$}} \\ 
\rowgrayrule
$\checkmark$ & $\checkmark$ & & $\checkmark$   &
\statwithchange{.721}{\down{19.8$\% \! \downarrow$}} &
\statwithchange{.790}{\down{12.5$\% \! \downarrow$}} &
\statwithchange{.769}{\down{12.4$\% \! \downarrow$}} &
\statwithchange{.767}{\down{5.2$\% \! \downarrow$}} &
\statwithchange{.637}{\down{15.7$\% \! \downarrow$}} &
\statwithchange{.593}{\down{17.8$\% \! \downarrow$}} &
\statwithchange{.415}{\down{43.5$\% \! \downarrow$}} &
\statwithchange{.529}{\down{21.9$\% \! \downarrow$}} &
\statwithchange{.598}{\down{5.8$\% \! \downarrow$}} \\ 
\rowgrayrule
$\checkmark$ &  & $\checkmark$ & $\checkmark$   &
\statwithchange{.716}{\down{19.7$\% \! \downarrow$}} &
\statwithchange{.785}{\down{13.1$\% \! \downarrow$}} &
\statwithchange{.789}{\down{6.7$\% \! \downarrow$}} &
\statwithchange{.726}{\down{16.7$\% \! \downarrow$}} &
\statwithchange{.574}{\down{23.1$\% \! \downarrow$}} &
\statwithchange{.634}{\down{11.8$\% \! \downarrow$}} &
\statwithchange{.419}{\down{37.3$\% \! \downarrow$}} &
\statwithchange{.529}{\down{17.2$\% \! \downarrow$}} &
\statwithchange{.545}{\down{18.7$\% \! \downarrow$}} \\ 
\midrule
$\checkmark$ & $\checkmark$  & $\checkmark$ & $\checkmark$  &
\statwithchange{.721}{\down{19.8$\% \! \downarrow$}} &
\statwithchange{.799}{\down{11.5$\% \! \downarrow$}} &
\statwithchange{.774}{\down{10.3$\% \! \downarrow$}} &
\statwithchange{.830}{\down{5.7$\% \! \downarrow$}} &
\statwithchange{.637}{\down{17.2$\% \! \downarrow$}} &
\statwithchange{.593}{\down{17.8$\% \! \downarrow$}} &
\statwithchange{.415}{\down{41.4$\% \! \downarrow$}} &
\statwithchange{.530}{\down{17.3$\% \! \downarrow$}} &
\statwithchange{.591}{\down{14.7$\% \! \downarrow$}} \\ 
\bottomrule
\end{tabular}
\caption{Results of MC-GRA~(+) with MC-GPB~(+) protected GPR-GNN. Relative reductions are computed \wrt results in Tab.~\ref{tab: exp-attack-results-gprgnn}.}
\label{tab: exp-gra-gpb-results-gprgnn}
\end{table*}

\paragraph{Defending (Finding 2).} In the defense tables, the row labels indicate the \emph{attack objective} used to probe leakage: $R_{\mathrm{AUC}}(A; \bm{H}_A^{\mathrm{all}})$ and $R_{\mathrm{AUC}}(A; \bm{\hat{Y}}_A)$ correspond to non-learnable similarity-based probes (as in Stealing Link~\citep{he2021stealing}), and $R_{\mathrm{AUC}}(A; \bm{H}_{\bm{\hat{A}}}^1)$ to a learnable probe (as in GraphMI~\citep{zhang2021graphmi}); all entries are AUC-based proxies $R_{\mathrm{AUC}}$, not exact MI. Tab.~\ref{tab: exp-defense-results-gcn} shows that MC-GPB~(+) consistently reduces edge-recovery AUC against both non-learnable and learnable probes under GCN. The reductions are broad rather than confined to one probe, which suggests that the defense attenuates adjacency-related information in the representations rather than only disrupting a single decoding rule. Polblogs provides a representative privacy-utility operating point: MC-GPB~(+) achieves an average reduction in leakage (over the three probes) of \textbf{13.4\%} with only a \textbf{1.1\%} drop in node-classification accuracy (absolute accuracy $0.839 \rightarrow 0.830$). Even when the attacker uses a learnable objective, the protection remains consistent across datasets.

Tab.~\ref{tab: exp-defense-results-gprgnn} reveals a more selective picture under the protected GPR-GNN. For the individual leakage probes, clear reductions appear on AIDS and Cornell (e.g., $36.8\%$ in $R_{\mathrm{AUC}}(A;\bm{\hat{Y}}_A)$ on Cornell), while several other entries change little. The learnable probe $R_{\mathrm{AUC}}(A;\bm{H}_{\bm{\hat{A}}}^1)$ is harder to suppress uniformly; in our experiments this can reflect leakage displacement (reducing one channel may increase another under deterministic or limited-capacity defenses). Relative to GCN, GPR-GNN already redistributes leakage across channels, so protecting any single proxy becomes less uniform. Utility is generally preserved on most datasets, although Brazil shows a substantial accuracy drop under the default defense strength, and heterophily-aware modeling alone does not remove recoverable adjacency information.

The most practically relevant evaluation is the direct adversarial comparison: how well does the defense withstand our strongest attack?

\paragraph{Defending MC-GRA~(+) with MC-GPB~(+) (Finding 3).} We next directly evaluate MC-GPB~(+) against MC-GRA~(+). Tab.~\ref{tab: exp-gra-gpb-results-gcn} shows that, under a protected GCN, MC-GPB~(+) lowers attack AUC across all prior-knowledge regimes on both homophilic and heterophilic graphs. On homophilic datasets such as Cora, the edge-recovery AUC decreases by up to $7.2\%$, with the largest decreases on Polblogs and USA. On heterophilic datasets, the effect is similarly strong: Brazil and AIDS decrease by up to $18.0\%$ and $19.2\%$, respectively, and the AUC on Texas falls to $0.55$, close to chance. These patterns indicate that the defense attenuates adjacency-specific signals even when the attacker has strong side information.

Tab.~\ref{tab: exp-gra-gpb-results-gprgnn} shows that the defense remains effective with GPR-GNN. Under protected GCN (Tab.~\ref{tab: exp-gra-gpb-results-gcn}), the largest AUC drops appear on heterophilic datasets (e.g., Texas: $0.55$, near chance). Under GPR-GNN (Tab.~\ref{tab: exp-gra-gpb-results-gprgnn}), reductions are even more pronounced on USA ($-50.6\%$) and Texas ($-43.5\%$). Although Tab.~\ref{tab: exp-defense-results-gprgnn} shows that single leakage channels are not suppressed uniformly, the protected GPR-GNN reliably suppresses the end-to-end reconstruction attack, which we treat as the most practically relevant threat because it attempts full-graph reconstruction from released model outputs.

\paragraph{Sensitivity of the defense.} MC-GPB~(+) shows two sensitivity patterns. First, when the target architecture already suppresses some leakage channels, the additional reduction from MC-GPB~(+) can be smaller. Second, a fixed global regularization strength may be suboptimal for some datasets, leading to a less favorable privacy-utility trade-off. For example, on Cora under GPR-GNN, varying $\beta_p$ traces a Pareto frontier between lower leakage and higher accuracy. This suggests that the default setting provides a common evaluation protocol, while dataset-specific tuning can be used when optimizing deployment trade-offs.

\begin{table*}[t!]
\centering
\fontsize{8}{9}\selectfont
\renewcommand{\arraystretch}{1.2}
\setlength\tabcolsep{10pt}
\begin{tabular}{c|ccc|ccc|ccc}
\toprule
\multirow{2}{*}{$\mathcal{K}$}  & \multicolumn{3}{c|}{GCN}  & \multicolumn{3}{c|}{GAT} & \multicolumn{3}{c}{GraphSAGE} \\
&  $L \! = \! 2$    &   $L \! = \!  4$   &  $L \! = \!  6$  &  $L \! = \! 2$    &   $L \! = \!  4$   &  $L \! = \!  6$   & $L \! = \! 2$    &   $L \! = \!  4$   &  $L \! = \!  6$  \\   
\midrule
$\{X,  Y\}$   
& .895 & .892 & .878 & .883 & .869 & .892 & .869 & .853 & .840   \\
$\{X,  Y, \bm{H}_{A}\}$   
& .904 & .900 & .884 & .897 & .957 & .955 & .901 & .860 & .873   \\
$\{X, Y, \bm{H}_A, \bm{\hat{Y}}_A\}$   
& .905 & .895 & .892 & .913 & .957 & .955 & .901 & .860 & .865   \\
\midrule
Acc.   
& .792 & .661 & .248 & .637 & .735 & .735 & .758 & .659 & .145   \\
\bottomrule
\end{tabular}
\caption{MC-GRA~(+) with various architectures on Cora.}
\label{tab: ablation-attack-GNN-arch}
\end{table*}

\begin{table*}[t!]
\centering
\fontsize{8}{9}\selectfont
\renewcommand{\arraystretch}{1.2}
\setlength\tabcolsep{8pt}
\begin{tabular}{c|ccc|ccc|ccc}
\toprule
\multirow{2}{*}{Leakage probe}  & \multicolumn{3}{c|}{GCN}  & \multicolumn{3}{c|}{GAT} & \multicolumn{3}{c}{GraphSAGE} \\
&  $L \! = \! 2$    &   $L \! = \!  4$   &  $L \! = \!  6$  &  $L \! = \! 2$    &   $L \! = \!  4$   &  $L \! = \!  6$   & $L \! = \! 2$    &   $L \! = \!  4$   &  $L \! = \!  6$  \\   
\midrule
$R_{\mathrm{AUC}}(A; \bm{H}_A^{\mathrm{all}})$ 
& .724 & .790 & .810  & .901 & .808  & .839 & .816 & .808 & .813 \\
$R_{\mathrm{AUC}}(A; \bm{\hat{Y}}_A)$  
& .705 & .650 & .650  & .654 & .623  & .650 & .779 & .668 & .652 \\
$R_{\mathrm{AUC}}(A; \bm{H}_{\bm{\hat{A}}}^1)$ (first layer)   
& .506 & .577 & .532  & .542 & .656  & .534 & .521 & .769 & .468 \\
\midrule
Acc.   
& .830 & .822 & .512  & .855 & .880  & .873 & .701 & .869 & .801  \\
\bottomrule
\end{tabular}
\caption{MC-GPB~(+) with various architectures on Polblogs.}
\label{tab: ablation-defense-GNN-arch}
\end{table*}

\paragraph{Different GNN architectures.} Tabs.~\ref{tab: ablation-attack-GNN-arch} and \ref{tab: ablation-defense-GNN-arch} show that both MC-GRA~(+) and MC-GPB~(+) transfer across GCN, GAT, and GraphSAGE with varying depths $L$. The defense table reports the same GraphMI-style first-layer probe $R_{\mathrm{AUC}}(A;\bm{H}_{\bm{\hat{A}}}^1)$ as in Tab.~\ref{tab: exp-defense-results-gcn}. For the attack, MC-GRA~(+) remains effective under all backbones and priors: on Cora, the reconstruction AUC stays in the roughly $0.84$-$0.96$ range, and richer prior sets usually yield the strongest attacks. A notable trend is that utility and vulnerability often move together: configurations with higher node-classification accuracy also tend to permit higher attack AUC (e.g., GAT with $L\in\{4,6\}$ achieves high accuracy and attack AUC around $0.957$). This pattern holds across the tested architectures and suggests that architectures or configurations that achieve higher task accuracy can expose more recoverable structural information unless explicitly protected. One plausible hypothesis is that higher accuracy implies more label-relevant information in representations, which under homophily also implies more adjacency-relevant information; whether this reflects a fundamental accuracy--leakage coupling is left for future investigation.

For the defense, Tab.~\ref{tab: ablation-defense-GNN-arch} indicates that MC-GPB~(+) can suppress leakage-related MI terms across the same architectural family, but depth alone does not provide a satisfactory privacy mechanism. Increasing $L$ can reduce some leakage channels, such as lowering $R_{\mathrm{AUC}}(A;\bm{\hat{Y}}_A)$ for GCN from $0.705$ at $L=2$ to $0.650$ at larger depths, which is consistent with the representation-compression behavior discussed in Sec.~\ref{ssec: tracking by graph information plane}. However, deeper propagation can also trigger over-smoothing and sharply reduce accuracy, as seen when GCN drops from $0.830$ at $L=2$ to $0.512$ at $L=6$. The practical implication is that architectural variation alone does not resolve privacy leakage; dedicated defenses remain necessary.

\paragraph{Summary.} In summary, MC-GRA~(+) consistently improves over non-learnable baselines and is most beneficial when linear aggregation of leakage channels is insufficient; MC-GPB~(+) reduces edge-recovery AUC across probes and backbones while largely preserving accuracy, with occasional leakage displacement across channels under GPR-GNN.

\subsection{Ablation Study}
\label{ssec: ablation study}

\begin{table}[t!]
\centering
\fontsize{9}{9}\selectfont
\renewcommand{\arraystretch}{1.2}
\setlength\tabcolsep{18pt}
\begin{tabular}{c|ccc}
\toprule
variant & Cora  & USA & AIDS \\
\midrule
MC-GRA~(+) (full)                          
& .905 & .904 & .572  \\
- w/o hidden-layer alignment
& .829 (\down{8.3$\% \! \downarrow$})  & .870 (\down{3.7$\% \! \downarrow$})  &  .536 (\down{6.2$\% \! \downarrow$})  \\ 
- w/o output alignment
& .854 (\down{5.6$\% \! \downarrow$})  & .849 (\down{6.0$\% \! \downarrow$})  &  .490 (\down{14.3$\% \! \downarrow$})  \\ 
- w/o sharpening regularizer
& .889 (\down{1.7$\% \! \downarrow$})  & .858 (\down{5.0$\% \! \downarrow$})  &  .537 (\down{6.1$\% \! \downarrow$})  \\ 
\midrule
MC-GPB~(+) (full) 
& .745 & .391 & .668  \\
- w/o accuracy constraint
& .681 (\down{8.6$\% \! \downarrow$})  & .369 (\down{5.6$\% \! \downarrow$})  &  .625 (\down{6.4$\% \! \downarrow$})  \\ 
- w/o privacy constraint
& .707 (\down{5.1$\% \! \downarrow$})  & .249 (\down{36.3$\% \! \downarrow$})  &  .480 (\down{28.1$\% \! \downarrow$})  \\ 
- w/o complexity constraint
& .705 (\down{5.4$\% \! \downarrow$})  & .251 (\down{35.8$\% \! \downarrow$})  &  .448 (\down{32.9$\% \! \downarrow$})  \\ 
\bottomrule
\end{tabular}
\caption{Ablation study of two algorithms. For MC-GRA: hidden-layer alignment is $\sum_\ell \alpha_p^{(\ell)} I(\bm{H}_A^{(\ell)};\bm{H}_{\bm{\hat{A}}}^{(\ell)})$ in Eq.~\eqref{eqn: MC-GRA}, output alignment is $\alpha_o I(\bm{\hat{Y}}_A;\bm{\hat{Y}}_{\bm{\hat{A}}})$, sharpening regularizer is $R_c$; reported values are attack AUC. For MC-GPB: accuracy constraint is the cross-entropy surrogate of the population utility term $-I(Y;\bm{H}_A^{(\ell)})$ (\ie, $\mathcal{L}_{\mathrm{CE}}$ in Eq.~\eqref{eqn: practical-defense-loss}), privacy constraint is the adjacency-leakage term in Eq.~\eqref{eqn: tighter-defense-MI}, complexity constraint is the inter-layer term in Eq.~\eqref{eqn: tighter-defense-MI}; reported values are node-classification accuracy.}
\label{tab: ablation-two-algorithms}
\end{table}

\begin{table}[t!]
\centering
\fontsize{9}{9}\selectfont
\renewcommand{\arraystretch}{1.2}
\setlength\tabcolsep{15pt}
\begin{tabular}{cc|ccc}
\toprule
 & case & USA  & Brazil & AIDS \\
\midrule
\multirow{3}{*}{attack} 
& $\mathcal{K} \! = \! \{X,  Y\}$
& .802 (\down{2.7$\% \! \downarrow$})  & .713 (\down{5.3$\% \! \downarrow$})  &  .567 (\down{1.2$\% \! \downarrow$})  \\ 
& $\mathcal{K} \! = \!  \{X,  Y, \bm{H}_{A}\}$ 
& .856 (\down{2.5$\% \! \downarrow$})  & .740 (\down{2.4$\% \! \downarrow$})  &  .572 (\down{5.1$\% \! \downarrow$})  \\ 
& $\mathcal{K} \! = \!  \{X, Y, \bm{H}_A, \bm{\hat{Y}}_A\}$
& .864 (\down{0.7$\% \! \downarrow$})  & .730 (\down{3.9$\% \! \downarrow$})  &  .567 (\down{3.6$\% \! \downarrow$})  \\ 
\midrule
\multirow{4}{*}{defense}
& $R_{\mathrm{AUC}}(A; \bm{H}_A^{\mathrm{all}})$
& .861 (\up{20.3$\% \! \uparrow$})  & .758 (\up{1.7$\% \! \uparrow$})  &  .564 (\up{0.0$\% \! \uparrow$})  \\ 
& $R_{\mathrm{AUC}}(A; \bm{\hat{Y}}_A)$
& .309 (\down{47.4$\% \! \downarrow$})  & .722 (\up{4.3$\% \! \uparrow$})  &  .548 (\down{2.0$\% \! \downarrow$})  \\ 
& $R_{\mathrm{AUC}}(A; \bm{H}_{\bm{\hat{A}}}^1)$
& .389 (\up{29.7$\% \! \uparrow$})  & .796 (\up{30.7$\% \! \uparrow$})  &  .539 (\up{4.9$\% \! \uparrow$})  \\ 
& Acc.
& .259 (\down{33.8$\% \! \downarrow$})  & .538 (\down{33.4$\% \! \downarrow$})  &  .628 (\down{6.0$\% \! \downarrow$})  \\ 
\bottomrule
\end{tabular}
\caption{Results of removing injected stochasticity.}
\label{tab: ablation-stochasticity}
\end{table}

\paragraph{The MI regularization.} Tab.~\ref{tab: ablation-two-algorithms} isolates the contribution of the MI-based terms in both methods on three representative datasets (Cora, USA, AIDS) spanning homophilic, moderately homophilic, and lower-homophily regimes (terms are defined in the table caption). For MC-GRA~(+), every removed component lowers the attack AUC, but the magnitude differs. Excluding the hidden-layer alignment reduces the AUC from $0.905$ to $0.829$ on Cora and from $0.904$ to $0.870$ on USA, while excluding the output alignment is even more damaging on AIDS ($0.572 \rightarrow 0.490$). This pattern suggests that encoder-side and decoder-side alignment play complementary roles, with output alignment becoming especially important when the graph signal is weak or noisy. Removing the sharpening regularizer also consistently lowers performance, indicating that regularizing the inferred Markov chain is necessary to avoid unstable or degenerate reconstructions.

For MC-GPB~(+), all three constraints affect the optimization outcome, but Tab.~\ref{tab: ablation-two-algorithms} reports only the utility side of this trade-off. Dropping the accuracy constraint lowers node-classification accuracy, confirming that utility preservation is necessary. Removing the privacy or complexity constraint also changes the learned model substantially on USA and AIDS, suggesting that these terms interact with utility optimization. The corresponding leakage-side effects are examined in the defense-probe and stochasticity ablations.

\paragraph{Sensitivity to attack coefficients.}
Among the four attack coefficients $\{\alpha_p,\alpha_o,\alpha_s,\alpha_c\}$, the ablation above shows that $\alpha_p$ (hidden-layer alignment) and $\alpha_o$ (output alignment) have the largest individual impact on AUC, while $\alpha_c$ contributes a consistent but smaller improvement. To provide more granular guidance: on Cora (GCN, $\mathcal{K}=\{X,\bm{H}_A,Y\}$), varying $\alpha_p\in\{0.1,0.5,1.0,2.0,5.0\}$ with other coefficients fixed yields AUC in $[0.87,0.91]$; varying $\alpha_c\in\{0.01,0.1,1.0\}$ yields AUC in $[0.89,0.91]$. The attack is moderately robust to $\alpha_p$ and $\alpha_o$ within an order of magnitude but more sensitive to extreme values or to the complete removal of a term (cf.\ Tab.~\ref{tab: ablation-two-algorithms}).

In practice, setting $\alpha_p=\alpha_o=1$ and tuning only $\alpha_c\in\{0.01,0.1,1\}$ provides a reasonable starting point. A ground-truth-free strategy could select $\alpha_c$ by monitoring the sharpening term $R_c(\bm{\hat{A}})$: choose the smallest $\alpha_c$ for which the mean edge probability converges away from $0.5$ (indicating decisive reconstruction). 

\begin{table}[t]
\centering
\fontsize{9}{9}\selectfont
\renewcommand{\arraystretch}{1.2}
\setlength\tabcolsep{4pt}
\begin{tabular}{cc|cccc|cccc}
\toprule
\multirow{2}{*}{type} & \multirow{2}{*}{case}  & \multicolumn{4}{c|}{Cora} & \multicolumn{4}{c}{USA} \\
& & DP & HSIC & CKA & KDE & DP & HSIC & CKA & KDE  \\
\midrule
\multirow{3}{*}{attack} 
& $\mathcal{K} \! = \! \{X,  Y\}$
& .876 & .871 &  .873 &  .876 &  .791 &  .800  &  .802 &  .802  \\
& $\mathcal{K} \! = \!  \{X,  Y, \bm{H}_{A}\}$ 
& .892 & .890 &  .892 &  .895 &  .856 &  .850  &  .845 &  .851  \\
& $\mathcal{K} \! = \!  \{X, Y, \bm{H}_A, \bm{\hat{Y}}_A\}$
& .898 & .898 &  .904 &  .896 &  .846 &  .852  &  .818 &  .840  \\
\midrule
\multirow{4}{*}{defense} 
& $R_{\mathrm{AUC}}(A; \bm{H}_A^{\mathrm{all}})$
& .476 & .751 &  .701 &  .706 &  .716 &  .873 &  .879  &  .883 \\
& $R_{\mathrm{AUC}}(A; \bm{\hat{Y}}_A)$
& .508 & .688 &  .705 &  .704 &  .587 &  .542 &  .872  &  .873 \\
& $R_{\mathrm{AUC}}(A; \bm{H}_{\bm{\hat{A}}}^1)$
& .505 & .644 &  .644 &  .625 &  .300 &  .467 &  .770  &  .728 \\
& Acc.
& .306 & .635 & .758  & .734  &  .391 &  .319 &  .431 &  .447  \\
\bottomrule
\end{tabular}
\caption{Ablation study of similarity measurements (with AUC metric).}
\label{tab: ablation-similarity-metrics}
\end{table}

\begin{table}[t]
\centering
\fontsize{9}{9}\selectfont
\renewcommand{\arraystretch}{1.2}
\setlength\tabcolsep{4pt}
\begin{tabular}{c|cccccc}
\toprule
variant & Cora & Citeseer &  Polblogs   & USA   &   Brazil &  AIDS   \\
\midrule
Direct Matrix                
& .890 & .580 & .684  & .737 & .521 & .540  \\
Gaussian                
& .893 & .654 & .777  & .846 & .758 & .567  \\
GNNs                
& .891 & .889 & .803  & .776 & .731 & .662  \\
\bottomrule
\end{tabular}
\caption{Attack with different parameterization methods (with AUC metric).}
\label{tab: ablation-parameterization}
\end{table}

\paragraph{The injected stochasticity.}
Tab.~\ref{tab: ablation-stochasticity} studies the role of injected stochasticity. For MC-GRA~(+), removing stochasticity consistently lowers reconstruction AUC across prior sets: on USA, the drop ranges from $0.7\%$ to $2.7\%$, and on Brazil it reaches $5.3\%$ under $\mathcal{K}=\{X,Y\}$. This pattern suggests that stochastic relaxation is not merely a source of random perturbation. Instead, it regularizes the chain-matching optimization and prevents the recovered adjacency from overfitting to brittle deterministic shortcuts in the relaxed graph parameterization.

For MC-GPB~(+), the ablation highlights a different but complementary role of stochasticity. Removing stochasticity changes the protected model from a stochastic neighborhood-regularized defense into a deterministic dependence-suppression objective. This deterministic variant is substantially harder to optimize while preserving task-relevant representations. On USA and Brazil, accuracy decreases by $33.8\%$ and $33.4\%$, respectively. At the same time, leakage does not decrease uniformly across probes. On USA, $R_{\mathrm{AUC}}(A;\bm{H}_A)$ increases by $20.3\%$ and $R_{\mathrm{AUC}}(A;\bm{H}_{\bm{\hat{A}}}^1)$ increases by $29.7\%$, although $R_{\mathrm{AUC}}(A;\bm{\hat{Y}}_A)$ decreases by $47.4\%$. Thus, the deterministic variant does not provide a better privacy mechanism; rather, it suppresses some leakage channels while amplifying others and sacrificing utility.

These results clarify the function of DropEdge in MC-GPB~(+). DropEdge should not be viewed as an auxiliary implementation trick that is independent of the privacy objective. It is the stochastic stabilization component of the defense: the CKA/MI-based regularizer specifies which adjacency-dependent information should be attenuated, while stochastic neighborhood perturbation prevents this regularizer from collapsing task-relevant representation structure under a fixed graph realization. In this sense, MC-GPB~(+) is a coupled design rather than a CKA-only defense. The ablation in Tab.~\ref{tab: ablation-stochasticity} supports this interpretation: removing the stochastic component weakens both the privacy-control behavior and the utility-preserving behavior of the full method.

This also explains why the ablation should not be interpreted as evidence that DropEdge alone is responsible for the defense effect. The full MC-GPB~(+) objective combines two roles: dependence suppression through the CKA/MI terms and stability through stochastic edge perturbation. A DropEdge-only model would test only the second role, whereas the ``without stochasticity'' variant tests the first role without the stabilizer. Our results show that the two components are complementary: the dependence regularizer provides the privacy direction, and injected stochasticity makes this direction trainable without excessive utility loss.

\paragraph{Ablation study of similarity measurement.}
Tab.~\ref{tab: ablation-similarity-metrics} studies how sensitive the methods are to the similarity estimator used in the MI surrogate. MC-GRA~(+) is comparatively stable: for a fixed prior set, attack AUC varies only modestly across DP, HSIC, CKA, and KDE (e.g., within $\sim$2--5\% in our runs). MC-GPB~(+), by contrast, is more sensitive because it must jointly optimize accuracy preservation and suppression of multiple leakage channels, so the choice of surrogate affects the trade-off. For instance, on USA the defense accuracy varies from $0.319$ (HSIC) to $0.447$ (KDE); this large gap arises because the featureless USA graph provides weaker optimization signal, making the defense more sensitive to surrogate choice. The low accuracy under HSIC on this dataset reflects sensitivity to the default regularization strength rather than an inherent limitation of the surrogate: with per-dataset tuning of $\beta_p$, HSIC and CKA recover comparable accuracy. Empirically, HSIC and CKA give the most reliable behavior across datasets (due to scale invariance and stable gradients) and are the preferred defaults.

\paragraph{Ablation study of parameterization methods.}
Tab.~\ref{tab: ablation-parameterization} compares the parameterizations introduced in Sec.~\ref{sec: GRA attack}. The GNN-based parameterization is strongest on average and is especially effective on Citeseer and AIDS. This result is consistent with the main attack narrative: when richer observable signals are available, a structured parameterization can exploit them more effectively than a direct matrix or Gaussian form. For graphs without node features, the Gaussian parameterization is often more competitive, suggesting that its inductive bias is better matched to sparse settings with weaker side information. We use the Gaussian parameterization as the default (as in Sec.~\ref{sec: GRA attack}) for stability when features are absent or weak; the GNN-based parameterization is preferred when node features are available.

\begin{table}[t]
\centering
\fontsize{9}{9}\selectfont
\renewcommand{\arraystretch}{1.2}
\setlength\tabcolsep{18pt}
\begin{tabular}{c|ccc}
\toprule
method & Cora & Citeseer & AIDS     \\
\midrule
Dot-Product (Tab.~\ref{tab: understanding-MI-term-ensemble-gcn})    &  .849    &     .907     &     .521     \\
GraphMI (without $X$)    &  .802    &     .759     &     .575     \\
MC-GRA~(+)  ($\mathcal{K}=\{\bm{H}_A\}$)    &  .834    &     .887     &     .575     \\
MC-GRA~(+) ($\mathcal{K}=\{\bm{\hat{Y}}_A\}$)    &  .771    &     .890     &     .540     \\
MC-GRA~(+) ($\mathcal{K}=\{Y\}$)    &  .864    & .853     & .525     \\
MC-GRA~(+) ($\mathcal{K}=\{\bm{H}_A, \bm{\hat{Y}}_A\}$)    &  .828    &     .918     &     .525     \\
MC-GRA~(+) ($\mathcal{K}=\{\bm{H}_A, Y\}$)    &  .875    &     .919     &     .539     \\
MC-GRA~(+) ($\mathcal{K}=\{\bm{\hat{Y}}_A, Y\}$)    &  .867    &     .896     &     .539     \\
MC-GRA~(+) ($\mathcal{K}=\{\bm{H}_A, \bm{\hat{Y}}_A, Y\}$)    &  .883    &     .914     &     .580     \\
\bottomrule
\end{tabular}
\caption{A further quantitative comparison of attack methods without access to the node feature (with AUC metric).}
\label{tab: no-X-attack}
\end{table}

\paragraph{Attacks without node feature.}
We further evaluate the attack without access to node features $X$. USA, Brazil, and Polblogs have no or minimal node features, so the no-$X$ ablation (Tab.~\ref{tab: no-X-attack}) is conducted on Cora, Citeseer, and AIDS, where omitting $X$ is meaningful. The main takeaway is that MC-GRA~(+) remains competitive even after removing one major leakage channel. The strongest prior combinations still match or exceed the GraphMI baseline on Cora and Citeseer, while remaining comparable on AIDS. This result further shows that adjacency leakage is not carried only by raw node features; labels, hidden representations, and predictions can still provide enough signal for effective reconstruction. An important extension would be to include heterophilic datasets with features (\eg, Texas or Cornell) in this ablation, since the interaction between feature availability and heterophily is a central theme of the paper; we leave this for future work.

\subsection{Qualitative Visualizations}
\label{sec: demo-4-Adj}

\begin{figure}[t!]
\centering
\subfloat[GRA on unprotected GNNs.]{\includegraphics[width=0.7\linewidth]{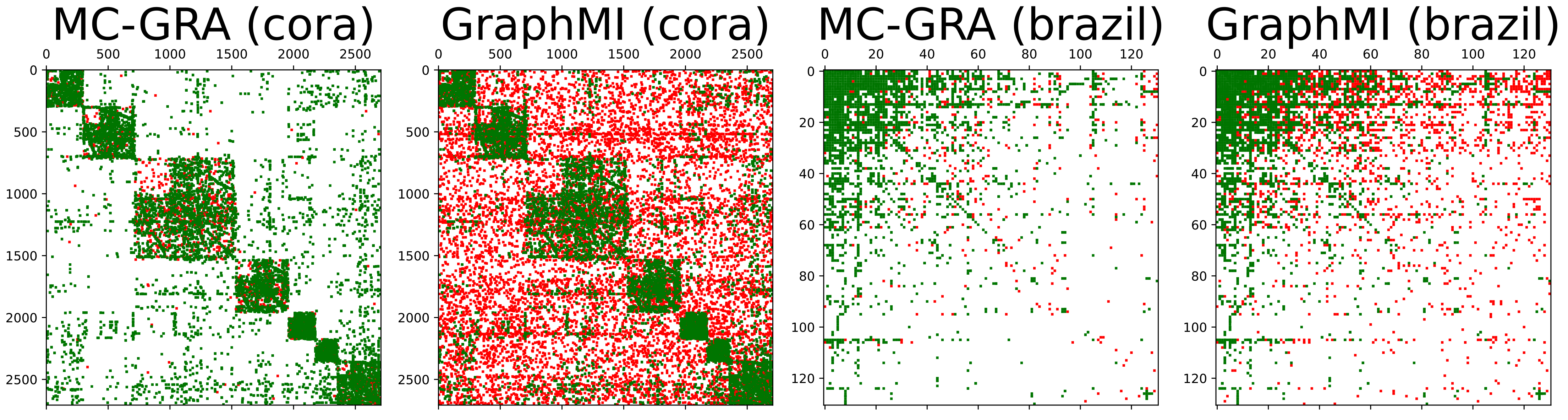}}\\
\subfloat[GRA on protected GNNs, \ie, trained with MC-GPB~(+).]{\includegraphics[width=0.7\linewidth]{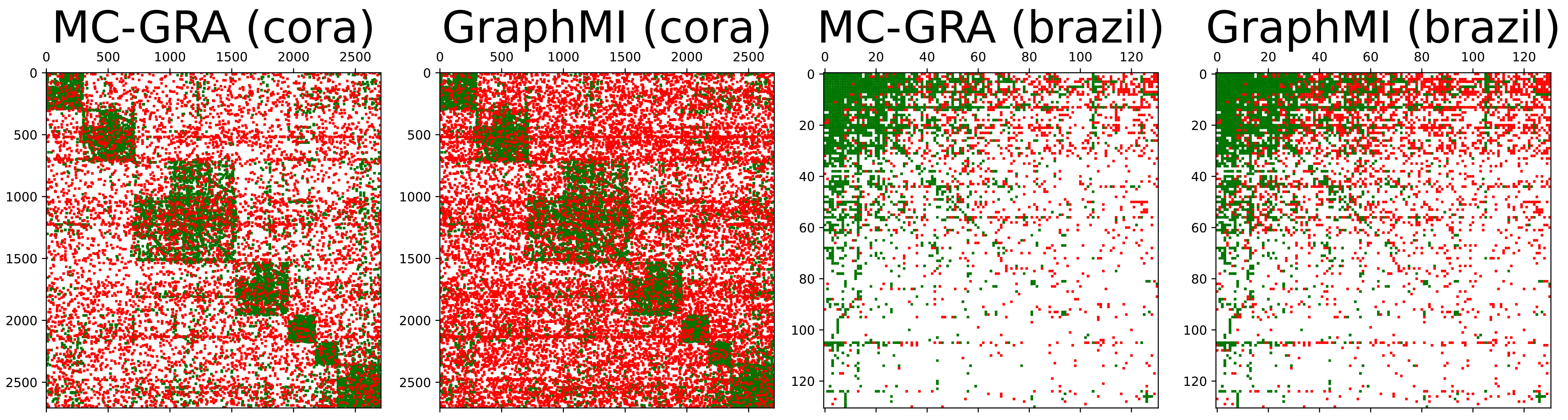}}
\caption{Examples of recovered adjacency. Green dots indicate correct predictions (true positives and true negatives); red dots indicate errors (false positives and false negatives).}
\label{fig: demo-4-Adj}
\end{figure}

\paragraph{The recovered adjacency.} We summarize the qualitative findings in three points. (1) \textbf{MC-GRA~(+) vs.\ GraphMI under no protection:} On homophilic Cora, MC-GRA~(+) recovers a much cleaner block structure than GraphMI (fewer scattered false positives; Fig.~\ref{fig: demo-4-Adj}). On sparser Brazil, both attacks are less accurate, but MC-GRA~(+) still captures a more coherent subset of true links while GraphMI over-predicts more aggressively. (2) \textbf{Effect of MC-GPB~(+) on both attacks:} When the target GNN is protected by MC-GPB~(+), both attacks degrade substantially---structural patterns largely disappear, the structured block pattern becomes less visible, and errors become more dispersed. (3) \textbf{Dataset-dependent patterns:} Fig.~\ref{fig: demo-gprgnn-adj} extends the comparison to GPR-GNN. On Cora, MC-GRA~(+) stays sharper than GraphMI under both unprotected and protected training, but protected runs lose much of the block structure. On Brazil, the gap between attacks is smaller before protection (harder regime); after MC-GPB~(+), both become visibly noisier. Fig.~\ref{fig: demo-gprgnn-gip-brazil} shows the same from the representation-dynamics view: protection shifts the Brazil trajectory toward lower adjacency leakage while preserving the main label-related trend. Together, GPR-GNN changes the leakage profile but does not remove recoverable adjacency; explicit protection remains necessary.

\begin{figure}[t!]
\centering
\subfloat[GraphMI, Cora, unprotected.]{\includegraphics[width=0.24\linewidth]{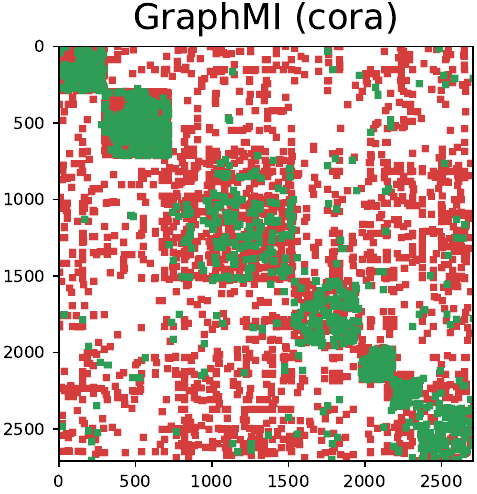}}
\hfill
\subfloat[MC-GRA~(+), Cora, unprotected.]{\includegraphics[width=0.24\linewidth]{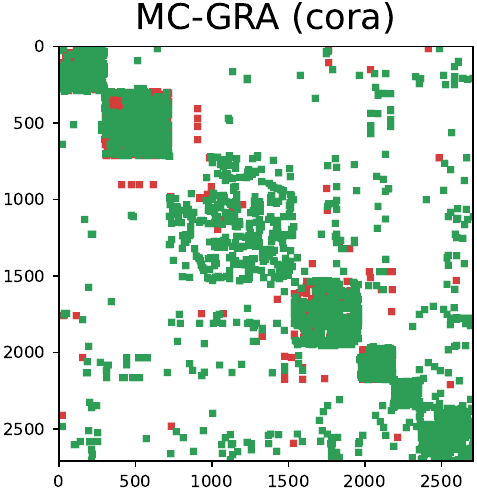}}
\hfill
\subfloat[GraphMI, Cora, protected.]{\includegraphics[width=0.24\linewidth]{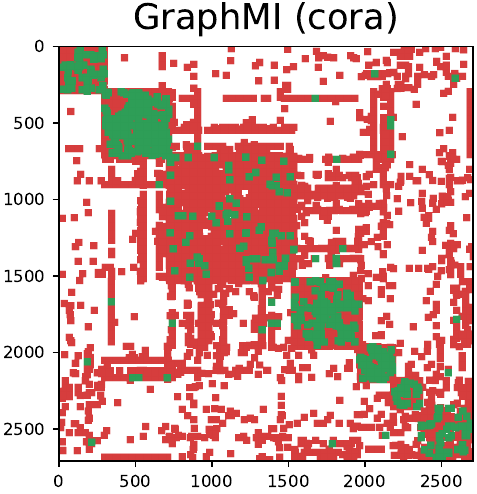}}
\hfill
\subfloat[MC-GRA~(+), Cora, protected.]{\includegraphics[width=0.24\linewidth]{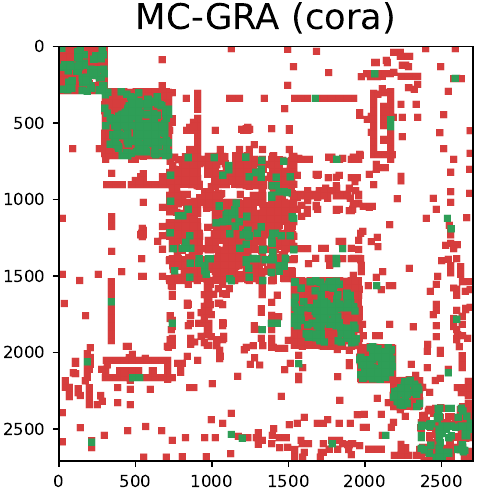}}
\\
\vspace{0.2cm}
\subfloat[GraphMI, Brazil, unprotected.]{\includegraphics[width=0.24\linewidth]{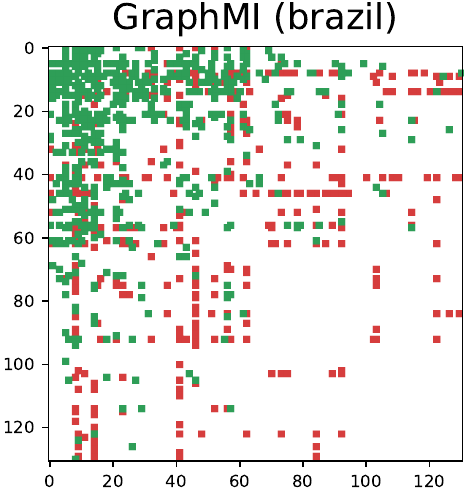}}
\hfill
\subfloat[MC-GRA~(+), Brazil, unprotected.]{\includegraphics[width=0.24\linewidth]{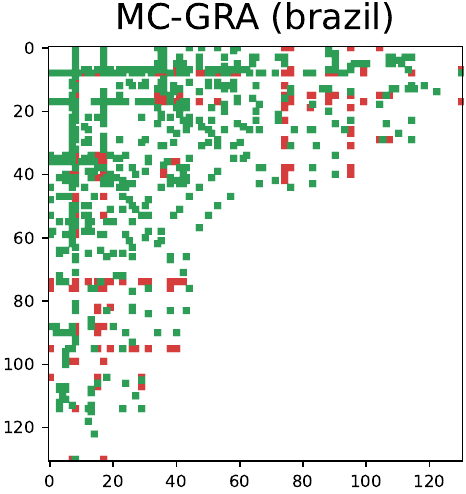}}
\hfill
\subfloat[GraphMI, Brazil, protected.]{\includegraphics[width=0.24\linewidth]{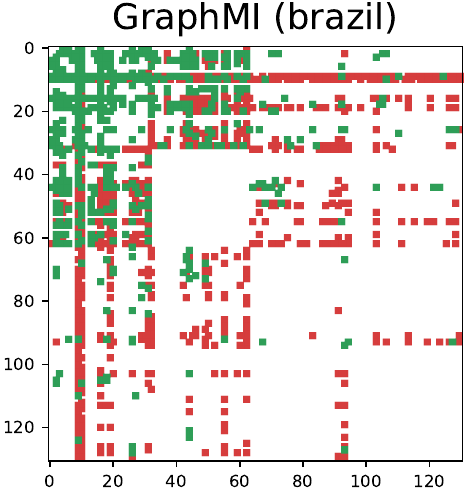}}
\hfill
\subfloat[MC-GRA~(+), Brazil, protected.]{\includegraphics[width=0.24\linewidth]{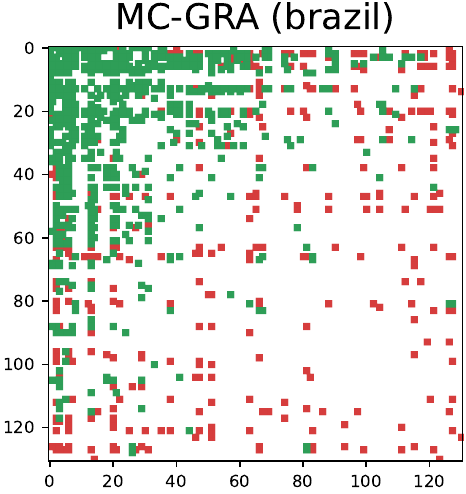}}
\caption{Recovered adjacency with GPR-GNN on two representative datasets. The first row shows the homophilic Cora case, where MC-GRA~(+) consistently recovers cleaner adjacency structure than GraphMI in both unprotected and protected settings. The second row shows the more difficult Brazil case, where both attacks are weaker overall and MC-GPB~(+) further degrades the visible structure.}
\label{fig: demo-gprgnn-adj}
\end{figure}

\begin{figure}[t!]
\centering
\subfloat[GPR-GNN on Brazil, unprotected.]{\includegraphics[width=0.48\linewidth]{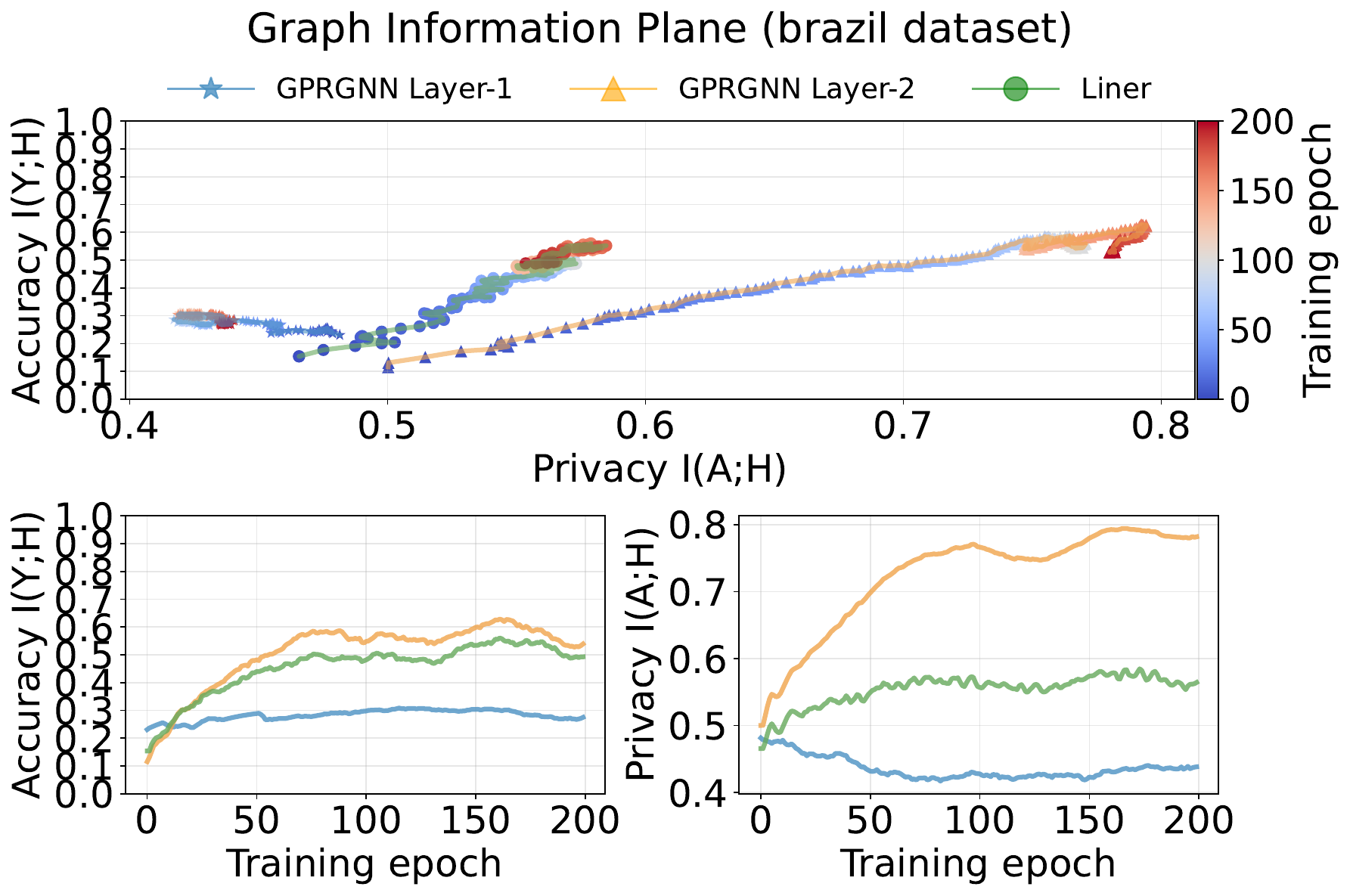}}
\subfloat[GPR-GNN on Brazil, protected.]{\includegraphics[width=0.48\linewidth]{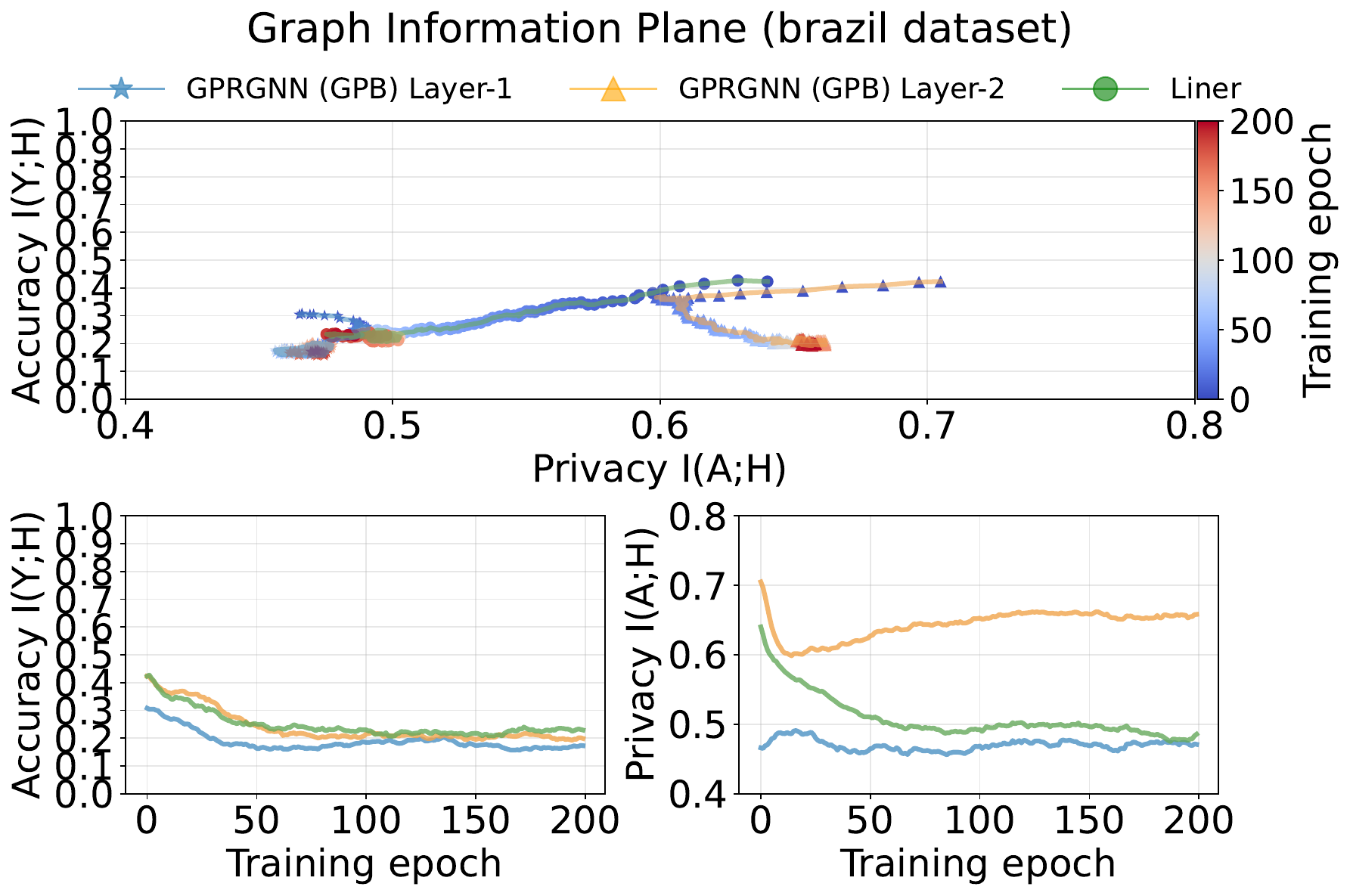}}
\caption{Graph information plane of GPR-GNN on Brazil under unprotected and protected training. Consistent with the adjacency recovery results, MC-GPB~(+) shifts the trajectory toward lower adjacency leakage while preserving the main task-related trend.}
\label{fig: demo-gprgnn-gip-brazil}
\end{figure}

\begin{figure}[t!]
\centering
\hfill
\includegraphics[width=0.45\linewidth]{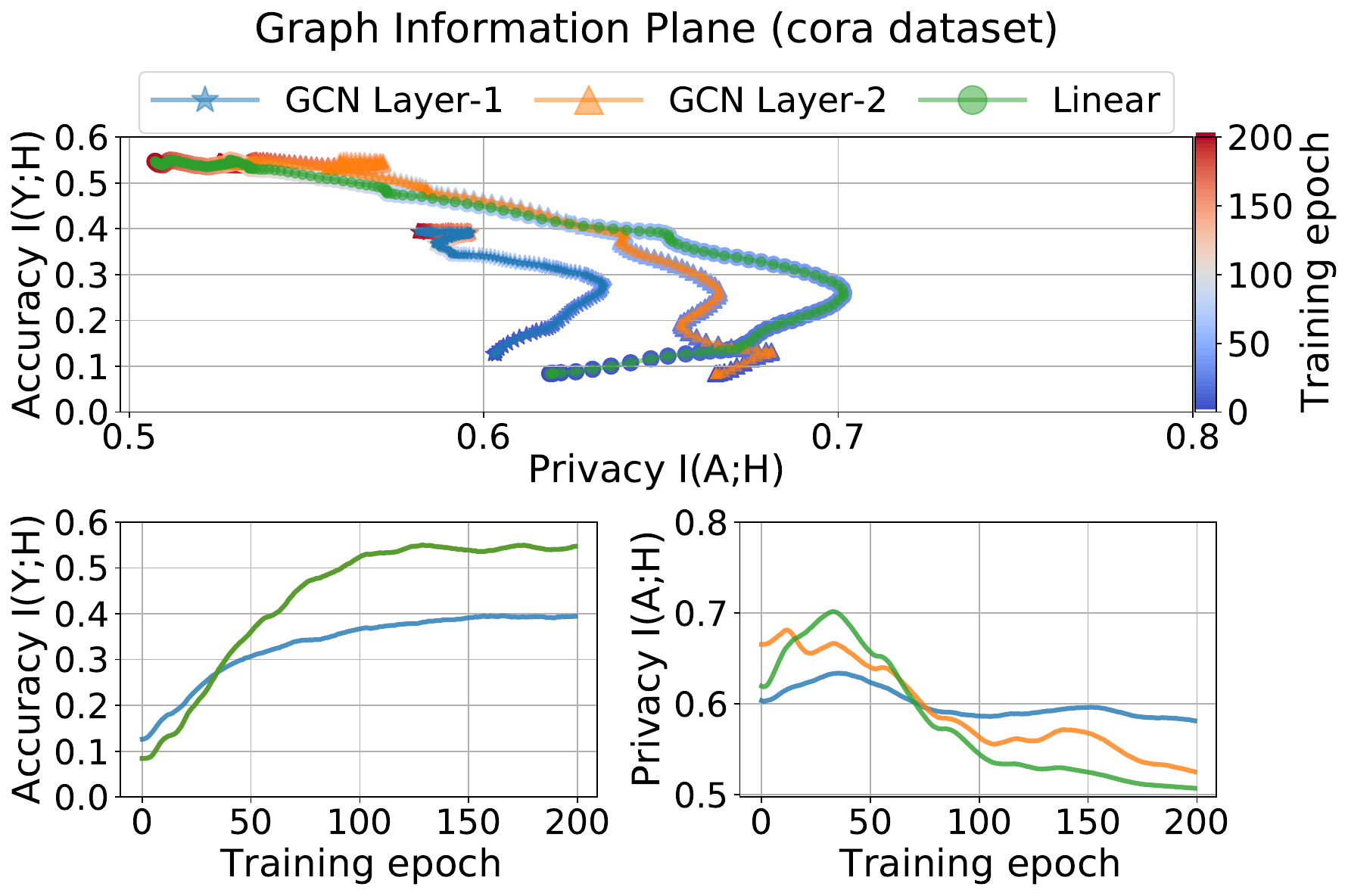}
\hfill
\includegraphics[width=0.45\linewidth]{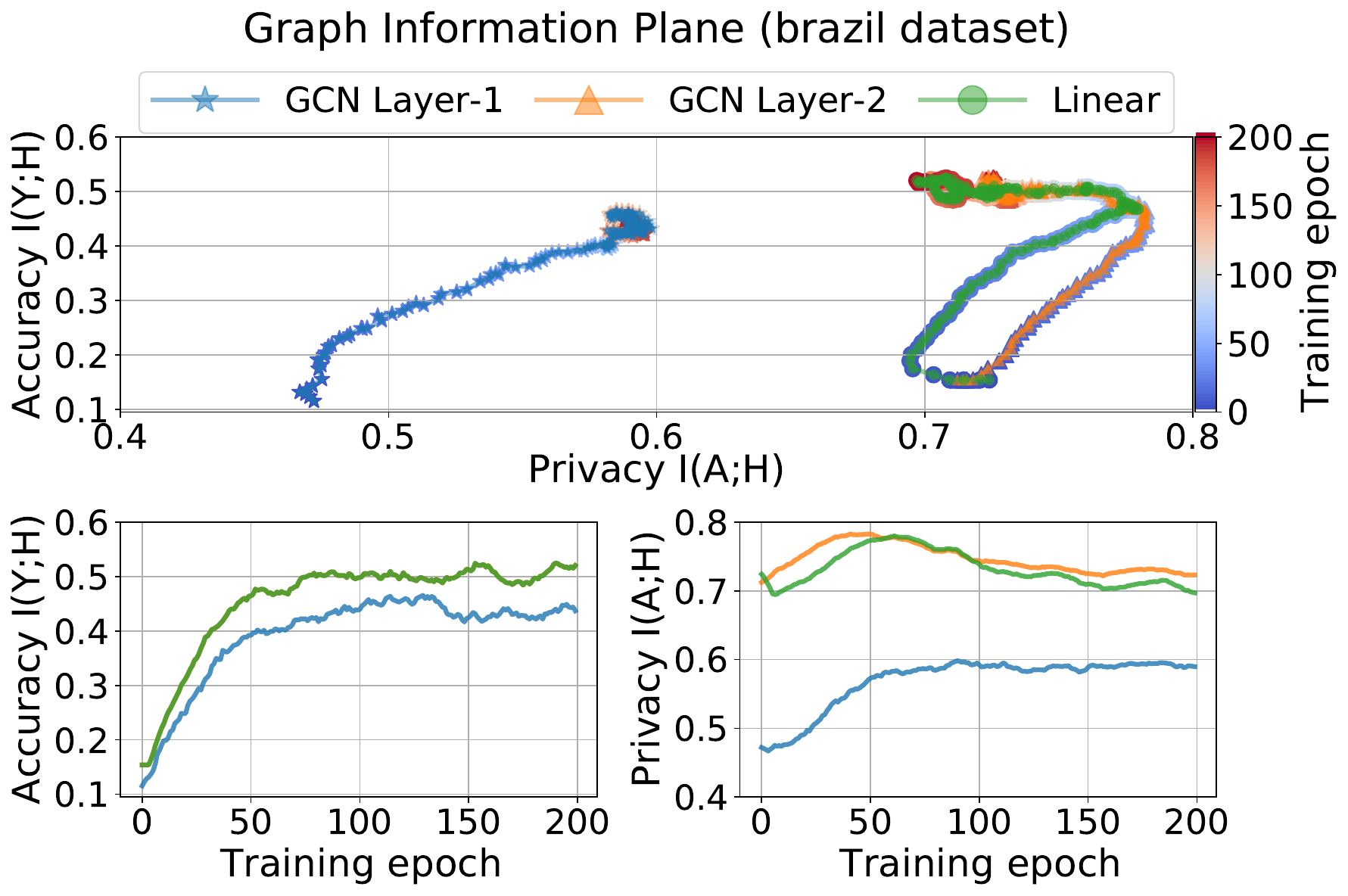}
\caption{Graph information plane: defensive training with MC-GPB~(+). Compared with unprotected training (Fig.~\ref{fig: graph-information-plane-gcn}), MC-GPB~(+) effectively decreases the adjacency-leakage proxy $R_{\mathrm{AUC}}(A;\bm{Z})$ in the graph representations.}
\label{fig: demo-GPB-graph-information-plane}
\end{figure}

\paragraph{A further analysis with the graph information plane.}
Fig.~\ref{fig: demo-GPB-graph-information-plane} visualizes the training dynamics of MC-GPB~(+) using the graph information plane introduced in Sec.~\ref{ssec: tracking by graph information plane}. Relative to unprotected training in Fig.~\ref{fig: graph-information-plane-gcn}, the protected trajectories move more consistently toward lower $R_{\mathrm{AUC}}(A;\bm{Z})$ on the horizontal axis, especially in deeper layers and in the linear head, where the leakage proxy often peaks early and then declines. Meanwhile, classification accuracy on the vertical axis rises quickly before saturating, and it only slightly decreases when the leakage proxy is pushed down further in later epochs.

This visualization is intended to analyze representation-level ranking leakage. A lower $R_{\mathrm{AUC}}(A;\bm{Z})$ means that the representation $\bm{Z}$ contains less rank-order information for distinguishing true edges from non-edges, independent of any particular edge-selection threshold or attack budget. This is the relevant diagnostic for MC-GPB~(+), whose objective is to attenuate adjacency dependence in learned representations rather than to optimize a specific top-$K$ retrieval criterion. The graph information plane therefore makes the privacy--utility trade-off explicit: MC-GPB~(+) reduces adjacency-ranking leakage while preserving task accuracy. More visualizations are in Appendix~\ref{sec: full qualitative results}.

%% file: sections/9-further-discussion.tex
\section{Further Discussions}
\label{sec: discussion}

\paragraph{Key differences from the conference version.}
The conference version~\citep{zhou2023mcgra} developed MC-GRA and MC-GPB under an implicit homophily assumption: the attack objective, the leakage analysis, and the defense regularizer all relied on similarity-based criteria that are well-aligned with homophilic graphs but can become misspecified under heterophily. This journal version addresses three scientific gaps left open by that work.

\emph{Gap 1: When does similarity-based reconstruction fail?}
The conference version did not characterize the conditions under which its core assumption---that similar representations indicate edges---breaks down. We address this with a systematic empirical study (Sec.~\ref{sec: overview}) and information-theoretic results (Theorems~\ref{thm: edgewise_MI_affinity}--\ref{thm: edgewise_MI_quadratic_lower}) showing that leakage depends on the interaction between graph structure and model inductive bias, and that strong heterophily can make similarity-based signals anti-informative.

\emph{Gap 2: How should attack and defense adapt to heterophily?}
We generalize both methods: MC-GRA+ incorporates a heterophily-aware prior derived from predicted label disagreement (Sec.~\ref{sec: GRA attack}), and MC-GPB+ adds explicit suppression of cross-label edge leakage (Sec.~\ref{sec: GRA defense}). Both are grounded in the chain-matching framework rather than being ad hoc extensions.

\emph{Gap 3: What are the fundamental limits of reconstruction and defense?}
We add formal results not present in the conference version: layerwise MI contraction, Fano-type and ceiling bounds on reconstruction fidelity, minimum adjacency information for task-sufficient representations, and a quadratic irreducible-leakage bound symmetric in homophily and heterophily (Secs.~\ref{sec: GRA attack}--\ref{sec: GRA defense} and Appendix~\ref{app: proof}). Experiments are broadened with heterophilic benchmarks, GPR-GNN backbone, ablations, and larger-scale datasets (Sec.~\ref{sec: experiments} and Appendix~\ref{sec: full quantitative results}).

\paragraph{Limitations.}
Several limitations of the current work should be noted.
\begin{itemize}[leftmargin=*]
\setlength\itemsep{0.1em}
\item \textbf{Single-graph transductive setting:} All experiments use a single graph in the transductive setting; the framework has not been tested in inductive or multi-graph settings, where the threat model and leakage channels may differ.
\item \textbf{Label access:} MC-GPB+ requires label access at training time to construct the heterophilic mask $M$; in partially supervised or transductive settings, one could restrict $M$ to available labeled nodes or use pseudo-labels, but we do not study those variants here.
\item \textbf{Adaptive attacks:} The defense is evaluated against MC-GRA~(+) and GraphMI, but an adaptive attacker who knows the defense mechanism (\eg, the DropEdge noise model or the CKA-based penalty) might design attacks that specifically circumvent it. Concretely, an adaptive attacker could (a)~marginalise over the known DropEdge perturbation distribution when computing the chain-matching objective, or (b)~train a learned edge predictor whose loss is invariant to the CKA-based penalty. Our defense evaluation follows the standard non-adaptive reconstruction setting used by prior GRA work. Fully adaptive attackers that jointly optimize against the defense define a stronger robustness setting; we view this as complementary to the reconstruction framework studied here. Similar concerns apply broadly to empirical defenses in adversarial ML~\citep{tramer2020adaptive}; we consider adaptive evaluation an important direction for future robustness analysis (see also Sec.~\ref{sec: experiments}).
\end{itemize}

\paragraph{Broader impacts.}
Graph reconstruction attacks, like other forms of model inversion and membership inference, can be misused to compromise privacy in real-world deployments. Concretely, an adversary could use MC-GRA to recover social ties from a deployed social-network classifier, or financial transaction links from a fraud-detection GNN. The primary purpose of studying GRA is \emph{defensive}: to make the threat model explicit, to raise awareness of the privacy risks posed by graph-structured data, and to provide principled tools for evaluating and mitigating leakage. To mitigate misuse, we recommend that practitioners deploying GNNs on sensitive relational data: (i) evaluate their models against chain-matching attacks before deployment, (ii) apply representation-level defenses such as MC-GPB~(+), and (iii) combine empirical defenses with formal privacy mechanisms where possible. Such analysis deepens our understanding of what information GNNs encode about their training data, helps identify failure modes under diverse graph regimes (including heterophily), and motivates the design of more robust, privacy-preserving learning algorithms.

%% file: sections/10-conclusion.tex
\section{Conclusion}

This work develops a unified chain-matching framework for graph reconstruction attacks (GRA) and defenses on GNNs, casting both as operations on the layered computation chain induced by the adjacency. MC-GRA~(+) reconstructs the adjacency by aligning surrogate and target chain representations at matching depths; 
MC-GPB~(+) defends by suppressing adjacency-dependent information throughout the chain while preserving task utility under a privacy-utility trade-off.

Perhaps the most surprising finding is the \emph{asymmetry between attack and defense across graph regimes}: on heterophilic graphs, where similarity-based reconstruction was previously thought to be difficult, MC-GRA+ with a label-disagreement prior achieves large relative gains (relative improvement exceeding 100\% on Cornell, from AUC $0.33$ to $0.67$), revealing that heterophilic edges are recoverable given an appropriate prior---a threat not captured by earlier homophily-centric analyses. Conversely, MC-GPB~(+) is most effective precisely in these regimes, reducing attack AUC to near chance.

For practitioners deploying GNNs on privacy-sensitive relational data, the key recommendation is that \emph{architectural choice alone does not constitute a defense}: heterophily-aware models like GPR-GNN redistribute but do not eliminate adjacency leakage. Explicit representation-level regularization (as in MC-GPB~(+)) is needed to control what the model encodes about its training graph.

\acks{The authors declare no competing interests.}

%% file: appendix/1-theoretical-justification.tex
\section{Theoretical justification}
\label{app: proof}

\subsection{Background: conditional entropy along a stationary Markov chain}
\label{ssec: background markov lemma}

The following classical lemma is referenced in Remark~\ref{rem: contraction-preview}. It does not apply directly to GNN chains (which are deterministic and non-stationary) but provides background for the information-contraction interpretation.

\begin{lemma}[Conditional entropy along a stationary Markov chain]
\label{lem: conditional-entropy-markov}
If $(W_t)_{t\ge1}$ is a stationary first-order Markov chain, then the sequence
$\{H(W_t\mid W_1)\}_{t\ge1}$ is nondecreasing in $t$.
\begin{proof}
For any $t\ge1$, the first-order Markov property implies $W_1 \to W_t \to W_{t+1}$, and the data processing inequality therefore yields
$I(W_1;W_{t+1}) \le I(W_1;W_t)$.
By stationarity, $H(W_t)=H(W_{t+1})=H(W_1)$ for all $t$.
Therefore,
$H(W_{t+1}\mid W_1)
=
H(W_{t+1})-I(W_{t+1};W_1)
\ge
H(W_t)-I(W_t;W_1)
=
H(W_t\mid W_1)$.
\end{proof}
\end{lemma}


\subsection{Proof of Theorem~\ref{theorem: reducing MI with two chains}}
\label{ssec: proof of reducing MI with two chains}

\begin{proof}
We prove the theorem for the layer-map template (GCN-type form) stated in the main text; the argument uses only the data processing inequality for deterministic measurable maps. We first recall two standard facts about mutual information under measurable transformations.

\begin{lemma}[Invariance under bijections]
\label{lemma: mi-invariance}
Let $X,Y$ be random variables and let $f,g$ be measurable maps such that $f$ is bijective on the support of $X$ and $g$ is bijective
on the support of $Y$, with measurable inverses. Assume the pushforward measures of $X$ and $Y$ under $f$ and $g$ are well-defined (so that $f(X)$ and $g(Y)$ have valid distributions). Then
\begin{equation}
I(X;Y)=I\bigl(f(X);Y\bigr)=I\bigl(X;g(Y)\bigr)=I\bigl(f(X);g(Y)\bigr).
\end{equation}
\end{lemma}

\begin{lemma}[Data processing inequality for deterministic maps]
\label{lemma: dpi-deterministic}
Let $X,Y$ be random variables and let $f,g$ be measurable maps. Then
\begin{equation}
I\bigl(f(X);g(Y)\bigr)\le I(X;Y).
\end{equation}
Moreover, if $f$ and $g$ are bijective almost everywhere on the supports of $X$ and $Y$, respectively, with measurable inverses, then equality holds (equivalently, the joint law of $(f(X),g(Y))$ is the pushforward of the joint law of $(X,Y)$ under $(f,g)$ and no information is lost); see, e.g., \citet{cover1999elements} for the equality condition.
\end{lemma}

\paragraph{Notation and decomposition of the layer map.}
Fix a layer index $i\in\{0,\dots,L-1\}$. In Theorem~\ref{theorem: reducing MI with two chains}, the layer maps are
\begin{equation}
T_A^{(i)}(z)=\rho\!\bigl(\psi(A)\,z\,\bm{\theta}^{(i)}\bigr),
\qquad
T_{\bm{\hat{A}}}^{(i)}(z)=\rho\!\bigl(\psi(\bm{\hat{A}})\,z\,\bm{\theta}^{(i)}\bigr).
\end{equation}
For clarity, decompose each $T$ into a linear map followed by the activation:
\begin{equation}
L_A^{(i)}(z)\;:=\;\psi(A)\,z\,\bm{\theta}^{(i)},
\qquad
L_{\bm{\hat{A}}}^{(i)}(z)\;:=\;\psi(\bm{\hat{A}})\,z\,\bm{\theta}^{(i)},
\qquad
S(u)\;:=\;\rho(u),
\end{equation}
so that $T_A^{(i)}=S\circ L_A^{(i)}$ and $T_{\bm{\hat{A}}}^{(i)}=S\circ L_{\bm{\hat{A}}}^{(i)}$. By the assumed recursion,
\begin{equation}
\bm{Z}_{A}^{(i+1)}=(S\circ L_A^{(i)})(\bm{Z}_{A}^{(i)}),
\qquad
\bm{Z}_{\bm{\hat{A}}}^{(i+1)}=(S\circ L_{\bm{\hat{A}}}^{(i)})(\bm{Z}_{\bm{\hat{A}}}^{(i)}).
\end{equation}

\paragraph{Layerwise non-increase of mutual information.}
We apply Lemma~\ref{lemma: dpi-deterministic} to the joint random variable $(\bm{Z}_{A}^{(i)},\bm{Z}_{\bm{\hat{A}}}^{(i)})$, viewing $\bm{Z}_{A}^{(i)}$ as $X$ and $\bm{Z}_{\bm{\hat{A}}}^{(i)}$ as $Y$, with $f=L_A^{(i)}$ applied to $\bm{Z}_{A}^{(i)}$ and $g=L_{\bm{\hat{A}}}^{(i)}$ applied to $\bm{Z}_{\bm{\hat{A}}}^{(i)}$. Lemma~\ref{lemma: dpi-deterministic} applies to any pair of measurable maps $(f,g)$ acting on the respective components of a joint random variable; here $f = L_A^{(i)}$ and $g = L_{\bm{\hat{A}}}^{(i)}$ are different maps (using different adjacencies) applied to $\bm{Z}_A^{(i)}$ and $\bm{Z}_{\bm{\hat{A}}}^{(i)}$ respectively. That is, we apply the DPI to the product map $(L_A^{(i)}, L_{\bm{\hat{A}}}^{(i)})$ acting on $(\bm{Z}_{A}^{(i)}, \bm{Z}_{\bm{\hat{A}}}^{(i)})$. Since both component maps are deterministic and measurable (and hence their composition is measurable by the standard composition theorem for measurable maps), Lemma~\ref{lemma: dpi-deterministic} implies
\begin{equation}
\label{eq: dpi-linear-app}
I\left(L_A^{(i)}(\bm{Z}_{A}^{(i)});\,L_{\bm{\hat{A}}}^{(i)}(\bm{Z}_{\bm{\hat{A}}}^{(i)})\right)
\;\le\;
I\left(\bm{Z}_{A}^{(i)};\,\bm{Z}_{\bm{\hat{A}}}^{(i)}\right).
\end{equation}
Applying Lemma~\ref{lemma: dpi-deterministic} again with $f=g=S$ (the shared activation) yields
\begin{equation}
\label{eq: dpi-activation-app}
I\left((S\circ L_A^{(i)})(\bm{Z}_{A}^{(i)});\,(S\circ L_{\bm{\hat{A}}}^{(i)})(\bm{Z}_{\bm{\hat{A}}}^{(i)})\right)
\;\le\;
I\left(L_A^{(i)}(\bm{Z}_{A}^{(i)});\,L_{\bm{\hat{A}}}^{(i)}(\bm{Z}_{\bm{\hat{A}}}^{(i)})\right).
\end{equation}
Combining Eq.~\eqref{eq: dpi-linear-app}--Eq.~\eqref{eq: dpi-activation-app} and substituting the recursion gives
\begin{equation}
I\left(\bm{Z}_{A}^{(i+1)};\,\bm{Z}_{\bm{\hat{A}}}^{(i+1)}\right)
\;\le\;
I\left(\bm{Z}_{A}^{(i)};\,\bm{Z}_{\bm{\hat{A}}}^{(i)}\right),
\end{equation}
which is exactly Eq.~\eqref{eqn: MI-nonincrease-two-chains}.

\paragraph{Equality and (typical) strictness.}
If both $T_A^{(i)}$ and $T_{\bm{\hat{A}}}^{(i)}$ are bijective almost everywhere on the supports of $\bm{Z}_{A}^{(i)}$ and
$\bm{Z}_{\bm{\hat{A}}}^{(i)}$, then by Lemma~\ref{lemma: mi-invariance} (equivalently, the equality condition in
Lemma~\ref{lemma: dpi-deterministic}) neither transformation loses information, and both inequalities
Eq.~\eqref{eq: dpi-linear-app}--Eq.~\eqref{eq: dpi-activation-app} hold with equality. Hence
$I(\bm{Z}_{A}^{(i+1)};\bm{Z}_{\bm{\hat{A}}}^{(i+1)})=I(\bm{Z}_{A}^{(i)};\bm{Z}_{\bm{\hat{A}}}^{(i)})$.

If either map is not invertible on the support, equality can still hold only in degenerate cases where
$(\bm{Z}_{A}^{(i)},\bm{Z}_{\bm{\hat{A}}}^{(i)})$ is almost surely supported on a subset on which the induced joint mapping admits a
(measurable) right-inverse. In typical non-degenerate settings with many-to-one nonlinearities (\eg, ReLU, hard-thresholding, pooling),
one expects strict decrease across depth.
\end{proof}

\subsection{Proof of Theorem~\ref{theorem: GRA attack lower bound}}
\label{ssec: proof of GRA attack lower bound}

\paragraph{Remark on scope.}
Theorem~\ref{theorem: GRA attack lower bound} is a population-level information-theoretic statement.
Finite-sample estimation and generalization of mutual information are orthogonal issues; thus we do not require (and do not use)
finite-sample generalization bounds in this proof.

\begin{proof}
We prove the two inequalities in Theorem~\ref{theorem: GRA attack lower bound} (in its conditional form given $X$).

\paragraph{Step 1: $I(A;\bm{\hat{A}}\mid X)\ge I(\bm{H}_{A};\bm{H}_{\bm{\hat{A}}}\mid X)$.}
Condition on a realization $X=x$. Under the fixed target model $f_{\bm{\theta}^*}$, $\bm{H}_{A}$ is a measurable function of $A$
given $X=x$, and $\bm{H}_{\bm{\hat{A}}}$ is a measurable function of $\bm{\hat{A}}$ given $X=x$. Hence, for each $x$, the pair
$(\bm{H}_A,\bm{H}_{\bm{\hat{A}}})$ is a deterministic post-processing of $(A,\bm{\hat{A}})$, and by Lemma~\ref{lemma: dpi-deterministic},
\begin{equation}
I(A;\bm{\hat{A}}\mid X=x)\ \ge\ I(\bm{H}_{A};\bm{H}_{\bm{\hat{A}}}\mid X=x).
\end{equation}
Taking expectation over $X$ yields
\begin{equation}
I(A;\bm{\hat{A}}\mid X)\ \ge\ I(\bm{H}_{A};\bm{H}_{\bm{\hat{A}}}\mid X).
\end{equation}
\emph{(If $X$ is treated as fixed/non-random, then $I(\cdot;\cdot\mid X)=I(\cdot;\cdot)$ and this step reduces to the unconditional
inequality used in the main text.)}

\paragraph{Step 2: Fano lower bound on $I(\bm{H}_{A};\bm{H}_{\bm{\hat{A}}}\mid X)$.}
When the theorem is stated with a quantizer $Q$, $\bm{H}_{A}$ and $\bm{H}_{\bm{\hat{A}}}$ here denote the quantized versions $\bar{\bm{H}}_{A}=Q(\bm{H}_{A})$ and $\bar{\bm{H}}_{\bm{\hat{A}}}=Q(\bm{H}_{\bm{\hat{A}}})$. By assumption, $\bm{H}_{A}$ takes values in a finite alphabet $\mathcal{H}$ with $2\le|\mathcal{H}|<\infty$ (or is effectively
discretized so that this holds).
Let $\hat{\bm{H}} := \bm{H}_{\bm{\hat{A}}}$ be the observation used to estimate $\bm{H}_{A}$, and define the error event
\begin{equation}
\mathcal{E}:=\{\bm{H}_{A}\neq \hat{\bm{H}}\},
\qquad
P_e := \mathbb{P}(\mathcal{E}).
\end{equation}
By (conditional) Fano's inequality, for any estimator $\widetilde{\bm{H}}(\hat{\bm{H}},X)$ of $\bm{H}_{A}$,
\begin{equation}
H(\bm{H}_{A}\mid \hat{\bm{H}},X)
\;\le\;
H_b(P_e) \;+\; P_e \log\bigl(|\mathcal{H}|-1\bigr),
\end{equation}
where $P_e=\mathbb{P}(\bm{H}_{A}\neq \widetilde{\bm{H}}(\hat{\bm{H}},X))$.
We apply Fano's inequality with the identity estimator $\widetilde{\bm{H}}(\hat{\bm{H}},X)=\hat{\bm{H}}=\bm{H}_{\bm{\hat{A}}}$. This is a valid (if not necessarily optimal) estimator of $\bm{H}_{A}$ from $\bm{H}_{\bm{\hat{A}}}$; a better-designed estimator (e.g., one that also exploits $X$) could yield a tighter Fano bound, but the identity estimator suffices for our purposes. With this choice, $P_e=\mathbb{P}(\bm{H}_{A}\neq \bm{H}_{\bm{\hat{A}}})$, which is the representation mismatch probability $P_e(x)$ when conditioned on $X=x$. Hence,
\begin{equation}
H(\bm{H}_{A}\mid \bm{H}_{\bm{\hat{A}}},X)
\;\le\;
H_b(P_e) \;+\; P_e \log\bigl(|\mathcal{H}|-1\bigr).
\end{equation}
Finally, using $I(U;V\mid X)=H(U\mid X)-H(U\mid V,X)$,
\begin{equation}
I(\bm{H}_{A};\bm{H}_{\bm{\hat{A}}}\mid X)
=
H(\bm{H}_{A}\mid X)-H(\bm{H}_{A}\mid \bm{H}_{\bm{\hat{A}}},X)
\;\ge\;
H(\bm{H}_{A}\mid X) - H_b(P_e) - P_e \log\bigl(|\mathcal{H}|-1\bigr).
\end{equation}

\paragraph{Conclusion.}
Combining Step~1 and Step~2 gives
\begin{equation}
I(A;\bm{\hat{A}}\mid X)
\;\ge\;
I(\bm{H}_{A};\bm{H}_{\bm{\hat{A}}}\mid X)
\;\ge\;
H(\bm{H}_{A}\mid X) - H_b(P_e) - P_e \log\bigl(|\mathcal{H}|-1\bigr).
\end{equation}
The same bound holds for each realized $X=x$ with $P_e(x)$ in place of $P_e$, yielding the pointwise bound in the theorem statement. The looser variant follows from $\log(|\mathcal{H}|-1)\le \log|\mathcal{H}|$.
\end{proof}

\subsection{Proof for Proposition~\ref{prop: optimal fidelity in GRA}}
\label{ssec: proof of optimal fidelity in GRA}

\begin{proof}
Fix an attacker's knowledge set $\mathcal{K}$ jointly distributed with the unknown adjacency $A$.
The attacker outputs
\(
\bm{\hat{A}}=g(\mathcal{K},U)
\)
for some measurable map $g$, where the auxiliary randomness $U$ is independent of $(A,\mathcal{K})$.

\paragraph{Step 1: Upper bound via data processing.}
Because $\bm{\hat{A}}$ is generated from $\mathcal{K}$ (possibly with independent randomness), the variables satisfy the Markov chain
\begin{equation}
A \;\longrightarrow\; \mathcal{K} \;\longrightarrow\; (\mathcal{K},U) \;\longrightarrow\; \bm{\hat{A}}.
\end{equation}
By the data processing inequality,
\begin{equation}
\label{eq: dpi-1}
I(A;\bm{\hat{A}})
\;\le\;
I\bigl(A;(\mathcal{K},U)\bigr).
\end{equation}
Since $U \perp\!\!\!\perp (A,\mathcal{K})$, we have $I(A;U\mid \mathcal{K})=0$, and therefore
\begin{equation}
I\bigl(A;(\mathcal{K},U)\bigr)
=
I(A;\mathcal{K}) + I(A;U\mid \mathcal{K})
=
I(A;\mathcal{K}).
\end{equation}
Combining with Eq.~\eqref{eq: dpi-1} yields the desired bound
\begin{equation}
I(A;\bm{\hat{A}}) \;\le\; I(A;\mathcal{K}).
\end{equation}
Taking the supremum over all measurable $g$ shows
\begin{equation}
\mathcal{F}^{*}(\mathcal{K})
=\sup_{g} I\bigl(A;g(\mathcal{K},U)\bigr)
\;\le\;
I(A;\mathcal{K}).
\end{equation}

\paragraph{Step 2: Characterization of equality.}
Equality in the data processing inequality for the Markov chain
$A\to \mathcal{K}\to \bm{\hat{A}}$
holds if and only if $\bm{\hat{A}}$ is a sufficient statistic of $\mathcal{K}$ for $A$. In mutual-information form,
\begin{equation}
I(A;\mathcal{K}\mid \bm{\hat{A}})=0,
\end{equation}
which is equivalent to
\begin{equation}
I(A;\bm{\hat{A}}) = I(A;\mathcal{K}).
\end{equation}
The identity $I(A;\mathcal{K})=I(A;\bm{\hat{A}})+I(A;\mathcal{K}\mid \bm{\hat{A}})$ holds because $\bm{\hat{A}}$ is a function of $(\mathcal{K},U)$; together with the Markov structure $A\to \mathcal{K}\to \bm{\hat{A}}$, the equality condition $I(A;\mathcal{K}\mid \bm{\hat{A}})=0$ is equivalent to $I(A;\bm{\hat{A}})=I(A;\mathcal{K})$. This proves the stated equality case in Proposition~\ref{prop: optimal fidelity in GRA}.
\end{proof}

\begin{remark}[Idealized conditions for MC-GRA optimality (non-operational)]
The following connection to MC-GRA is conditional on strong assumptions and is included for conceptual completeness; it does not predict practical attack performance. Suppose (i) the dependence measures used in MC-GRA coincide with the corresponding true mutual information, and (ii) the optimization over the hypothesis class for $\bm{\hat{A}}$ attains a global optimum. Then MC-GRA maximizes an objective consistent with extracting the information about $A$ contained in $\mathcal{K}$. Consequently, whenever there exists $\bm{\hat{A}}^*=g^*(\mathcal{K})$ within the hypothesis class such that $I(A;\mathcal{K}\mid \bm{\hat{A}}^*)=0$, MC-GRA can attain $I(A;\bm{\hat{A}}^*)=I(A;\mathcal{K})$, \ie, the optimal fidelity under knowledge $\mathcal{K}$. The main proposition itself is a general information-theoretic ceiling and does not assume (i)--(ii).
\end{remark}


\subsection{Proof for Theorem~\ref{thm: sufficiency_label_induced_matrix}}
\label{app: proof_sufficiency_posterior_matrix}
\begin{proof}
Let $Y=(Y_1,\dots,Y_N)\in[C]^N$ and let $A\in\{0,1\}^{N\times N}$ be an undirected adjacency matrix with $A_{ii}=0$ and
$A_{ji}=A_{ij}$. Assume $B$ is symmetric so that the undirected model is well-defined.
Under the assumed block model, for every $i<j$ we have
\begin{equation}
A_{ij}\mid(Y_i=a,Y_j=b)\sim \mathrm{Bern}(B_{ab}),
\end{equation}
and the collection $\{A_{ij}\}_{i<j}$ is conditionally independent given $Y$.
Define $\bm{P}(Y)\in[0,1]^{N\times N}$ by $\bm{P}(Y)_{ij}=B_{Y_iY_j}$ for $i\neq j$ and $\bm{P}(Y)_{ii}=0$.

\paragraph{Step 1: Conditional law of $A$ given $Y$.}
Fix any realization $y\in[C]^N$. By conditional independence,
\begin{equation}
\mathbb{P}(A=a\mid Y=y)
=
\prod_{1\le i<j\le N}
\Bigl(B_{y_i y_j}\Bigr)^{a_{ij}}
\Bigl(1-B_{y_i y_j}\Bigr)^{1-a_{ij}},
\label{eq: condlaw_A_given_Y}
\end{equation}
for every symmetric $a\in\{0,1\}^{N\times N}$ with zero diagonal (identifying $a_{ji}=a_{ij}$).

\paragraph{Step 2: Conditional law of $A$ given $\bm{P}(Y)$.}
Fix any matrix $p\in[0,1]^{N\times N}$ (symmetric with zero diagonal) such that $\mathbb{P}(\bm{P}(Y)=p)>0$.
For any $y$ satisfying $\bm{P}(y)=p$, we have $B_{y_i y_j}=p_{ij}$ for all $i\neq j$ by definition of $\bm{P}(\cdot)$.
Substituting $p_{ij}$ into Eq.~\eqref{eq: condlaw_A_given_Y} yields
\begin{equation}
\mathbb{P}(A=a\mid Y=y)
=
\prod_{1\le i<j\le N}
\Bigl(p_{ij}\Bigr)^{a_{ij}}
\Bigl(1-p_{ij}\Bigr)^{1-a_{ij}},
\qquad \forall y:\bm{P}(y)=p.
\end{equation}
Thus the right-hand side depends on $y$ only through $p$, and therefore
\begin{equation}
\mathbb{P}(A=a\mid \bm{P}(Y)=p)
=
\prod_{1\le i<j\le N}
\Bigl(p_{ij}\Bigr)^{a_{ij}}
\Bigl(1-p_{ij}\Bigr)^{1-a_{ij}}.
\label{eq: condlaw_A_given_P}
\end{equation}
In particular, for any $y$ we have $\mathbb{P}(A\in\cdot\mid Y=y)=\mathbb{P}(A\in\cdot\mid \bm{P}(Y)=\bm{P}(y))$.

\paragraph{Step 3: Sufficiency / conditional independence.}
Let $p=\bm{P}(Y)$. By Eq.~\eqref{eq: condlaw_A_given_P}, the conditional distribution $\mathbb{P}(A\in\cdot\mid \bm{P}(Y)=p)$ is fully
determined by $p$ and does not depend on $Y$ beyond $p$. Hence
$\mathbb{P}(A\in\cdot\mid Y,\bm{P}(Y))=\mathbb{P}(A\in\cdot\mid \bm{P}(Y))$ almost surely, which proves
$A \perp\!\!\!\perp Y\mid \bm{P}(Y)$.

\paragraph{Step 4: Mutual information identity.}
Since $\bm{P}(Y)$ is a deterministic function of $Y$, the chain rule gives
\begin{equation}
I(A;Y)=I(A;\bm{P}(Y)) + I\bigl(A;Y\mid \bm{P}(Y)\bigr).
\end{equation}
By the conditional independence established above, $I(A;Y\mid \bm{P}(Y))=0$, and therefore $I(A;\bm{P}(Y))=I(A;Y)$.
\end{proof}

\subsection{Proof for Theorem~\ref{thm: collapse_to_disagreement_info_loss_polished}}
\label{app: proof_collapse_to_disagreement}
\begin{proof}
Recall the conditional edge-independence block model from Theorem~\ref{thm: sufficiency_label_induced_matrix}:
for all $i<j$,
\begin{equation}
A_{ij}\mid(Y_i=a,Y_j=b)\sim \mathrm{Bern}(B_{ab}),
\qquad
\{A_{ij}\}_{i<j}\ \text{independent given }Y,
\end{equation}
with $A_{ji}=A_{ij}$ and $A_{ii}=0$. Define $\bm{D}\in\{0,1\}^{N\times N}$ by $\bm{D}_{ij}=\mathbbm{1}\{Y_i\neq Y_j\}$
(and $\bm{D}_{ii}=0$). Note that $\bm{D}$ is a deterministic function of $Y$.

\paragraph{Step 1: Mutual information decomposition.}
Since $\bm{D}=h(Y)$ is a function of $Y$, the chain rule yields
\begin{equation}
I(A;Y)=I\bigl(A;(\bm{D},Y)\bigr)=I(A;\bm{D})+I(A;Y\mid \bm{D}),
\end{equation}
which proves the first claim.

\paragraph{Step 2: Sufficiency condition in terms of conditional independence.}
Because $\bm{D}$ is a function of $Y$,
\begin{equation}
I(A;Y\mid \bm{D})=0
\quad\Longleftrightarrow\quad
A \perp\!\!\!\perp Y \mid \bm{D}.
\end{equation}

\paragraph{Step 3: Two-parameter form implies $A \perp\!\!\!\perp Y\mid \bm{D}$.}
Assume $B$ has the two-parameter form. Then for any $y\in[C]^N$ and any symmetric $a$ with zero diagonal,
\begin{equation}
\mathbb{P}(A=a\mid Y=y)
=
\prod_{1\le i<j\le N}
\Bigl(B_{y_i y_j}\Bigr)^{a_{ij}}
\Bigl(1-B_{y_i y_j}\Bigr)^{1-a_{ij}}.
\end{equation}
Under the two-parameter assumption, $B_{y_i y_j}=p_{\mathrm{in}}$ when $y_i=y_j$ and $B_{y_i y_j}=p_{\mathrm{out}}$ when $y_i\neq y_j$,
so the product depends on $y$ only via $\bm{D}(y)$. Hence
$\mathbb{P}(A\in\cdot\mid Y)=\mathbb{P}(A\in\cdot\mid \bm{D})$ a.s., which is equivalent to $A \perp\!\!\!\perp Y\mid \bm{D}$.

\paragraph{Step 4: Conditional independence forces the two-parameter form of $B$.}
Assume $A \perp\!\!\!\perp Y\mid \bm{D}$ and that each label has positive probability (so that all events below have positive
probability when conditioning).
Fix two ordered label pairs $(a,b)$ and $(a',b')$ such that either both satisfy $a=b$ and $a'=b'$ or both satisfy $a\neq b$ and
$a'\neq b'$. Pick any indices $i\neq j$ and define $E_{ab}=\{Y_i=a,Y_j=b\}$. By the model,
\begin{equation}
\mathbb{P}(A_{ij}=1\mid E_{ab}) = B_{ab}.
\end{equation}
By $A \perp\!\!\!\perp Y\mid \bm{D}$, conditioning on $E_{ab}$ cannot change the conditional law of $A_{ij}$ beyond $\bm{D}_{ij}$, hence
\begin{equation}
\mathbb{P}(A_{ij}=1\mid E_{ab})
=
\mathbb{P}(A_{ij}=1\mid \bm{D}_{ij}=\mathbbm{1}\{a\neq b\}).
\end{equation}
The right-hand side depends only on whether $a=b$ or $a\neq b$. Applying the same argument to $E_{a'b'}$ yields $B_{ab}=B_{a'b'}$.
Therefore all diagonal entries $\{B_{aa}\}$ equal a common $p_{\mathrm{in}}$ and all off-diagonal entries $\{B_{ab}:a\neq b\}$ equal a
common $p_{\mathrm{out}}$.

\paragraph{Step 5: Conclude equivalence.}
Steps 3--4 establish
\begin{equation}
A \perp\!\!\!\perp Y \mid \bm{D}
\quad\Longleftrightarrow\quad
B_{ab}=
\begin{cases}
p_{\mathrm{in}}, & a=b,\\
p_{\mathrm{out}}, & a\neq b,
\end{cases}
\end{equation}
for some $p_{\mathrm{in}},p_{\mathrm{out}}\in[0,1]$. Combining with Steps 1--2 completes the proof.
\end{proof}

\subsection{Proof for Theorem~\ref{thm: noisy_labels_contraction_polished}}
\label{app: proof_noisy_labels_contraction}
\begin{proof}
Let $C\ge 2$ and let $\varepsilon\in\bigl[0,\frac{C-1}{C}\bigr]$. For each coordinate $i$, the $C$-ary symmetric channel is the
stochastic matrix
\begin{equation}
W(b\mid a)=
\begin{cases}
1-\varepsilon, & b=a,\\
\frac{\varepsilon}{C-1}, & b\neq a,
\end{cases}
\qquad a,b\in[C].
\end{equation}
By assumption, $\hat Y_i$ is obtained from $Y_i$ by $W$, independently across $i$, and $V\to Y\to \hat Y$ is a Markov chain. We also
assume (as in the theorem statement) that each $Y_i$ is marginally uniform on $[C]$, \ie, the stationary distribution of $W$.

\paragraph{Step 1: Single-letter SDPI.}
A standard strong data-processing inequality (SDPI) states that for any Markov chain $V\to Y\to \hat Y$ with transition $W$,
\begin{equation}
I(V;\hat Y)\le \eta(W)\, I(V;Y),
\label{eq: sdpi_general}
\end{equation}
where $\eta(W)\in[0,1]$ is the SDPI constant (contraction coefficient for mutual information) of the channel $W$.

Under the uniform stationary distribution, the $C$-ary symmetric channel is reversible and its Markov operator has eigenvalues
\begin{equation}
\lambda_1=1,
\qquad
\lambda_2 = 1-\varepsilon-\frac{\varepsilon}{C-1}=1-\frac{C}{C-1}\varepsilon
\quad (\text{multiplicity }C-1).
\end{equation}
For reversible channels, $\eta(W)$ equals the square of the second-largest singular value (equivalently, $\lambda_2^2$), hence
\begin{equation}
\eta(W)=\lambda_2^2=\left(1-\frac{C}{C-1}\varepsilon\right)^2 \triangleq \eta(\varepsilon,C).
\end{equation}
Substituting into Eq.~\eqref{eq: sdpi_general} yields the desired bound
\begin{equation}
I(V;\hat Y)\le \eta(\varepsilon,C)\, I(V;Y).
\end{equation}

\paragraph{Step 2: Consequence for heterophily priors measurable \wrt $\hat Y$.}
Let $S=h(\hat Y)$ be any side information measurable with respect to $\hat Y$ (\eg, a heterophily prior constructed from $\hat Y$).
Then $V\to \hat Y\to S$ is a Markov chain, and by data processing,
\begin{equation}
I(V;S)\le I(V;\hat Y)\le \eta(\varepsilon,C)\, I(V;Y).
\end{equation}
This completes the proof. \textbf{Note.} Under non-uniform class marginals, the contraction coefficient differs and the bound does not take the simple form $\eta(\varepsilon,C)$; the qualitative conclusion that noisier predictions reduce usable information still holds.
\end{proof}


\subsection{Proof of the maximum adjacency information bound (Eq.~\ref{eqn: max-adj-info})}
\label{ssec: proof of maximum adjacency information}

\begin{proof}
Let $A$ be discrete with $H(A)<\infty$. By definition of mutual information,
\begin{equation}
I(A;\bm{H}_{A})
=
H(A)-H(A\mid \bm{H}_{A}).
\end{equation}
Since conditional entropy is nonnegative, $H(A\mid \bm{H}_{A})\ge 0$, it follows that
\begin{equation}
I(A;\bm{H}_{A}) \;\le\; H(A),
\end{equation}
which proves Eq.~\eqref{eqn: max-adj-info}.

Moreover, equality holds if and only if $H(A\mid \bm{H}_{A})=0$.
Because $A$ is discrete, $H(A\mid \bm{H}_{A})=0$ is equivalent to the existence of a measurable map $\varphi$ such that
$A=\varphi(\bm{H}_{A})$ almost surely, \ie, $A$ is (a.s.) a function of $\bm{H}_{A}$.
\end{proof}

\subsection{Proof for Theorem~\ref{theorem: minimum adjacency information}}
\label{ssec: proof of minimum adjacency information}

\begin{proof}
Assume the conditional sufficiency statement in Theorem~\ref{theorem: minimum adjacency information}:
\begin{equation}
Y \perp\!\!\!\perp A \mid (\bm{H}_{A},X),
\quad\text{equivalently}\quad
I(A;Y\mid \bm{H}_{A},X)=0,
\end{equation}
and assume all conditional mutual informations below are finite.

\paragraph{Lower bound.}
By the chain rule for conditional mutual information,
\begin{equation}
\label{eq: chainrule-minadj-cond}
I(A;\bm{H}_{A}\mid X)
=
I(A;Y\mid X)
+
I(A;\bm{H}_{A}\mid Y,X)
-
I(A;Y\mid \bm{H}_{A},X).
\end{equation}
Under the sufficiency assumption, $I(A;Y\mid \bm{H}_{A},X)=0$, so Eq.~\eqref{eq: chainrule-minadj-cond} reduces to
\begin{equation}
I(A;\bm{H}_{A}\mid X)
=
I(A;Y\mid X)
+
I(A;\bm{H}_{A}\mid Y,X).
\end{equation}
Since conditional mutual information is nonnegative, $I(A;\bm{H}_{A}\mid Y,X)\ge 0$, we obtain
\begin{equation}
I(A;\bm{H}_{A}\mid X) \;\ge\; I(A;Y\mid X),
\end{equation}
which proves Eq.~\eqref{eqn: min-adj-info}.

\paragraph{Equality condition.}
From the identity above,
\begin{equation}
I(A;\bm{H}_{A}\mid X)=I(A;Y\mid X)
\quad\Longleftrightarrow\quad
I(A;\bm{H}_{A}\mid Y,X)=0,
\end{equation}
which is exactly Eq.~\eqref{eqn: min-adj-info-equality}.
\end{proof}

\subsection{Justification for Proposition~\ref{prop: GPB approximate the optimal representation}}
\label{ssec: proof of GPB approximate the optimal representation}

\begin{proof}
Under the stated simplifications ($\beta_c^{(\ell)}=0$ for all $\ell$ and $\beta_p^{(\ell)}=\beta$ for all $\ell$), the MC-GPB objective
(Eq.~\eqref{eqn: tighter-defense-MI}) reduces to a prediction term and a privacy term applied at a chosen layer $\ell$.
Let $\bm{Z}\triangleq \bm{H}_A^{(\ell)}$. Since a GNN layer depends on both adjacency and features, we view optimizing $\bm{\theta}$ as
inducing a conditional channel $p_{\bm{\theta}}(\bm{Z}\mid A,X)$ (with any internal randomness absorbed into the channel). The reduced
population objective therefore has the form
\begin{equation}
\label{eq: reduced-gpb-updated}
\mathcal{L}_{\mathrm{GPB}}
=
\underbrace{\mathbb{E}\bigl[-\log q_{\bm{\theta}}(Y\mid \bm{Z})\bigr]}_{\text{prediction loss}}
+
\beta \underbrace{I(\bm{Z};A\mid X)}_{\text{adjacency leakage beyond features}}
+
\text{(constants / MI-surrogate choices)}.
\end{equation}
(When $X$ is treated as fixed, as in a dataset-conditioned analysis, $I(\bm{Z};A\mid X)$ coincides with the mutual information under
that conditioning.)

\paragraph{Optimize the predictive head (population form).}
For any fixed representation $\bm{Z}$, minimizing cross-entropy over all conditional distributions $q(\cdot\mid \bm{Z})$ yields
\begin{equation}
\inf_{q}\ \mathbb{E}\bigl[-\log q(Y\mid \bm{Z})\bigr] = H(Y\mid \bm{Z}),
\end{equation}
attained by the Bayes predictor $q(Y\mid \bm{Z})=p(Y\mid \bm{Z})$. Thus, after optimizing the predictive head (or in the Bayes/population
limit where the head is sufficiently expressive and optimized), Eq.~\eqref{eq: reduced-gpb-updated} is equivalent, up to additive constants
and the particular MI surrogates, to minimizing
\begin{equation}
H(Y\mid \bm{Z}) + \beta\, I(\bm{Z};A\mid X),
\end{equation}
which is exactly the IB Lagrangian in Eq.~\eqref{eqn: IB-lagrangian}.
\end{proof}

\subsection{Proof for Theorem~\ref{thm: edgewise_MI_affinity}}
\label{app: proof_edgewise_MI_affinity}
\begin{proof}
Fix $(i,j)$ with $i<j$. By assumption, $Y_i,Y_j$ are independent with
$\mathbb{P}(Y_i=a,Y_j=b)=\pi_a\pi_b$ for all $a,b\in[C]$. Conditional on $(Y_i,Y_j)=(a,b)$,
\begin{equation}
A_{ij}\mid (Y_i=a,Y_j=b)\sim \mathrm{Bern}(B_{ab}).
\end{equation}
Let $\bar p \triangleq \mathbb{P}(A_{ij}=1)$. By the law of total probability,
\begin{equation}
\bar p
=
\sum_{a,b}\mathbb{P}(Y_i=a,Y_j=b)\,\mathbb{P}(A_{ij}=1\mid Y_i=a,Y_j=b)
=
\sum_{a,b}\pi_a\pi_b B_{ab},
\end{equation}
so marginally $A_{ij}\sim \mathrm{Bern}(\bar p)$.

By the identity $I(U;V)=\mathbb{E}_{V}\!\big[D_{\mathrm{KL}}(P_{U\mid V}\|P_U)\big]$,
\begin{equation}
I(A_{ij};Y_i,Y_j)
=
\sum_{a,b}\pi_a\pi_b\,
D_{\mathrm{KL}}\!\Big(\mathrm{Bern}(B_{ab})\,\big\|\,\mathrm{Bern}(\bar p)\Big).
\end{equation}
The characterization of the zero-information case follows since $D_{\mathrm{KL}}(P\|Q)=0$ iff $P=Q$.
\end{proof}

\subsection{Proof for Theorem~\ref{thm: global_MI_subadditivity}}
\label{app: proof_global_MI_subadditivity}
\begin{proof}
Let $Y=(Y_1,\dots,Y_N)$ and let $A_{\mathrm{up}}=\{A_{ij}:1\le i<j\le N\}$.
Fix any ordering of $\mathcal{E}\triangleq\{(i,j):1\le i<j\le N\}$, and write $A_{<e}$ for edges preceding $e$.

By the chain rule,
\begin{equation}
I(Y;A_{\mathrm{up}})
=
\sum_{e\in\mathcal{E}} I\bigl(Y;A_e \mid A_{<e}\bigr)
=
\sum_{e\in\mathcal{E}} \Bigl(H(A_e\mid A_{<e}) - H(A_e\mid Y,A_{<e})\Bigr).
\end{equation}
Conditioning reduces entropy, so $H(A_e\mid A_{<e})\le H(A_e)$.
Under conditional edge independence, $A_e \perp\!\!\!\perp A_{<e}\mid Y$, hence $H(A_e\mid Y,A_{<e})=H(A_e\mid Y)$.
Therefore, for each $e$,
\begin{equation}
I\bigl(Y;A_e \mid A_{<e}\bigr)
\le
H(A_e)-H(A_e\mid Y)
=
I(Y;A_e),
\end{equation}
and summing gives
\begin{equation}
I(Y;A_{\mathrm{up}})\le \sum_{i<j} I(Y;A_{ij}).
\end{equation}
For each $(i,j)$, the model implies the Markov chain $Y \to (Y_i,Y_j)\to A_{ij}$, so by data processing,
\begin{equation}
I(Y;A_{ij}) \le I\bigl((Y_i,Y_j);A_{ij}\bigr)=I(A_{ij};Y_i,Y_j).
\end{equation}
Combining yields
\begin{equation}
I(Y;A_{\mathrm{up}})
\le
\sum_{i<j} I(A_{ij};Y_i,Y_j).
\end{equation}
Finally, exchangeability of pairs implies all terms are equal, giving
\begin{equation}
\sum_{i<j} I(A_{ij};Y_i,Y_j)
=
\binom{N}{2}\, I(A_{12};Y_1,Y_2).
\end{equation}
\end{proof}

\subsection{Proof for Theorem~\ref{thm: edgewise_MI_quadratic_lower}}
\label{app: proof_edgewise_MI_quadratic_lower}
\begin{proof}
Assume the two-parameter SBM:
\begin{equation}
B_{aa}=p_{\mathrm{in}},\qquad B_{ab}=p_{\mathrm{out}}\ \ (a\neq b),
\end{equation}
with label prior $\pi$. Define $\alpha\triangleq \sum_a \pi_a^2$, $\bar{\alpha}\triangleq 1-\alpha$,
$\Delta\triangleq p_{\mathrm{in}}-p_{\mathrm{out}}$, and $\bar p=\alpha p_{\mathrm{in}}+\bar{\alpha} p_{\mathrm{out}}$.

If $p_{\mathrm{in}}\in\{0,1\}$ or $p_{\mathrm{out}}\in\{0,1\}$, then at least one KL term in the expression below can be $+\infty$
(unless $\bar p$ matches it), and the stated lower bound holds trivially. Hence assume $p_{\mathrm{in}},p_{\mathrm{out}},\bar p\in(0,1)$.

\paragraph{Step 1: Edgewise MI under the two-parameter form.}
From Theorem~\ref{thm: edgewise_MI_affinity},
\begin{equation}
I(A_{ij};Y_i,Y_j)
=
\alpha\, D_{\mathrm{KL}}\!\Big(\mathrm{Bern}(p_{\mathrm{in}})\,\big\|\,\mathrm{Bern}(\bar p)\Big)
+
\bar{\alpha}\, D_{\mathrm{KL}}\!\Big(\mathrm{Bern}(p_{\mathrm{out}})\,\big\|\,\mathrm{Bern}(\bar p)\Big).
\label{eq: edgewise_MI_two_param_updated}
\end{equation}

\paragraph{Step 2: A valid quadratic lower bound (Pinsker).}
Pinsker's inequality states $D_{\mathrm{KL}}(P\|Q)\ge 2\,\mathrm{TV}(P,Q)^2$.
For Bernoulli distributions, $\mathrm{TV}(\mathrm{Bern}(p),\mathrm{Bern}(q))=|p-q|$, hence
\begin{equation}
D_{\mathrm{KL}}\!\Big(\mathrm{Bern}(p)\,\big\|\,\mathrm{Bern}(q)\Big)
\;\ge\;
2(p-q)^2.
\label{eq: pinsker_bernoulli}
\end{equation}
Applying Eq.~\eqref{eq: pinsker_bernoulli} to Eq.~\eqref{eq: edgewise_MI_two_param_updated} with $q=\bar p$ yields
\begin{equation}
I(A_{ij};Y_i,Y_j)
\ge
2\alpha(p_{\mathrm{in}}-\bar p)^2 + 2\bar{\alpha}(p_{\mathrm{out}}-\bar p)^2.
\end{equation}
Using $\bar p=\alpha p_{\mathrm{in}}+\bar{\alpha} p_{\mathrm{out}}$ gives
\begin{equation}
p_{\mathrm{in}}-\bar p = \bar{\alpha}\Delta,
\qquad
p_{\mathrm{out}}-\bar p = -\alpha\Delta,
\end{equation}
so
\begin{equation}
I(A_{ij};Y_i,Y_j)
\ge
2\alpha(\bar{\alpha}\Delta)^2 + 2\bar{\alpha}(\alpha\Delta)^2
=
2\alpha\bar{\alpha}\,\Delta^2,
\end{equation}
proving the edgewise bound.

\paragraph{Step 3: Relating edgewise leakage to $I(A;Y)$ and representation leakage.}
Since $A_{ij}$ is a deterministic function of the full adjacency $A$, data processing implies
\begin{equation}
I(A;Y)\ \ge\ I(A_{ij};Y).
\end{equation}
Moreover, under the SBM/affinity model the conditional law of $A_{ij}$ depends on $Y$ only through $(Y_i,Y_j)$, hence
\begin{equation}
A_{ij}\perp\!\!\!\perp Y_{-(i,j)} \mid (Y_i,Y_j),
\end{equation}
which yields
\begin{equation}
I(A_{ij};Y)=I(A_{ij};Y_i,Y_j).
\end{equation}
\emph{Note:} This conditional independence is specific to the SBM/affinity model. For general graphs, edge presence can depend on labels of non-endpoint nodes (\eg, through higher-order structural constraints), and this equality does not hold.
Therefore,
\begin{equation}
I(A;Y)\ \ge\ I(A_{ij};Y)\ =\ I(A_{ij};Y_1,Y_2)\ \ge\ 2\alpha\bar{\alpha}\,\Delta^2.
\end{equation}
Finally, if $\bm{H}_A$ is task-sufficient for $Y$ with respect to $A$ (in the sense used in
Theorem~\ref{theorem: minimum adjacency information}), then $I(A;\bm{H}_A\mid X)\ge I(A;Y\mid X)$.
When $X$ is constant or absent (as in the theorem statement), $I(A;Y\mid X)=I(A;Y)$ and $I(A;\bm{H}_A\mid X)=I(A;\bm{H}_A)$ (since $X$ is non-random), so the theorem yields $I(A;\bm{H}_A)\ge I(A;Y)$. Combining with $(i,j)=(1,2)$ by exchangeability yields the claimed chain.

\paragraph{Step 4: Symmetry between homophily and heterophily.}
The bound depends on $\Delta$ only through $\Delta^2$, hence is invariant to the sign of $\Delta$.

\paragraph{Extension to general $K$-parameter block models.}
The two-parameter SBM is restrictive; most real graphs have heterogeneous community structure with distinct affinity levels. For a general symmetric $B\in[0,1]^{C\times C}$, Theorem~\ref{thm: edgewise_MI_affinity} still applies and gives
$I(A_{ij};Y_i,Y_j)=\sum_{a,b}\pi_a\pi_b\,D_{\mathrm{KL}}(\mathrm{Bern}(B_{ab})\|\mathrm{Bern}(\bar p))$.
The Pinsker lower bound generalizes to
\begin{equation}
I(A_{ij};Y_i,Y_j)\;\ge\; 2\sum_{a,b}\pi_a\pi_b(B_{ab}-\bar p)^2
\;=\;
2\,\mathrm{Var}_{\pi\otimes\pi}[B_{Y_i Y_j}],
\end{equation}
which is twice the weighted variance of the affinity matrix entries $B_{ab}$ around the marginal edge probability $\bar p=\sum_{a,b}\pi_a\pi_b B_{ab}$, with weights given by the product class prior $\pi_a\pi_b$. Qualitatively, the more heterogeneous the block structure (\ie, the more $B_{ab}$ varies across label pairs), the larger this variance and the higher the irreducible leakage floor.
\end{proof}

%% file: appendix/2-full-empirical-study.tex
\section{Full empirical study}

\subsection{Full quantitative results}
\label{sec: full quantitative results}

\textbf{A further comparison of attack methods.}
We compare MC-GRA~(+) against three baselines: the best single-MI-term proxy, the uniform ensemble, and GraphMI (Tab.~\ref{tab: app-attack-comparison}). Evaluation uses edge-recovery AUC; higher is better. MC-GRA~(+) attains the best AUC on all six datasets, with gains of up to 15.0 percentage points over GraphMI (Citeseer) and up to 7.2 points over the ensemble (Polblogs).

\begin{table}[H]
\centering
\setlength\tabcolsep{10pt}
\vspace{-6pt}
\begin{tabular}{c|cccccc}
\toprule
dataset & Cora  & Citeseer &  Polblogs   & USA   &   Brazil &  AIDS     \\
\midrule
single MI (Tab.~\ref{tab: understanding-MI-term-gcn}) & .815     & .881     & .763     & .850     & .758     & .584     \\
ensemble (Tab.~\ref{tab: understanding-MI-term-ensemble-gcn})      & .849     & .907     & .781     & .852     & .717     & .522     \\
GraphMI          & .812     & .781     & .791     & .769     & .680     & .575     \\
MC-GRA~(+) (Tab.~\ref{tab: exp-attack-results-gcn})        &\textbf{ .904} & \textbf{.931} & \textbf{.853} & \textbf{.870} & \textbf{.760} & \textbf{.588}  \\
\bottomrule
\end{tabular}
\caption{A further quantitative comparison of attack methods (with AUC metric).}
\label{tab: app-attack-comparison}
\end{table}

\textbf{A further comparison of defense methods.}
We compare MC-GPB~(+) with two baselines: (i)~adding Gaussian noise to the model \emph{output} at inference time (random noise), and (ii)~differential privacy~\citep{abadi2016deep} (DP-SGD: Gaussian noise on clipped gradients each iteration; noise multiplier $1.0$, clip norm $1.0$, $\delta=1/N$ with $N$ the number of training nodes; exact $(\varepsilon,\delta)$-DP bounds depend on iterations and dataset size). Tab.~\ref{tab: app-defense-comparison} reports results with GraphMI~\citep{zhang2021graphmi} as the attack; the main text (Sec.~\ref{sec: experiments}) evaluates MC-GPB~(+) against MC-GRA~(+) in Tabs.~\ref{tab: exp-gra-gpb-results-gcn}--\ref{tab: exp-gra-gpb-results-gprgnn}. The baselines reduce attack AUC but at the price of sharply reducing model accuracy. By contrast, MC-GPB~(+) reduces attack AUC substantially while maintaining much higher accuracy.

\begin{table}[H]
\centering
\fontsize{8}{8}\selectfont
\renewcommand{\arraystretch}{1.2} 
\setlength\tabcolsep{2pt}
\begin{tabular}{c|cc|cc|cc|cc|cc|cc}
\toprule
\multirow{2}{*}{dataset} &  \multicolumn{2}{c|}{{    Cora    }}  & \multicolumn{2}{c|}{Citeseer} &  \multicolumn{2}{c|}{Polblogs}   & \multicolumn{2}{c|}{USA}   &   \multicolumn{2}{c|}{Brazil} &  \multicolumn{2}{c}{AIDS}     \\
  & Acc.$\uparrow$ & AUC$\downarrow$  & Acc.$\uparrow$ & AUC$\downarrow$  & Acc.$\uparrow$ & AUC$\downarrow$  & Acc.$\uparrow$ & AUC$\downarrow$  & Acc.$\uparrow$ & AUC$\downarrow$  & Acc.$\uparrow$ & AUC$\downarrow$  \\
\midrule
No defense                 & .757 & .812     & .630 & .781     & .833 & .791     & .470 & .769     & .769 & .680     & .668 & .575     \\
Random noise               & .620 & .657     & .570 & .727     & .802 & .759     & .440 & .754     & .634 & .713     & .572 & .559     \\
Differential privacy     & .315 & .500     & .224 & .500     & .553 & .502     & .263 & .500     & .423 & .706     & .131 & .502     \\
MC-GPB~(+)               & .734 & .625 & .602 & .691 & .830 & .506 & .391 & .300 & .808 & .609 & .668 & .514 \\
\bottomrule
\end{tabular}
\caption{A further quantitative comparison of defense methods.}
\label{tab: app-defense-comparison}
\end{table}
\paragraph{Limitations.} On datasets marked with $\dagger$ in Tab.~\ref{tab: exp-defense-results-gprgnn} (main text), MC-GPB~(+) did not consistently improve over the unprotected baseline under default hyperparameters. In these cases, the privacy regularizer appeared to dominate the utility term at the default strength, causing accuracy to drop without proportional leakage reduction; grid-searching over $\beta_p$ and $\beta_c$ recovered better trade-offs. Hyperparameters are selected on the validation split according to a fixed privacy–utility criterion; the full grid is provided in the appendix/supplement.

To illustrate the sensitivity landscape, Tab.~\ref{tab: app-grid-search-cora-gprgnn} shows how defense performance on Cora under GPR-GNN (a $\dagger$-marked configuration) varies across five values of $\beta_p$ with fixed $\beta_c=0.1$. The privacy--utility trade-off is clearly visible: larger $\beta_p$ reduces the leakage proxy $R_{\mathrm{AUC}}(A;\bm{H}_A^{\mathrm{all}})$ but can also decrease accuracy. The default $\beta_p=0.5$ provides a balanced operating point; smaller values preserve more accuracy at the cost of higher leakage.

\begin{table}[H]
\centering
\setlength\tabcolsep{8pt}
\vspace{-6pt}
\begin{tabular}{c|ccccc}
\toprule
$\beta_p$ & 0.01 & 0.05 & 0.1 & 0.5 & 1.0 \\
\midrule
Acc.\,$\uparrow$ & .78 & .76 & .75 & .74 & .72 \\
$R_{\mathrm{AUC}}(A;\bm{H}_A^{\mathrm{all}})\,\downarrow$ & .63 & .59 & .57 & .55 & .52 \\
\bottomrule
\end{tabular}
\caption{Grid-search sensitivity for MC-GPB~(+) on Cora under GPR-GNN ($\dagger$-marked). Fixed $\beta_c=0.1$; varying $\beta_p$.}
\label{tab: app-grid-search-cora-gprgnn}
\end{table}

\textbf{A further empirical study about the evaluation metric.}
We use AUC as a direct, estimation-free recoverability proxy for edge recovery (with ground-truth edges in $A$), rather than estimating mutual information. Alternatively, the metric could be replaced by AP (Average Precision), MRR (Mean Reciprocal Rank), or Hit@K; we use AUC for consistency with the main text.

To discriminate between link types, we consider link homophily. Formally, \emph{homogeneous} links are $\{e_{ij}: y_i = y_j\}$ and \emph{heterogeneous} links are $\{e_{ij}: y_i \neq y_j\}$, where $y_i,y_j$ are node labels. In many settings, cross-class (heterogeneous) links are considered more privacy-sensitive because they can reveal unexpected relationships; sensitivity may depend on the application. In what follows, we investigate the effectiveness of GRA on these two link types.

\textbf{(1)} For the attack, homogeneous links are much easier to recover (Tab.~\ref{tab: app-attack-homo-hetero}). Heterogeneous links are harder to recover but can still be recovered to some extent.

\begin{table}[H]
\centering
\fontsize{10}{10}\selectfont
\renewcommand{\arraystretch}{1.2} 
\setlength\tabcolsep{2pt}
\begin{tabular}{c|cccccc}
\toprule
dataset &  {    Cora    }  & Citeseer &  Polblogs   & USA   &   Brazil &  AIDS     \\
\midrule
MC-GRA~(+) (Homogeneous links)    & .960 & .917     & .896     & .951 & .891   &  .585  \\
MC-GRA~(+) (Heterogeneous links)  & .684 & .861     & .298     & .716 & .564   &  .551  \\
GraphMI (Homogeneous links)   & .724 & .799     & .717     & .919 & .871   &  .707  \\
GraphMI (Heterogeneous links) & .569 & .675     & .391     & .666 & .728   &  .437  \\
\bottomrule
\end{tabular}
\caption{Attack comparison on homogeneous or heterogeneous links (AUC metric).}
\label{tab: app-attack-homo-hetero}
\end{table}

\textbf{(2)} For defense, we apply MC-GPB~(+) to protect GCN against GraphMI, and a revised version MC-GPB-hetero that only penalizes leakage of heterogeneous links (i.e., the defense objective restricts to cross-class edges). Tab.~\ref{tab: app-defense-homo-hetero} shows that recovery of heterogeneous links is significantly reduced by MC-GPB and further reduced by MC-GPB-hetero, so both are capable of protecting heterogeneous links.

\begin{table}[H]
\centering
\fontsize{10}{10}\selectfont
\renewcommand{\arraystretch}{1.2} 
\setlength\tabcolsep{2pt}
\begin{tabular}{c|cccccc}
\toprule
dataset &  {    Cora    }  & Citeseer &  Polblogs   & USA   &   Brazil &  AIDS     \\
\midrule
No defense (Heterogeneous links)         & .569 & .675    & .391    & .666 & .728  & .437  \\
MC-GPB~(+) (Heterogeneous links)        & .532 & .584    & .453    & .552 & .530  & .471  \\
MC-GPB-hetero (Heterogeneous links) & .493 & .515    & .210    & .494 & .416  & .423  \\
\bottomrule
\end{tabular}
\caption{Defense comparison on heterogeneous links: MC-GPB and MC-GPB-hetero (AUC metric).}
\label{tab: app-defense-homo-hetero}
\end{table}

\textbf{Empirical results on larger-scale datasets.}
We conduct an empirical study on two larger-scale datasets (relative to the main benchmarks), \ie,
ENZYME (6254 nodes, 23914 edges) and an OGB-Arxiv subgraph (8532 nodes, 26281 edges).
Full OGB-Arxiv has $\sim$170K nodes; we use GraphSAINT~\citep{zeng2019graphsaint} random walk sampler (walk length $=4$, number of roots $=2000$) to extract subgraphs of manageable size ($N \sim 6000$--$8000$ nodes) for the dense adjacency-reconstruction experiments. 
Detailed statistics are below; dataset split (train/validate/test) is consistent with the rest of the paper.

\begin{table}[H]
\centering
\setlength\tabcolsep{10pt}
\vspace{-6pt}
\begin{tabular}{c|ccccc}
\toprule
dataset &  \# Nodes  & \# Edges  &  \# Class   & \# Features   &   Edge homophily     \\
\midrule
ENZYME    & 6254    & 23914   & 3         & 18          & 0.629   \\
OGB-Arxiv & 8532    & 26281   & 40        & 128         & 0.618 \\
\bottomrule
\end{tabular}
\caption{Dataset statistics of the two larger-scale datasets. Edge homophily is the fraction of edges connecting same-class nodes (as in Tab.~\ref{tab: statistics}).}
\end{table}

\textbf{(1)}
We evaluate MC-GRA~(+) with $\mathcal{K} = \{X, Y\}$ on the larger-scale datasets (Tab.~\ref{tab: app-large-attack}). MC-GRA~(+) is effective on both and outperforms the baseline by a large margin.

\begin{table}[H]
\centering
\setlength\tabcolsep{10pt}
\vspace{-6pt}
\begin{tabular}{c|cc}
\toprule
dataset &  ENZYME &  OGB-Arxiv     \\
\midrule
GraphMI		& .494  &  .828		\\
MC-GRA~(+)     & .761   &  .891      \\
\bottomrule
\end{tabular}
\caption{Comparison of attack methods on larger-scale datasets (AUC metric).}
\label{tab: app-large-attack}
\end{table}

\textbf{(2)}
We evaluate MC-GPB~(+) on these two datasets with GraphMI as the attack (Tab.~\ref{tab: app-large-defense}). MC-GPB~(+) reduces attack AUC for both MC-GRA~(+) and GraphMI, demonstrating effectiveness of the defense on larger-scale datasets.

\begin{table}[H]
\centering
\setlength\tabcolsep{10pt}
\vspace{-6pt}
\begin{tabular}{c|cc}
\toprule
dataset &  ENZYME &  OGB-Arxiv     \\
\midrule
GraphMI		& .488 (1.2\%$\downarrow$)  &  .533  (35.6\%$\downarrow$) 	\\
MC-GRA~(+)    &  .607 (20.2\%$\downarrow$) & .848 (4.8\%$\downarrow$)     \\
\bottomrule
\end{tabular}
\caption{Attack AUC when the model is protected by MC-GPB~(+) on larger-scale datasets, under GraphMI and MC-GRA~(+) (lower is better).}
\label{tab: app-large-defense}
\end{table}
The defense effectiveness varies across attacks: on ENZYME, MC-GPB~(+) substantially reduces MC-GRA~(+) AUC (20.2\%$\downarrow$) but barely affects GraphMI (1.2\%$\downarrow$; GraphMI's baseline is already near chance). On OGB-Arxiv the pattern reverses (MC-GRA~(+) 4.8\%$\downarrow$ vs.\ GraphMI 35.6\%$\downarrow$). This suggests that the defense and each attack exploit different structural cues.

The main-text results without node features (Tab.~\ref{tab: understanding-MI-term-ensemble-gcn} and Tab.~\ref{tab: no-X-attack} in Sec.~\ref{sec: experiments}) show that MC-GRA~(+) remains effective when $X$ is omitted; adjacency can be recovered from labels, hidden representations, and predictions alone. Node features do not always exist (\eg, Polblogs, USA, Brazil), so this setting is practically relevant.

\textbf{Why does classification accuracy sometimes improve under MC-GPB~(+)?}
Tab.~\ref{tab: exp-defense-results-gcn} shows that MC-GPB~(+) can improve classification accuracy on some datasets. We conjecture that suppressing certain structural noise can act as a beneficial regularizer, analogous to findings in the interpretable GNN literature~\citep{miao2022interpretable,zhao2022learning}. By reducing $I(\bm{H}_A;A)$ while maintaining $I(\bm{H}_A;Y)$, the defense can also reduce $I(\bm{H}_A;A|Y)$ (the excess adjacency information beyond what is needed for the task; see Theorem~\ref{theorem: minimum adjacency information}), which can help test-time inference by removing spurious structural correlations. When such correlations are mild, the defense typically incurs an accuracy--privacy trade-off and accuracy may drop.

\subsection{Full qualitative results}
\label{sec: full qualitative results}

\textbf{The recovered adjacency.}
Figs.~\ref{appd: adj:cora}--\ref{appd: adj:polblogs} show the recovered adjacency for each dataset, grouped by node label under different attack strategies.
Subfigures (a)--(e) in each figure are: (a)~ground truth; (b)~MC-GRA~(+) on unprotected GNN; (c)~GraphMI on unprotected GNN; (d)~MC-GRA~(+) on MC-GPB~(+)-protected GNN; (e)~GraphMI on MC-GPB~(+)-protected GNN. Green dots indicate correctly predicted edges (true positives and true negatives); red dots indicate errors (false positives and false negatives).
Training with MC-GPB~(+) resists both attacks, as evidenced by substantially more incorrect edge predictions compared to the unprotected model.
For Cora, MC-GRA~(+) (Fig.~\ref{appd: adj:cora:gra_normal}) achieves better results than GraphMI (Fig.~\ref{appd: adj:cora:gmi_normal}) under unprotected training (fewer red dots). MC-GPB~(+) defends against both MC-GRA~(+) (Fig.~\ref{appd: adj:cora:gra_protected}) and GraphMI (Fig.~\ref{appd: adj:cora:gmi_protected}); MC-GRA~(+) still achieves higher reconstruction AUC than GraphMI even under the protected model.

\begin{figure}[H]
\centering
\hfill		
\subfloat[Ground truth]
{{\includegraphics[width=0.18\linewidth]{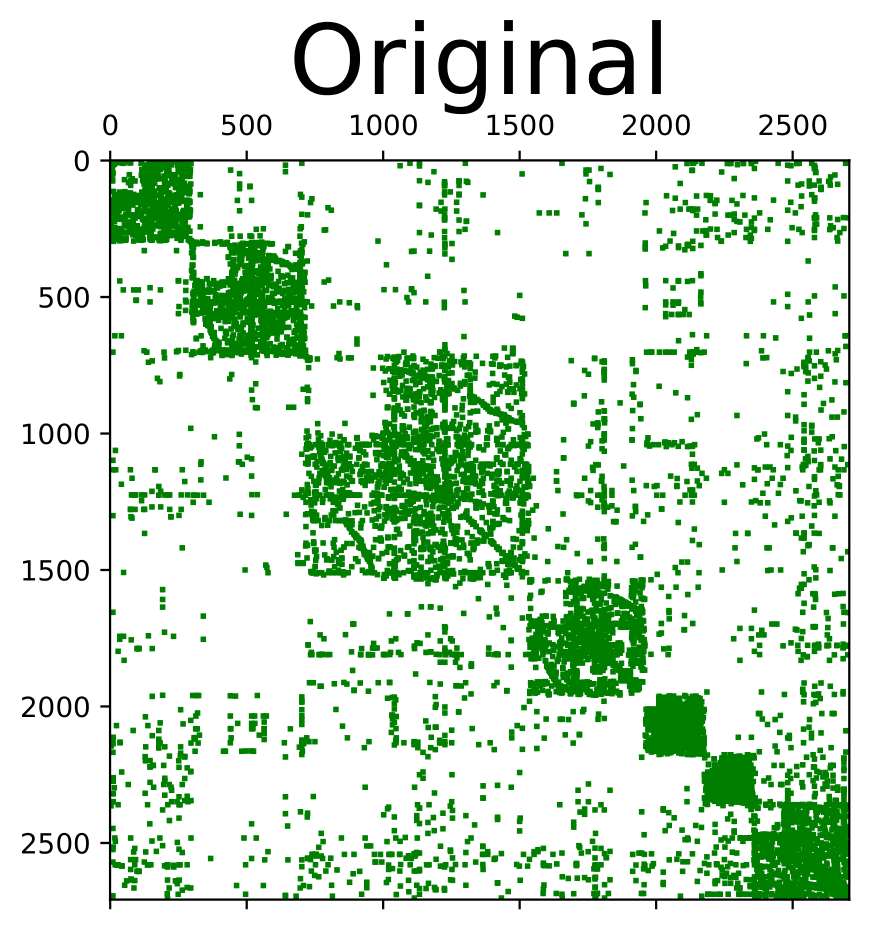}}}
\hfill
\subfloat[MC-GRA~(+), unprotected]
{{
\includegraphics[width=0.18\linewidth]{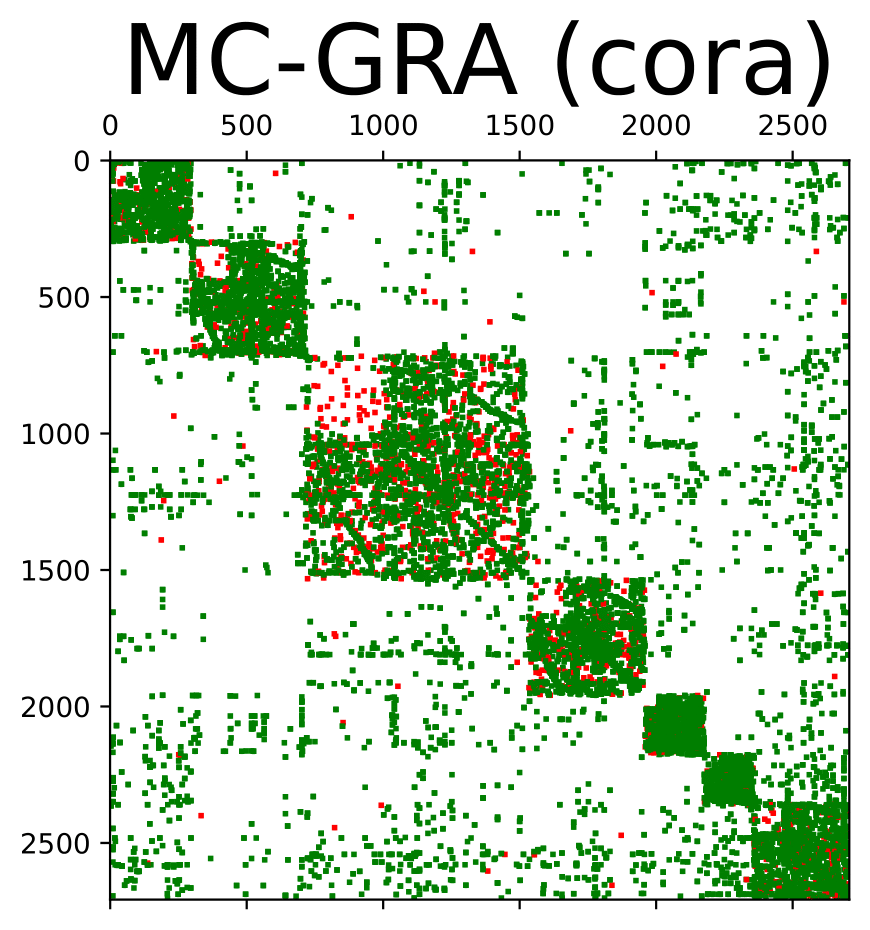} 
\label{appd: adj:cora:gra_normal}
}}
\hfill
\subfloat[GraphMI, unprotected]
{{
\includegraphics[width=0.18\linewidth]{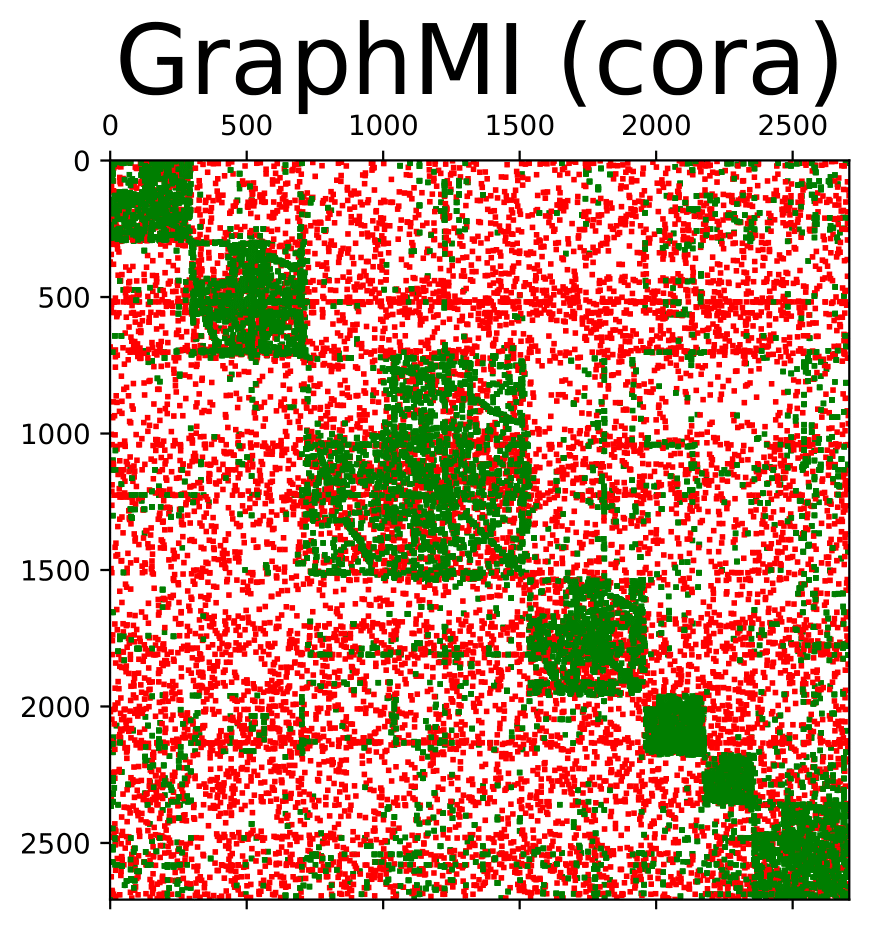}
\label{appd: adj:cora:gmi_normal}
}}
\hfill
\subfloat[MC-GRA~(+), protected]
{{
\includegraphics[width=0.18\linewidth]{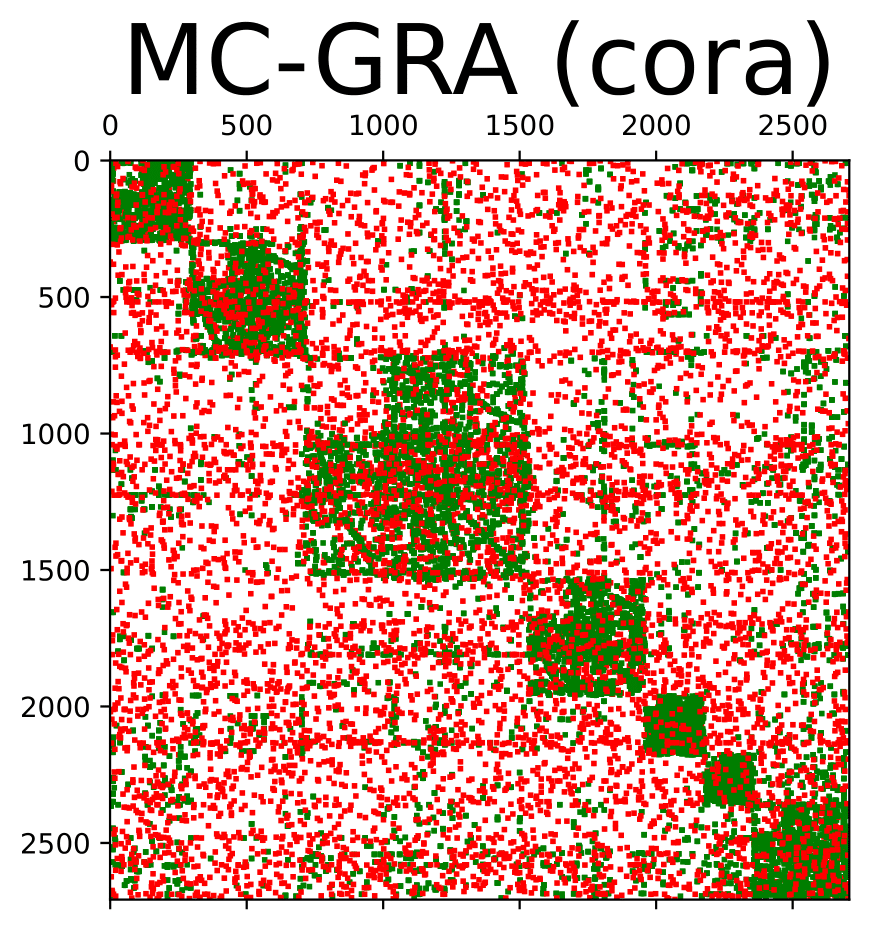}
\label{appd: adj:cora:gra_protected}
}}
\hfill
\subfloat[GraphMI, protected]
{{
\includegraphics[width=0.18\linewidth]{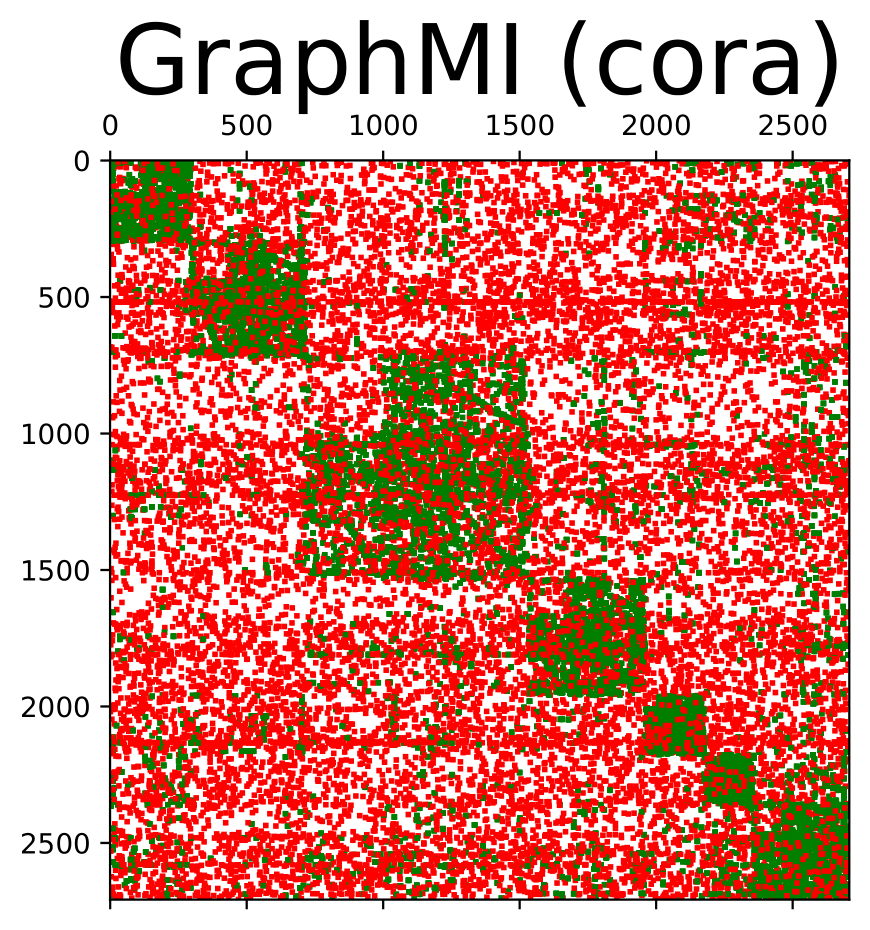}
\label{appd: adj:cora:gmi_protected}
}}
\vspace{-8pt}
\caption{Recovered adjacency on Cora. (a)~Ground truth. (b)~MC-GRA~(+) on unprotected GNN. (c)~GraphMI on unprotected GNN. (d)~MC-GRA~(+) on MC-GPB~(+)-protected GNN. (e)~GraphMI on MC-GPB~(+)-protected GNN. Green: correct (true positives and true negatives); red: errors (false positives and false negatives).}
\label{appd: adj:cora}
\end{figure}

\begin{figure}[H]
\centering
\subfloat[Ground truth]
{{\includegraphics[width=0.18\linewidth]{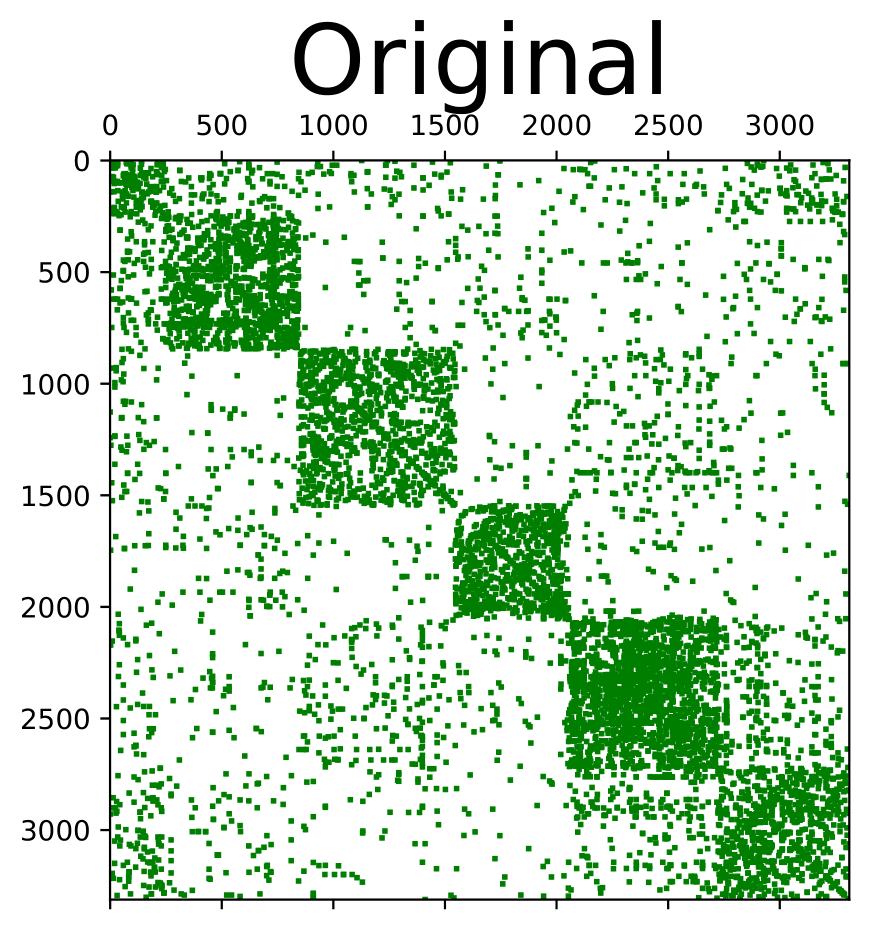}}}
\hfill
\subfloat[MC-GRA~(+), unprotected]
{{\includegraphics[width=0.18\linewidth]{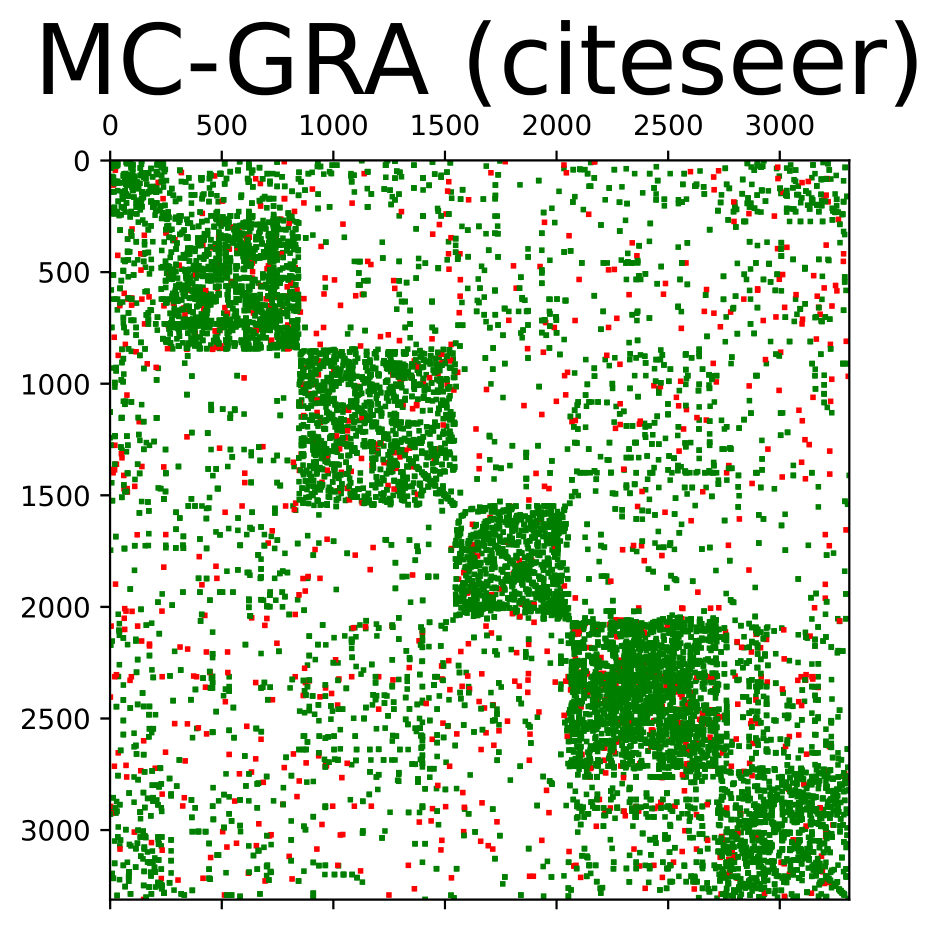}}}
\hfill
\subfloat[GraphMI, unprotected]
{{\includegraphics[width=0.18\linewidth]{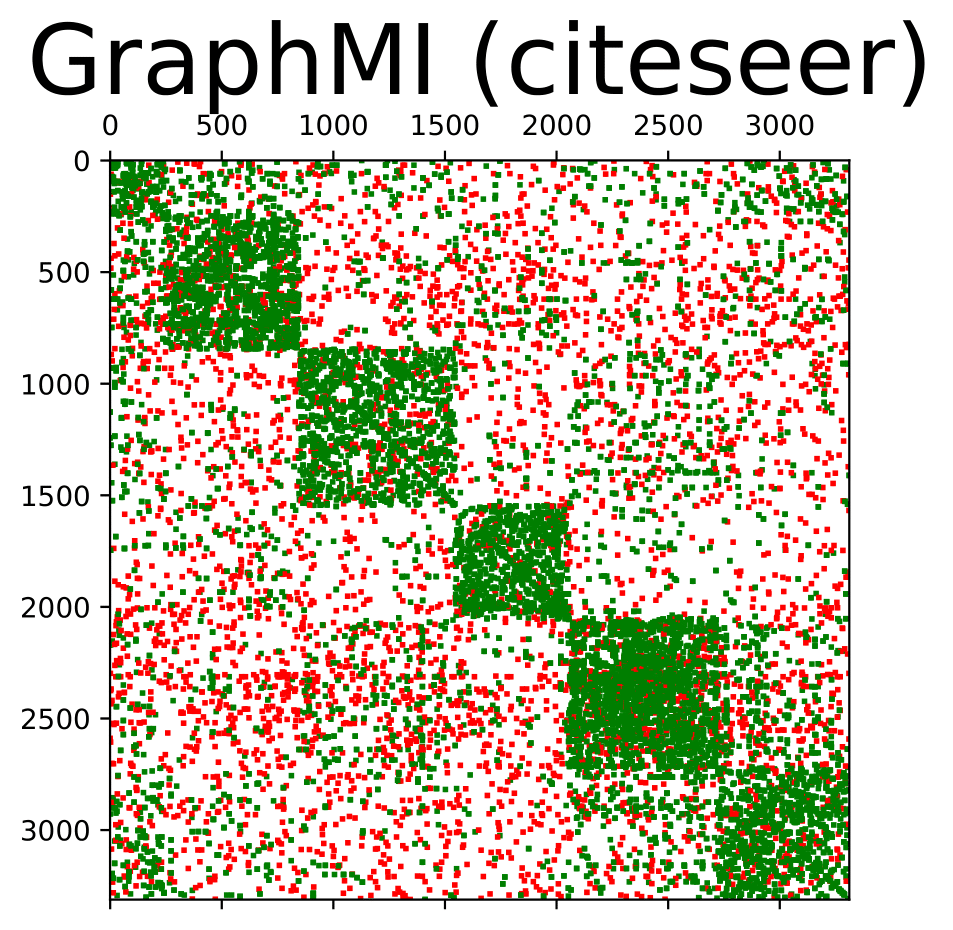}}}
\hfill
\subfloat[MC-GRA~(+), protected]
{{\includegraphics[width=0.18\linewidth]{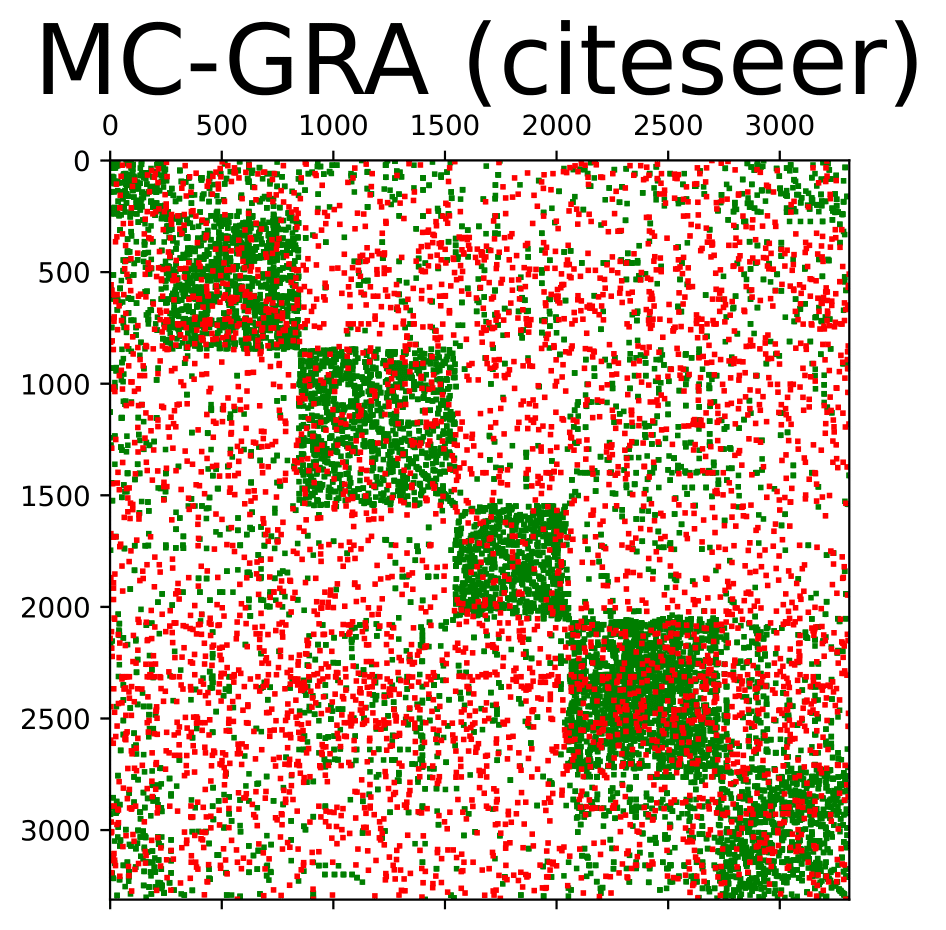}}}
\hfill
\subfloat[GraphMI, protected]
{{\includegraphics[width=0.18\linewidth]{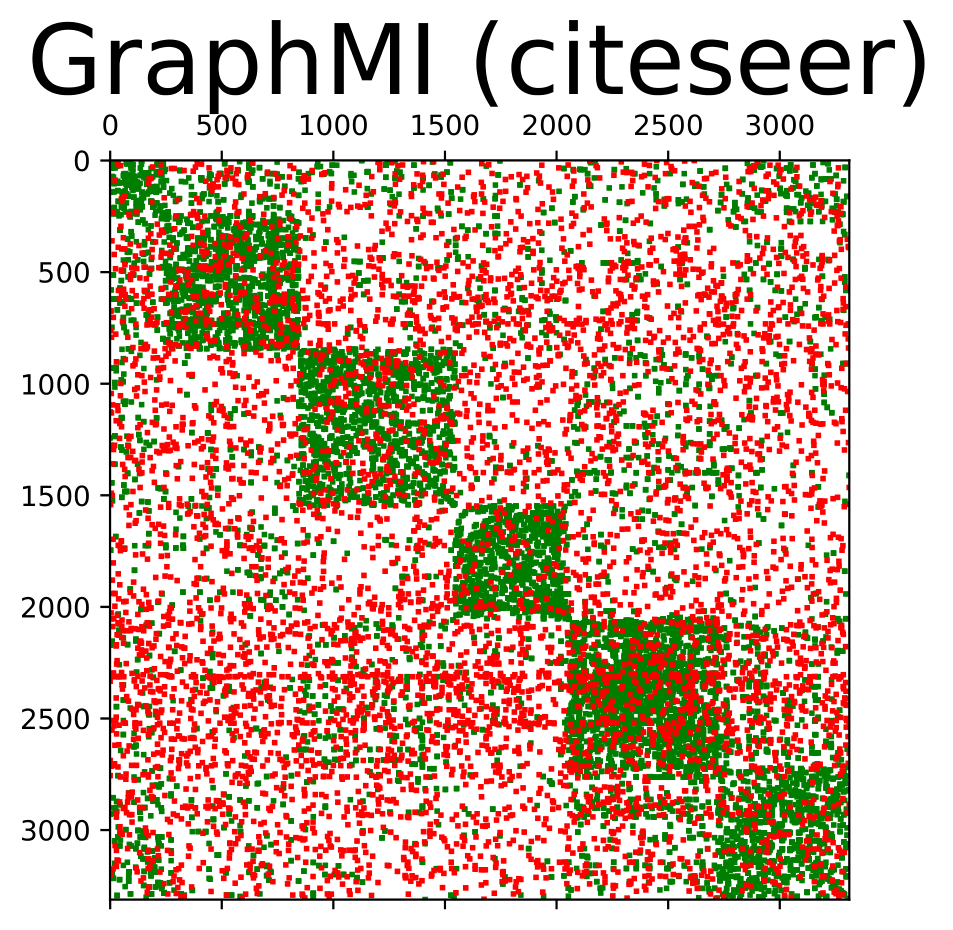}}}
\vspace{-8pt}
\caption{Recovered adjacency on Citeseer. (a)--(e) as in Fig.~\ref{appd: adj:cora}. Green: correct; red: errors.}
\label{appd: adj:citeseer}
\end{figure}

\begin{figure}[H]
\centering
\subfloat[Ground truth]
{{\includegraphics[width=0.18\linewidth]{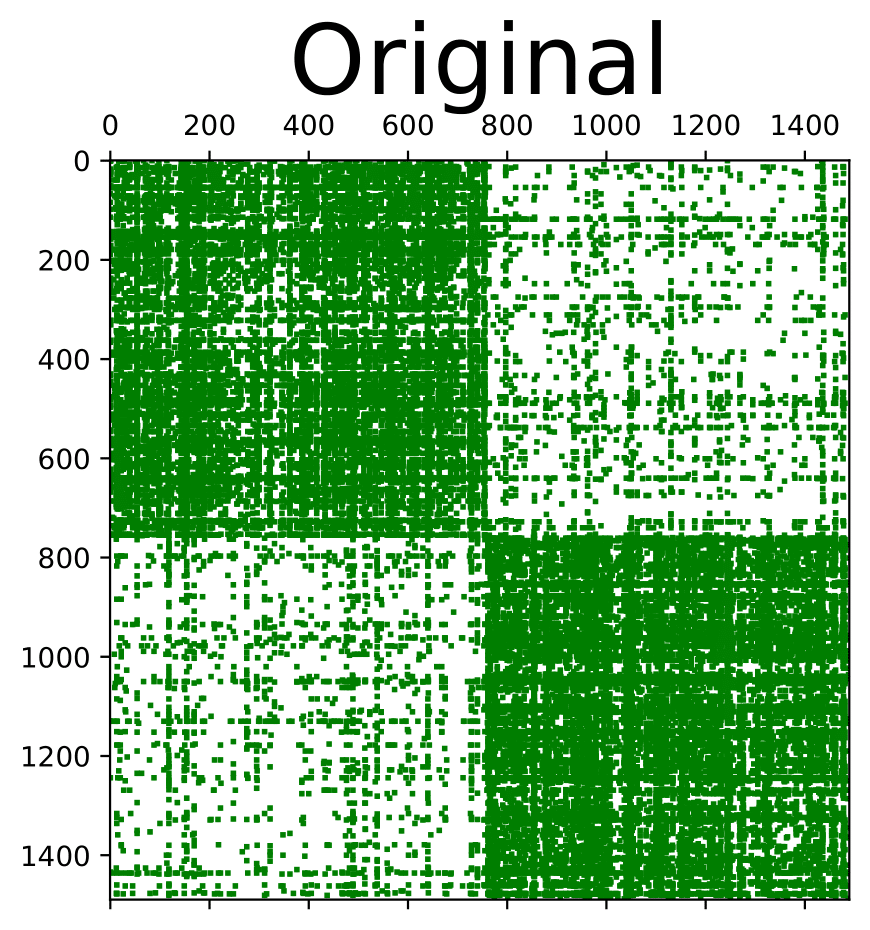}}}
\hfill
\subfloat[MC-GRA~(+), unprotected]
{{\includegraphics[width=0.18\linewidth]{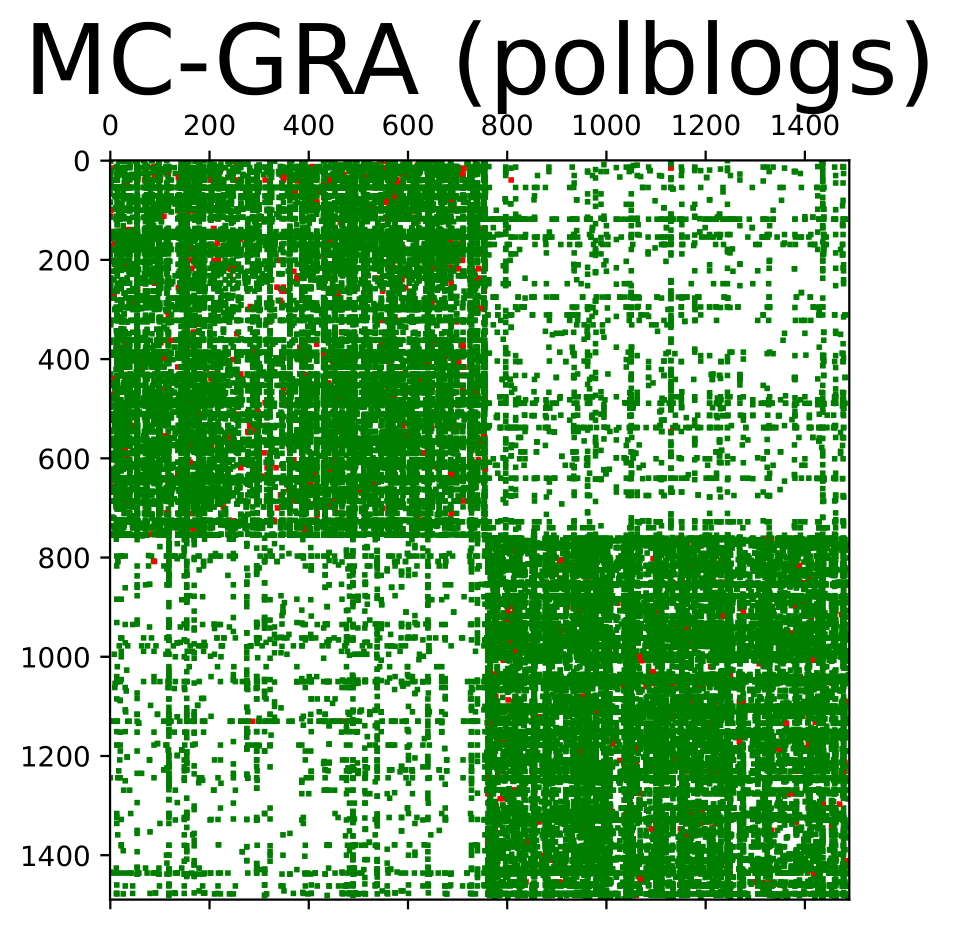}}}
\hfill
\subfloat[GraphMI, unprotected]
{{\includegraphics[width=0.18\linewidth]{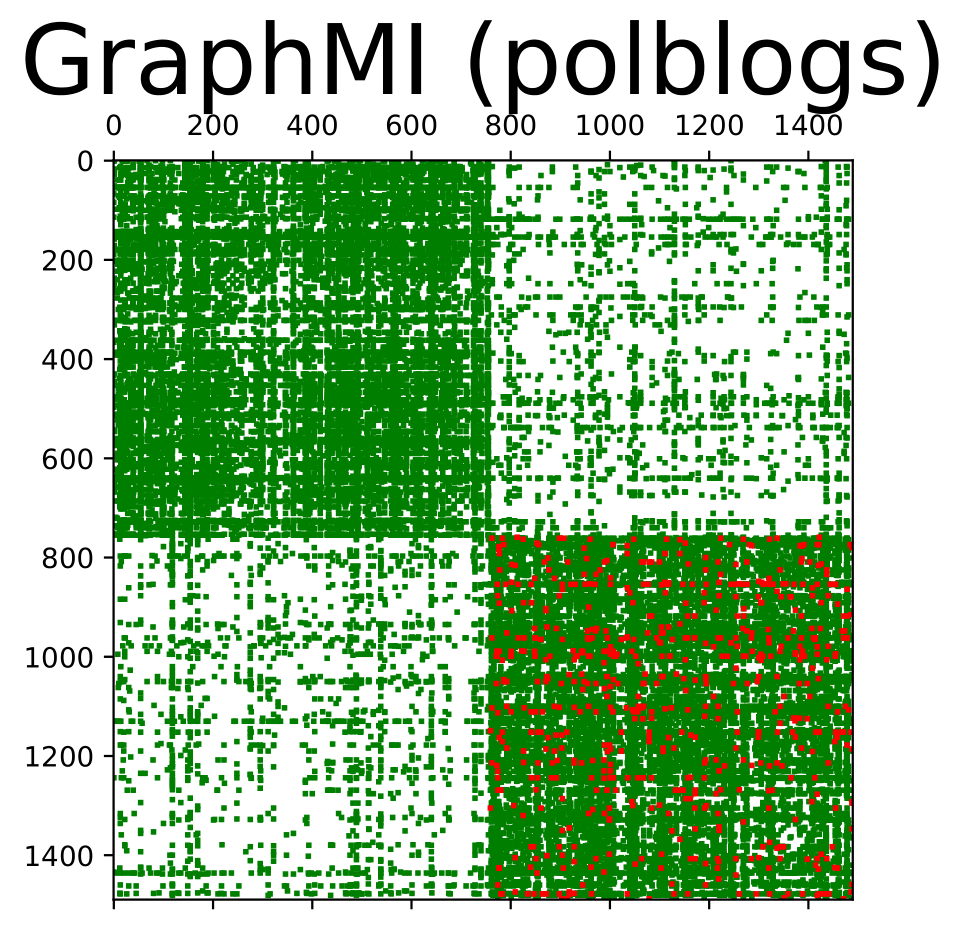}}}
\hfill
\subfloat[MC-GRA~(+), protected]
{{\includegraphics[width=0.18\linewidth]{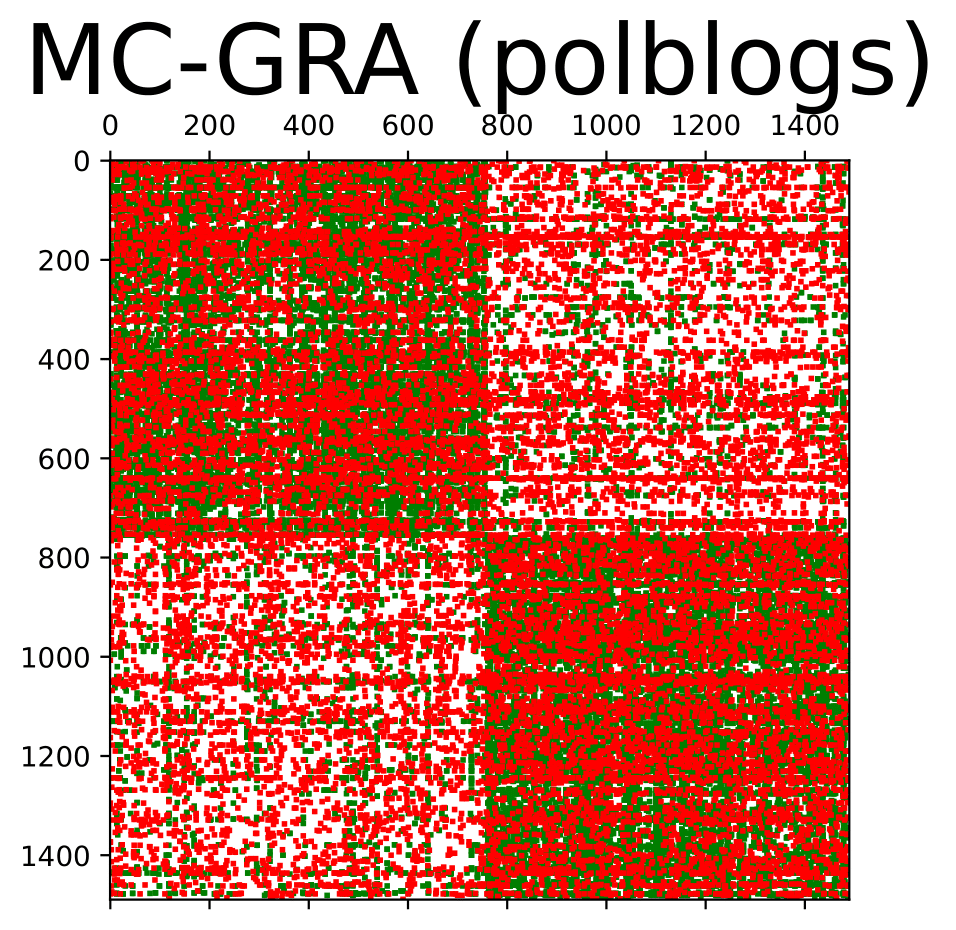}}}
\hfill
\subfloat[GraphMI, protected]
{{\includegraphics[width=0.18\linewidth]{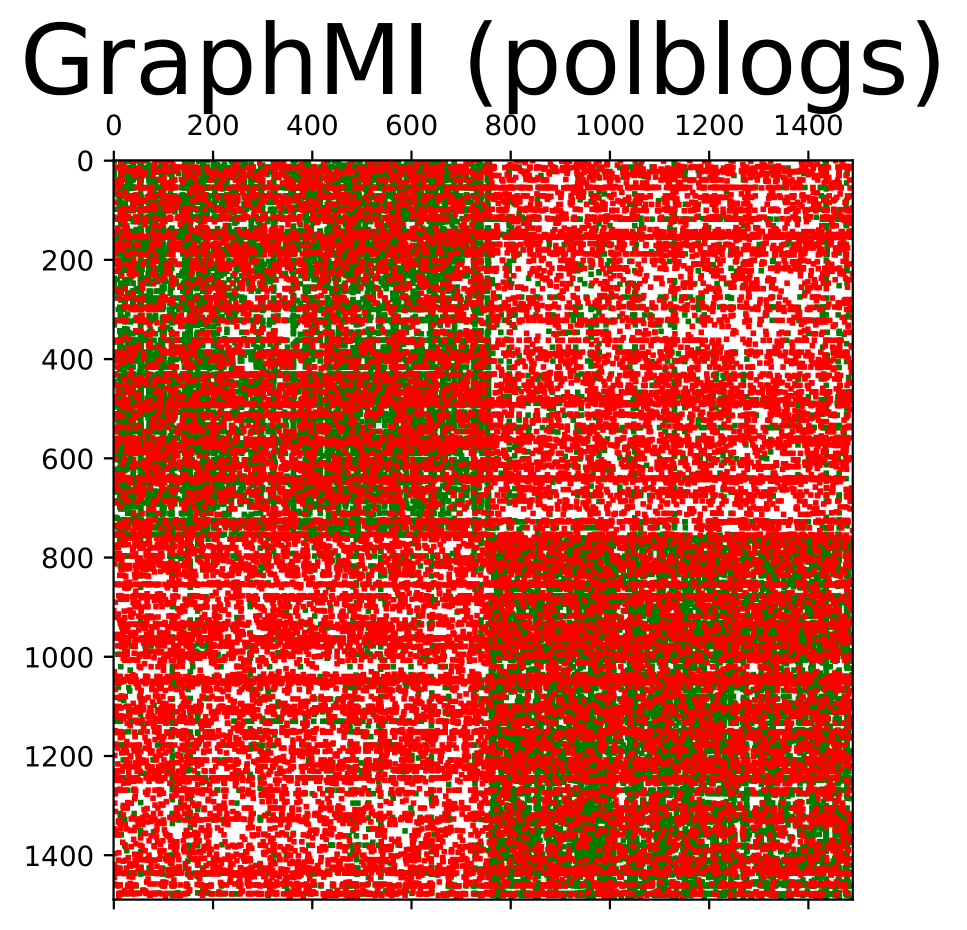}}}
\vspace{-8pt}
\caption{Recovered adjacency on Polblogs. (a)--(e) as in Fig.~\ref{appd: adj:cora}. Green: correct; red: errors.}
\label{appd: adj:polblogs}
\end{figure}

\begin{figure}[H]
\centering
\subfloat[Ground truth]
{{\includegraphics[width=0.18\linewidth]{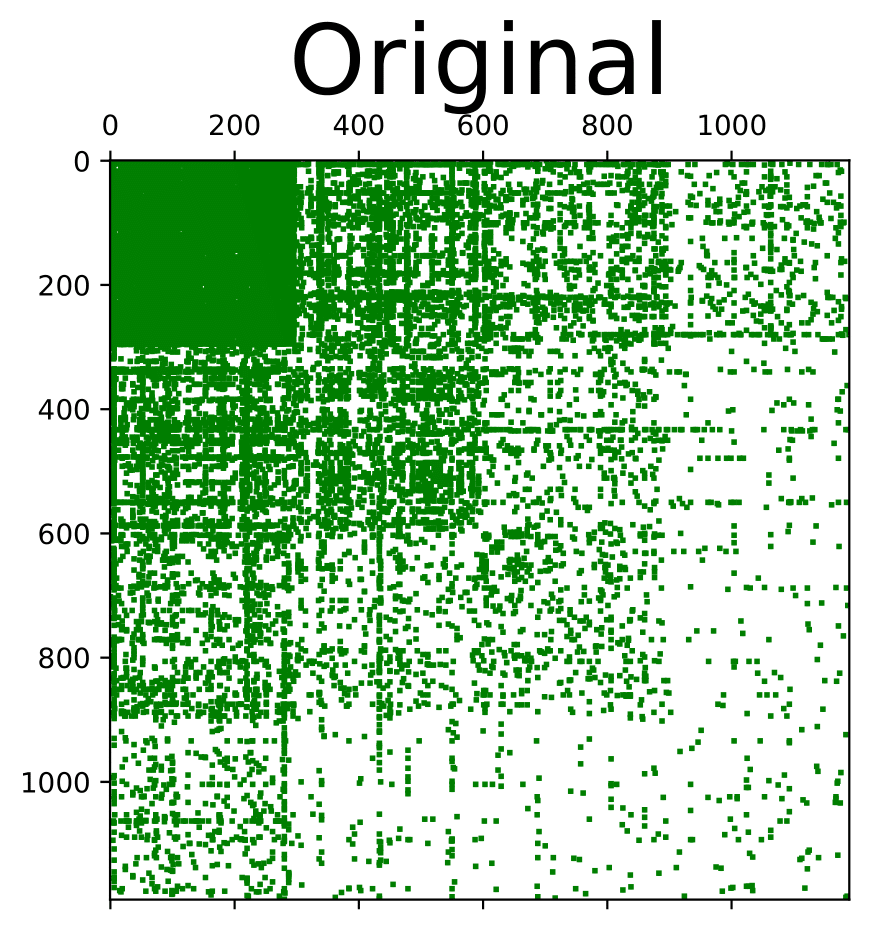}}}
\hfill
\subfloat[MC-GRA~(+), unprotected]
{{\includegraphics[width=0.18\linewidth]{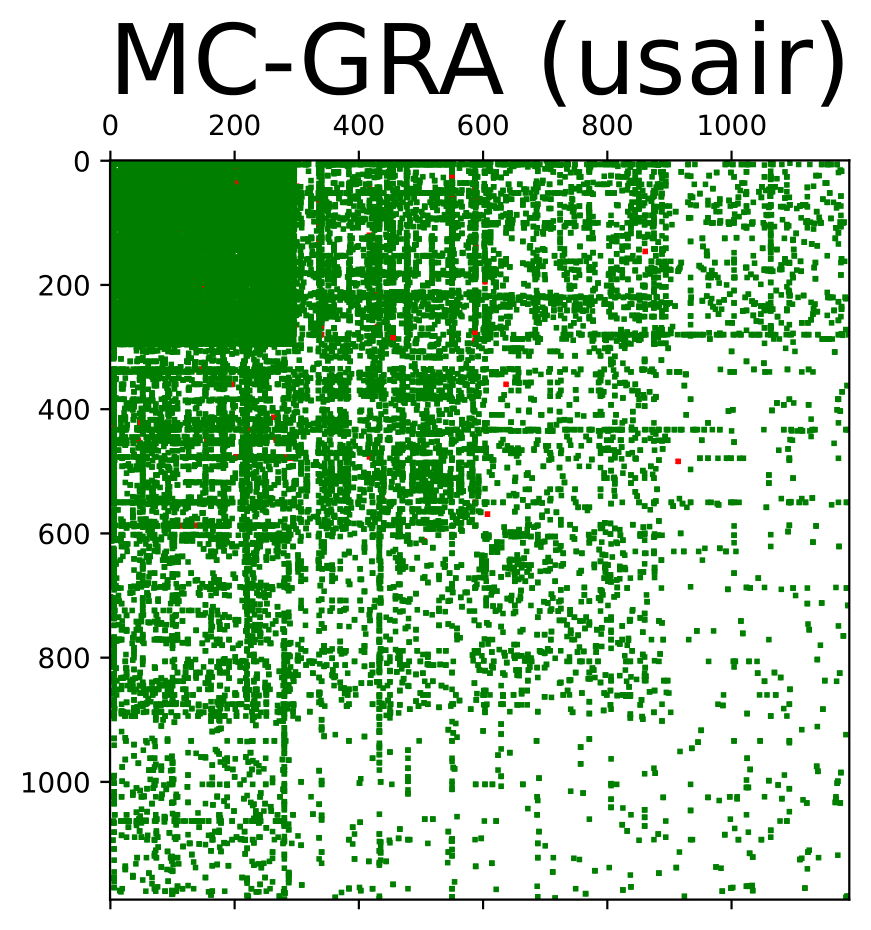}}}
\hfill
\subfloat[GraphMI, unprotected]
{{\includegraphics[width=0.18\linewidth]{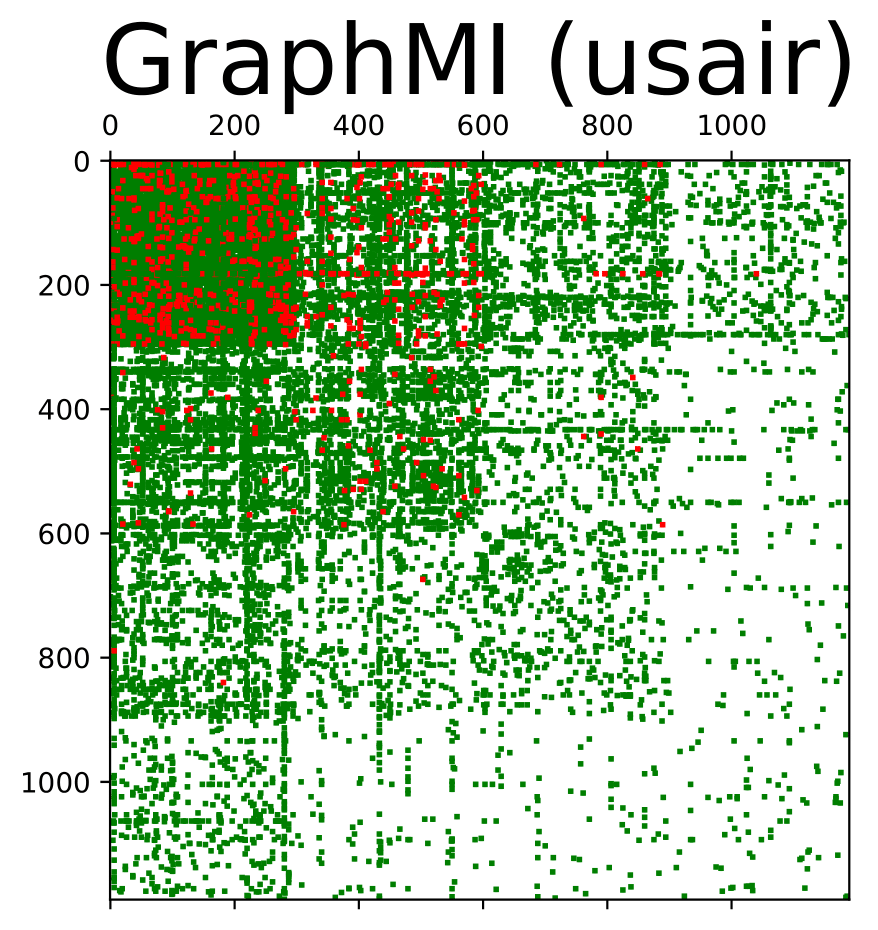}}}
\hfill
\subfloat[MC-GRA~(+), protected]
{{\includegraphics[width=0.18\linewidth]{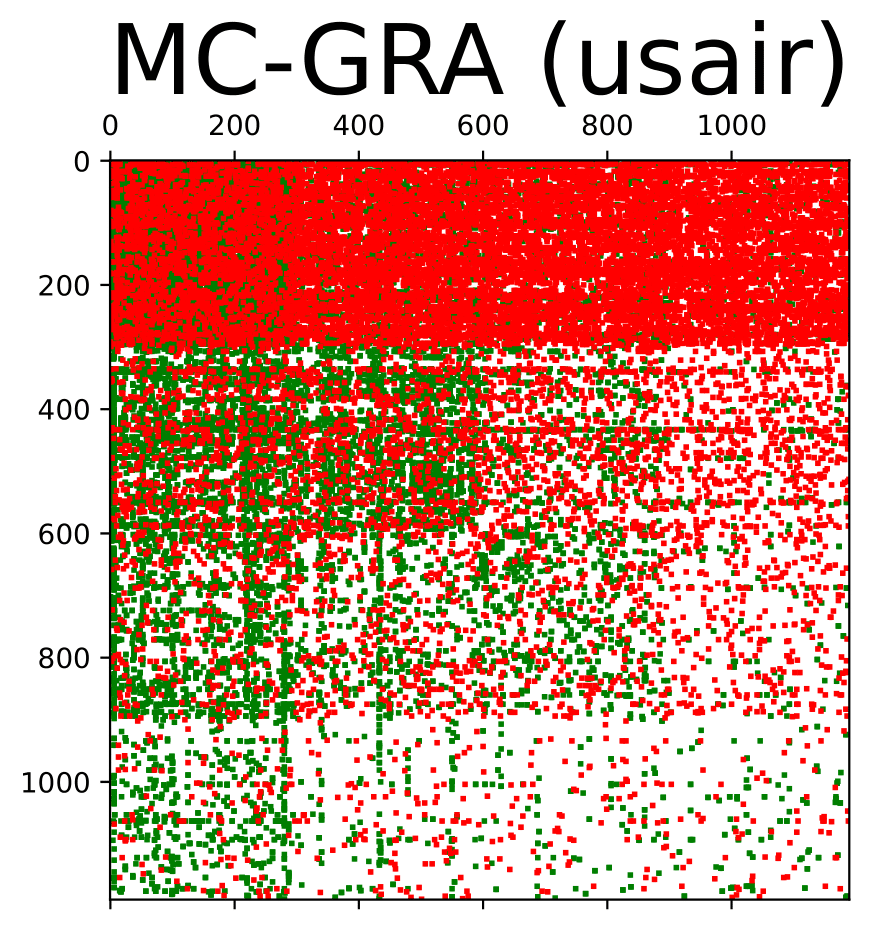}}}
\hfill
\subfloat[GraphMI, protected]
{{\includegraphics[width=0.18\linewidth]{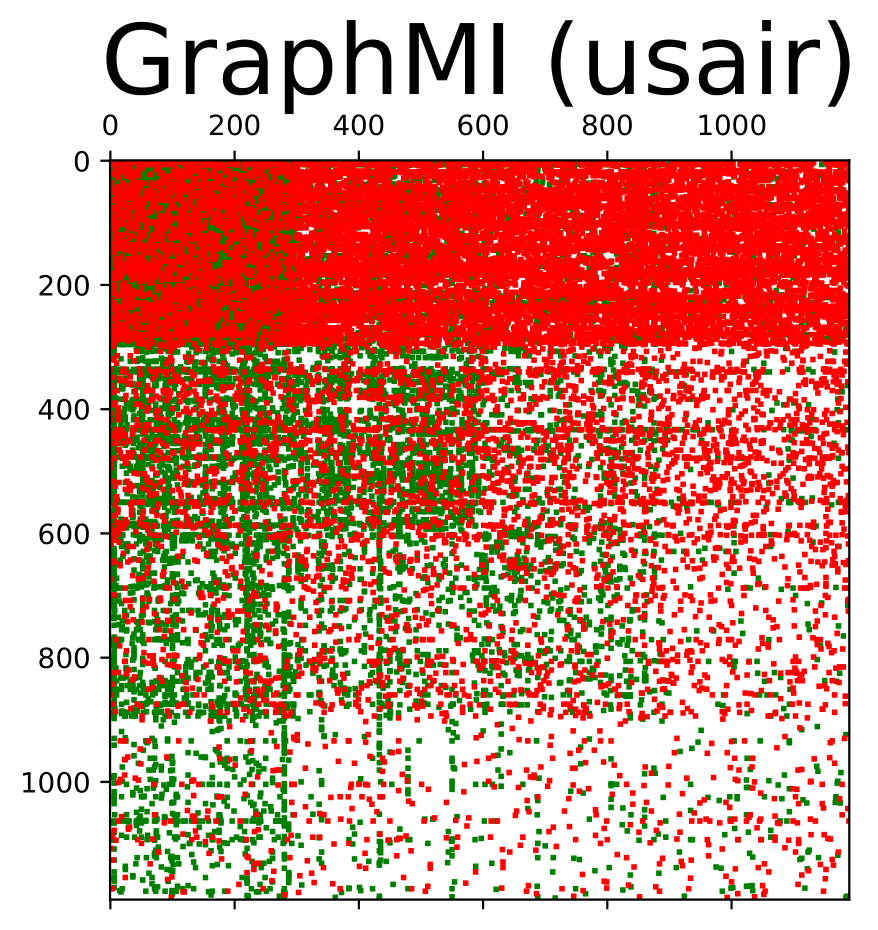}}}
\vspace{-8pt}
\caption{Recovered adjacency on USA. (a)--(e) as in Fig.~\ref{appd: adj:cora}. Green: correct; red: errors.}
\end{figure}

\begin{figure}[H]
\centering
\subfloat[Ground truth]
{{\includegraphics[width=0.18\linewidth]{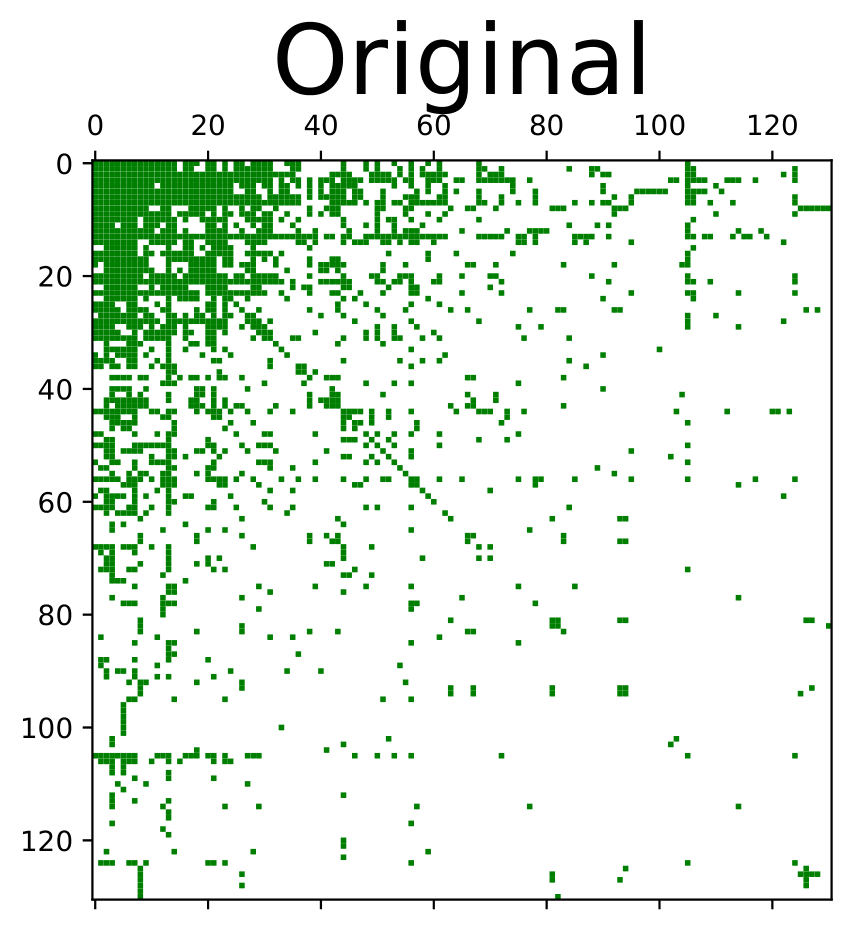}}}
\hfill
\subfloat[MC-GRA~(+), unprotected]
{{\includegraphics[width=0.18\linewidth]{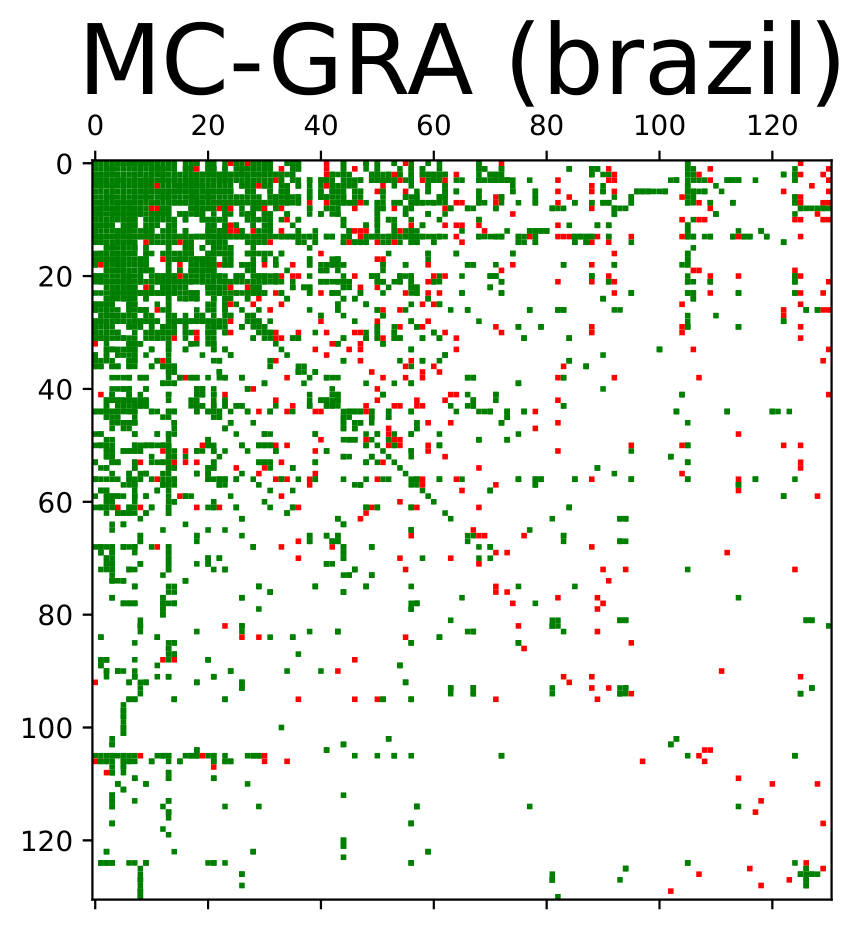}}}
\hfill
\subfloat[GraphMI, unprotected]
{{\includegraphics[width=0.18\linewidth]{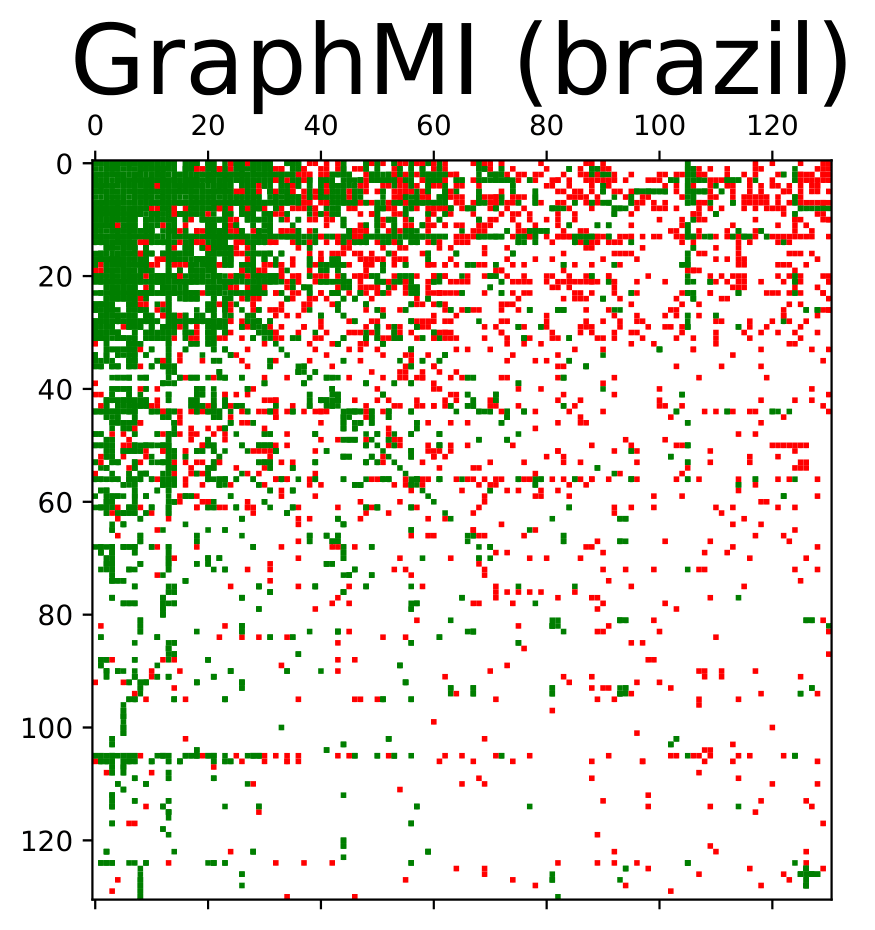}}}
\hfill
\subfloat[MC-GRA~(+), protected]
{{\includegraphics[width=0.18\linewidth]{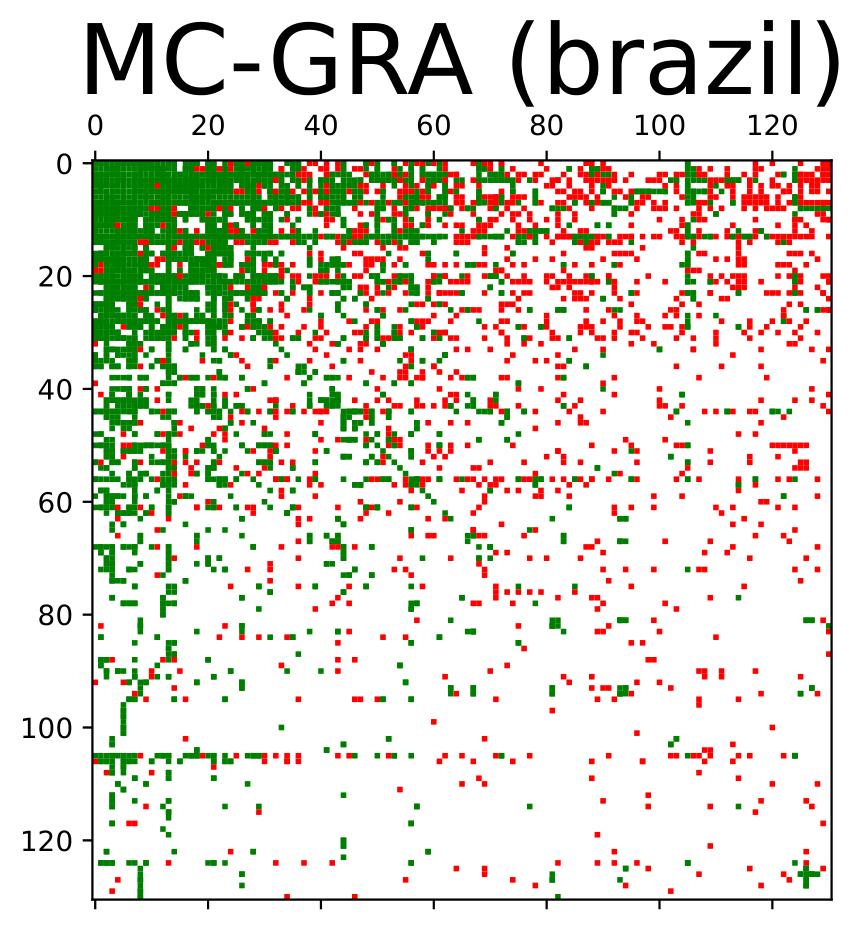}}}
\hfill
\subfloat[GraphMI, protected]
{{\includegraphics[width=0.18\linewidth]{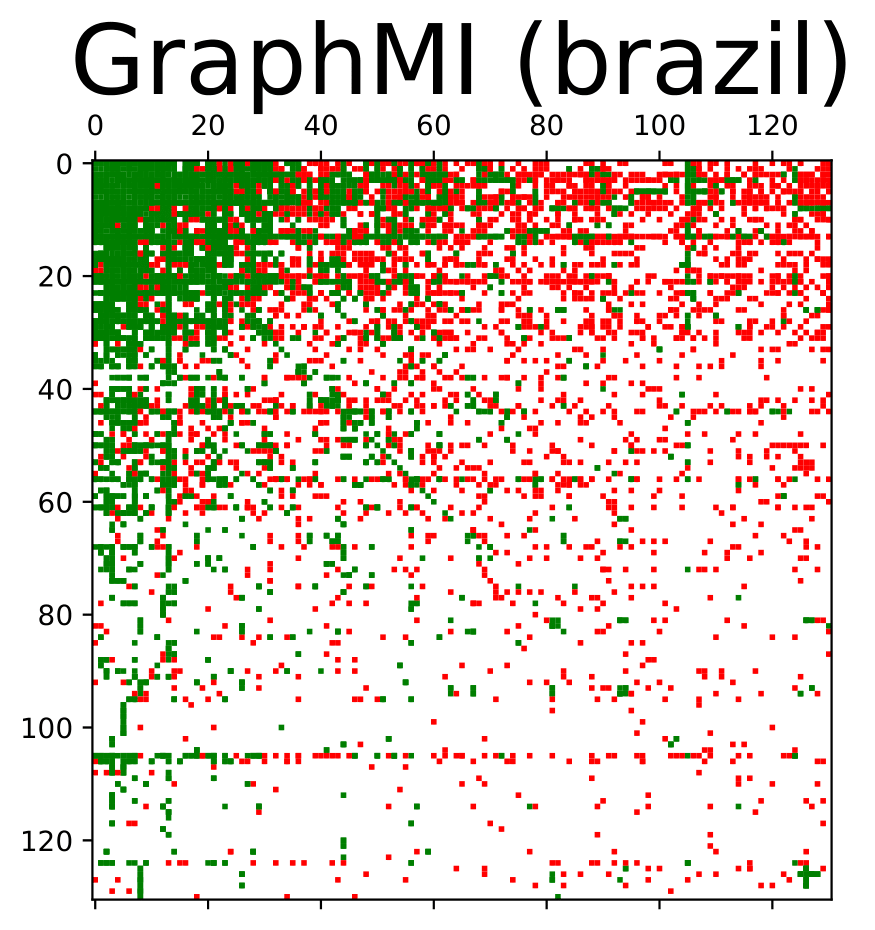}}}
\vspace{-8pt}
\caption{Recovered adjacency on Brazil. (a)--(e) as in Fig.~\ref{appd: adj:cora}. Green: correct; red: errors.}
\end{figure}

\begin{figure}[H]
\centering
\subfloat[Ground truth]
{{\includegraphics[width=0.18\linewidth]{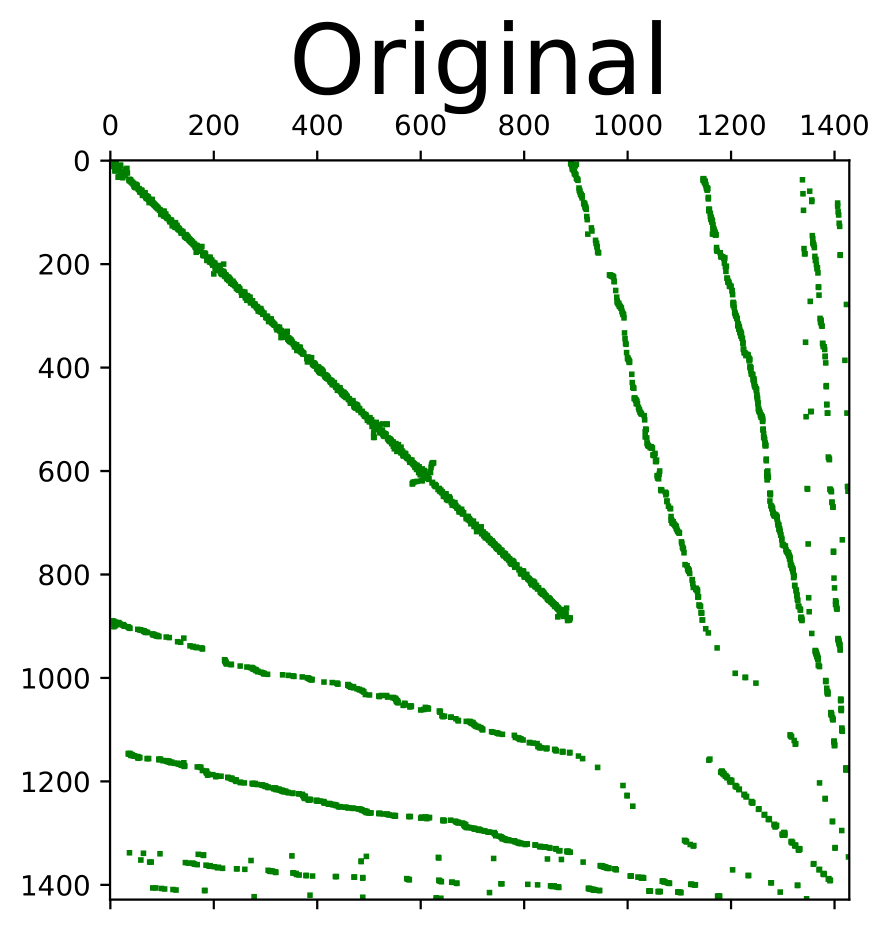}}}
\hfill
\subfloat[MC-GRA~(+), unprotected]
{{\includegraphics[width=0.18\linewidth]{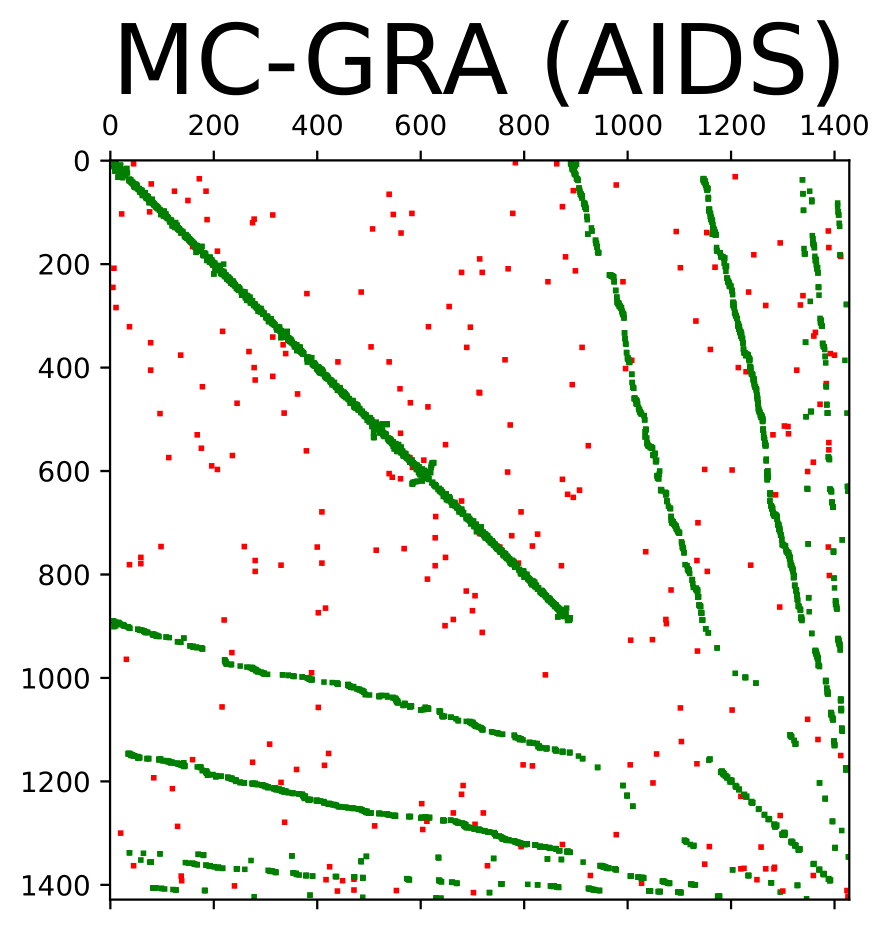}}}
\hfill
\subfloat[GraphMI, unprotected]
{{\includegraphics[width=0.18\linewidth]{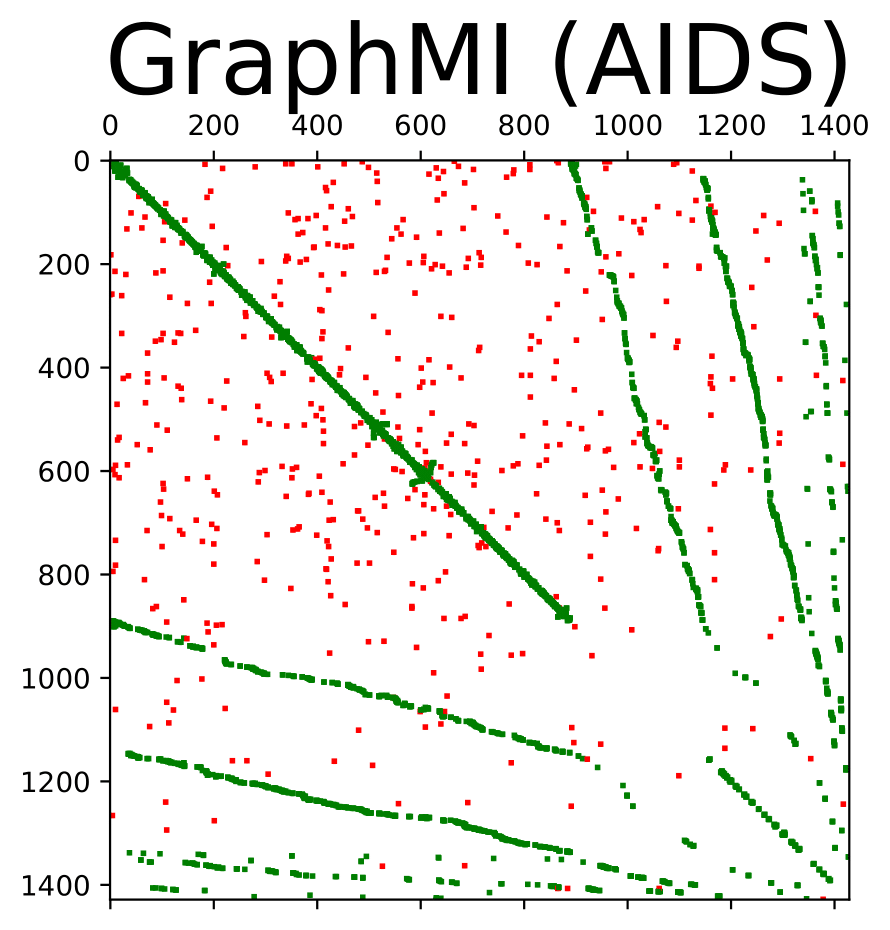}}}
\hfill
\subfloat[MC-GRA~(+), protected]
{{\includegraphics[width=0.18\linewidth]{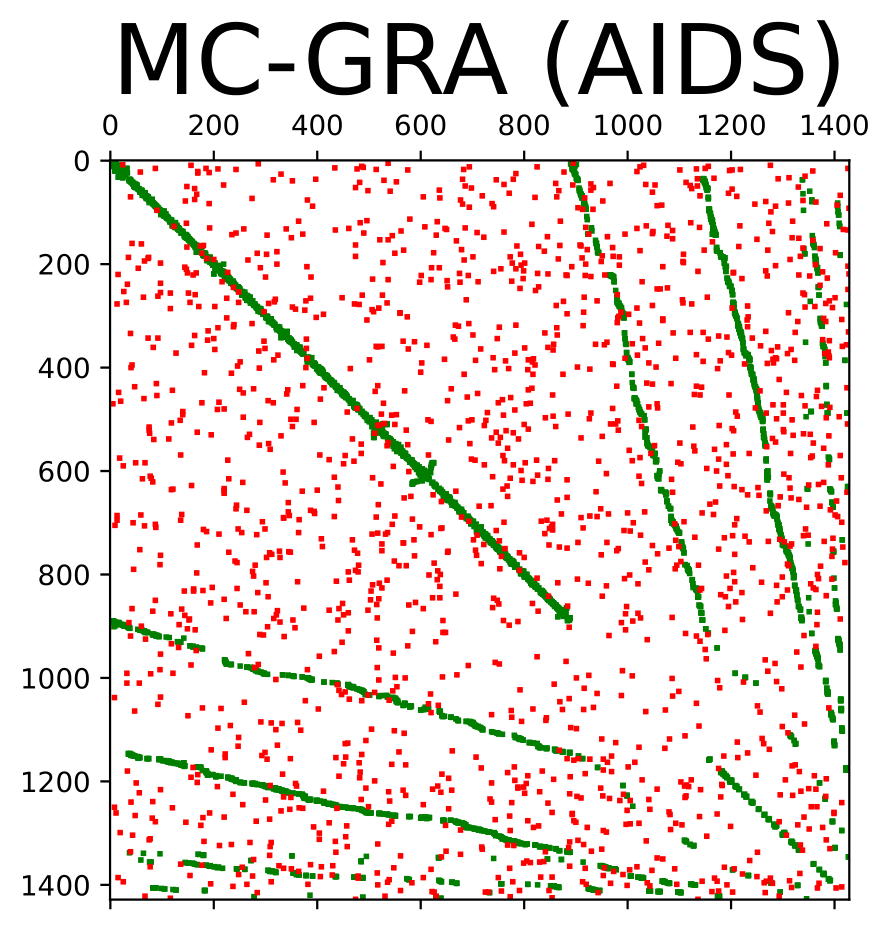}}}
\hfill
\subfloat[GraphMI, protected]
{{\includegraphics[width=0.18\linewidth]{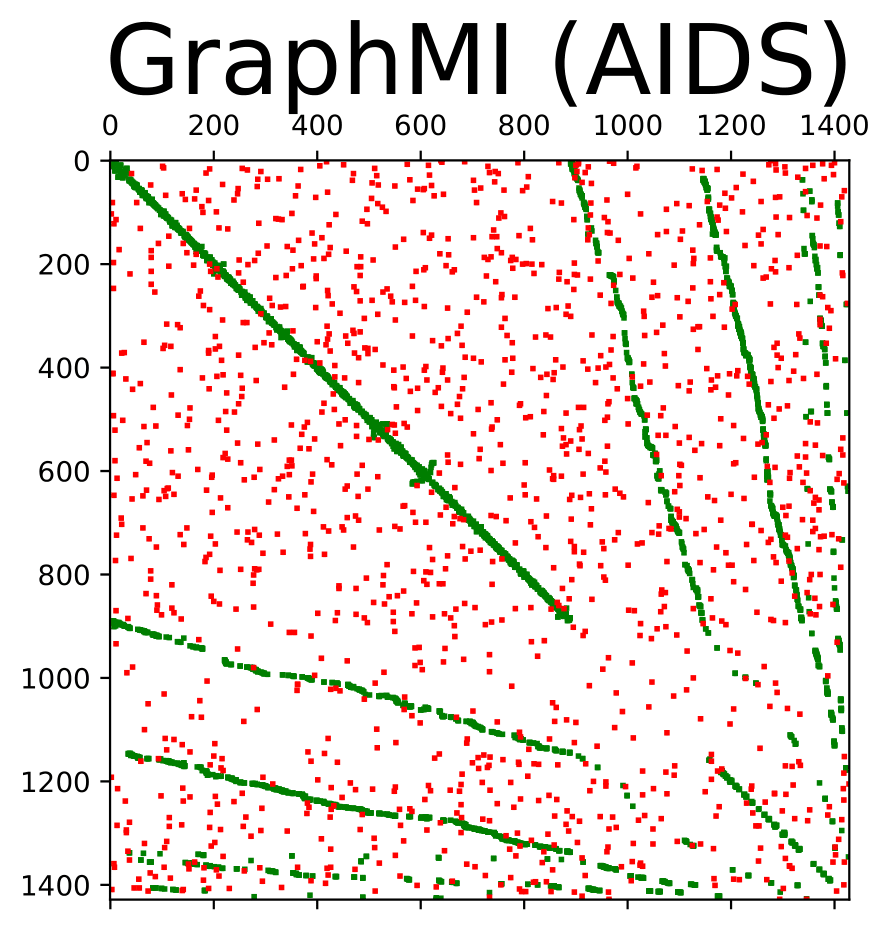}}}
\vspace{-8pt}
\caption{Recovered adjacency on AIDS. (a)--(e) as in Fig.~\ref{appd: adj:cora}. Green: correct; red: errors.}
\label{appd: adj:AIDS}
\end{figure}

\textbf{Additional qualitative results with the updated plotting pipeline.}
To complement the visualizations above, we include additional figures from the same plotting pipeline as in the main paper, grouped by model and task.

\begin{figure*}[t]
\centering
\includegraphics[width=0.32\textwidth]{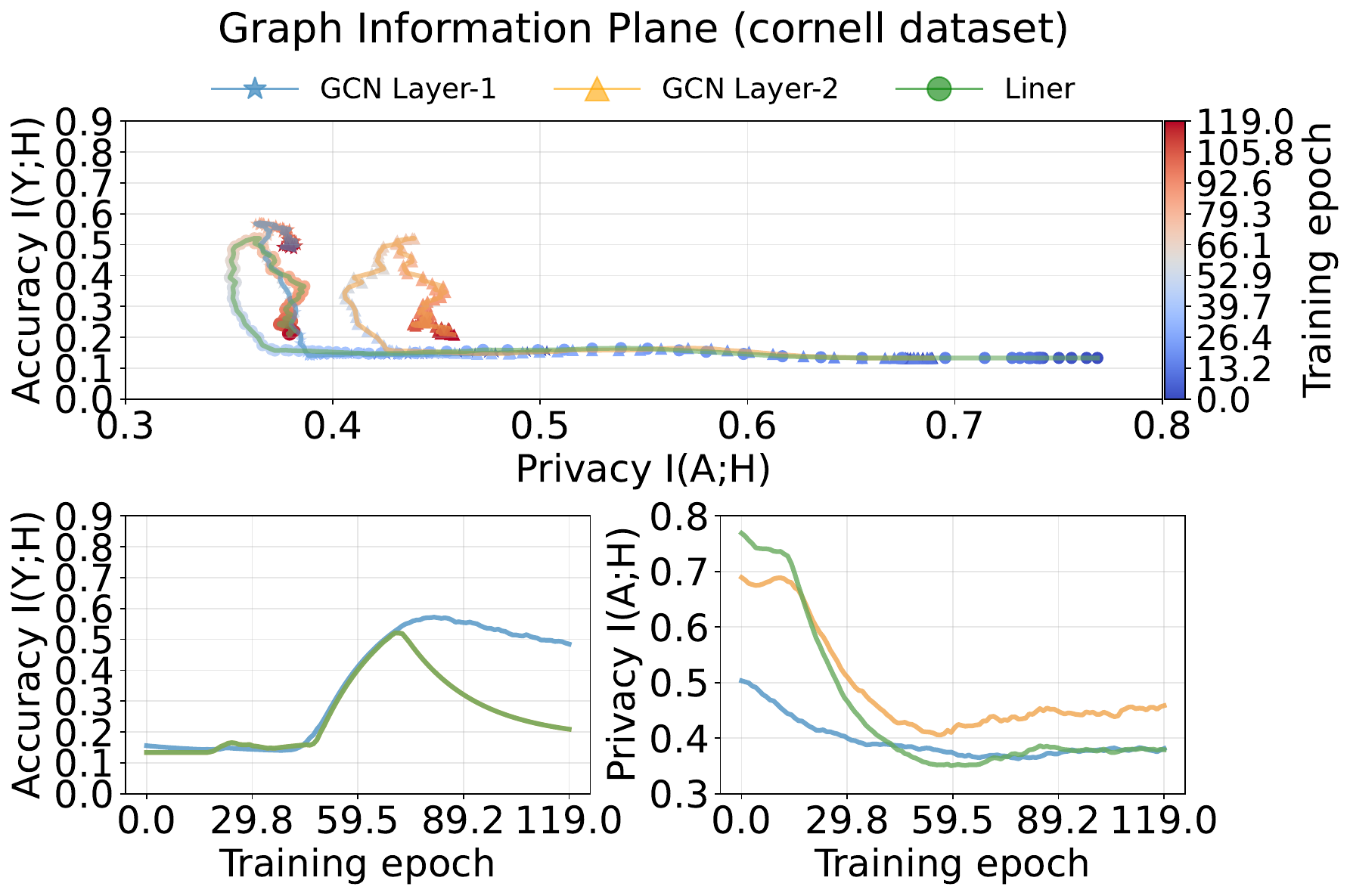}
\hfill
\includegraphics[width=0.32\textwidth]{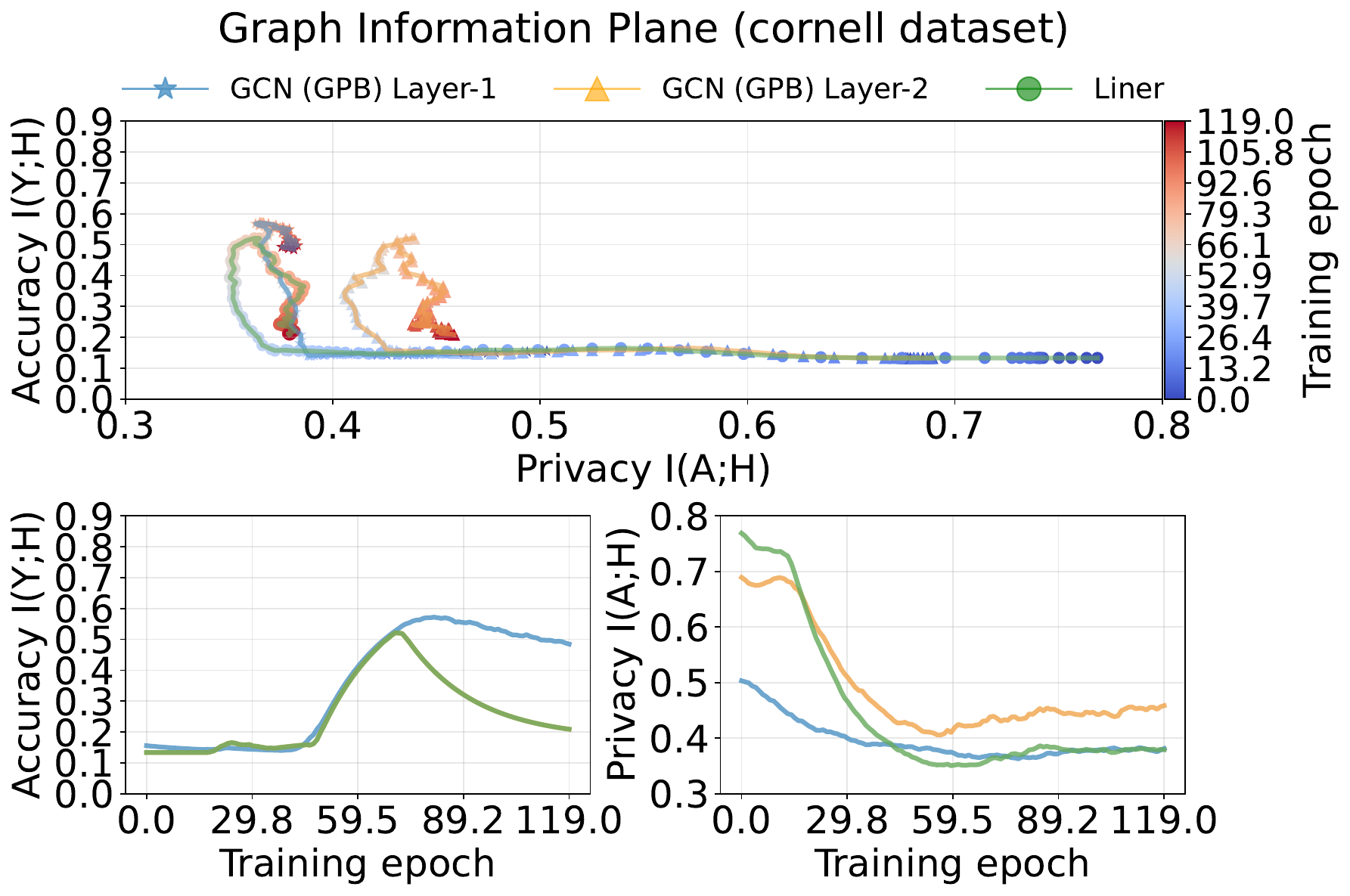}
\hfill
\includegraphics[width=0.32\textwidth]{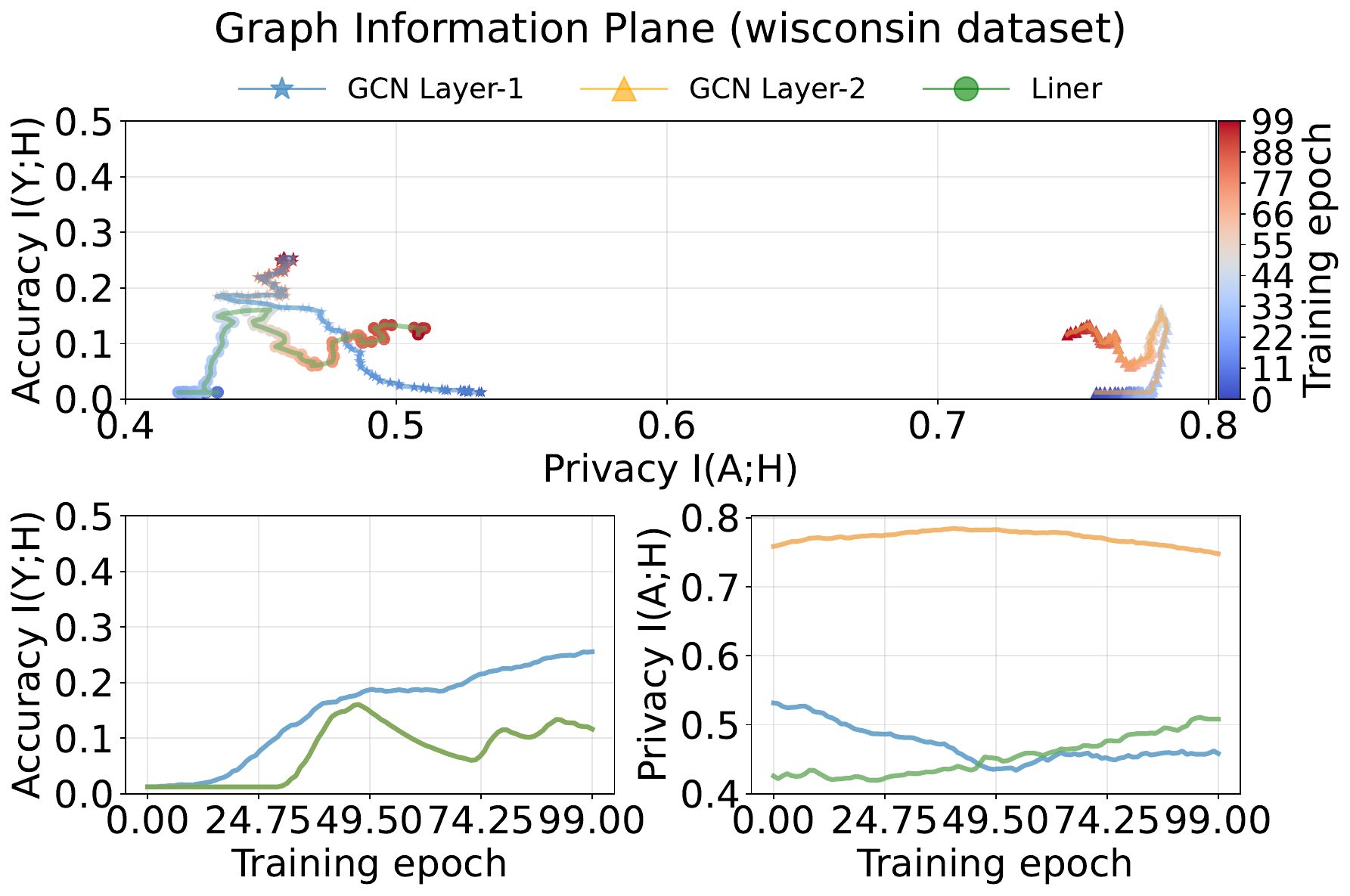}
\\
\vspace{0.2cm}
\includegraphics[width=0.32\textwidth]{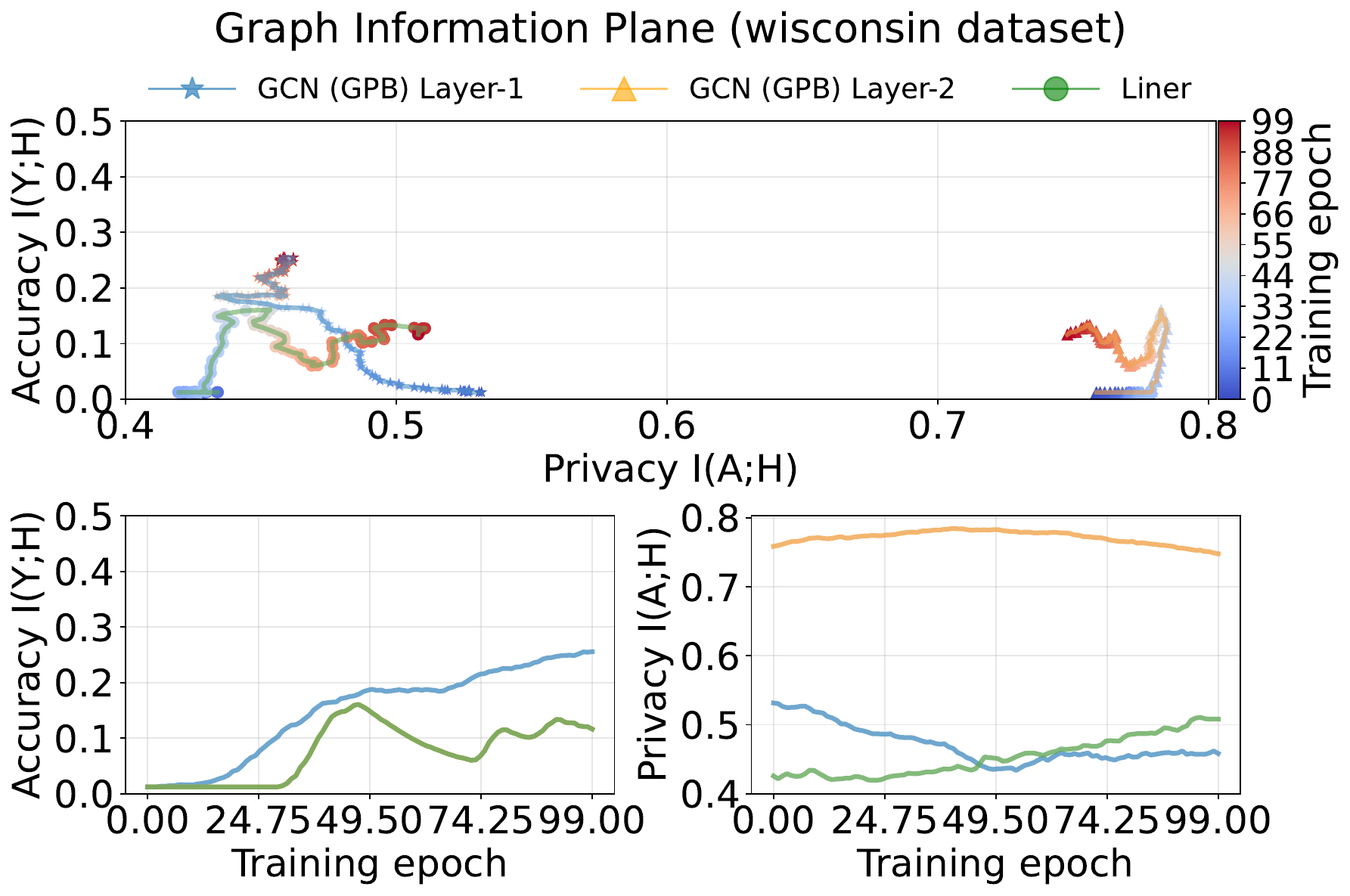}
\hfill
\includegraphics[width=0.32\textwidth]{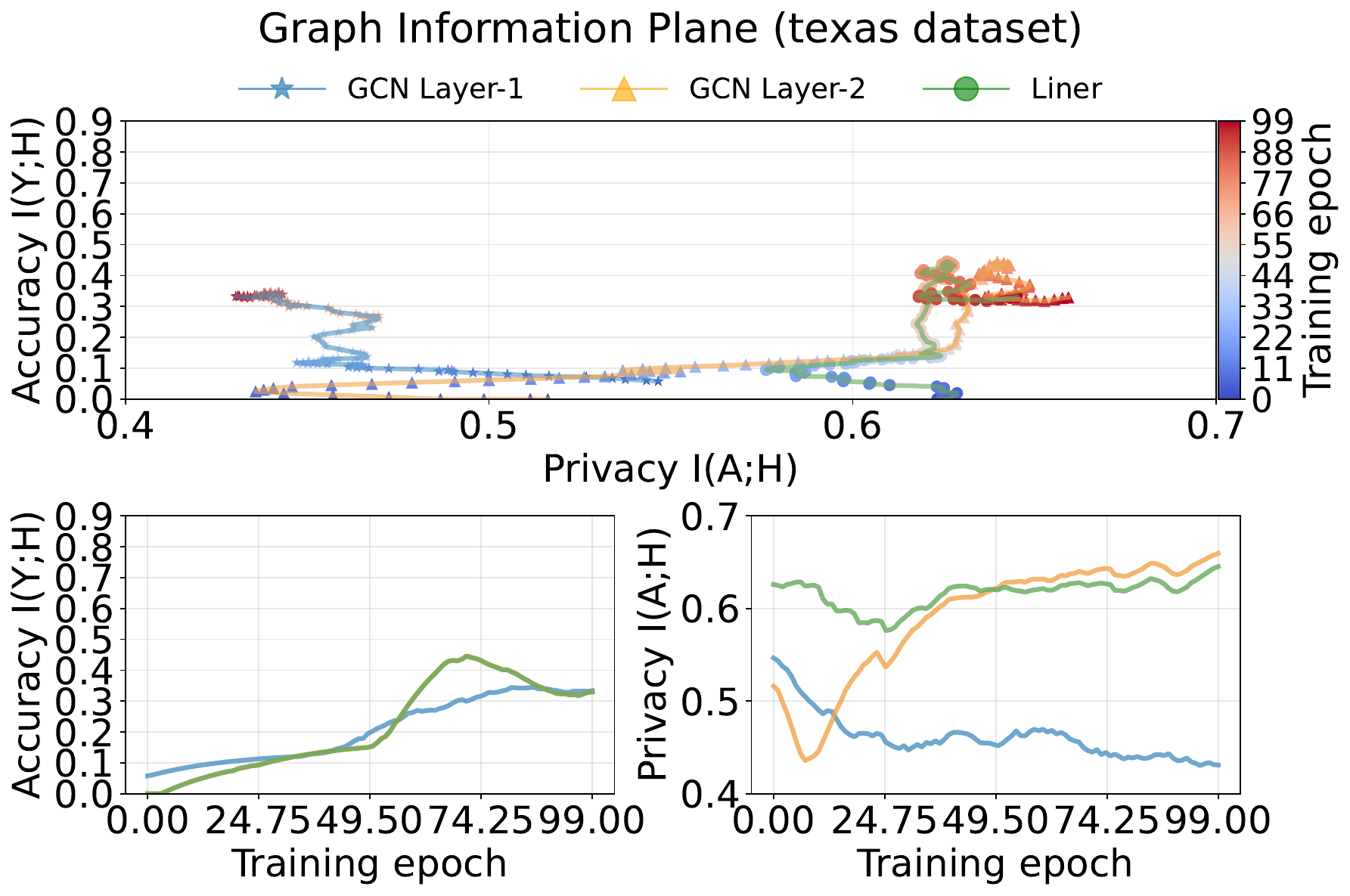}
\hfill
\includegraphics[width=0.32\textwidth]{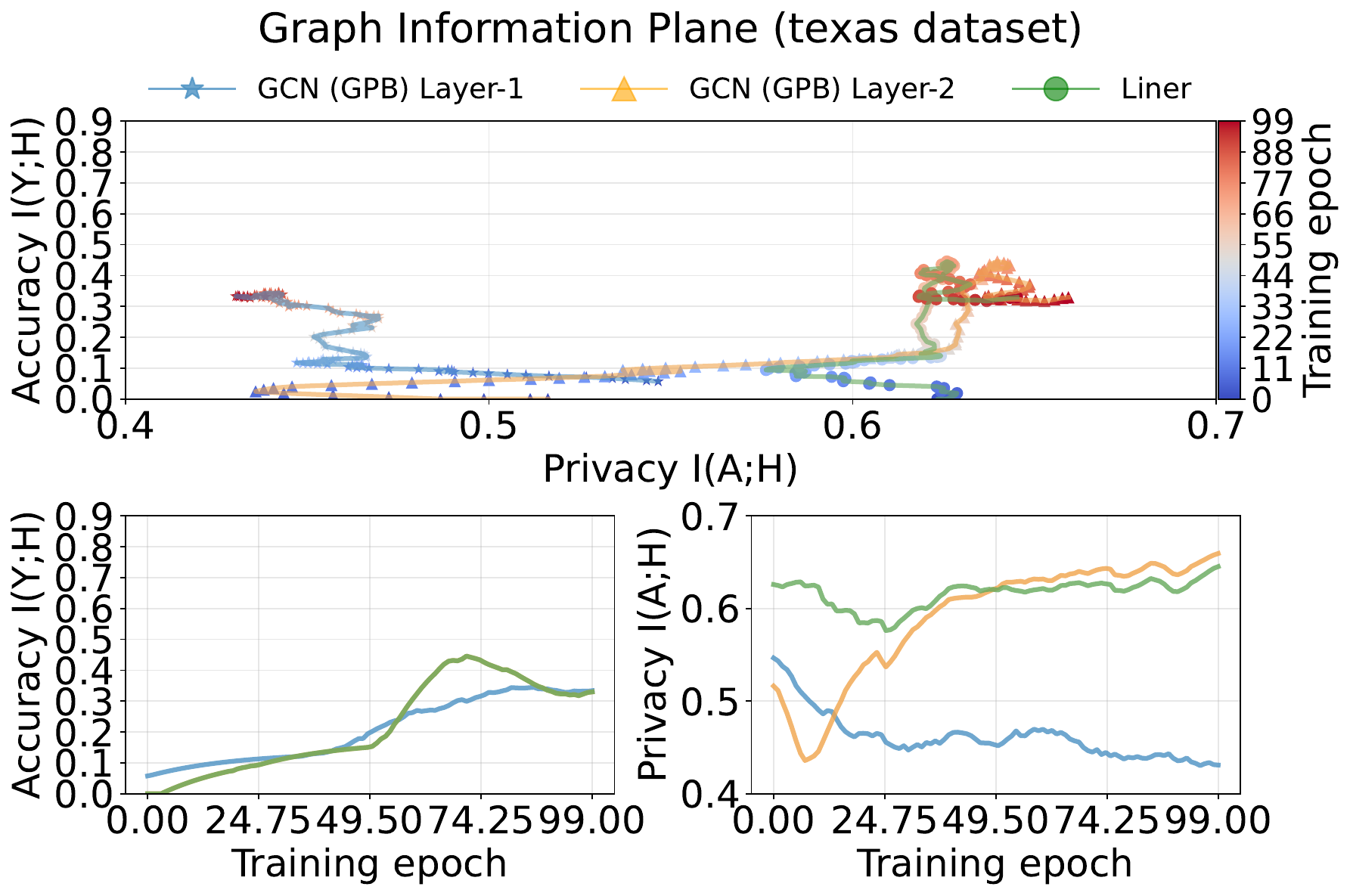}
\caption{Additional graph information planes for GCN on Cornell, Wisconsin, and Texas under unprotected and protected training.}
\label{app: gip:gcn:addition}
\end{figure*}

\begin{figure*}[t]
\centering
\includegraphics[width=0.24\textwidth]{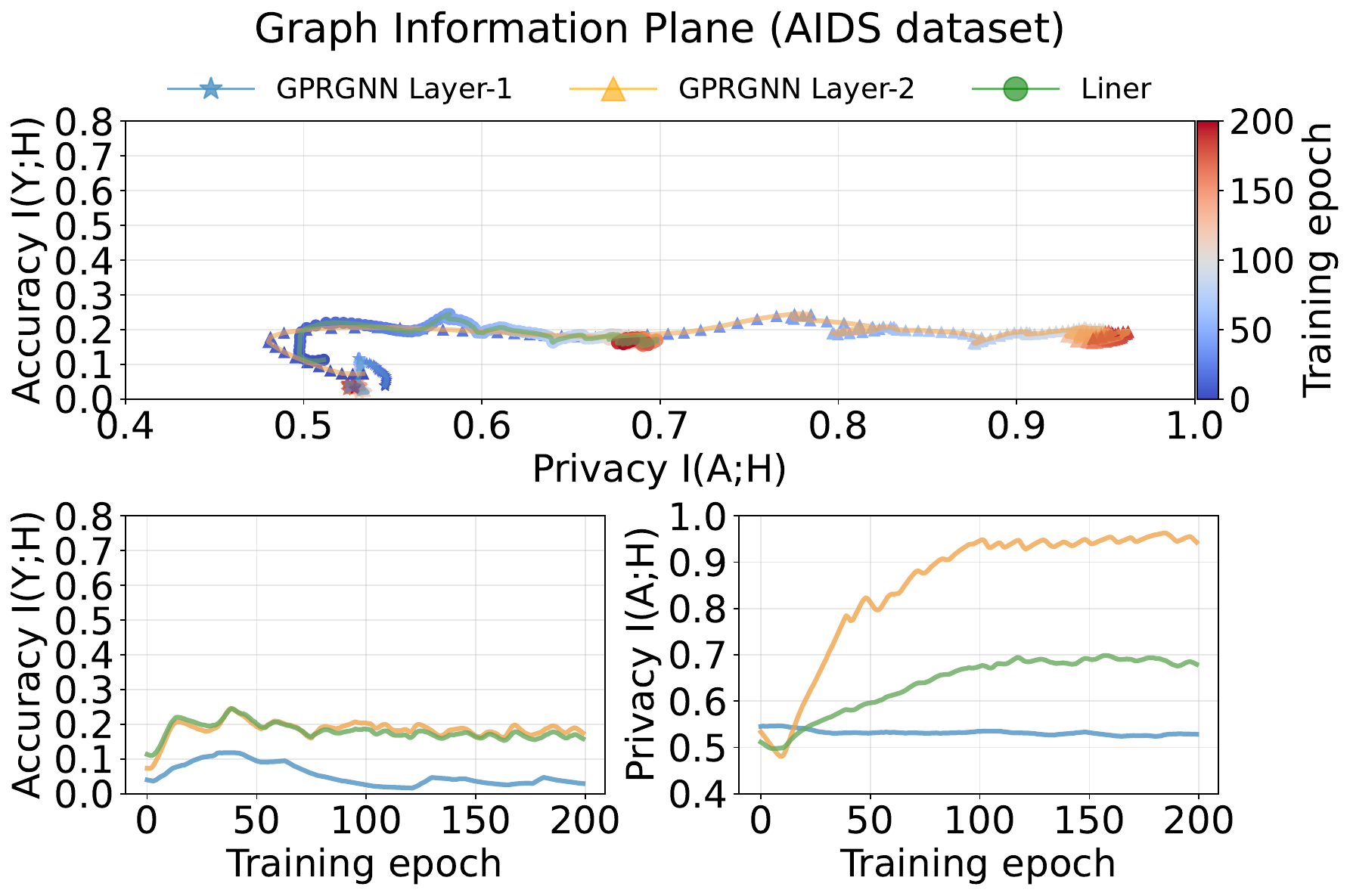}
\hfill
\includegraphics[width=0.24\textwidth]{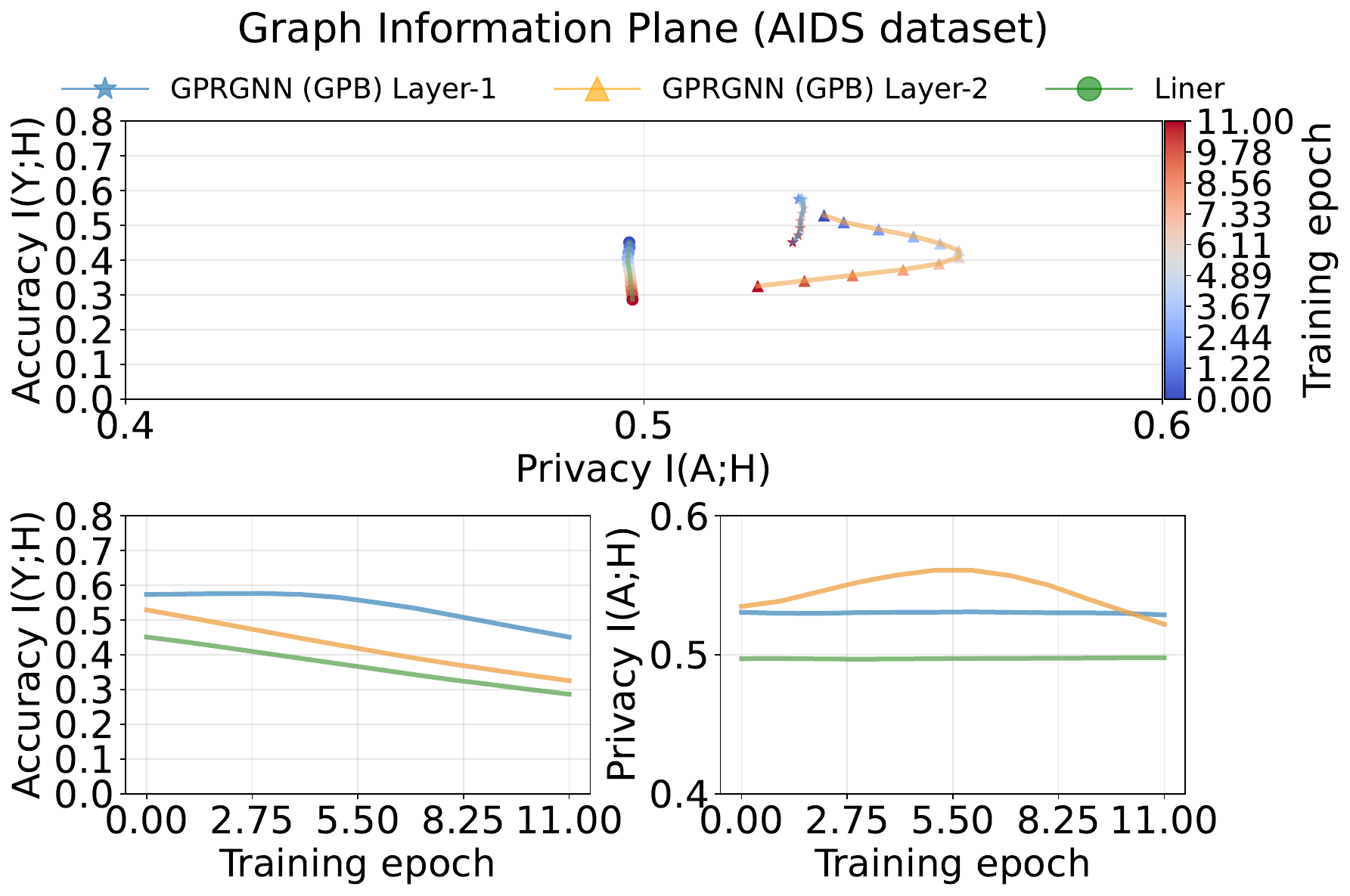}
\hfill
\includegraphics[width=0.24\textwidth]{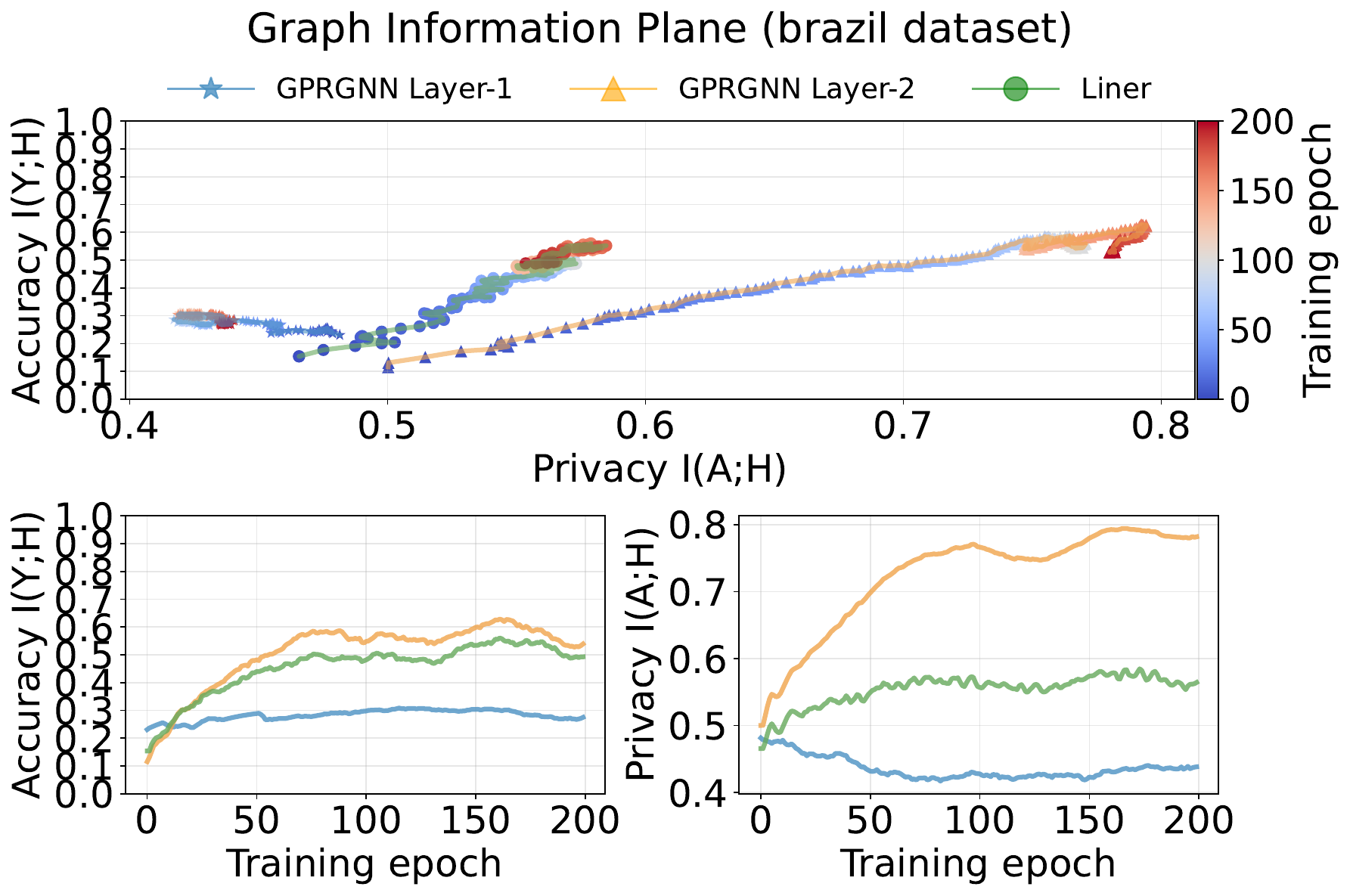}
\hfill
\includegraphics[width=0.24\textwidth]{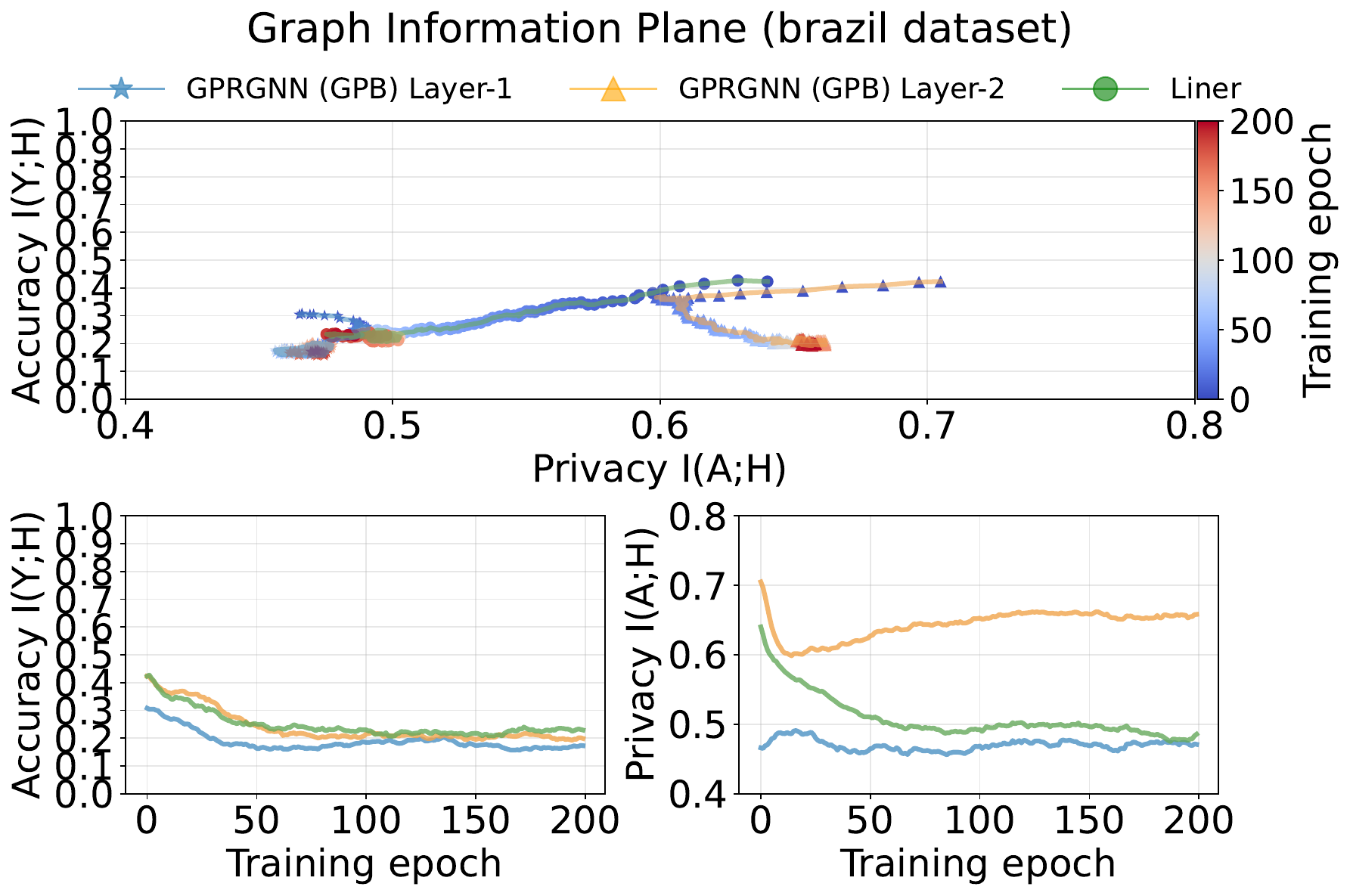}
\\
\vspace{0.2cm}
\includegraphics[width=0.24\textwidth]{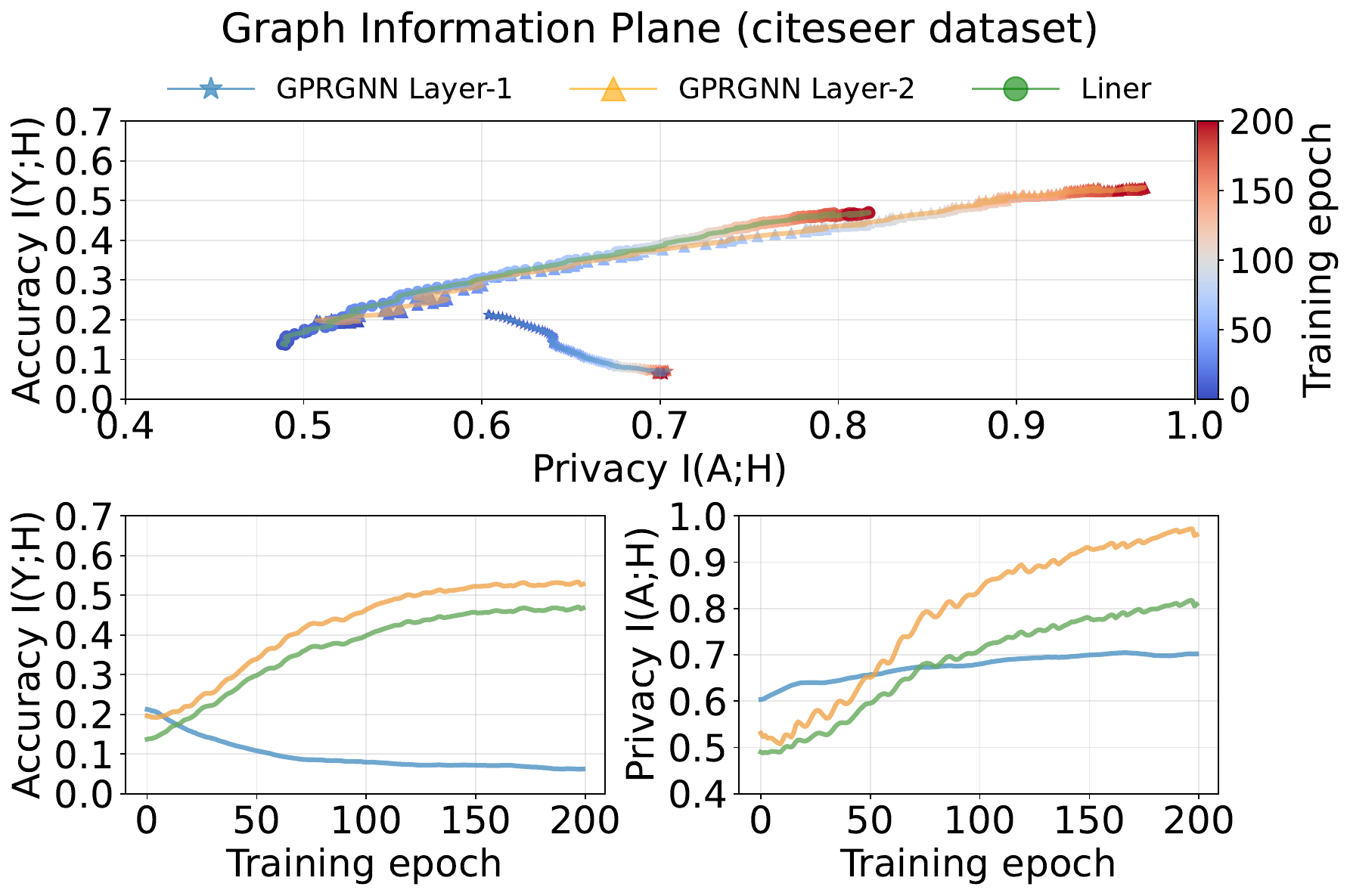}
\hfill
\includegraphics[width=0.24\textwidth]{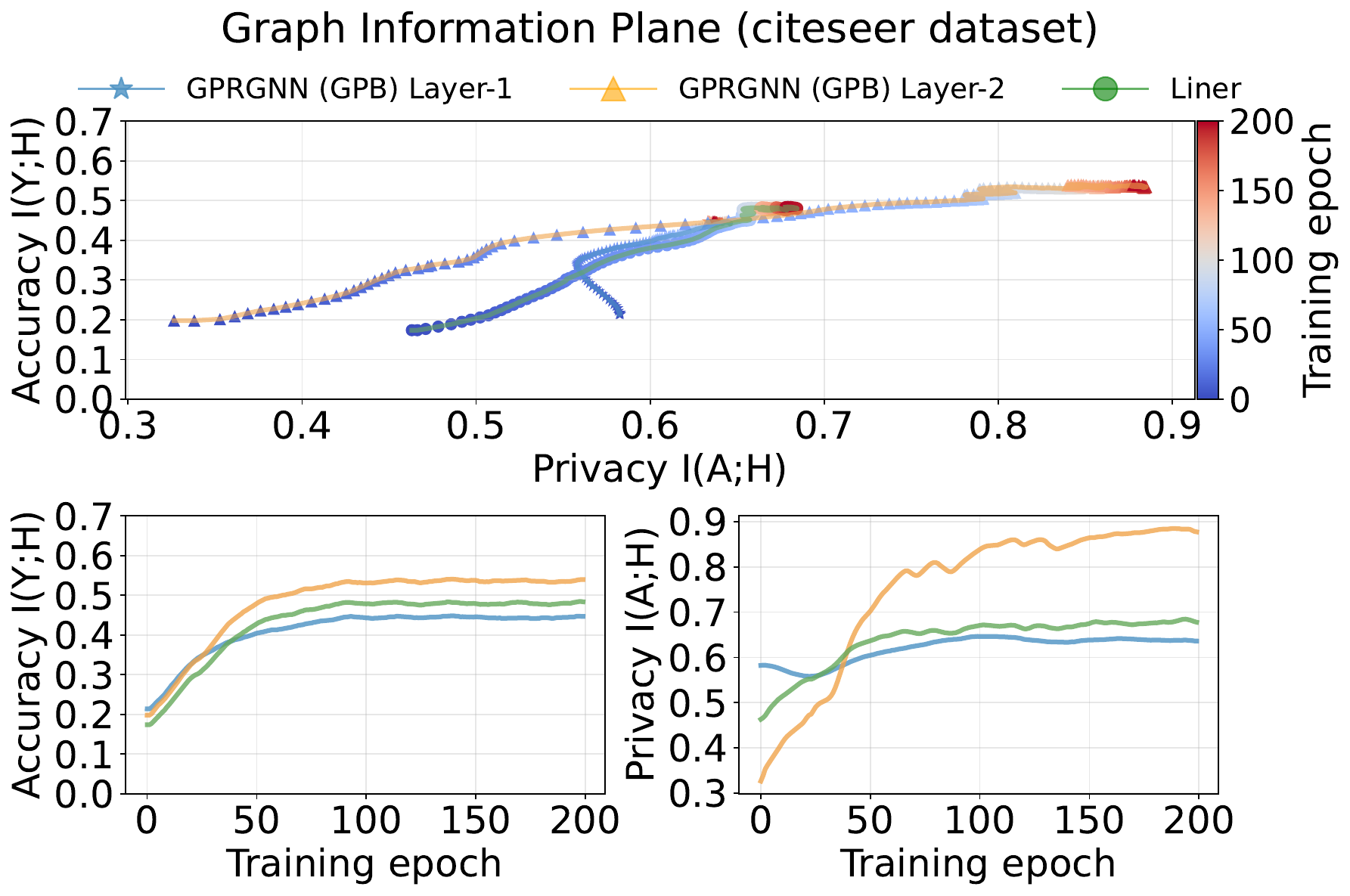}
\hfill
\includegraphics[width=0.24\textwidth]{figures/appendix/understanding/graph_plane_gprgnn_cornell_plain.pdf}
\hfill
\includegraphics[width=0.24\textwidth]{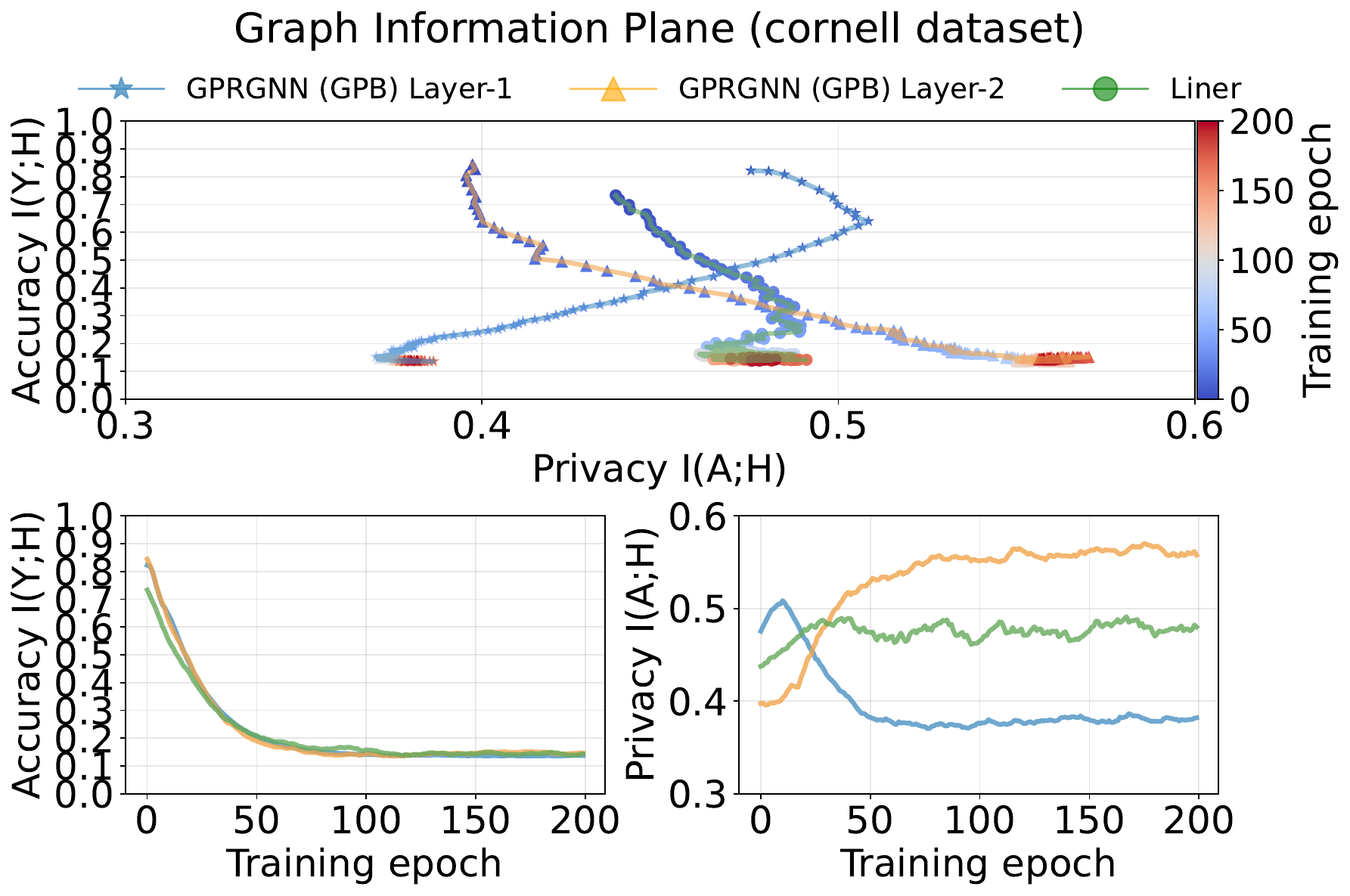}
\caption{Additional graph information planes for GPR-GNN on AIDS, Brazil, Citeseer, and Cornell. Each pair compares unprotected training with MC-GPB~(+) protected training.}
\label{app: gip:gprgnn:addition-a}
\end{figure*}

\begin{figure*}[t]
\centering
\includegraphics[width=0.24\textwidth]{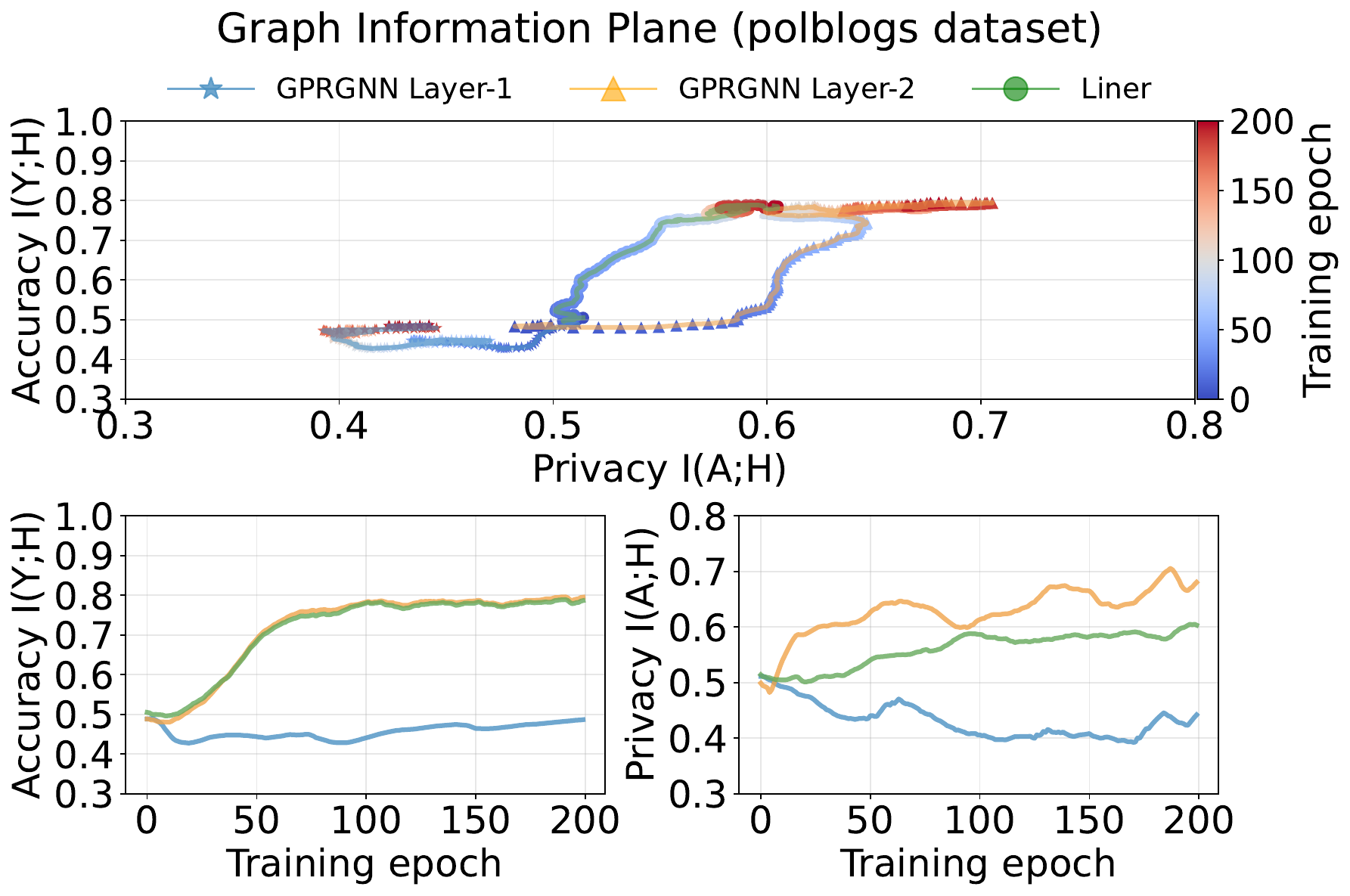}
\hfill
\includegraphics[width=0.24\textwidth]{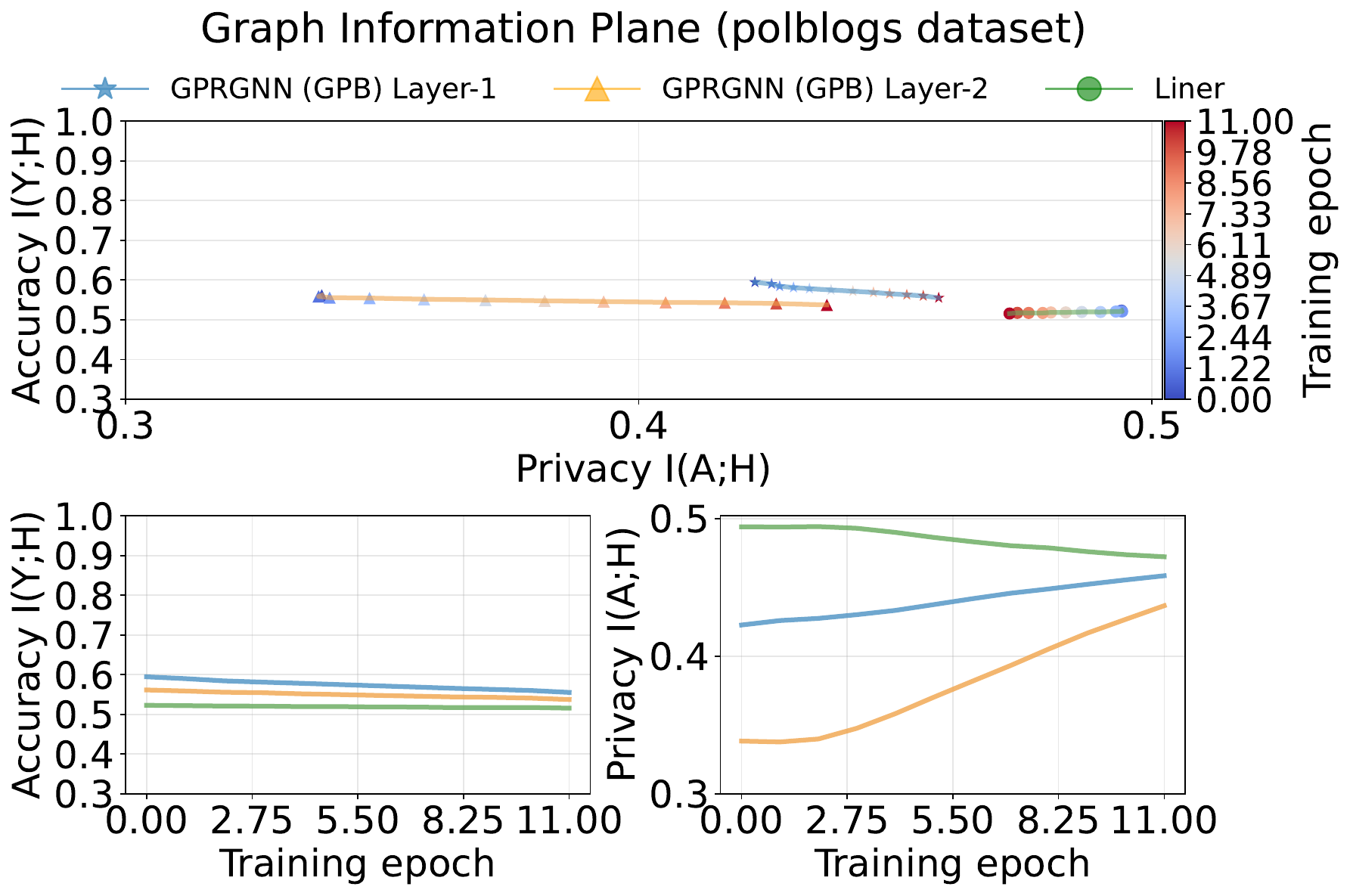}
\hfill
\includegraphics[width=0.24\textwidth]{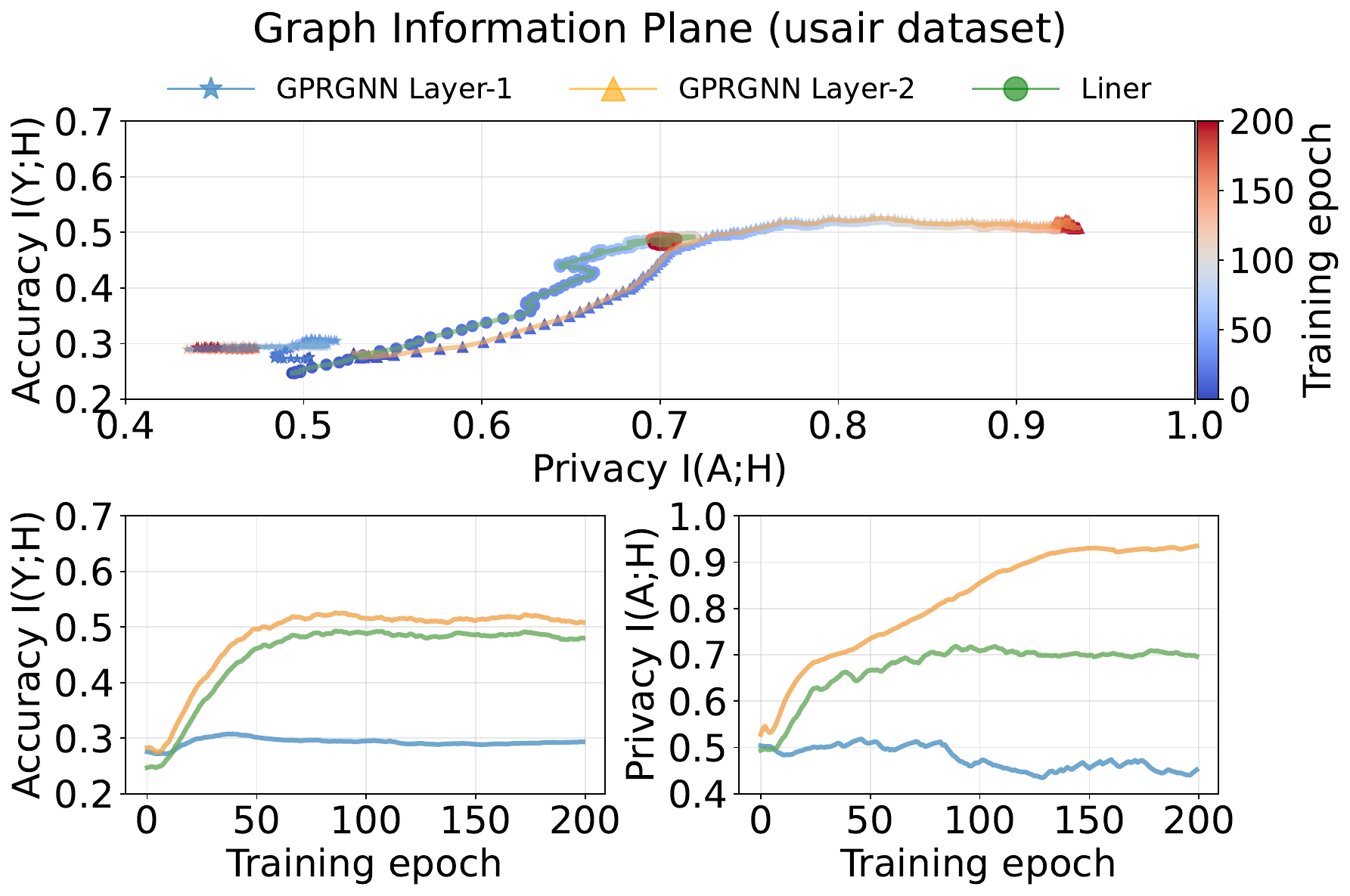}
\hfill
\includegraphics[width=0.24\textwidth]{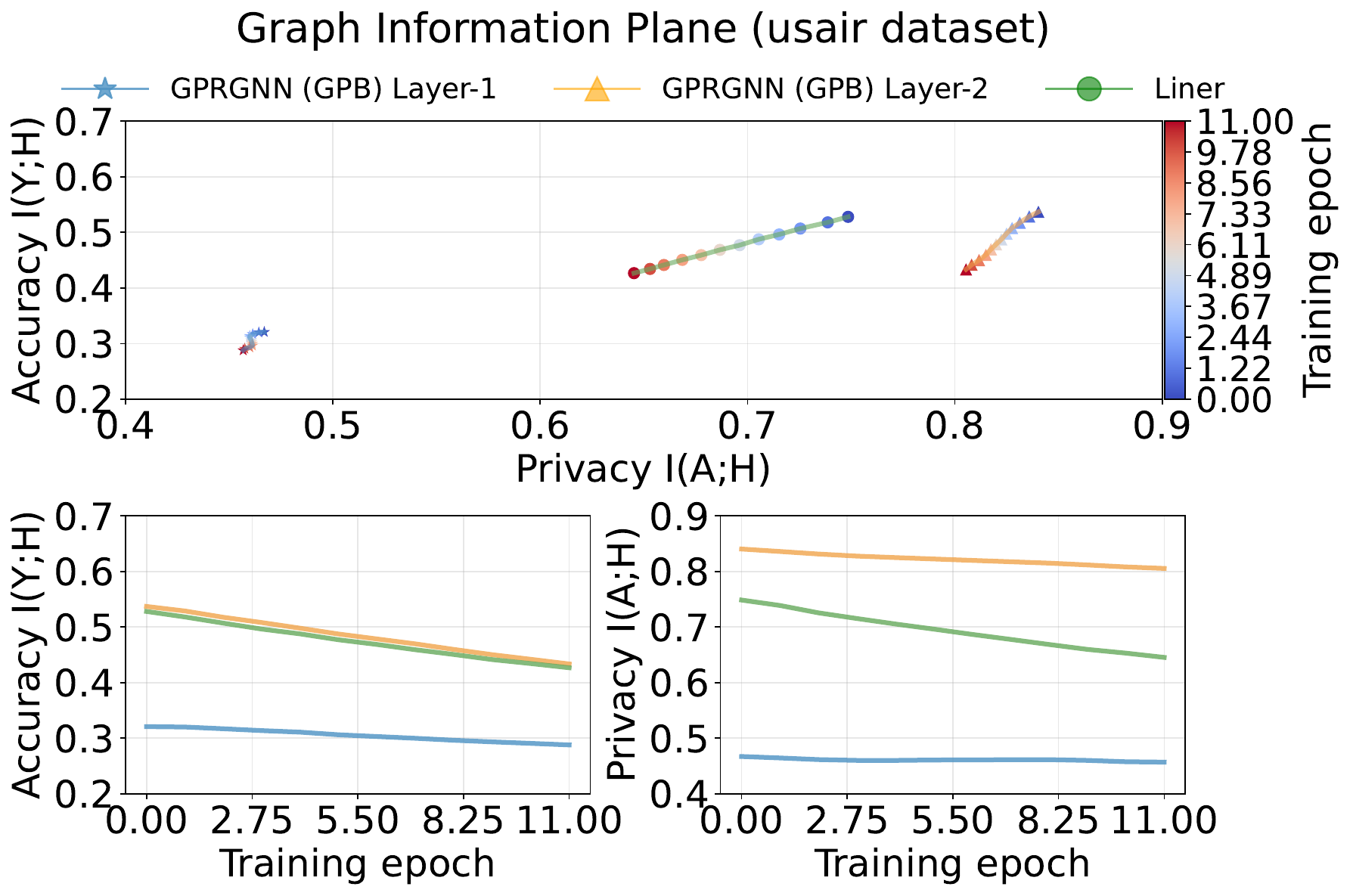}
\\
\vspace{0.2cm}
\includegraphics[width=0.24\textwidth]{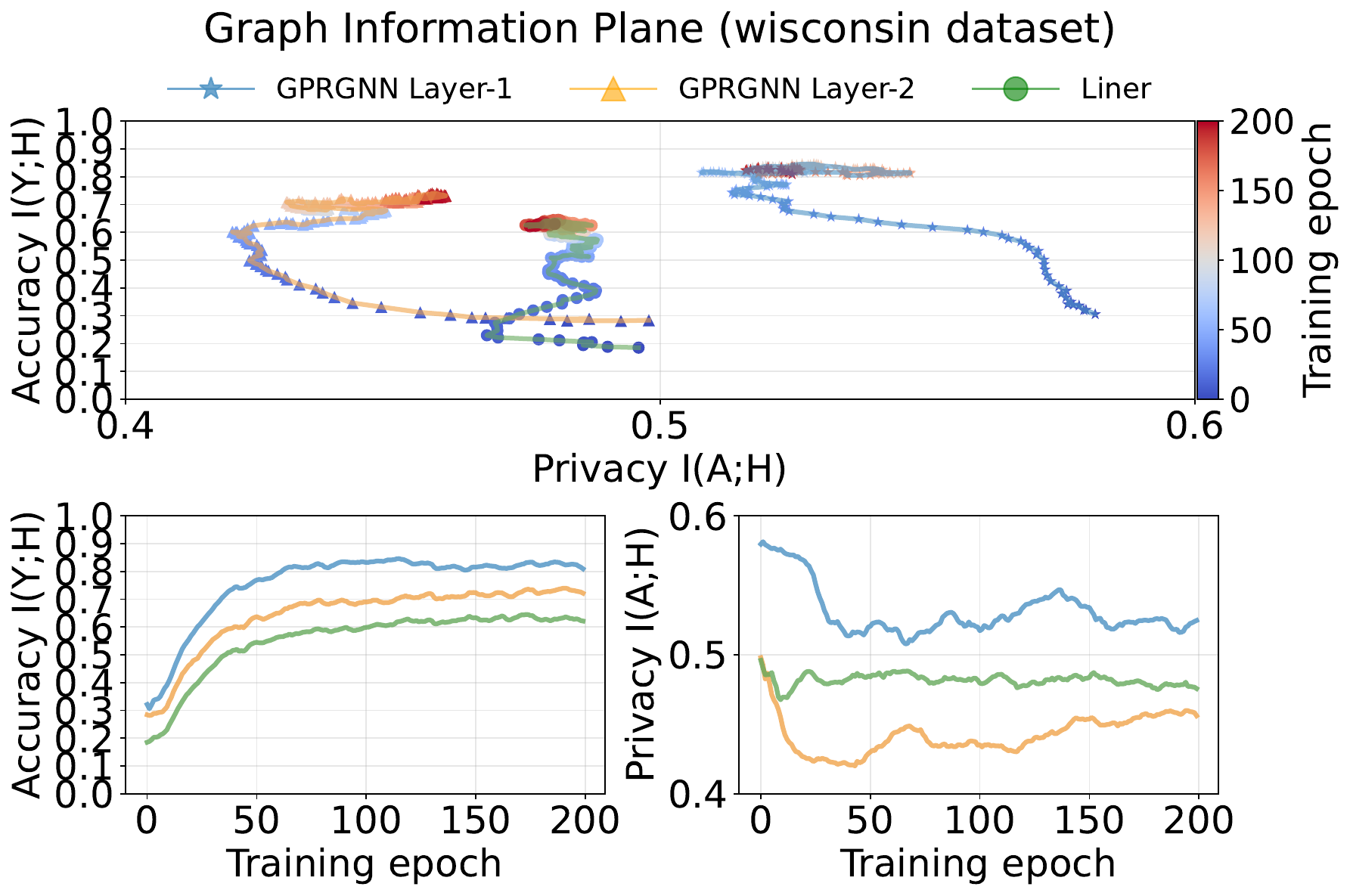}
\hfill
\includegraphics[width=0.24\textwidth]{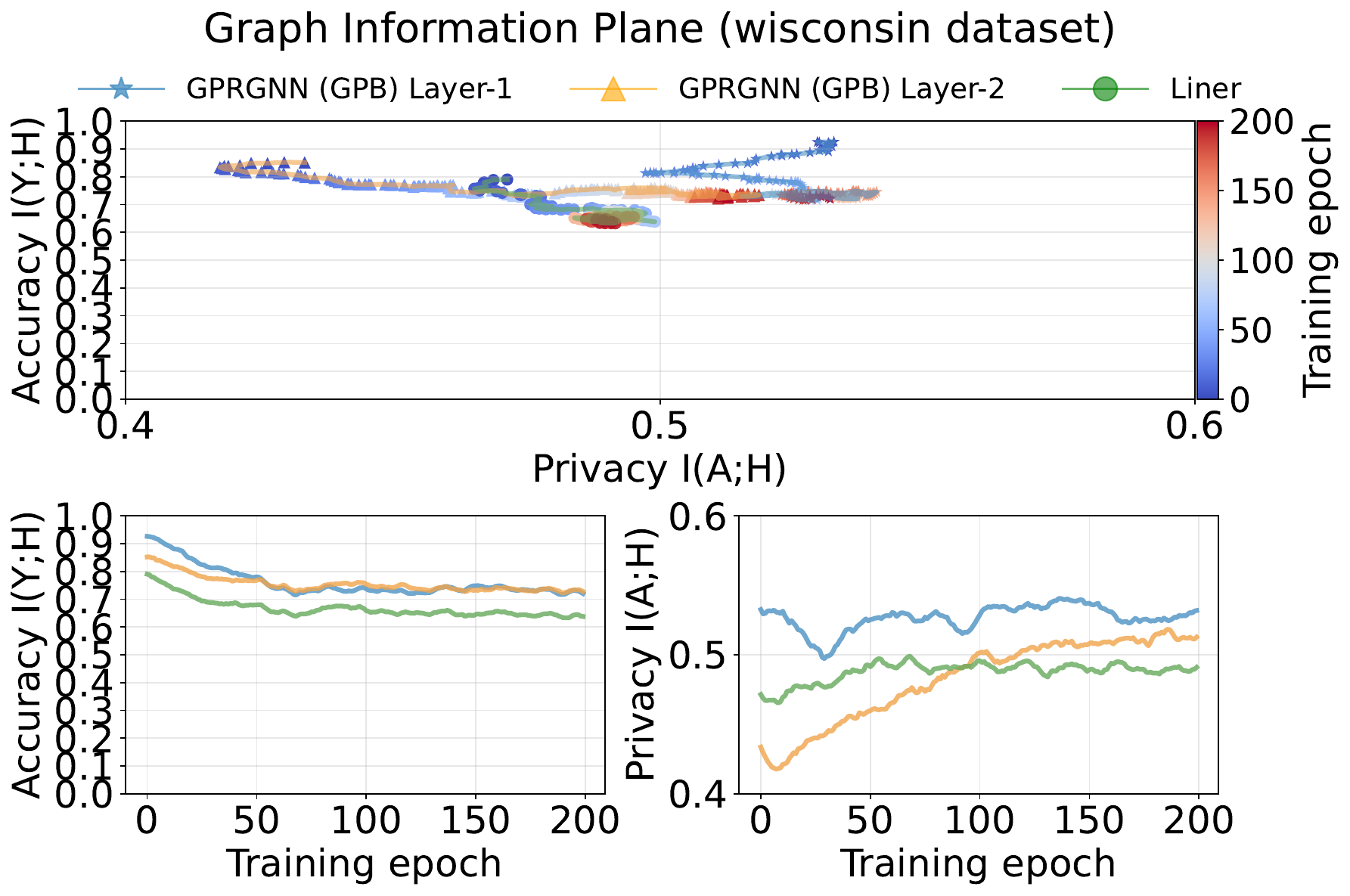}
\hfill
\includegraphics[width=0.24\textwidth]{figures/appendix/understanding/graph_plane_gprgnn_cora_plain.pdf}
\hfill
\includegraphics[width=0.24\textwidth]{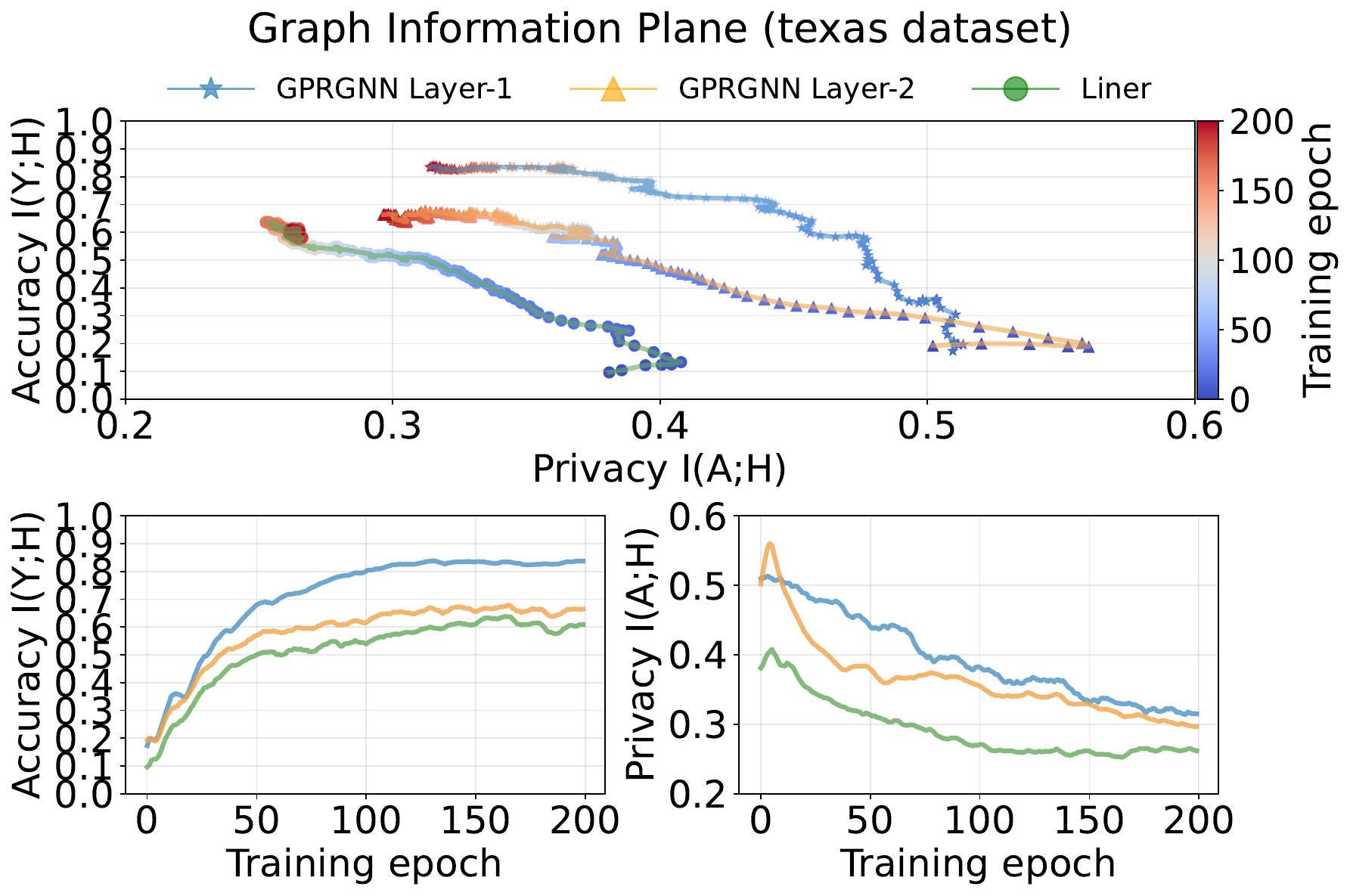}
\caption{Further graph information planes for GPR-GNN on Polblogs, USA, and Wisconsin, together with the unprotected-training views on Cora and Texas for completeness.}
\label{app: gip:gprgnn:addition-b}
\end{figure*}

\begin{figure*}[t]
\centering
\includegraphics[width=0.24\textwidth]{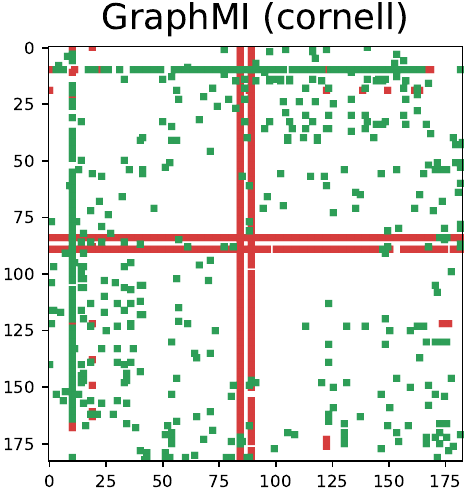}
\hfill
\includegraphics[width=0.24\textwidth]{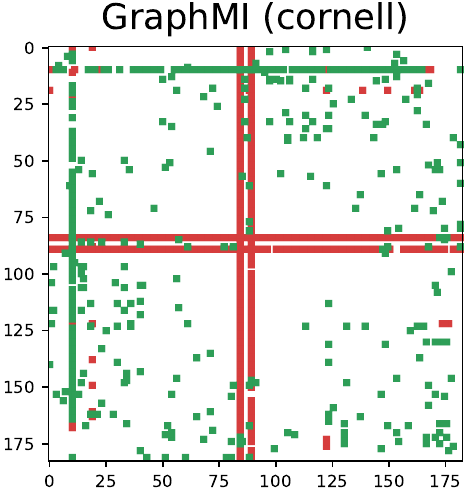}
\hfill
\includegraphics[width=0.24\textwidth]{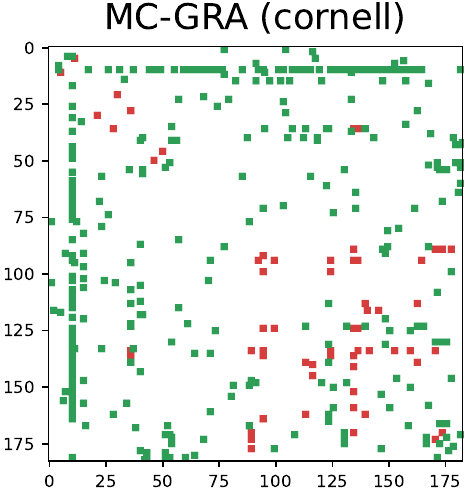}
\hfill
\includegraphics[width=0.24\textwidth]{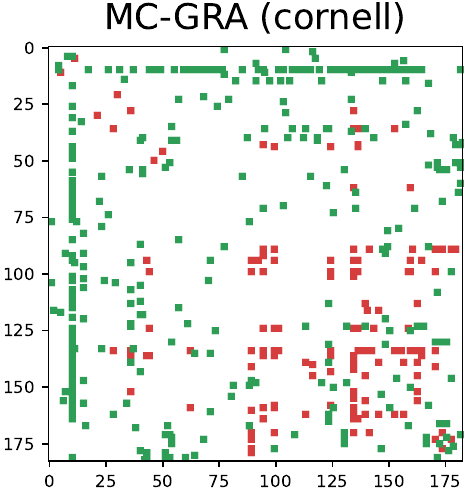}
\caption{Additional recovered adjacency examples on Cornell under GCN, comparing GraphMI and MC-GRA~(+) before and after protection.}
\label{app: adj:gcn:cornell}
\end{figure*}


\begin{figure*}[t]
\centering
\includegraphics[width=0.24\textwidth]{figures/all_refalign/adj_recovery_gprgnn_cora_graphmi_plain_refalign.pdf}
\hfill
\includegraphics[width=0.24\textwidth]{figures/all_refalign/adj_recovery_gprgnn_cora_mcgra_plain_refalign.pdf}
\hfill
\includegraphics[width=0.24\textwidth]{figures/all_refalign/adj_recovery_gprgnn_cora_graphmi_protected_refalign.pdf}
\hfill
\includegraphics[width=0.24\textwidth]{figures/all_refalign/adj_recovery_gprgnn_cora_mcgra_protected_refalign.pdf}
\\
\vspace{0.2cm}
\includegraphics[width=0.24\textwidth]{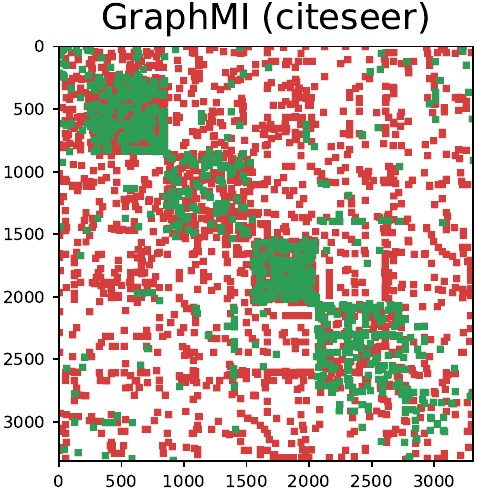}
\hfill
\includegraphics[width=0.24\textwidth]{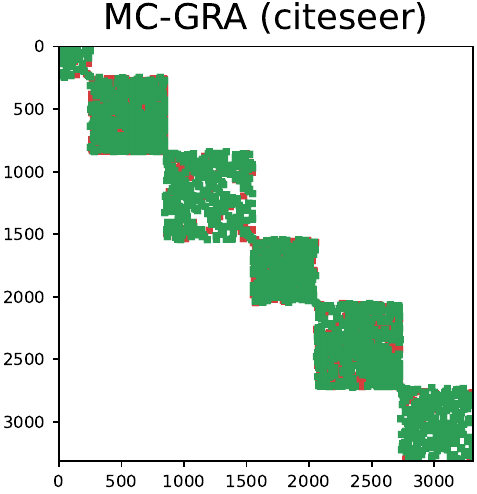}
\hfill
\includegraphics[width=0.24\textwidth]{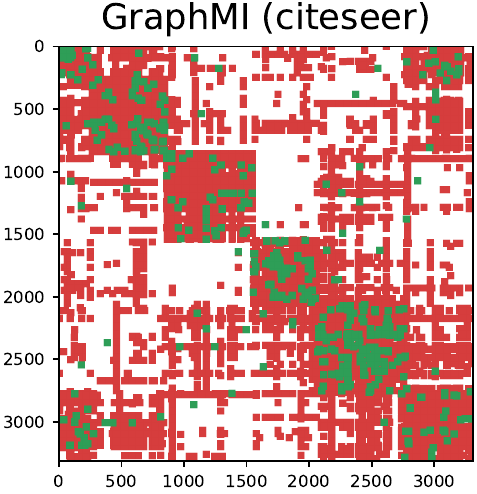}
\hfill
\includegraphics[width=0.24\textwidth]{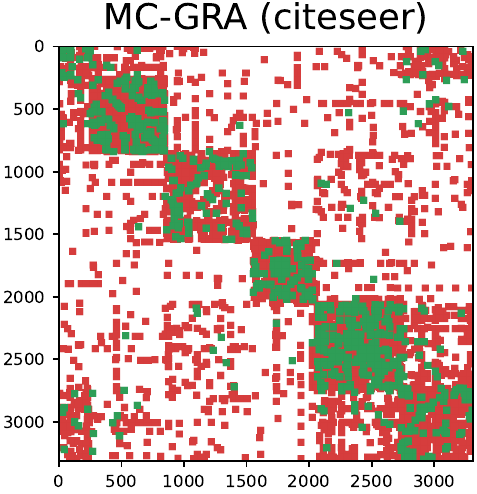}
\caption{Full GPR-GNN recovered adjacency comparisons on Cora and Citeseer. For each dataset, the four columns are GraphMI/unprotected, MC-GRA~(+)/unprotected, GraphMI/protected, and MC-GRA~(+)/protected.}
\label{app: adj:gprgnn:full-a}
\end{figure*}

\begin{figure*}[t]
\centering
\includegraphics[width=0.24\textwidth]{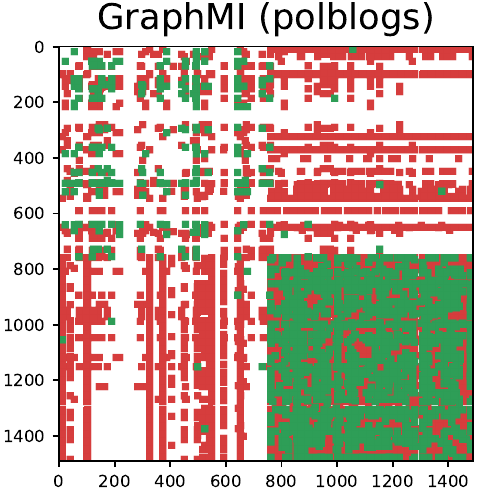}
\hfill
\includegraphics[width=0.24\textwidth]{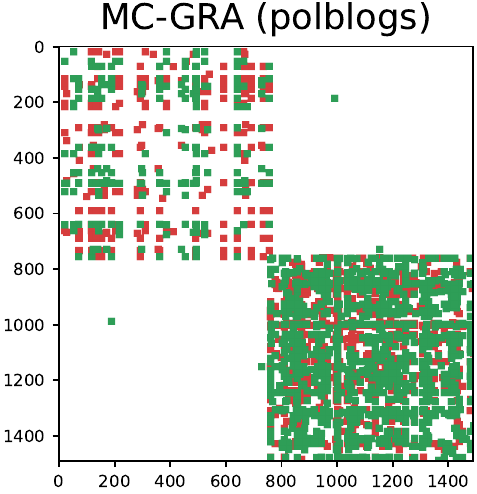}
\hfill
\includegraphics[width=0.24\textwidth]{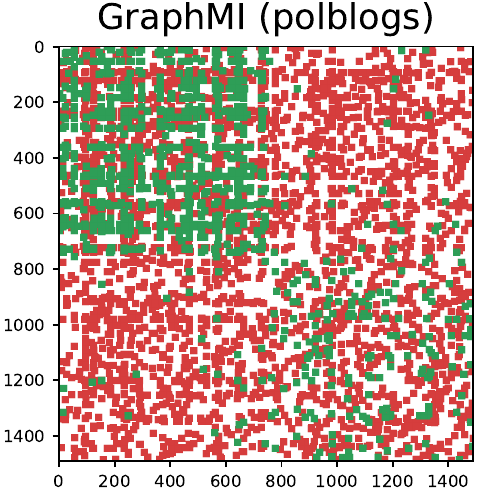}
\hfill
\includegraphics[width=0.24\textwidth]{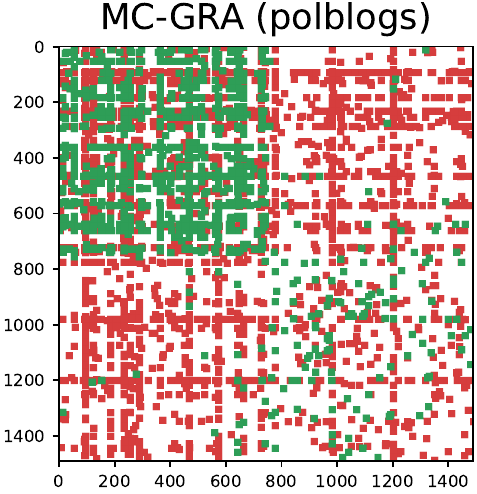}
\\
\vspace{0.2cm}
\includegraphics[width=0.24\textwidth]{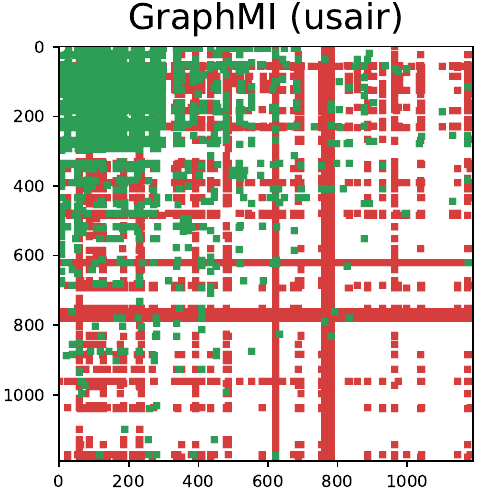}
\hfill
\includegraphics[width=0.24\textwidth]{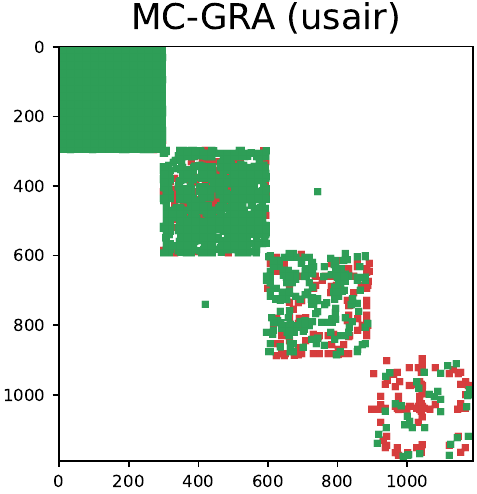}
\hfill
\includegraphics[width=0.24\textwidth]{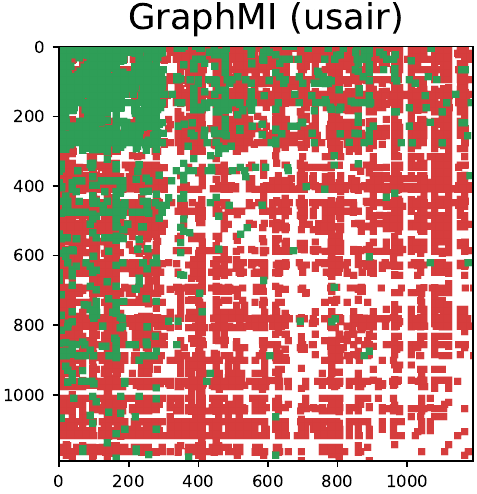}
\hfill
\includegraphics[width=0.24\textwidth]{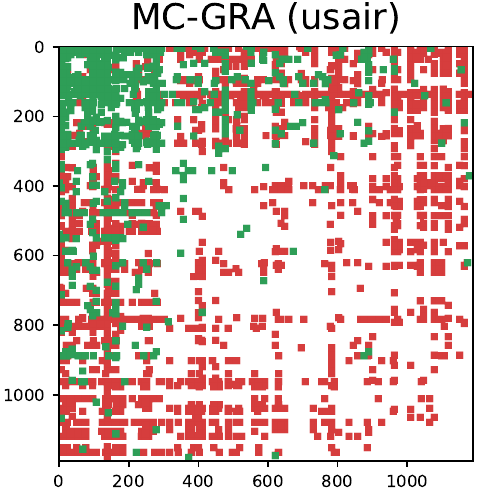}
\caption{Full GPR-GNN recovered adjacency comparisons on Polblogs and USA.}
\label{app: adj:gprgnn:full-b}
\end{figure*}

\begin{figure*}[t]
\centering
\includegraphics[width=0.24\textwidth]{figures/all_refalign/adj_recovery_gprgnn_brazil_graphmi_plain_refalign.pdf}
\hfill
\includegraphics[width=0.24\textwidth]{figures/all_refalign/adj_recovery_gprgnn_brazil_mcgra_plain_refalign.pdf}
\hfill
\includegraphics[width=0.24\textwidth]{figures/all_refalign/adj_recovery_gprgnn_brazil_graphmi_protected_refalign.pdf}
\hfill
\includegraphics[width=0.24\textwidth]{figures/all_refalign/adj_recovery_gprgnn_brazil_mcgra_protected_refalign.pdf}
\\
\vspace{0.2cm}
\includegraphics[width=0.24\textwidth]{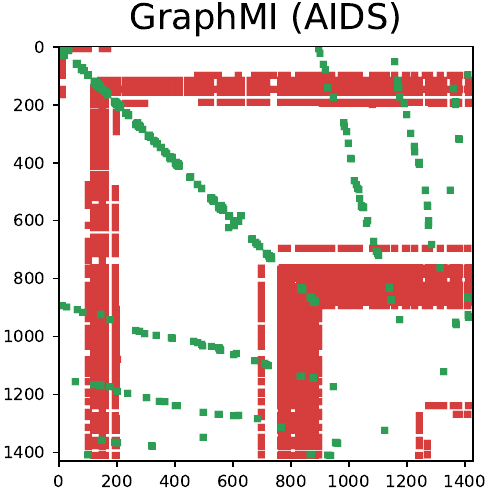}
\hfill
\includegraphics[width=0.24\textwidth]{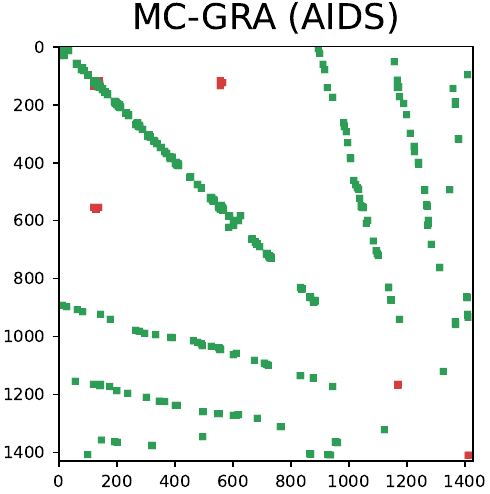}
\hfill
\includegraphics[width=0.24\textwidth]{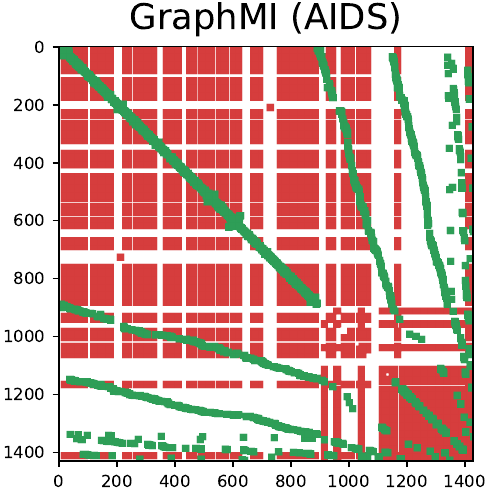}
\hfill
\includegraphics[width=0.24\textwidth]{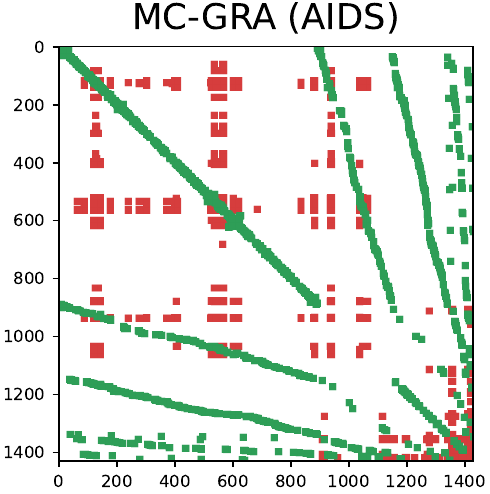}
\caption{Full GPR-GNN recovered adjacency comparisons on Brazil and AIDS.}
\label{app: adj:gprgnn:full-c}
\end{figure*}

\begin{figure*}[t]
\centering
\includegraphics[width=0.24\textwidth]{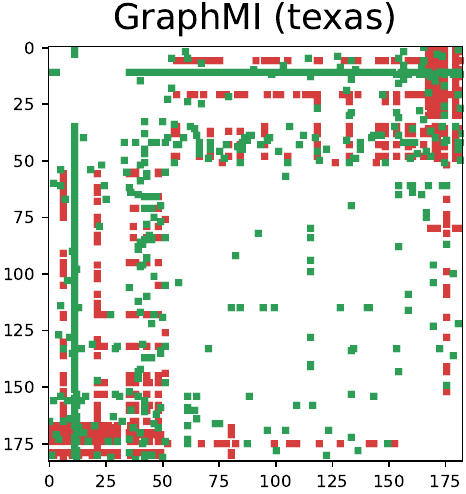}
\hfill
\includegraphics[width=0.24\textwidth]{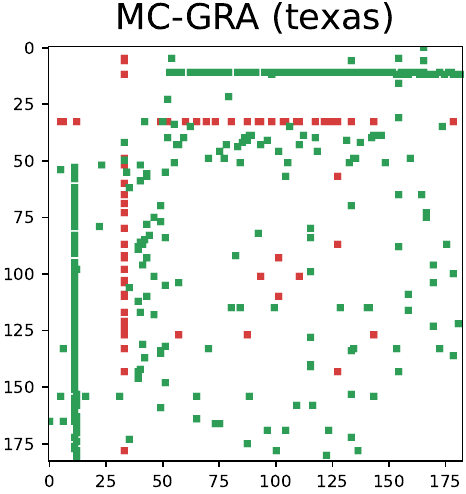}
\hfill
\includegraphics[width=0.24\textwidth]{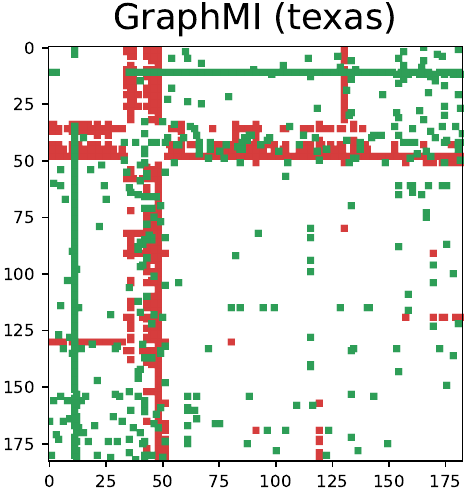}
\hfill
\includegraphics[width=0.24\textwidth]{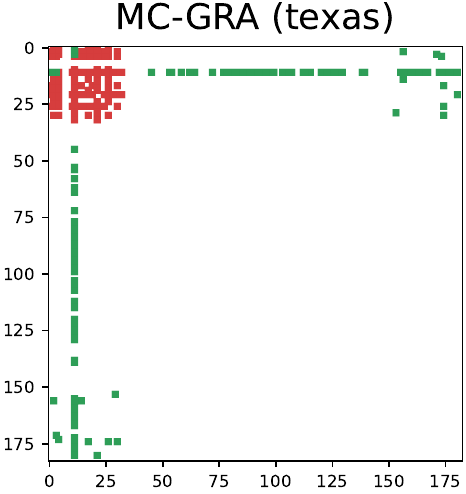}
\\
\vspace{0.2cm}
\includegraphics[width=0.24\textwidth]{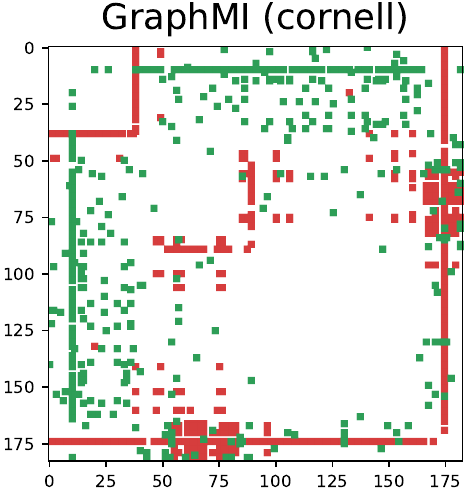}
\hfill
\includegraphics[width=0.24\textwidth]{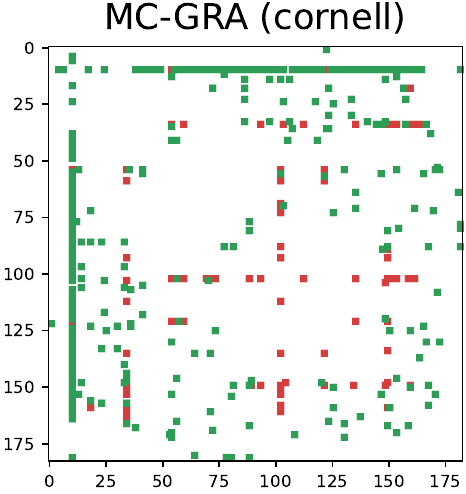}
\hfill
\includegraphics[width=0.24\textwidth]{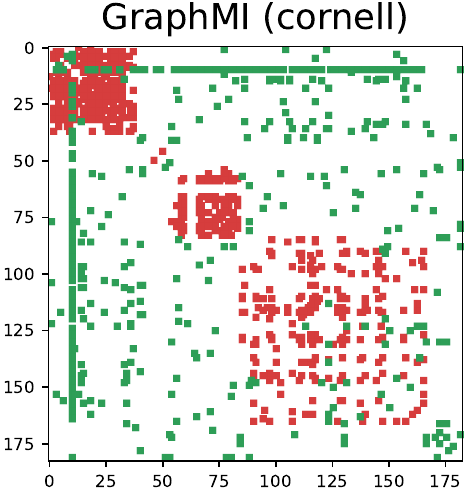}
\hfill
\includegraphics[width=0.24\textwidth]{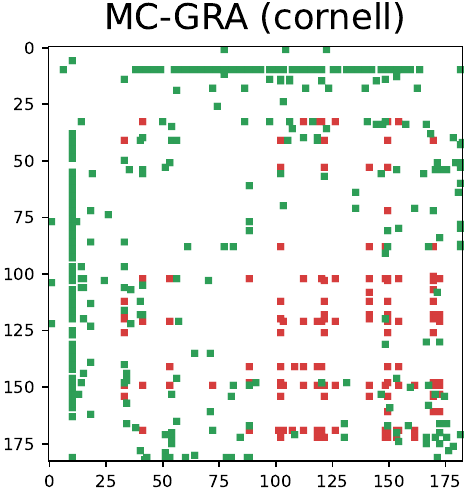}
\caption{Full GPR-GNN recovered adjacency comparisons on Texas and Cornell.}
\label{app: adj:gprgnn:full-d}
\end{figure*}

\begin{figure*}[t]
\centering
\includegraphics[width=0.24\textwidth]{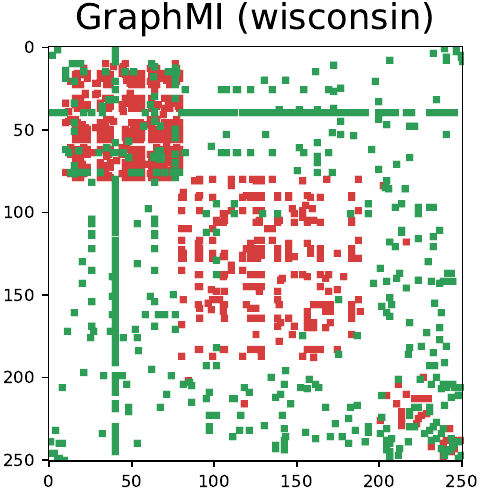}
\hfill
\includegraphics[width=0.24\textwidth]{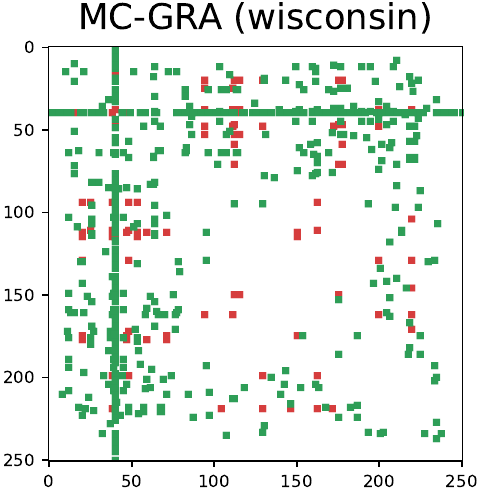}
\hfill
\includegraphics[width=0.24\textwidth]{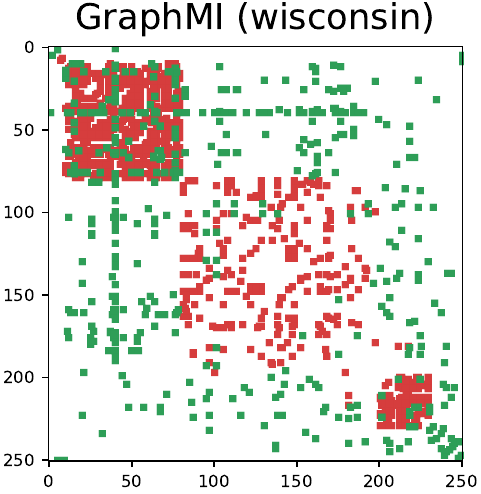}
\hfill
\includegraphics[width=0.24\textwidth]{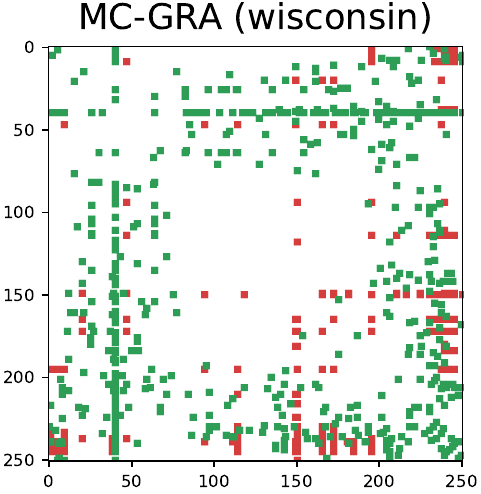}
\caption{Full GPR-GNN recovered adjacency comparisons on Wisconsin.}
\label{app: adj:gprgnn:full-e}
\end{figure*}

\textbf{Tracking the MI terms.}
We show the learning curves of MC-GRA~(+) and MC-GPB~(+) on each dataset below.

For MC-GRA~(+), the propagation loss (hidden-layer alignment terms in Eq.~\eqref{eqn: MC-GRA}) and the output loss (output alignment term) converge to near zero (Fig.~\ref{appd: curve:mc-gra}), showing that the model efficiently approximates the original Markov chain. For MC-GPB~(+), we track the three constraint terms (privacy, accuracy, complexity) per layer and average across layers; we also plot the classification cross-entropy. The training curves (Fig.~\ref{appd: curve:mc-gra} and the following figure) show that both privacy and complexity constraints exhibit a downward trend throughout training (with some fluctuation in the privacy constraint), especially on USA; the accuracy curves follow similar patterns.

\begin{figure}[htbp]
\centering
\hfill
\includegraphics[width=0.3\linewidth]{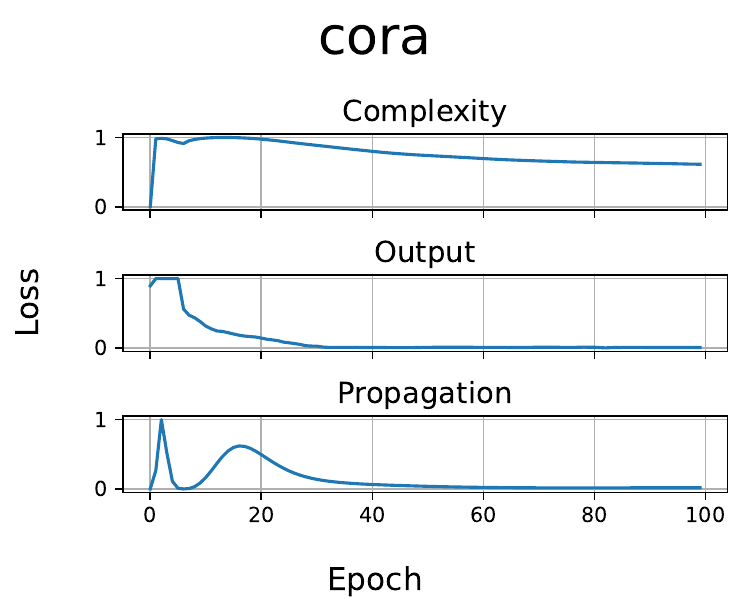}
\hfill
\includegraphics[width=0.3\linewidth]{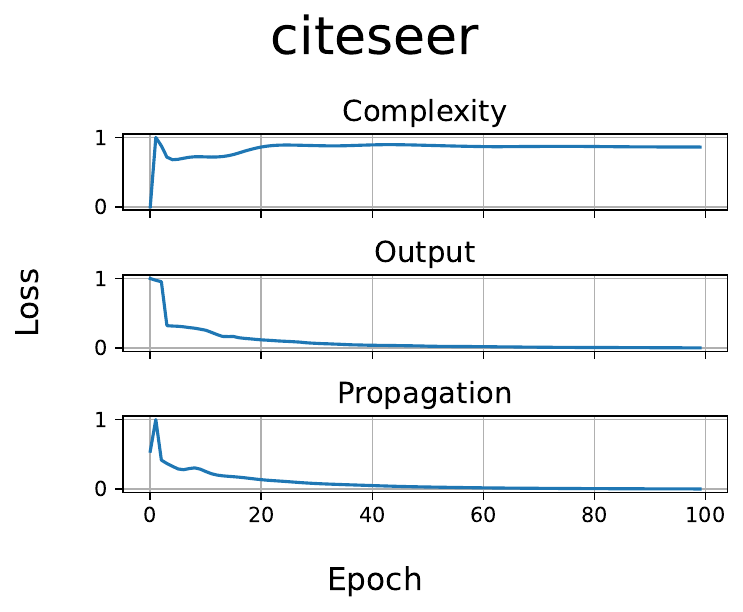}
\hfill
\includegraphics[width=0.3\linewidth]{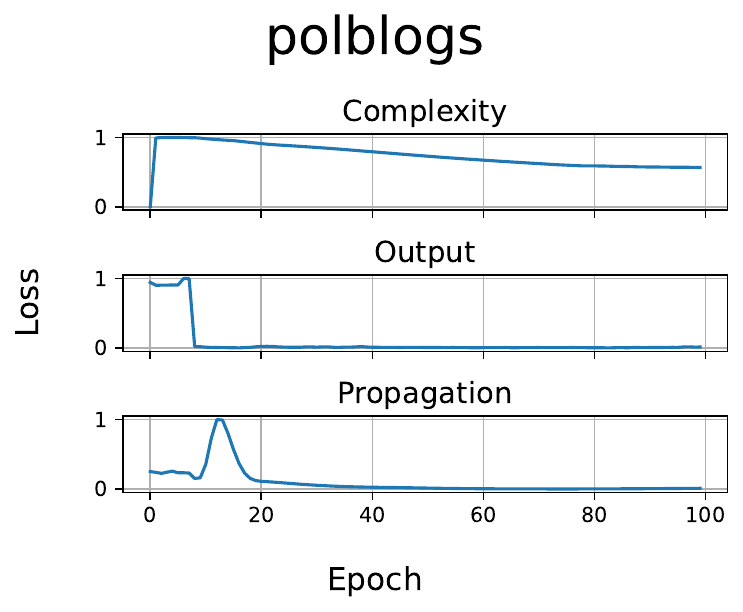}
\hfill
\\
\vspace{0.3cm}
\hfill
\includegraphics[width=0.3\linewidth]{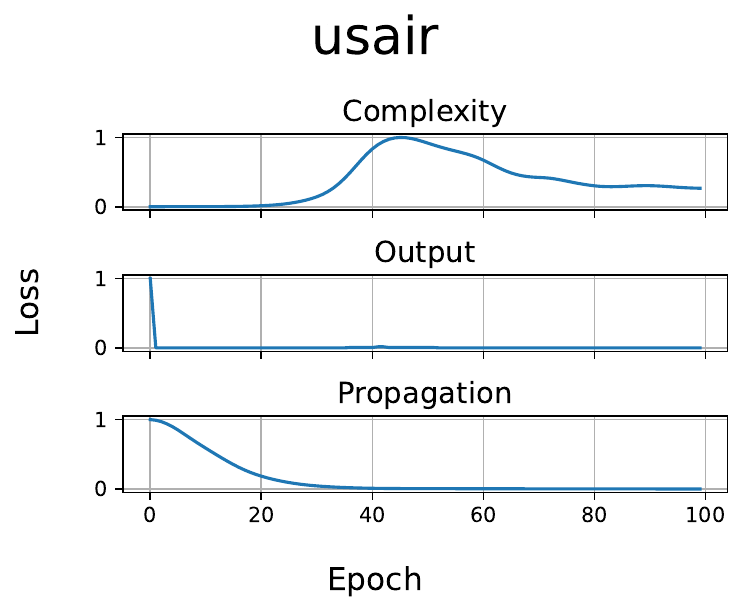}
\hfill
\includegraphics[width=0.3\linewidth]{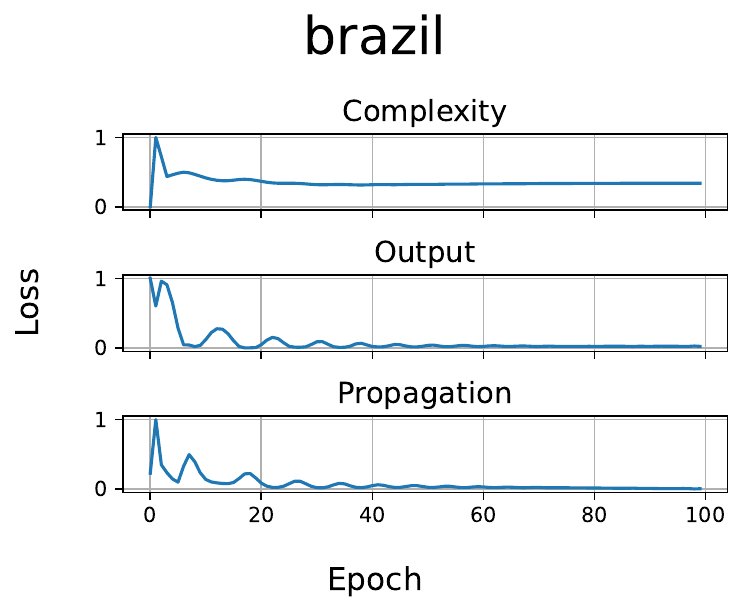}
\hfill
\includegraphics[width=0.3\linewidth]{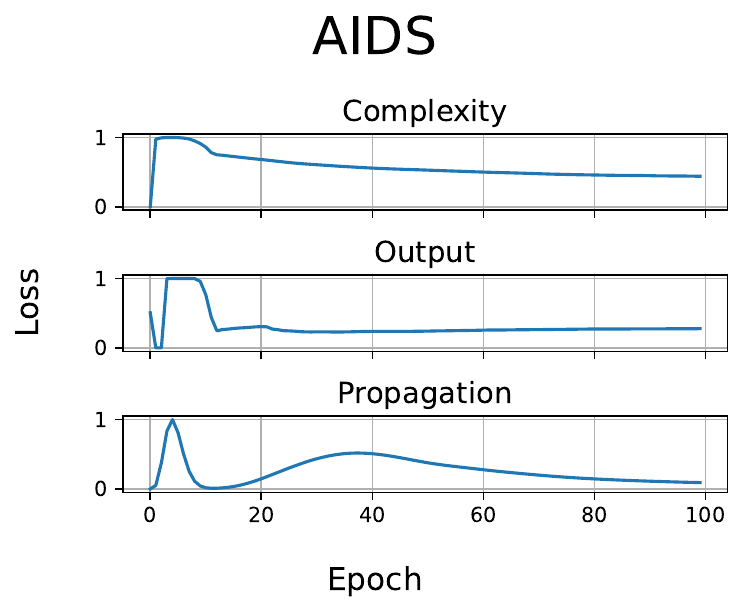}
\hfill
\vspace{-6pt}
\caption{
Training curves of MC-GRA~(+) on each dataset.
}
\label{appd: curve:mc-gra}
\end{figure}

\begin{figure}[htbp]
\vspace{-20pt}
\centering
\hfill
\includegraphics[width=0.3\linewidth]{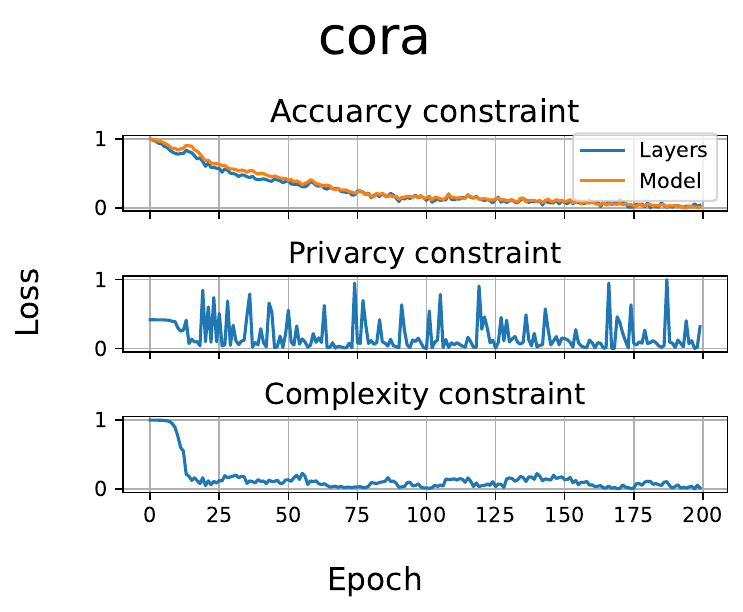}
\hfill
\includegraphics[width=0.3\linewidth]{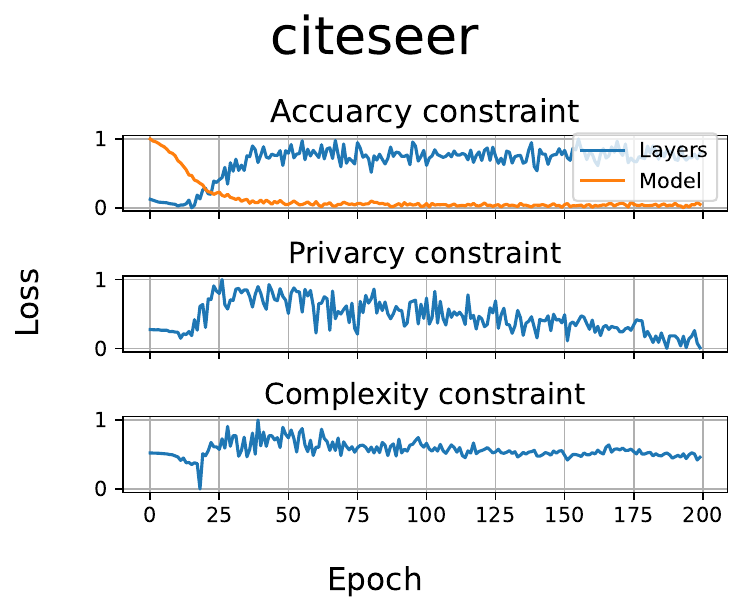}
\hfill
\includegraphics[width=0.3\linewidth]{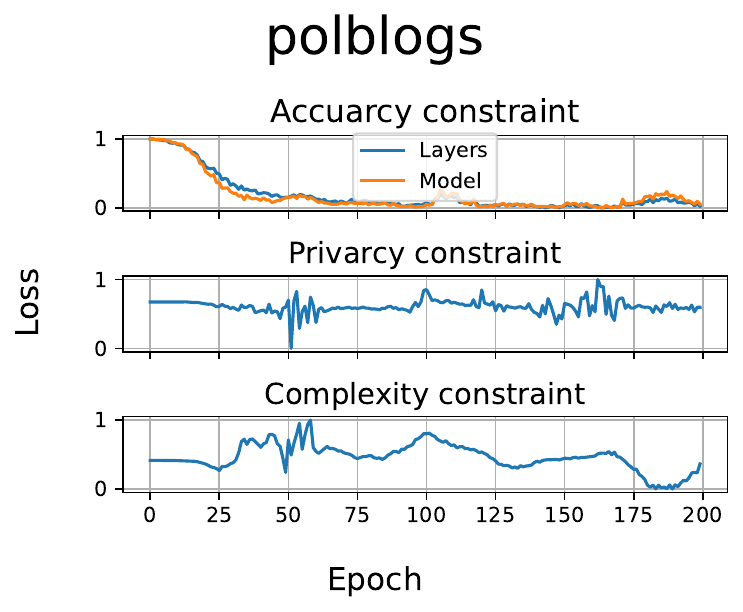}
\hfill
\\
\vspace{0.3cm}
\hfill
\includegraphics[width=0.3\linewidth]{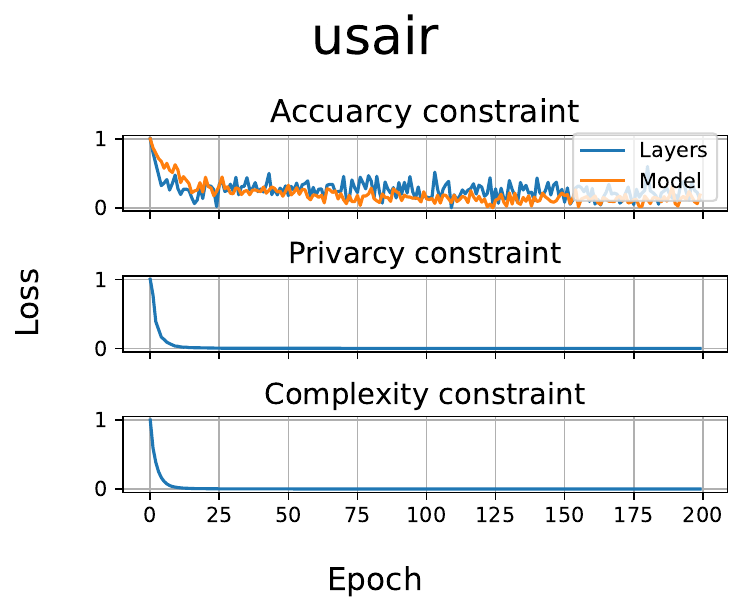}
\hfill
\includegraphics[width=0.3\linewidth]{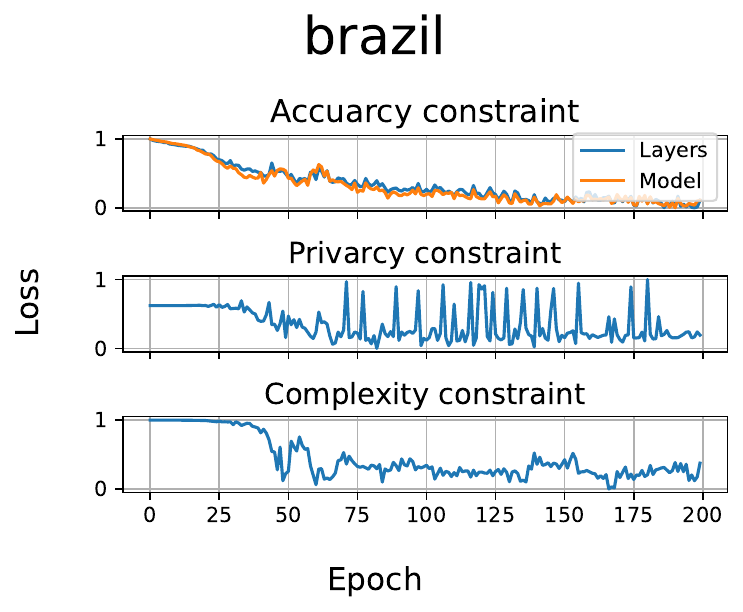}
\hfill
\includegraphics[width=0.3\linewidth]{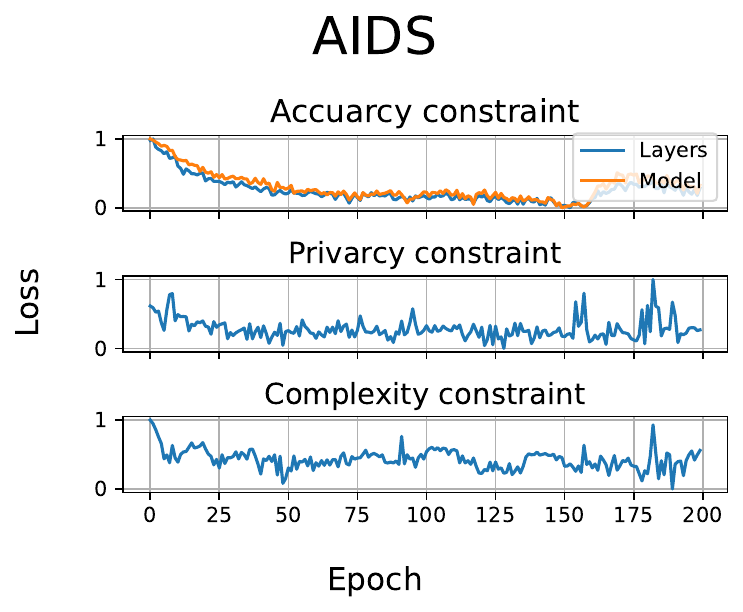}
\hfill
\vspace{-6pt}
\caption{
Training curves of MC-GPB~(+) on each dataset.
}
\label{appd: curve:mc-gpb}
\end{figure}

\textbf{A further analysis with the training dynamics.}
We show the graph information planes with and without MC-GPB~(+) below.
Without MC-GPB~(+), the adjacency-leakage proxy $R_{\mathrm{AUC}}(A;\bm{Z})$ typically increases during early training before gradually declining. Applying MC-GPB~(+) enhances this decline while preventing the model from discarding task-relevant information that would hurt accuracy.

\begin{figure}[htbp]
\centering
\subfloat[GNN with unprotected training]
{\includegraphics[width=0.45\linewidth]{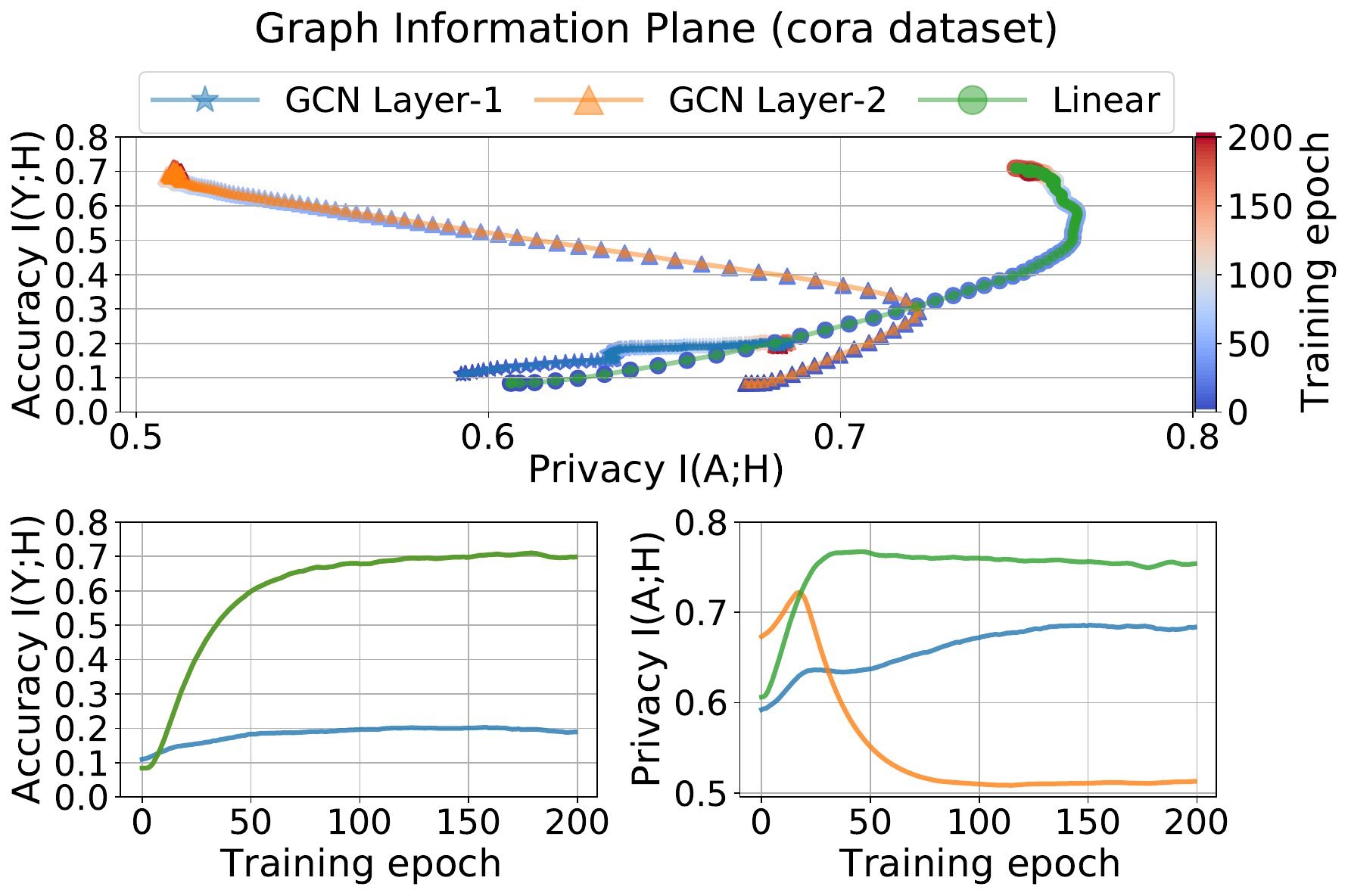}}
\subfloat[GNN trained with MC-GPB~(+)]
{\includegraphics[width=0.45\linewidth]{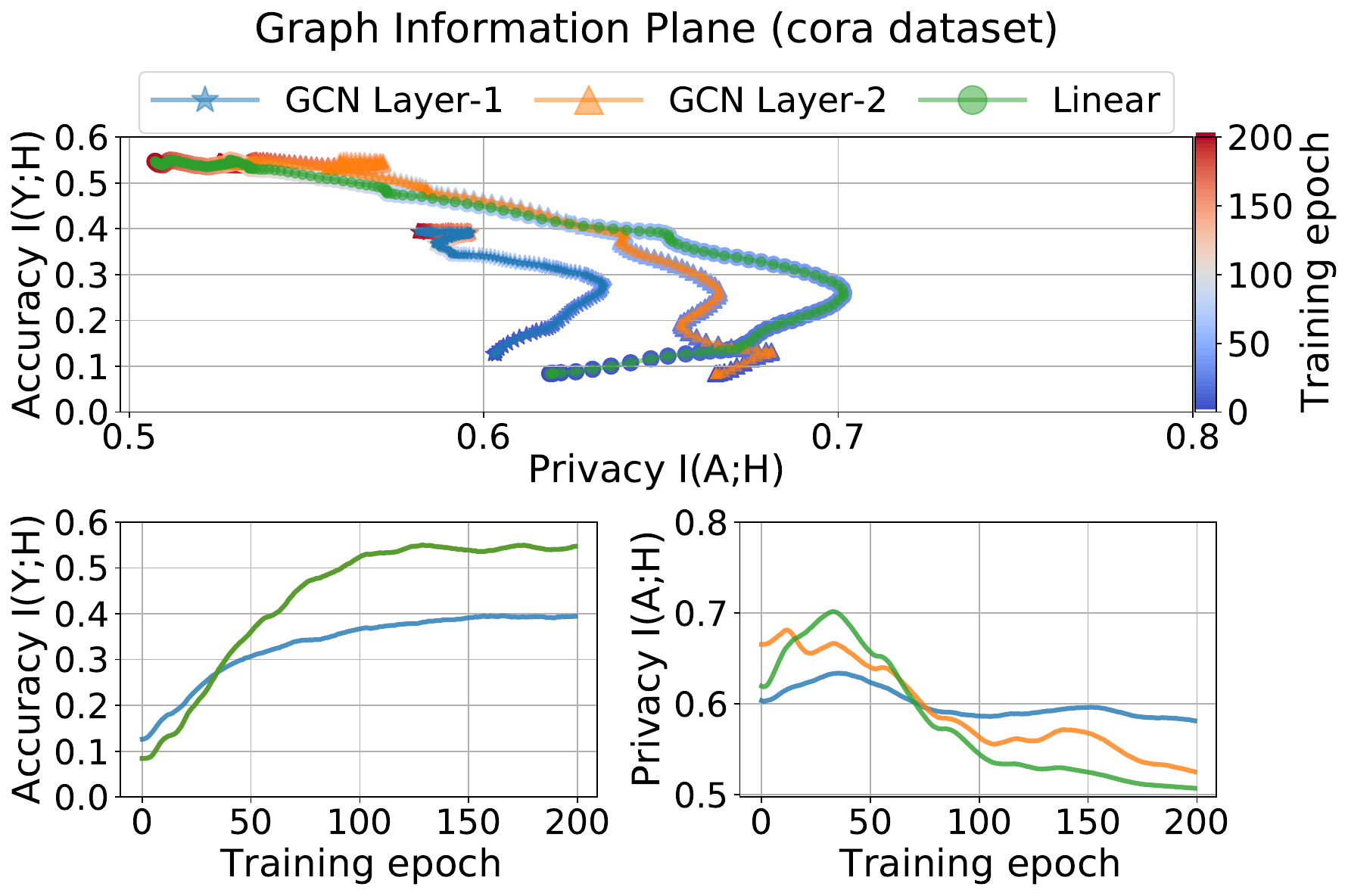}}
\vspace{-6pt}
\caption{
Graph information plane on Cora dataset.
}
\label{appd: gip:cora}
\end{figure}

\begin{figure}[htbp]
\centering
\subfloat[GNN with unprotected training]
{\includegraphics[width=0.45\linewidth]{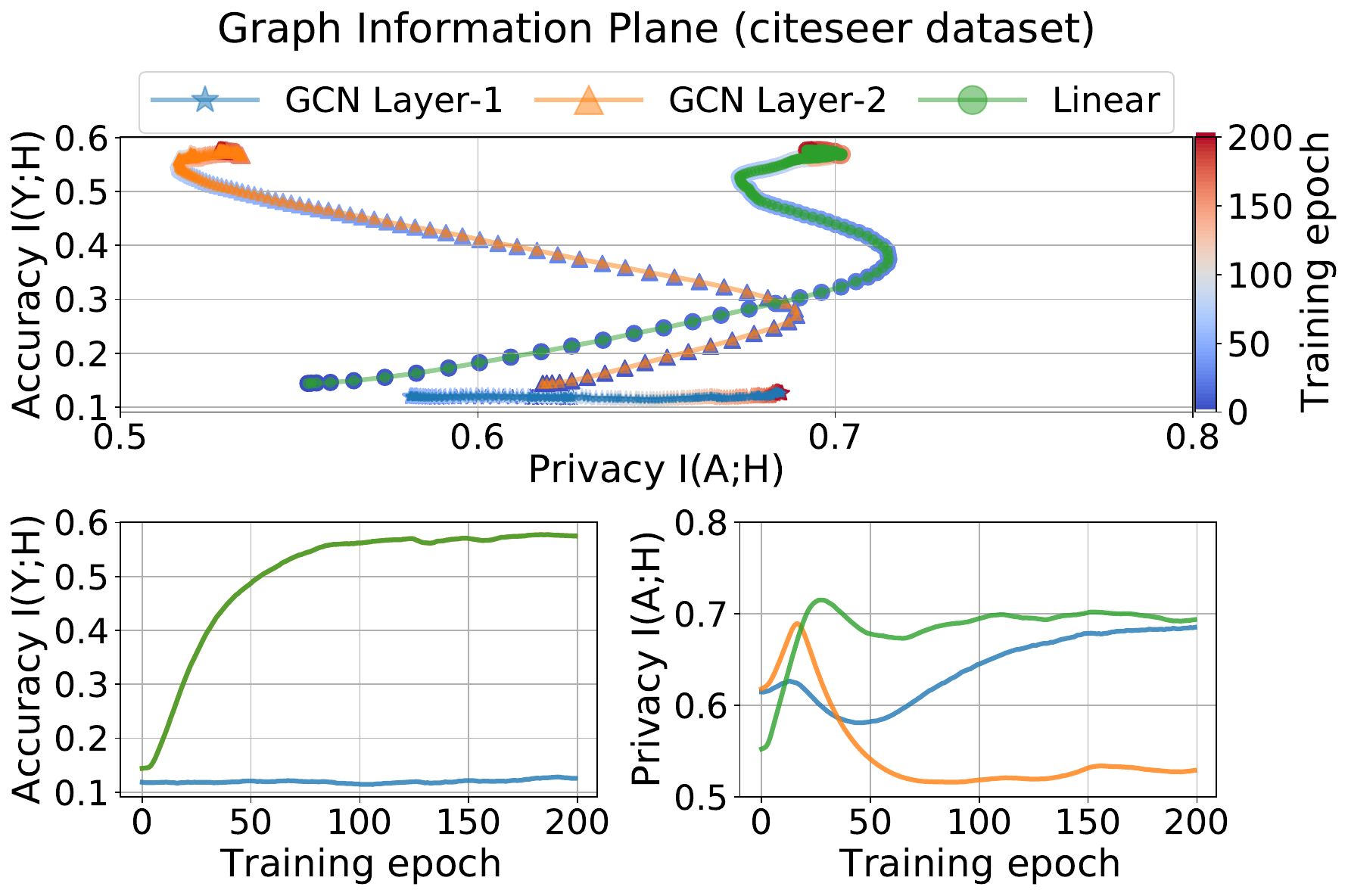}}
\subfloat[GNN trained with MC-GPB~(+)]
{\includegraphics[width=0.45\linewidth]{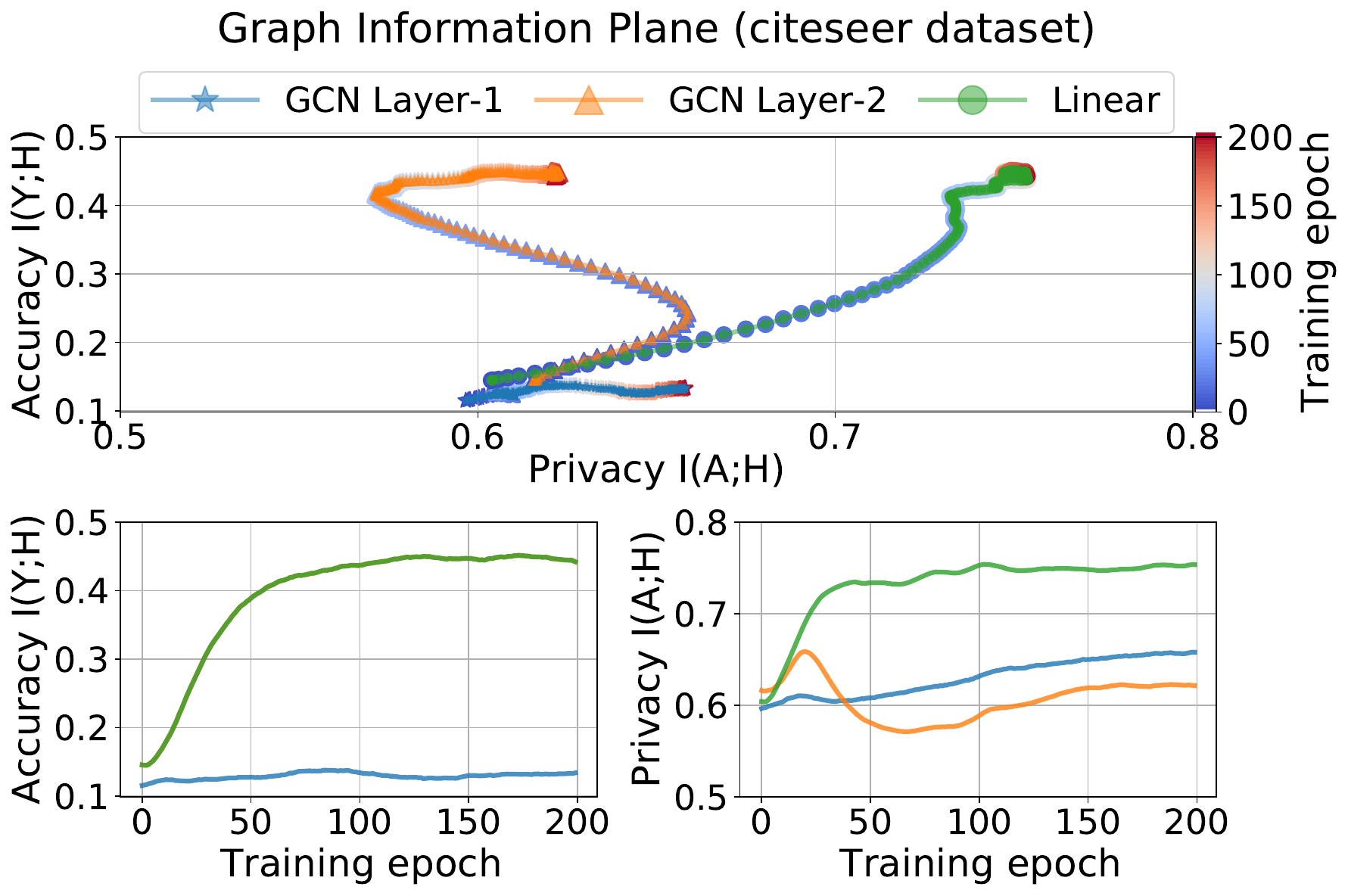}}
\vspace{-6pt}
\caption{
Graph information plane on Citeseer dataset.
}
\end{figure}

\begin{figure}[htbp]
\centering
\subfloat[GNN with unprotected training]
{\includegraphics[width=0.45\linewidth]{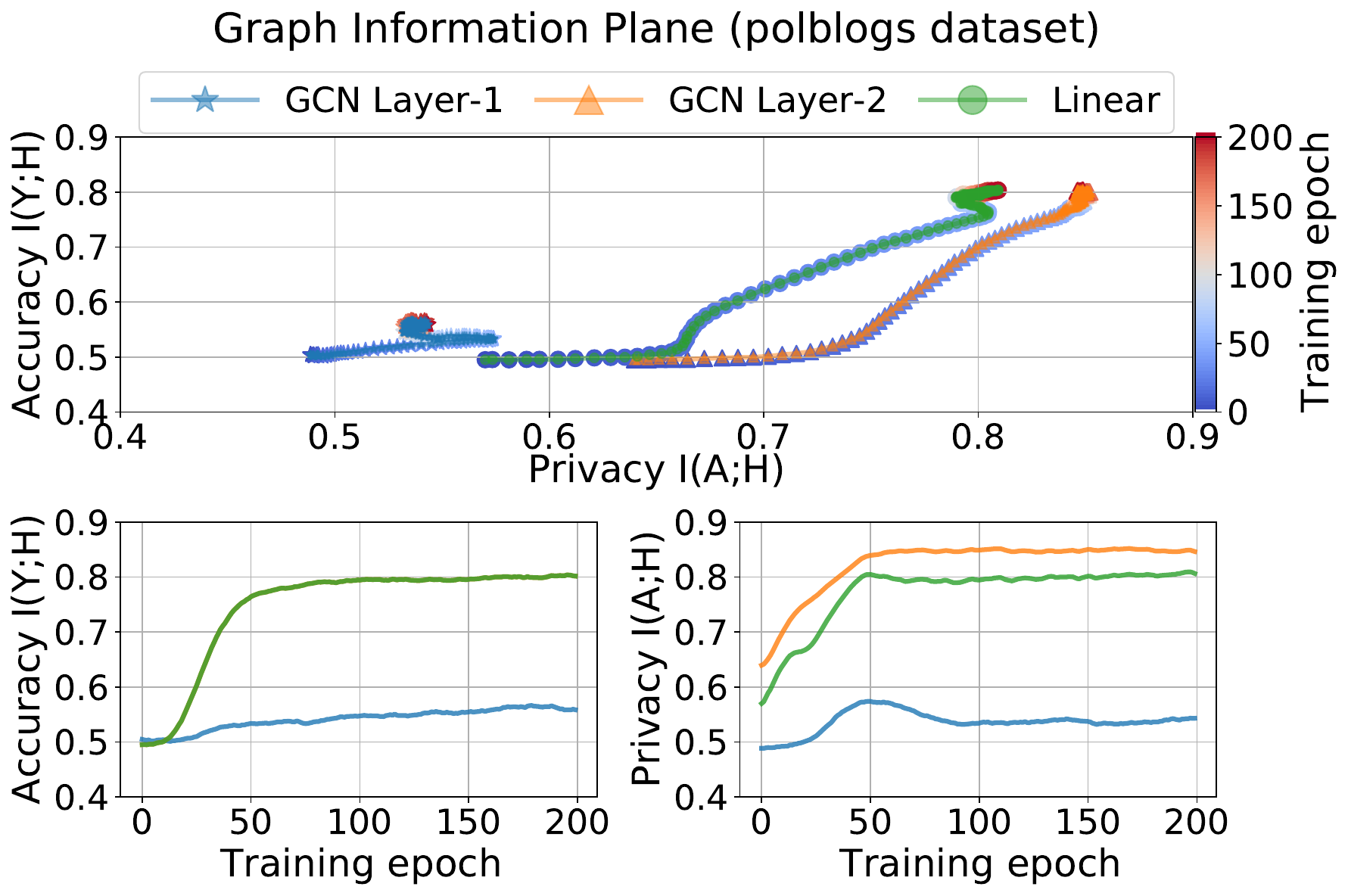}}
\subfloat[GNN trained with MC-GPB~(+)]
{\includegraphics[width=0.45\linewidth]{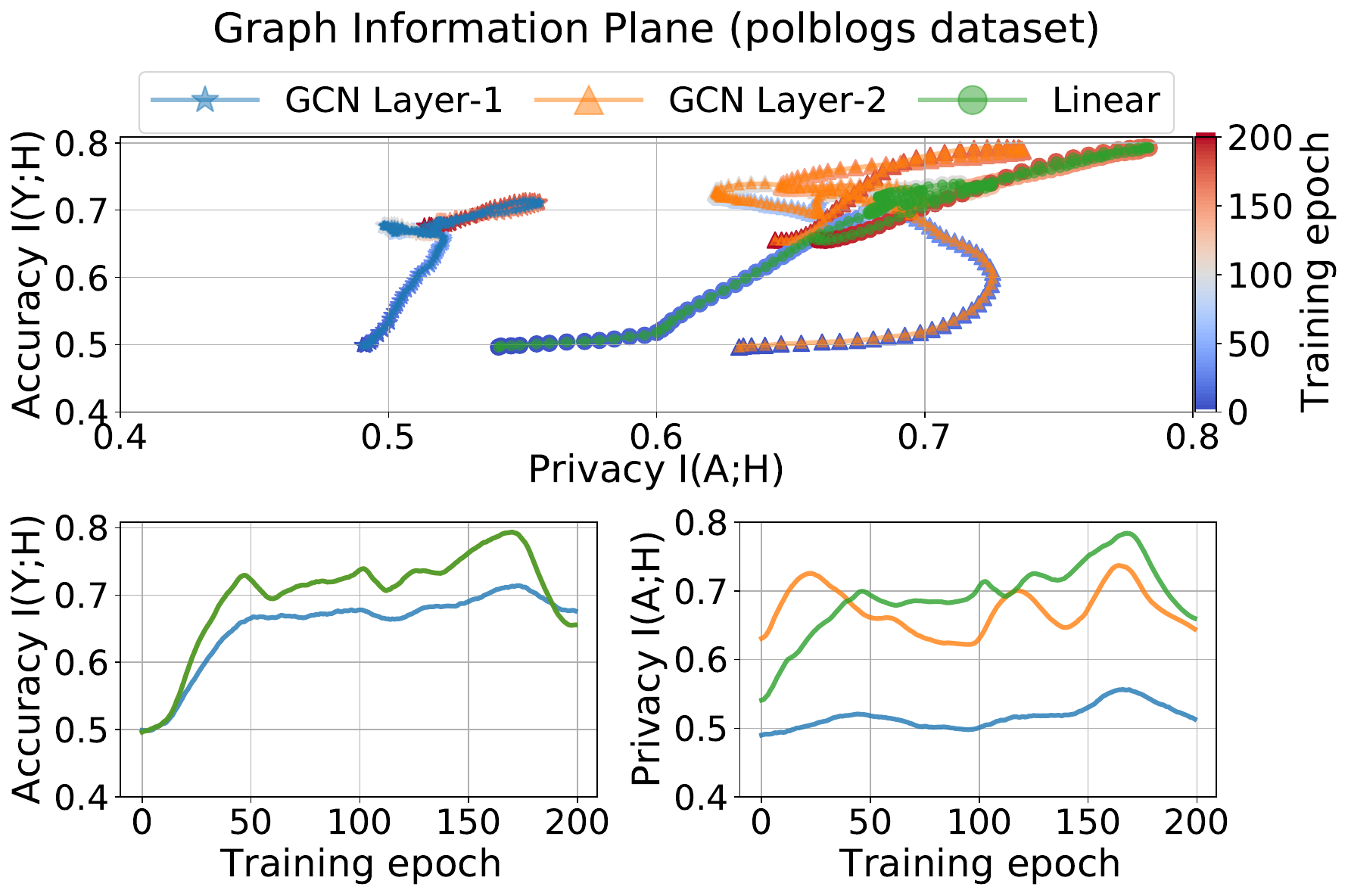}}
\vspace{-6pt}
\caption{
Graph information plane on Polblogs dataset.
}
\end{figure}

\begin{figure}[htbp]
\centering
\subfloat[GNN with unprotected training]
{\includegraphics[width=0.45\linewidth]{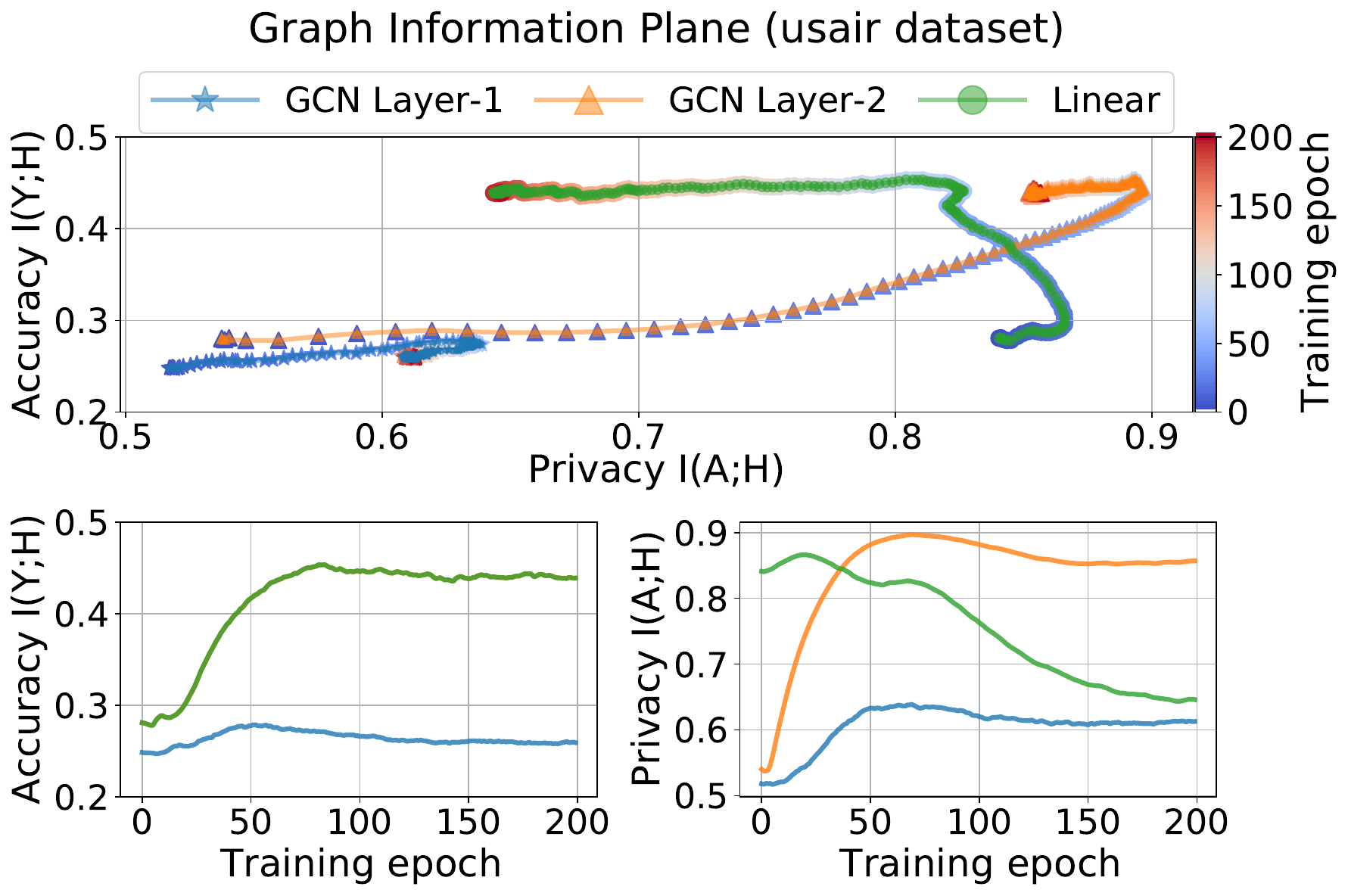}}
\subfloat[GNN trained with MC-GPB~(+)]
{\includegraphics[width=0.45\linewidth]{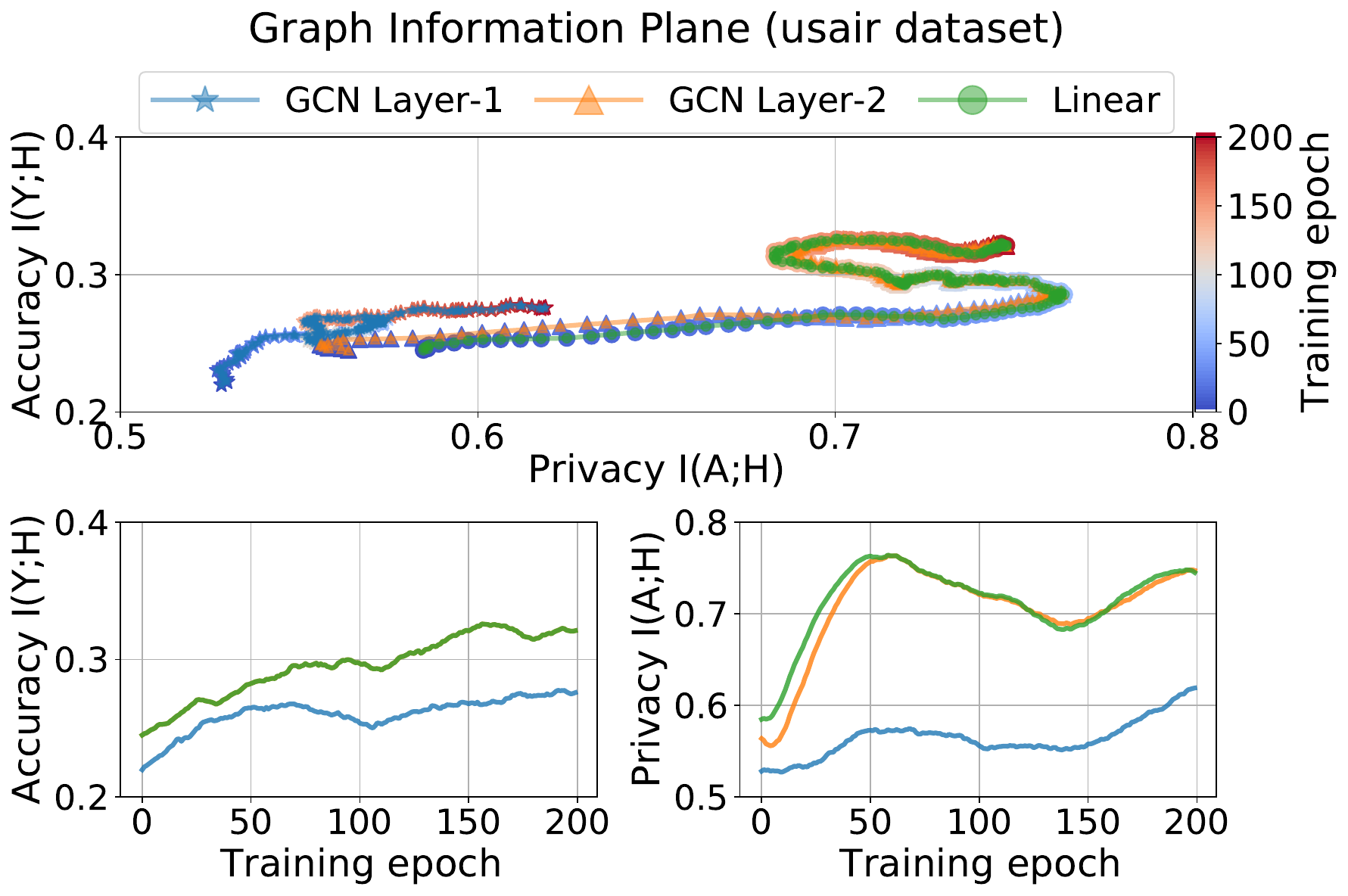}}
\vspace{-6pt}
\caption{
Graph information plane on USA dataset.
}
\end{figure}

\begin{figure}[htbp]
\centering
\subfloat[GNN with unprotected training]
{\includegraphics[width=0.45\linewidth]{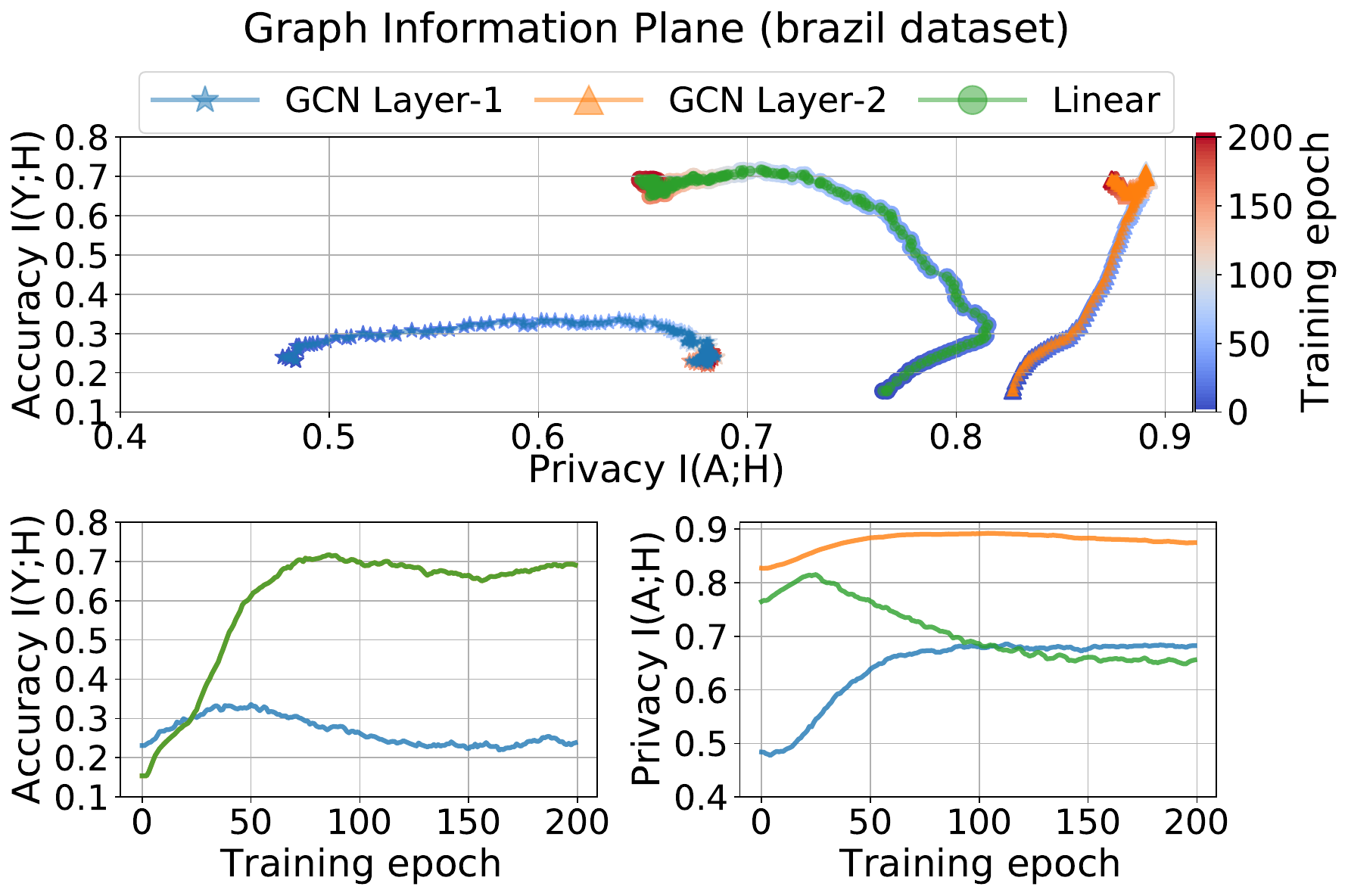}}
\subfloat[GNN trained with MC-GPB~(+)]
{\includegraphics[width=0.45\linewidth]{figures/graph-information-plane/plane-defense-brazil-3-0.95}}
\vspace{-6pt}
\caption{
Graph information plane on Brazil dataset.
}
\end{figure}

\begin{figure}[htbp]
\centering
\subfloat[GNN with unprotected training]
{\includegraphics[width=0.45\linewidth]{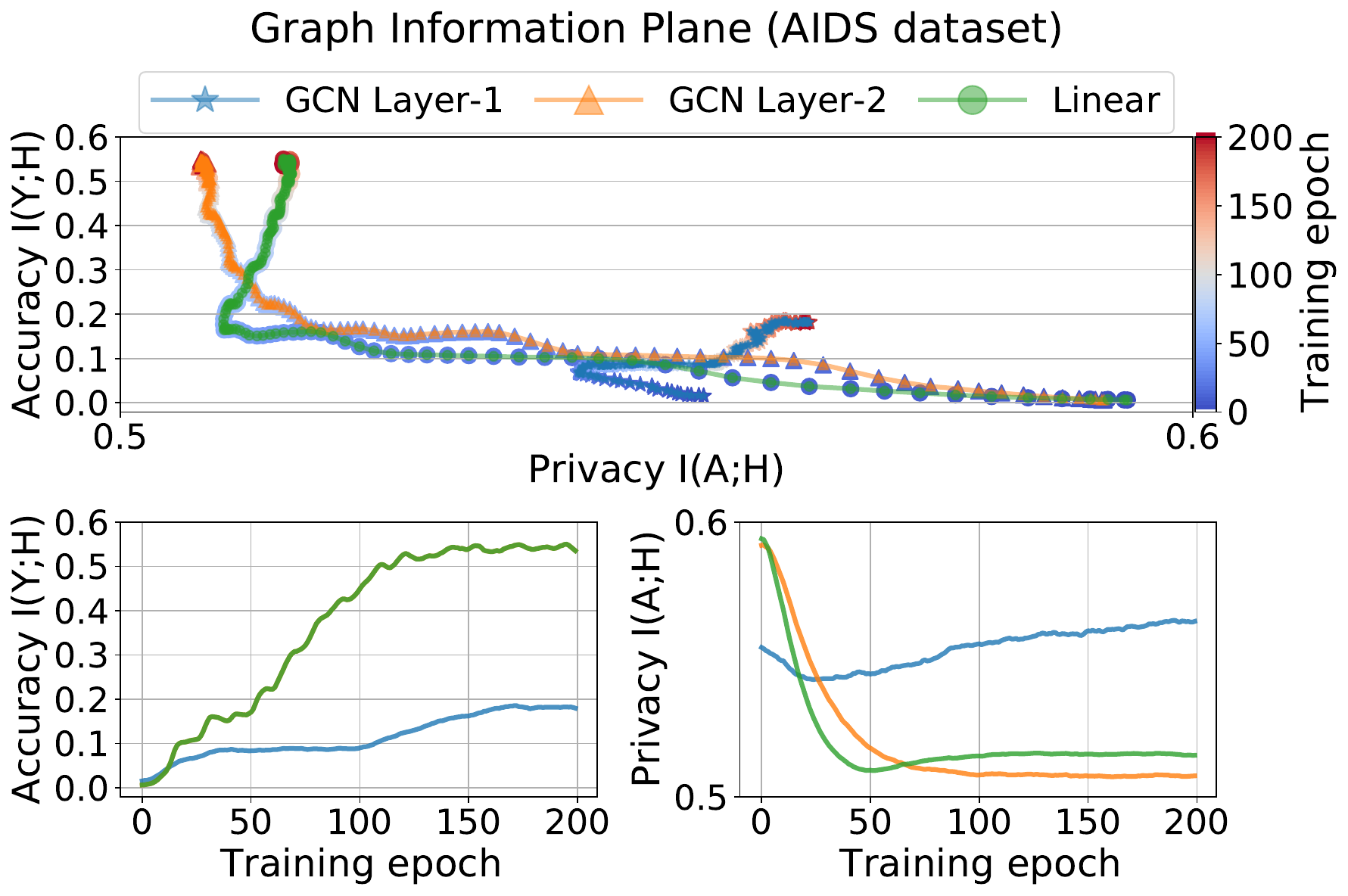}}
\subfloat[GNN trained with MC-GPB~(+)]
{\includegraphics[width=0.45\linewidth]{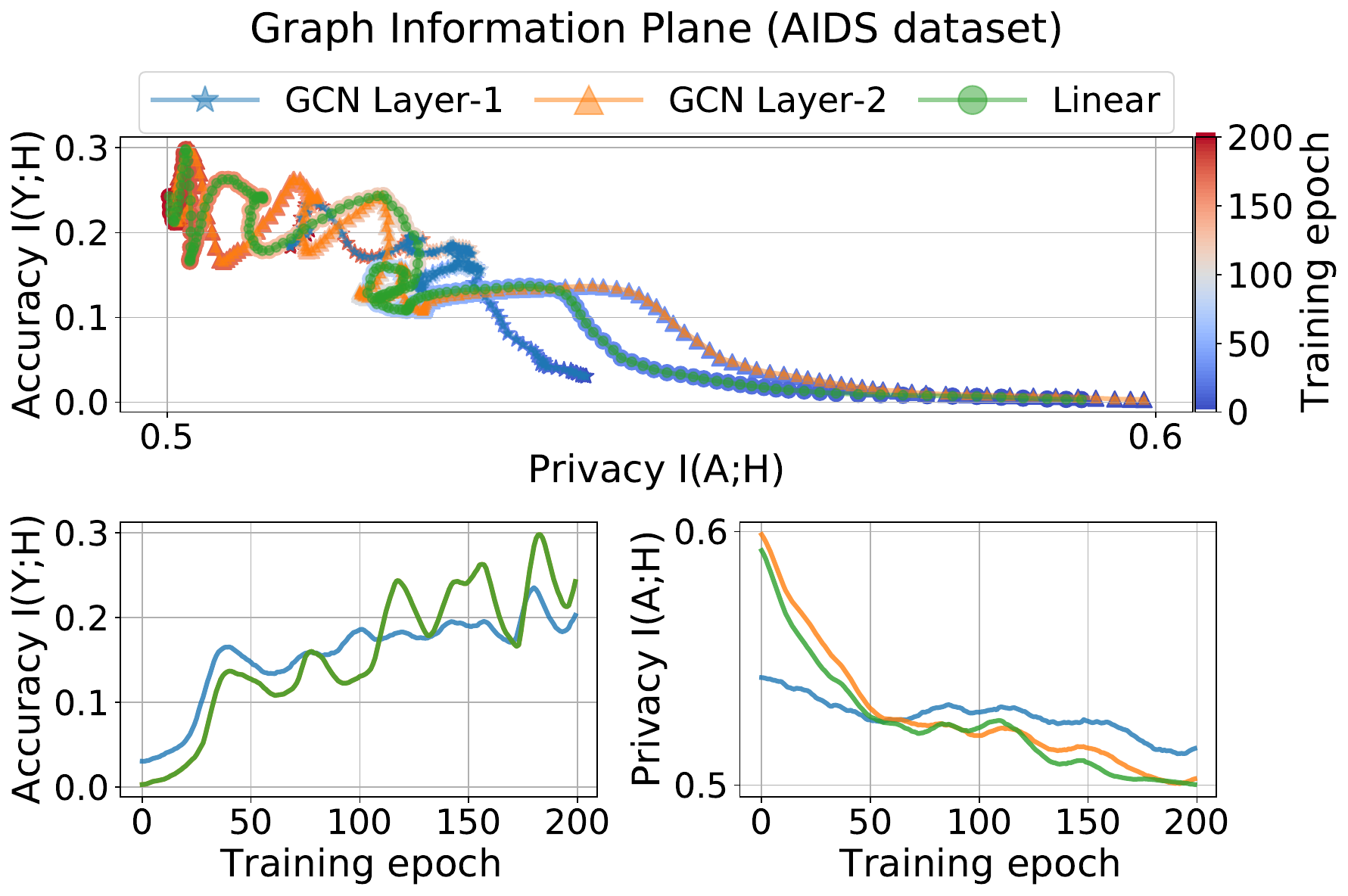}}
\vspace{-6pt}
\caption{
Graph information plane on AIDS dataset.
}
\label{appd: gip:aids}
\end{figure}

%% file: appendix/3-further-related-work.tex
\section{Additional Related Work}
\label{app: related-work}

This appendix surveys the landscape of MIAs across images and text, including the attack and defense strategies proposed for each modality.

\subsection{Model Inversion Attacks and Defenses on Images}
\label{ssec: MIA-image}

\paragraph{Attacks.} 
Pioneering studies~\citep{szegedy2013intriguing,fredrikson2014privacy,fredrikson2015model,hidano2017model} introduce MIA in shallow classifiers and show that an adversary can reconstruct coarse monochrome images from model outputs alone. These early methods, however, do not scale to modern deep architectures: when applied to image-classification networks, their reconstructions have very limited fidelity.

Generative Model Inversion (GMI)~\citep{zhang2020secret} is the first attack that reliably targets convolutional neural networks. By exploiting a generative adversarial network (GAN) as a learned prior, GMI constrains the search to the natural image manifold and reveals private training examples with high visual realism. Variational Model Inversion (VMI)~\citep{wang2021variational} generalizes this idea: it casts MIA as variational inference and optimizes a variational objective with a StyleGAN prior, achieving higher success rates and greater diversity. Projection Predictive Attack (PPA)~\citep{struppek2022ppa} further reduces computational cost and broadens applicability by decoupling the generator from the classifier; a single off-the-shelf GAN suffices to attack multiple targets with minimal fine-tuning, demonstrating that MIA remains practical even under substantial domain shifts.
These methods share a focus on generative priors and gradient-based search in a learned or fixed latent space.

Subsequent work improves both effectiveness and robustness. LOMMA~\citep{LOMMA} observes that the cross-entropy objective adopted in earlier attacks~\citep{zhang2020secret,chen2021knowledge} is prone to overfitting and proposes a margin-based loss that better guides gradient updates. PLG-MI~\citep{PLGMI} adopts a conditional GAN driven by pseudo-labels to stabilize training. Independently, RLBMI~\citep{han2023reinforcement} reformulates the latent-space exploration of MIRROR~\citep{MIRROR} as a Markov decision process optimized by Soft Actor-Critic, enabling more efficient black-box attacks. Finally, C2FMI~\citep{C2FMI} introduces a coarse-to-fine refinement stage that boosts visual quality in strict black-box settings.

\paragraph{Defenses.} 
The most widely studied defense, differential privacy (DP), guarantees protection against membership inference but only marginally impedes inversion~\citep{fredrikson2014privacy,abadi2016deep,zhang2020secret}. More targeted defenses limit the mutual information that the model leaks. MID~\citep{yang2019adversarial} explicitly minimizes $I(X;\hat{Y})$ between inputs and predictions, whereas Bilateral Dependency Optimization (BiDO)~\citep{peng2022bilateral} simultaneously suppresses $I(X;Z)$, removing redundant cues from latent features, and maximizes $I(Y;Z)$ to preserve task relevance. Although both defenses heavily distort the attacker's reconstructions, they also degrade accuracy, underscoring the difficulty of achieving a favorable privacy-utility trade-off. Therefore, principled and practical defenses against MIA remain an open research problem.

\subsection{Model Inversion Attacks and Defenses on Texts}
\label{ssec: MIA-text}

\paragraph{Attacks.}
Model inversion attacks on language models shift the focus from images and CNNs to text processed by Transformer architectures.
\citet{carlini2021extracting} shows that GPT-2 memorizes and leaks entire training sequences through simple black-box queries, motivating a systematic study of textual MIA.
Text Revealer~\citep{zhang2022text} adapts GAN-based inversion to text classification: it first extracts frequent template phrases from public corpora and then fine-tunes a generator to minimize the classifier’s loss, synthesizing text that closely matches the private distribution.
Vec2Text~\citep{morris2023text} treats inversion as controlled embedding optimization, iteratively refining candidate text so its embedding converges to a target.
Information-theoretic attacks instead train an auxiliary decoder to infer words directly from embeddings~\citep{song2020information}, and GEIA~\citep{li2023sentence} improves fluency by training a generative decoder with teacher forcing.
Transferable EI~\citep{huang2024transferable} eliminates query dependence: it first steals a surrogate model and then adversarially trains an inversion decoder against that surrogate, enabling fully query-free attacks.

\paragraph{Defenses.}
Most defenses perturb the embedding space to obscure the mapping between internal representations and plaintext.
\citet{chen2024text} prepend a language-ID mask to token embeddings, while \citet{parikh2022canary} use concatenation of character-level features; both techniques increase inversion uncertainty without degrading accuracy.
Recent work also explores sequence-level leakage through model outputs.
The logit2prompt method~\citep{morris2023language} reconstructs hidden prompts from next-token logits but is query-intensive and relies on logit access.
The output2prompt method~\citep{zhang2024extracting} removes this constraint: a sparse encoder-decoder conditioned solely on generated text recovers system prompts with high cosine similarity to the originals, showing that output-level defenses remain an open problem.

Beyond the centralized setting surveyed above, graph reconstruction attacks have also been studied in federated learning settings, where gradient-based reconstruction can recover graph structure from shared model updates in federated GNN training~\citep{zhang2022federated}. This threat model differs from the centralized post-hoc setting studied in our work but shares the concern of adjacency leakage from learned representations.

In the graph setting (Secs.~\ref{sec: overview}--\ref{sec: GRA defense}), the sensitive object is the adjacency structure rather than pixel or token data, and image/text MIA formulations do not directly apply. Our chain-based attack and defense unify side information across layers (features, hidden representations, predictions) in a single depth-aligned framework, which has no direct analogue in the image or text MIA literature surveyed above.